\documentclass[10pt]{article}
\usepackage[accepted]{tmlr}

\usepackage{hyperref}
\usepackage{url}
\usepackage{xcolor}
\usepackage{booktabs}
\usepackage{amsmath}
\usepackage{amssymb}
\usepackage{mathtools}
\usepackage{amsthm}
\usepackage{nicefrac}
\usepackage{amssymb}
\usepackage{empheq}
\usepackage{bm}
\usepackage[capitalise]{cleveref}
\usepackage{tikz}
\usepackage{pgfplots}
\usetikzlibrary{calc, shapes.arrows, arrows.meta, backgrounds, positioning}
\usetikzlibrary{decorations.text}
\pgfplotsset{compat=1.18}
\usepackage{adjustbox}
\usepackage{tikz-cd}
\usepackage{algorithm}
\usepackage{algpseudocode}
\usepackage{wrapfig}
\usepackage{enumitem}
\usepackage{multirow}
\usepackage{caption}
\usepackage{pifont}
\newcommand{\cmark}{\ding{51}}
\definecolor{light-gray}{gray}{0.8}
\newcommand{\xmark}{\color{light-gray}\ding{55}}
\usepackage{makecell}

\usepackage{accents}
\makeatletter
\DeclareRobustCommand{\cev}[1]{
  {\mathpalette\do@cev{#1}}
}
\newcommand{\do@cev}[2]{
  \vbox{\offinterlineskip
    \sbox\z@{$\m@th#1 x$}
    \ialign{##\cr
      \hidewidth\reflectbox{$\m@th#1\vec{}\mkern4mu$}\hidewidth\cr
      \noalign{\kern-\ht\z@}
      $\m@th#1#2$\cr
    }
  }
}
\makeatother
\DeclareMathSymbol{\shortminus}{\mathbin}{AMSa}{"39}

\newcommand*{\sT}{\mathsf{T}}
\newcommand*{\rd}{\mathrm{d}}
\newcommand*{\bP}{\mathbb{P}}

\newcommand*{\bR}{\mathbb{R}}
\newcommand*{\bE}{\mathbb{E}}
\newcommand*{\cP}{\mathcal{P}}
\newcommand*{\cG}{\mathcal{G}}

\newcommand*{\bN}{\mathbb{N}}
\newcommand*{\cA}{\mathcal{A}}
\newcommand*{\cE}{\mathcal{E}}
\newcommand*{\cL}{\mathcal{L}}
\newcommand*{\cS}{\mathcal{S}}
\newcommand*{\cR}{\mathcal{R}}
\newcommand*{\cF}{\mathcal{F}}
\newcommand*{\cJ}{\mathcal{J}}
\newcommand*{\cK}{\mathcal{K}}
\newcommand*{\cU}{\mathcal{U}}
\newcommand*{\cQ}{\mathcal{Q}}
\newcommand*{\cW}{\mathcal{W}}
\newcommand*{\cM}{\mathcal{M}}
\newcommand*{\cN}{\mathcal{N}}
\newcommand*{\cX}{\mathcal{X}}
\newcommand*{\cY}{\mathcal{Y}}
\newcommand*{\cT}{\mathcal{T}}
\newcommand*{\cC}{\mathcal{C}}
\newcommand*{\cD}{\mathcal{D}}
\newcommand*{\cO}{\mathcal{O}}

\newcommand*{\cZ}{\mathcal{Z}}
\newcommand*{\cV}{\mathcal{V}}

\newcommand*{\cBWUVP}{{cB$\mathbb{W}^2_2$\textsc{-uvp}}}
\newcommand*{\BWUVP}{{B$\mathbb{W}^2_2$\textsc{-uvp}}}
\newcommand*{\vcdot}{\boldsymbol{\cdot}}
\newcommand*{\dx}{\mathrm{d}x}
\newcommand*{\dy}{\mathrm{d}y}
\newcommand*{\dt}{\mathrm{d}t}

\newcommand*{\dWt}{\mathrm{d}W_{\mspace{-1mu}t}}

\newcommand*{\KL}{\mathrm{KL}}
\newcommand*{\BW}{\mathtt{BW}}
\newcommand*{\WFR}{\mathtt{WFR}}
\newcommand*{\FR}{\mathtt{FR}}
\newcommand*{\WFRgrad}{\mathtt{WFRgrad}}
\newcommand*{\nablaWFR}{\nabla_{\mspace{-3mu}\texttt{WFR}}}
\newcommand*{\Umu}{U_{\mkern-1.2mu\raisebox{1pt}{$\scriptstyle\mu$}}}
\newcommand*{\Unu}{U_{\mkern-1.2mu\raisebox{1pt}{$\scriptstyle\nu$}}}
\newcommand*{\Ftildesmallt}{\mathord{\text{\small$\widetilde{F}_t$}}}
\newcommand*{\deltaC}{\delta_{\mspace{-1mu}\scriptscriptstyle\mathcal{C}}}
\newcommand*{\deltaD}{\delta_{\mspace{-1mu}\scriptscriptstyle\mathcal{D}}}
\newcommand*{\dsupplus}{d^{\mspace{-.3mu}\scriptscriptstyle+}}
\newcommand*{\nablaW}{\nabla_{\!\scriptscriptstyle\mathbb{W}}}
\newcommand*{\vSppd}{\mathbf{S}_{\mspace{-2mu}\raisebox{2pt}{$\scriptscriptstyle+\mspace{-1.1mu}+$}}^{d}}
\newcommand*{\pcD}{\pi^\circ_{\mspace{-1mu}\raisebox{1pt}{$\scriptscriptstyle\mathcal{D}$}}}
\newcommand*{\pct}{\pi^\circ_{\raisebox{.9pt}{$\scriptstyle t$}}}
\newcommand*{\pcn}{\pi^\circ_{\raisebox{.9pt}{$\scriptstyle n$}}}
\newcommand*{\vecpct}{\vec{\pi}^\circ_{\raisebox{.9pt}{$\scriptstyle t$}}}
\newcommand*{\pctplusone}{\pi^\circ_{\raisebox{.9pt}{$\scriptstyle t+1$}}}
\newcommand*{\piT}{\pi_{\mspace{-1mu}\scriptscriptstyle T}}
\newcommand*{\pcT}{\pi^\circ_{\mspace{-1mu}\raisebox{1pt}{$\scriptscriptstyle T$}}}
\newcommand*{\piTplusone}{\pi_{\mspace{-1mu}{\scriptscriptstyle T}+1}}
\newcommand*{\etaT}{\eta_{\mspace{-1mu}\scriptscriptstyle T}}

\newcommand*{\nablaGat}{\nabla_{\!\textrm{G\^at}}}
\newcommand*{\nablaFre}{\nabla_{\!\textrm{Fr\'e}}}
\newcommand*{\gT}{g^{\mspace{-2mu}\raisebox{-0.2ex}{$\scriptstyle\mathsf{T}$}}}

\renewcommand*{\div}{\nabla\!\boldsymbol{\cdot}\mspace{-1.5mu}}
\newcommand*{\divsq}{\nabla^{\mspace{-.5mu}2}\mspace{-2mu}\boldsymbol{\cdot}\mspace{-1.8mu}}

\newcommand*{\fsupplus}{f^{\mspace{-0.4mu}{\scriptscriptstyle+}}}
\newcommand*{\fsupminus}{f^{\mspace{-.5mu}{\scriptscriptstyle-}}}

\newcommand*{\vsupplus}{v^{\mspace{-1mu}{\scriptscriptstyle+}}}
\newcommand*{\vsupminus}{v^{\mspace{-1mu}{\scriptscriptstyle-}}}

\newcommand*{\gammasupplus}{\gamma^{\mspace{-0.4mu}{\scriptscriptstyle+}}}
\newcommand*{\gammasupminus}{\gamma^{\mspace{-0.4mu}{\scriptscriptstyle-}}}
\newcommand*{\rhosupplus}{\rho^{\mspace{-1mu}{\scriptscriptstyle+}}}
\newcommand*{\rhosupminus}{\rho^{\mspace{-1mu}{\scriptscriptstyle-}}}
\newcommand*{\ggTscript}{{gg^{\mspace{-1mu}\raisebox{-0.2ex}{$\scriptscriptstyle\mathsf{T}$}}}}

\newcommand*{\SFMSink}{SF$^2$M-Sink}

\newcommand*{\Gibbs}{Gibbs}

\newcommand*{\Fisher}{Fisher}
\newcommand*{\Doob}{Doob}

\newcommand*{\Schrodinger}{Schr\"odinger}
\newcommand*{\Girsanov}{Girsanov}

\newcommand*{\MonteCarlo}{Monte Carlo}
\newcommand*{\Riemannian}{Riemannian}
\newcommand*{\Euclidean}{Euclidean}
\newcommand*{\Bregman}{Bregman}
\newcommand*{\Sinkhorn}{Sinkhorn}
\newcommand*{\Fenchel}{Fenchel}
\newcommand*{\Lebesgue}{Lebesgue}

\newcommand*{\Wasserstein}{Wasserstein}
\newcommand*{\Gaussian}{Gaussian}
\newcommand*{\LaSalle}{LaSalle}
\newcommand*{\Lyapunov}{Lyapunov}

\newcommand*{\Markovian}{Markovian}
\newcommand*{\FokkerPlanck}{Fokker–Planck}
\newcommand*{\RadonNikodym}{Radon-Nikodym}

\newcommand*{\Hessian}{Hessian}
\newcommand*{\KullbackLeibler}{Kullback-Leibler}
\newcommand*{\FisherRao}{Fisher--Rao}
\newcommand*{\WassersteinFisherRao}{Wasserstein-Fisher-Rao}
\newcommand*{\Jordan}{Jordan}
\newcommand*{\Kinderlehrer}{Kinderlehrer}
\newcommand*{\Otto}{Otto}

\newcommand*{\BuresWasserstein}{Bures--Wasserstein}
\newcommand*{\Langevin}{Langevin}
\newcommand*{\Borel}{Borel}
\newcommand*{\Wiener}{Wiener}
\newcommand*{\Danskin}{Danskin}

\newcommand*{\Brenier}{Brenier}
\newcommand*{\Brownian}{Brownian}
\newcommand*{\BenamouBrenier}{Benamou--Brenier}
\newcommand*{\Hellinger}{Hellinger}
\newcommand*{\HellingerKantorovich}{Hellinger--Kantorovich}
\newcommand*{\Gateaux}{G\^ateaux}
\newcommand*{\Frechet}{Fr\'echet}

\newcommand*{\Lipschitz}{Lipschitz}
\newcommand*{\Sobolev}{Sobolev}

\newcommand*{\Kolmogorov}{Kolmogorov}

\newcommand*{\one}{{\raisebox{.53pt}{\small\textcircled{\raisebox{-.23pt} {\hskip.06pt\scriptsize1}}}}}
\newcommand*{\two}{{\raisebox{.53pt}{\small\textcircled{\raisebox{-.2pt} {\hskip.1pt\scriptsize2}}}}}
\newcommand*{\three}{{\raisebox{.53pt}{\small\textcircled{\raisebox{-.25pt} {\hskip.15pt\scriptsize3}}}}}
\newcommand*{\four}{{\raisebox{.53pt}{\small\textcircled{\raisebox{-.25pt} {\scriptsize4}}}}}
\newcommand*{\five}{{\raisebox{.53pt}{\small\textcircled{\raisebox{-.25pt} {\hskip.1pt\scriptsize5}}}}}
\newcommand*{\subfigA}{(a)}
\newcommand*{\subfigB}{(b)}
\newcommand*{\subfigC}{(c)}

\newcommand*{\ie}{\textit{i.e.}}
\newcommand*{\eg}{\textit{e.g.}}

\DeclareMathOperator*{\argmin}{arg\,min}
\DeclareMathOperator*{\argmax}{arg\,max}
\DeclareMathOperator*{\minimize}{minimize}
\DeclareMathOperator*{\maximize}{maximize}

\DeclareMathOperator*{\dom}{dom}
\DeclareMathOperator*{\Lip}{Lip}

\theoremstyle{plain}
\newtheorem{theorem}{Theorem}
\newtheorem{proposition}{Proposition}
\newtheorem{lemma}{Lemma}
\theoremstyle{definition}
\newtheorem{definition}{Definition}
\newtheorem{assumption}{Assumption}
\newtheorem{remark}{Remark}

\newcommand*{\titletext}{
  Variational Online Mirror Descent for\\Robust Learning in \Schrodinger{} Bridge} 
\newcommand*{\titleonelinetext}{
  Variational Online Mirror Descent for Robust Learning in \Schrodinger{} Bridge} 

\title{\titletext{}}

\author{\name Dong-Sig Han,$^1$ Jaein Kim,$^2$ Hee Bin Yoo,$^2$ and  Byoung-Tak Zhang$^2$ \\
  \addr $^1$Department of Computing, Imperial College London\enspace
    $^2$Artificial Intelligence Institute, Seoul National University
}

\def\month{12}
\def\year{2025}
\def\openreview{\url{https://openreview.net/forum?id=g3SsM9FLpm}}

\begin{document}

\maketitle

\begin{abstract}
  The Schr\"{o}dinger bridge (SB) has evolved into a universal class of probabilistic generative models. In practice, however, estimated learning signals are innately uncertain, and the reliability promised by existing methods is often based on speculative optimal case scenarios. Recent studies regarding the Sinkhorn algorithm through mirror descent (MD) have gained attention, revealing geometric insights into solution acquisition of the SB problems. In this paper, we propose a variational online MD (OMD) framework for the SB problems, which provides further stability to SB solvers. We formally prove convergence and a regret bound for the novel OMD formulation of SB acquisition. As a result, we propose a simulation-free SB algorithm called Variational Mirrored Schr\"{o}dinger Bridge (VMSB) by utilizing the Wasserstein-Fisher-Rao geometry of the Gaussian mixture parameterization for Schr\"{o}dinger potentials. Based on the Wasserstein gradient flow theory, the algorithm offers tractable learning dynamics that precisely approximate each OMD step. In experiments, we validate the performance of the proposed VMSB algorithm across an extensive suite of benchmarks. VMSB consistently outperforms contemporary SB solvers on a wide range of SB problems, demonstrating the robustness as well as generality predicted by our OMD theory. 
\end{abstract}

\section{Introduction}

The \Schrodinger{} bridge (SB; \citealp{schrodinger2}) has emerged as a universal class of probabilistic generative models. Nevertheless, the methodologies used in SB algorithms remain somewhat \textit{atypical}, as each algorithm involves a series of sophisticated approaches to derive a solution. The excessive specialization in existing algorithms arises from the absence of a unified perspective, which has led contemporary studies to rest on overly idealized or unrealistic assumptions. To address the issue, a formal discussion of the SB algorithms' robustness, akin to discussions in classical optimization theory \citep{xu2008robust, duchi2021learning} is beneficial for ensuring the stability and reliability of their solutions against uncertainty, thereby improving overall performance. Considering the problem as an instance of ordinary optimization opens avenues for progress, particularly for improving stability. 

We draw inspiration from the rich geometric properties of the \Schrodinger{} bridge problem (SBP; \citealp{sbsurvey}) induced by \textit{entropy regularization} \citep{introeot}, which is closely related to the field of information geometry \citep{infogeo}. In this type of geometry, \textbf{mirror descent} (MD; \citealp{md}) provides a natural methodology for discrete updates, which has proven effective in various scenarios \citep{smd, shalev2012online, amid2020reparameterizing, amid2020winnowing, ghai2020exponentiated, mdirl}. For example, given a strictly convex function $\Omega: \bR^k \to \bR\cup\{+\infty\}$ on the Euclidean space, an update of classical MD with respect to a cost function $F: \bR^k \to \bR\cup\{+\infty\}$ is formulated as 
\begin{equation} \label{eq:md}
  \nabla\Omega(\hspace*{.2pt}w_{t+1}\hspace*{-.5pt})=\nabla\Omega(\hspace*{.2pt}w_t)-\eta_t\nabla\hspace*{-1pt} F(\hspace*{.2pt}w_t\hspace*{-.1pt}),
\end{equation}
where the gradient operator $\nabla\Omega(\hspace*{1pt}\cdot\hspace*{1pt})$ defines an injective mapping from the parameter space to the \textit{dual} space. Recent studies have shown that the \Sinkhorn{} algorithm for training SB models can be reinterpreted as mirror descent using log-\Schrodinger{} potentials as dual parameters \citep{cot, leger2021gradient, mdsem}. These prior analyses have predominantly focused on the optimal-case scenario with a fixed $F$ for simplicity. However, theoretical improvements for worst-case scenarios remain underexplored, particularly with respect to online convex learning methods \citep{zinkevich2003online} for unknown sequences of cost functions $\{F_t\}_{t=1}^\infty$, a setting known as \textbf{online mirror descent} (OMD; \citealp{omd_univ, omd_converge}). Since the SBP is an infinite-dimensional optimization \citep{sf} within optimal transport (OT), applying general OMD arguments to the SB context requires theoretical development and a computational strategy to ensure practical performance improvements. Leveraging the \Wasserstein{} gradient flow discovered by \Jordan{}, \Kinderlehrer{}, and \Otto{} (JKO; \citealp{jko}), this work proposes to solve these iterative optimization problems of OMD updates via gradient dynamics endowed with the \Wasserstein{} metric which has been canonically used in solving OT, replacing standard Euclidean gradients flows \citep{santambrogio2017euclidean}. We further establish a fundamental insight that OMD acts as the informational counterpart of gradient descent in Wasserstein geometry, thus providing a core theoretical rationale for our algorithm.  

\begin{figure}[t]
  \centering
  \begin{minipage}[b]{0.49\textwidth}
    \centering
    \definecolor{pastelgreen}{RGB}{205, 205, 205}
    \definecolor{pastelorange}{RGB}{205, 205, 205}
    \definecolor{superlightgray}{RGB}{235, 235, 235}
    \definecolor{verylightgray}{RGB}{185, 185, 185}
    \definecolor{medlightgray}{RGB}{165, 165, 165}
    \definecolor{somelightgray}{RGB}{125, 125, 125}
    \definecolor{someblue}{RGB}{45, 77, 139}
    \definecolor{lightblue}{RGB}{100, 144, 233}
    \definecolor{extremelightgray}{RGB}{245, 245, 245}
    \newcommand*{\PrimalSpace}{Primal Space}
    \newcommand*{\DualSpace}{Dual Space}
    \newcommand*{\Convert}{$\deltaC\Omega$}
    \newcommand*{\Reverse}{$\deltaD\Omega^\ast$}
    \tikzset{inner sep=0pt, outer sep=0pt}
    \centering 
    \begin{tikzpicture}[tight background, scale=1.28]
      \begin{scope}[xshift=-1.35cm, yshift=-0.27cm]
        \node[transform shape, scale=0.83] at (-0.18,1.48) {\PrimalSpace};
        \coordinate (O1) at (0, 0);
        \coordinate (O2) at (1.384, -0.8408);
        \coordinate (A1) at (0, 1.31);
        \coordinate (A2) at (1.236, 1.38);
        \coordinate (B1) at (-1.48, 0);
        \coordinate (B2) at (-0.68, -0.8096);
        \coordinate (C1) at (-1.42, 1.38);
        \coordinate (C2) at (0.2, 1.204);
        \coordinate (C3) at (0.89, -0.535);
        \coordinate (C4) at (-1.06, -0.06);
        \coordinate (D2) at (0.794, 0.795);
        \coordinate (D3) at (0.984, -0.707);
        \coordinate (D4) at (-1.18, -0.584);
        \coordinate (P0) at (-0.72, 0.96);
        \coordinate (P1) at (-0.0, 0.37);
        \coordinate (P2) at (0.5, -0.3);
        \shade[left color=extremelightgray!15, right color=red, bottom color=extremelightgray, shading angle=45, thin] (C2) to[out=-25, in=130] (D2) to[out=-70, in=85] (D3) to[out=180, in=-5] (D4) to[out=68, in=-95] (C4);
        \begin{scope}
          \clip (C2) to[out=-25, in=130] (D2) to[out=-70, in=85] (D3) to[out=180, in=-5] (D4) to[out=68, in=-95] (C4);
          \fill[verylightgray, opacity=0.5] (C3) circle (17pt);
        \end{scope}
        \draw[draw=black] (C2) to[out=-25, in=130] (D2) to[out=-70, in=85] (D3) to[out=180, in=-5] (D4) to[out=68, in=-95] (C4);
        \shade[bottom color=pastelgreen, top color=pastelgreen!1, shading angle=45, draw=somelightgray, very thin] (O1) to[out=100, in=-95] (A1) to[out=-5, in=180] (A2) to[out=-120, in=140] (O2) -- cycle;
        \shade[bottom color=pastelorange!1, top color=pastelorange, shading angle=-40, draw=somelightgray, very thin] (O1) to[out=-200, in=30] (B1) to[out=-10, in=122] (B2) to[out=30, in=150] (O2) -- cycle;
        \draw[black, black, line cap=round, shorten >=0.6pt] (O1) -- (O2);
        \shade[left color=extremelightgray!5, right color=extremelightgray!86, shading angle=45, draw=black, thin] (C1) to[out=0, in=165] (C2) to[out=-60, in=105] (C3) to[out=160, in=0] (C4) to[out=90, in=-70] cycle;
        \node[transform shape, scale=0.65] at (0.86, 1.135) {$\Pi^\perp_\nu$}; 
        \node[transform shape, scale=0.65] at (-0.5325, -0.45) {$\Pi^\perp_\mu$}; 
        \node[transform shape, scale=0.8] at (-1.07,1.18) {$\cC$};
        \begin{scope}
          \clip (C1) to[out=0, in=165] (C2) to[out=-60, in=105] (C3) to[out=160, in=0] (C4) to[out=90, in=-70] cycle;
          \fill[verylightgray, opacity=0.5] (C3) circle (17pt);
          \fill (P0) circle (1pt) node[above right=2pt, on grid, transform shape, scale=0.75] {$\pi_{\mspace{-1mu}\scriptscriptstyle t}$};
          \fill (P1) circle (1pt) node[above=3pt, on grid, transform shape, scale=0.75] {$\pi_{\mspace{-1mu}\scriptscriptstyle t+1}$};
          \fill (P2) circle (1pt) node[above left=2.5pt, on grid, transform shape, scale=0.75] {$\pct$};
          \begin{scope}[every path/.style={>={Stealth[scale=0.6]}, thick, dotted}]
            \draw (P0) to[out=-20, in=120] (C3);
          \end{scope}
        \end{scope}
        \fill (C3) circle (1pt) node[on grid, transform shape, right=6pt, scale=0.9]{$\pi^\ast$};
      \end{scope}
      \begin{scope}[xshift=1.72cm, yshift=-0.04cm]
        \begin{scope} 
          \coordinate (A) at (-1.05,  1.05);
          \coordinate (B) at (-1.05, -1.05);
          \coordinate (C) at ( 1.05, -1.05);
          \coordinate (D) at ( 1.05,  1.05);
          \coordinate (Z0) at (-0.8, 0.8);
          \coordinate (Z1) at (-0.5, 0.3);
          \coordinate (Z2) at (0.02, 0.05);
          \coordinate (Z3) at (0.35, -0.33);
          \coordinate (Zc) at (0.9,-0.9);
          \coordinate (U0) at (0.1, 1.1); 
          \coordinate (U1) at (0.05, 0.4);
          \coordinate (U2) at (0.3, -0.3);
          \coordinate (V0) at (-1.1, -0.1); 
          \coordinate (V1) at (-0.4, -0.05);
          \coordinate (V2) at (0.3, -0.3);
          \begin{scope}
            \node[transform shape, scale=0.83] at (-0.03,1.25) {\DualSpace};
            \clip (A) rectangle (C);          
            \fill[verylightgray, opacity=0.5] (Zc) circle (23pt);
            \node[transform shape, scale=0.55] at ($(Zc) + (-1.3,0.25)$) {\textsf{empirical}};
            \node[transform shape, scale=0.55] at ($(Zc) + (-1.3,0.05)$) {\textsf{estimates}};
            \draw[semithick] (A) rectangle (C);
            \fill (Z0) circle (1pt) node[right=3pt, on grid, transform shape, scale=0.6] {$\varphi_{\mspace{-1mu}\scriptscriptstyle t}\!\oplus\!\psi_{\mspace{-1mu}\scriptscriptstyle t}$};
            \fill (Z1) circle (1pt);
            \fill (Z2) circle (1pt);
            \fill (Zc) circle (1.5pt) node[left=3pt, on grid, transform shape, scale=0.6] {$\varphi_{\mspace{-1mu}\ast}\!\oplus\!\psi_{\mspace{-1mu}\ast}$};
            \draw[-{>[width=3pt, length=3pt]}, shorten >=1.4pt] (Z0) -- (Z1);
            \draw[-{>[width=3pt, length=3pt]}, shorten >=1.4pt] (Z1) -- (Z2);
            \draw[-{>[width=3pt, length=3pt]}, shorten >=1.4pt] (Z2) -- (Z3);          
            \draw[thin] (U0) .. controls (U1) and (U2) .. (Zc);
            \draw[thin] (V0) .. controls (V1) and (V2) .. (Zc);
          \end{scope}
          \begin{scope}
            \clip (A) -- ($(A) + (0.04, 0.05)$) -- ($(D) + (0.04, 0.05)$) -- (D) -- cycle;
            \draw[thin, fill=superlightgray] (A) -- ($(A) + (0.04, 0.05)$) -- ($(D) + (0.04, 0.05)$) -- (D) -- cycle;
          \end{scope}
          \begin{scope}
            \clip (D) -- ($(D) + (0.04, 0.05)$) -- ($(C) + (0.04, 0.05)$) -- (C) -- cycle;
            \draw[thin, fill=superlightgray] (D) -- ($(D) + (0.04, 0.05)$) -- ($(C) + (0.04, 0.05)$) -- (C) -- cycle;
          \end{scope}
        \end{scope}
        \node[transform shape, scale=0.8] at (0.82,0.82) {$\cD$};      
      \end{scope}
      \begin{scope}[xshift=0.12cm]
        \begin{scope}[yshift=0.26cm]
          \draw[transform shape, line width=12pt, -{Triangle[width=18pt, length=5.5pt]}] (-0.2, 0) -- (0.4,0);
          \draw[transform shape, white, line width=11pt, -{Triangle[width=15.5pt, length=4.5pt]}] (-0.186, 0) -- (0.383, 0);  
          \node[transform shape, scale=0.6] at (0.05, 0) {\Convert};
        \end{scope}
        \begin{scope}[yshift=-0.4cm]
          \draw[transform shape, line width=12pt, {Triangle[width=18pt, length=5.5pt]}-] (-0.2, 0) -- (0.4,0); 
          \draw[transform shape, white, line width=11pt, {Triangle[width=15.5pt, length=4.5pt]}-] (-0.183, 0) -- (0.386,0); 
        \node[transform shape, scale=0.6] at (0.15, 0) {\Reverse};
        \end{scope}    
      \end{scope}
    \end{tikzpicture}
    \caption{A schematic illustration. The primal and dual spaces $(\cC, \cD)$ retain bidirectional maps $(\deltaC \Omega, \deltaD \Omega^\ast\mspace{-1.5mu})$. $\Pi^\perp_\nu$ and $\Pi^\perp_\mu$ indicate projection spaces of $\gamma_1\pi = \mu$ and $\gamma_2\pi = \nu$, respectively. The current $\pi_t$ performs an online learning update following a ``unreliable'' leader $\pct$ in a region shaded in gray.} \label{fig:schema}  
  \end{minipage}
  \hfill 
  \begin{minipage}[b]{0.49\textwidth}
    \centering
    \captionsetup{type=table, position=above}
    \caption{A technical overview. Combining chracteristics of existing methods, VMSB offers a simulation-free SB solver that produces iterative OMD solutions. Our VMSB additionally provides a strong theoretical guarantee of convergence based on a regert analysis.} \label{tab:tech}
    \begin{adjustbox}{width=0.99\textwidth,center}
      \begin{tabular}{l c c c}
        \toprule
          & Iterative & Simulation-free & Regret analysis \\
        \midrule
        \makecell[l]{DSB \vspace*{-3pt} \\ {\scriptsize (De Bortoli et al.)}} & \cmark & \xmark & \xmark \\ 
        \makecell[l]{DSBM \vspace*{-3pt} \\ {\scriptsize (Shi et al.)}} & \cmark & \xmark & \xmark \\ 
        \makecell[l]{LightSB \vspace*{-3pt} \\ {\scriptsize (Korotin et al.)}} & \xmark & \cmark & \xmark \\ 
        \makecell[l]{LightSB-M \vspace*{-3pt} \\ {\scriptsize (Gushchi et al.)}} & \xmark & \cmark & \xmark \\ 
        \midrule
        \textbf{VMSB (ours)} & \cmark & \cmark & \cmark \\
        \bottomrule 
      \end{tabular}
    \end{adjustbox}
    \null
  \end{minipage}
\end{figure}

In this paper, we present a novel algorithm for SB acquisition through the lens of information geometry to further enhance stability in the learning process. As illustrated in \cref{fig:schema}, we leverage primal and dual geometries $(\cC, \cD)$ under the \Bregman{} potential $\Omega(\cdot) = \KL(\cdot\Vert e^{-c_\varepsilon}\mu\otimes\nu)$, where the transformations between the coupling $\pi_t$ and the dual potential $\varphi_t\oplus \psi_t$ are uniquely defined by first variation operators $(\deltaC, \deltaD)$ \citep{mdsem}. For online learning, we account for the solver’s optimization errors due to uncertainty (gray region), and we propose an online mirror descent framework to minimize the regret resulting from these errors. We show that this framework enables a complete form of online learning that tolerates unreliable empirical estimates from arbitrary data-driven SB solvers. To this end, we propose a robust, simulation-free SB algorithm, called \textbf{variational mirrored \Schrodinger{} bridge} (VMSB), with a new parametric update rule that we call \textbf{variational online mirror descent} (VOMD). We seek an \textit{exact} and \textit{computationally feasible} MD update mechanism by narrowing down the solution space to a subset of tractable distributions, which is often referred to as taking a \textit{variational} form \citep{vbi, vistat}. To solve the SBP in a computationally viable manner, our VOMD method utilizes a variational approach \citep{viwgf} based on the Wasserstein gradient flow with respect to the \WassersteinFisherRao{} (WFR) geometry. The proposed VMSB offers closed-form mirror descent updates by formulating the dynamics of iterative subproblems using \Wasserstein{} gradient flows, and our experiments indicate that VMSB generally outperforms existing simulation-free methods in various benchmarks.

\textbf{Our contributions.}\hspace*{6pt} This work builds a novel algorithm grounded in statistical learning theory, formally derived from the information geometric perspective for SBPs. Recently, a \Gaussian{} mixture parameterization of the \Schrodinger{} potentials has been proposed by \citet{lsb}. The simulation-free \textit{LightSB} solver is simple yet general, and its authors proved a universal approximation property for the parameterization. The expressiveness of the solver aligns with the geometric properties of \Gaussian{} variational inference and mixture models \citep{chen2018optimal, daudel2021mixture, viwgf, diao2023forward}. However, its shortcoming---and that of other \textit{simulation-free} solvers \citep{tong2024aistat, lsbm}---is the uncertainty of data-driven learning signals for non-convex objectives. We argue that this presents an opportunity for improvement by leveraging the rich geometric properties of the SBP in a variational form. To the best of our knowledge, VMSB is the first SB algorithm based on VOMD and inherits the theoretical essence of OMD. \cref{tab:tech} shows that VMSB is a simulation-free solver that admits a rigorous convergence analysis in general learning scenarios for probabilistic generative models. We verify the validity of our core theoretical principles through an extensive suite of SB benchmarks, including real-time online learning, standard benchmarks in optimal transport, and image-to-image translation tasks. Our main contributions are summarized below:
\begin{itemize}[leftmargin=*, noitemsep, partopsep=0pt, topsep=0pt, parsep=0pt]
  \item We develop a robust SB learning algorithm built upon the online mirror descent formulation, whose specific update rules follow \Wasserstein{}-2 dynamics derived from local MD objectives. Under mild assumptions, we formally prove the convergence of the proposed algorithm in general online learning scenarios (\S~\ref{sect:md}).
  \item We introduce a simulation-free SB method, leveraging the \WassersteinFisherRao{} geometry to ensure asymptotic stability within \Wasserstein{} gradient flows. The resulting VMSB algorithm admits closed-form dynamics, enabling an accurate and efficient implementation via the LightSB parameterization (\S~\ref{sect:vmsb}). 
  \item We validate our algorithm across diverse SB problem settings, highlighting the effectiveness of our VOMD-based framework in various scenarios, including online learning, classical optimal transport benchmarks, and image-to-image translation tasks. Empirical results consistently demonstrate that our proposed methods outperform existing simulation-free solvers, validating our theoretical assumptions and corresponding claims (\S~\ref{sect:expr}).
\end{itemize}

\section{Related Work} \label{sect:works}

\textbf{Simulation-free SB.}\hspace*{6pt} The \Schrodinger{} bridge problems are originated from a physical formulation for evolution of a dynamical system between measures \citep{sbp2mkp, onfree}. The study of SB has gained popularity due to its connection to entropy-regularized optimal transport (EOT; \citealp{cot, introeot}). Its association with optimal transport suggests various applications across various fields related to machine learning, such as image processing, natural language processing, and control systems \citep{swav, i2sb, wassword, likesb}. Historically, the most representative algorithm for SBP is Sinkhorn \citep{kullback}, and there has been progress in training \textit{simulation-based} SB with nonlinear networks \citep{sbml, dsb} by ``matching'' with simulation data of a half-bridge of forward and backward diffusion at each time. An SB solver is called as \textit{simulation-free} \citep{tong2024tmlr, tong2024aistat} if the solver is trained without samples from the simulation of SB diffusion processes. LightSB \citep{lsb} is a special type of simulation-free solver using the maximum likelihood method of Gaussian mixture models (GMMs). Building upon these advancements, our approach focuses on enhancing simulation-free SB solvers by leveraging geometric insights derived from the generalized dual geometry inherent to the SBP.

\textbf{MD and \Sinkhorn{}.}\hspace*{6pt} The \Bregman{} divergence \citep{bregman} is a family of statistical divergence that is particularly useful when analyzing constrained convex problems \citep{md_nonlin, cvx, fundamentals_cvx}. Notably, \citet{leger2021gradient} and \citet{mdsem} adopted the \Bregman{} divergence into entropic optimal transport and SB problems with probability measures, and the studies revealed that \Sinkhorn{} can be considered to be mirror descent with a constant step size $\eta \equiv 1$. In statistical geometries, the \Bregman{} divergence is a first-order approximation of a \Hessian{} structure \citep{geohess, bregdist}, which interprets MD as natural discretization on a gradient flow. \citet{wmf} introduced \Wasserstein{} mirror flow, and the results include a geometric interpretation of \Sinkhorn{} for unconstrained OT, \ie{}, when $\varepsilon\to 0$ from the entropic regularization setup. \citet{sf} formulated a \textit{half-iteration} of the \Sinkhorn{} algorithm into a mirror flow, \ie{}, $\eta_t \to 0$.  

\textbf{\Wasserstein{} gradient flows}\enspace have attracted considerable attention in machine learning where their intrinsic geometry is governed by the Wasserstein-2 metric. \citep{ambrosio2005gf, villani2009oldnew, santambrogio2017euclidean}. \citet{otto} introduced a formal \Riemannian{} structure to interpret various evolutionary equations as gradient flows with the \Wasserstein{} space, which is closely related to our variational approach. The mirror \Langevin{} dynamics is an early work describing the evolution of the \Langevin{} diffusion \citep{hsieh2018mirrored}, and was later incorporated in the geometry of the \Bregman{} \Wasserstein{} divergence \citep{rankin2023bregman}. We relate our methodology with recent approaches of variational inference on the \BuresWasserstein{} space \citep{viwgf, diao2023forward}. Utilizing \BuresWasserstein{} geometry, the \WassersteinFisherRao{} geometry \citep{liero2016, chizat2018, liero2018, viwgf} additionally provides ``liftings,'' which yield an interaction among measures. 
 
\textbf{Learning theory and Robustness.}\hspace*{6pt} Suppose we have time-varying costs $\{F_t\}_{t=1}^\infty$. We generally referred to learning through these signals as \textit{online learning} \citep{online}. Our interest lies in temporal costs defined in a probability space, where following the ordinary gradient may not be the best choice due to the geometric constraints \citep{infogeo, methodinfogeo}. In this sense, we primarily relate our work to OMD, an online learning form of mirror descent \citep{omd_univ,md_info,omd_converge}. The OMD algorithm provides a generalization of robust learning by seeking solutions that are optimal in a worst case sense, ensuring performance guarantees under adversarial or uncertain conditions \citep{xu2008robust, zinkevich2003online, madry2017towards}. Another relevant design of the online algorithm is the follow-the-regularized-leader (FTRL; \citealp{mcmahan2011follow}). OMD focuses on scheduling proximity of updates through $\{\eta_t\}_{t=1}^T$, whereas FTRL minimizes historical losses with a fixed proximity term. Meanwhile, there have been several notions of robustness. Distributionally robust optimization (DRO; \citealp{dro}) addresses robustness by solving a minimax objective under adversarial distributions, a concept originating from Knightian uncertainty. Our work addresses this via regret minimization, which is aligned with the core argument of Wasserstein distributionally robust regret optimization (DRRO; \citealp{drro}). The main difference of our work from DRRO is the premise of the ambiguity sets: whereas DRRO relies on Wasserstein balls, we model uncertainty in the dual geometry under statistical assumptions.

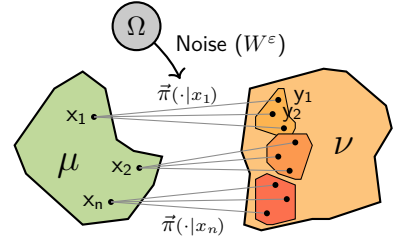
\begin{wrapfigure}{r}{0.3\textwidth}
  \definecolor{mugreen}{RGB}{192, 217, 155}
  \definecolor{nuorange}{RGB}{253, 196, 124}
  \definecolor{noisepurple}{RGB}{202, 202, 202}
  \definecolor{nuorange_1}{RGB}{253, 180, 80}
  \definecolor{nuorange_2}{RGB}{252, 153, 79}
  \definecolor{nuorange_n}{RGB}{252, 112, 78}
  \definecolor{pathgray}{RGB}{140, 140, 140}
  \tikzset{inner sep=0pt, outer sep=0pt}
  \vskip-17.5pt
  \centering 
  \begin{tikzpicture}[tight background]
    \begin{scope}[scale=0.75]
    \node[draw, thick, fill=noisepurple, circle, minimum width=18pt] (Om) at (-1,2.6) {$\Omega$};    
    \begin{scope}[xshift=-1.95cm]
      \draw[fill=mugreen, thick] (-1,0) -- (-1.2, 0.5) -- (-0.5,1.5) -- (0,1.8)  -- (0.4,1.5)  -- (0.6,0.6) -- (0.9,0.4) -- (1.4, 0.3) -- (1.25,-0.2) -- (0.8, -0.8) -- (0,-0.9) -- cycle;
      \coordinate (X1) at (0.2,  1);
      \coordinate (X2) at (1,  0.1);
      \coordinate (Xn) at (0.5, -0.5);
    \end{scope}
    \begin{scope}[xshift=2.3cm]
      \draw[fill=nuorange, thick] (1.1,1.2) -- (1.2,0.8) -- (0.9,0.3) -- (1,-0.5) -- (0.1, -.8) -- (-0.2, -.8) -- (-1.1,-1) -- (-1.4,-0.9) -- (-1.3,0) -- (-1.45,0.4) -- (-1.3, 1.5) -- (0,1.85) -- (0.4,1.8)  -- cycle;
      \draw[fill=nuorange_1, thin] (-0.5,0.8) -- (-0.5,0.9) -- (-0.6,1.) -- (-0.7,1.5) -- (-0.8,1.5) -- (-1.2,1) -- (-1.2,0.8) -- (-0.6,0.6) -- cycle;
      \draw[fill=nuorange_2, thin] (-0.3,0.7) -- (-0.2,0.5) --  (-0.4,-0.1) -- (-0.5,-0.1) -- (-0.8,0.) -- (-1,0.2) -- (-1,0.4) -- (-1,0.5) -- (-0.8,0.7) -- cycle;
      \draw[fill=nuorange_n, thin] (-0.5,-0.2) -- (-0.5,-0.8) -- (-0.9,-0.9) -- (-1,-0.9) -- (-1.2,-0.8) -- (-1.2,-0.1) -- (-1,0) -- cycle;
      \coordinate (X1_1) at (-0.8, 1.3);
      \coordinate (X1_2) at (-0.9, 1.05);
      \coordinate (X1_3) at (-0.7, 0.8);
      \coordinate (X2_1) at (-0.5, 0.55);
      \coordinate (X2_2) at (-0.8, 0.3);
      \coordinate (X2_3) at (-0.6, 0.05);
      \coordinate (Xn_1) at (-0.85, -0.2);
      \coordinate (Xn_2) at (-0.65, -0.45);
      \coordinate (Xn_3) at (-1, -0.7);
    \end{scope}
    \foreach \idx in {1, 2, n} {
      \fill (X\idx) circle (1.5pt);
      \node at ($(X\idx) - (0.3,0)$) {\small $\mathsf{x_\idx}$};
      \foreach \jdx in {1, 2, 3} {
        \draw[pathgray] (X\idx) -- (X\idx_\jdx);
        \fill (X\idx_\jdx) circle (1.5pt);
      }
    }
    \node[scale=0.95] at (1.95,1.35) {\small $\mathsf{y_1}$};
    \node[scale=0.95] at (1.75,1.05) {\small $\mathsf{y_2}$};
    \node[scale=0.95] (T1) at (-0.1,1.4) {\small${\vec{\pi}}$\scriptsize${(\cdot | x_1\mkern-2mu)}$};
    \node[scale=0.95] (Tn) at (0,-0.9) {\small${\vec{\pi}}$\scriptsize${(\cdot | x_n\mkern-2mu)}$};
    \node (mu) at (-2.2, 0.15) {\Large$\mu$}; 
    \node (nu) at (2.65, 0.5) {\Large$\nu$};
    \draw[thick, ->, shorten >= 2pt] (Om) to[bend left=10] node[above right=0.1cm, scale=0.85] {\textsf{Noise ($W^\varepsilon$)}} (T1);
  \end{scope}     
  \end{tikzpicture} 
  \caption{The SB problem.} \label{fig:setup}
  \vskip-13pt
\end{wrapfigure}

\section{General Properties of \Schrodinger{} Bridge Problems} \label{sect:prelim}

\textbf{Notation.}\hspace*{6pt} Let $\cP(\cS)$ ($\cP_2(\cS)$) denote the set of (absolutely continuous) \Borel{} probability measures on $\cS \subseteq \bR^d$ (with a finite second moment). For marginals $\mu,\nu\in\cP_2(\cS)$, $\Pi(\mu, \nu)$ denotes the set of couplings \citep{cot}. For a coupling $\pi$, we often use a shorthand notation $\vec{\pi}^x$ ($\cev{\pi}^y$) to denotes a conditional distribution for a sample data $\vec{\pi}(\cdot\vert x)$ (or $\cev{\pi}(\cdot\vert y)$; see \cref{fig:setup}). We use $\KL(\cdot\Vert\cdot)$ to denote the KL functional (\eg{}, formally defined at \cref{eq:eot}) and assume $+\infty$ if an argument is not absolutely continuous. We employ a notation $\bP([0,1], \cS)$ to denote a set of path measures from the time interval $[0,1]$. 

\textbf{The static SB problem and \Sinkhorn{}.}\hspace*{6pt} We first introduce a problem definition of a static variant of \Schrodinger{} bridge problems.  For a regularization coefficient $\varepsilon\in\bR^+$, the \textit{static} SB problem, or  the entropic optimal transport problem with a quadratic cost function $c(x,y) = \tfrac{1}{2}\lVert x-y\rVert^2$, is defined as optimization of finding a unique minimizer of optimal coupling $\pi^\ast$ with repect to the following objective 
\begin{equation}\label{eq:eot}
  \mathrm{OT}_{\mspace{-.5mu}\varepsilon}\mkern-.5mu(\mu, \nu)\coloneqq \! \inf_{\pi\in\Pi(\mu,\nu)}\mkern-1.2mu\iint_{\cS\times\cS}\!c(x,y)\,  d\pi(x,y) + \varepsilon \KL(\pi\Vert\mu\otimes\nu) 
\end{equation}
where $\mu\otimes\nu$ denotes, the product of marginals and $\KL(\cdot\Vert\cdot)$ in the regularization term denotes the \KullbackLeibler{} (KL) functional, or relative entropy,  defined as 
\[
\KL(\pi\Vert\mu\otimes\nu)\coloneqq
\begin{dcases}
  \iint_{\cS\times\cS}\mkern-1.2mu\log\mkern-1.2mu\frac{d\pi}{d(\mu\otimes\nu)}\,d\pi,
  & \pi\ll\mu\otimes\nu,\\
  \,\infty, & \pi\not\ll\mu\otimes\nu.
\end{dcases}
\]
It is well studied that the EOT problem \eqref{eq:eot} can be translated to a singular form of divergence minimization problem \citep{introeot}, \ie{}, $\KL(\pi \Vert \cR)$ where the reference measure $\cR$ is of \textit{\Gibbs{} parameterization}, satisfying the \RadonNikodym{} derivative of $\frac{d\cR}{d(\mu\otimes\nu)} = e^{-c_\varepsilon}$ with $c_\varepsilon = c/\varepsilon$.
We are particularly interested in its dual aspect, which contains a pair of \textit{log-\Schrodinger{} potentials} as an alternative representation for EOT solutions.
\begin{lemma}[Duality; \citealp{introeot}]\label{lem:dual}
  Suppose existence of optimal coupling $\pi^\ast\in\Pi(\mu,\nu)$ and log-\Schrodinger{} potentials $(\varphi^\ast\mspace{-3mu}, \psi^\ast)\in L^1(\mu)\!\times\!L^1(\nu)$. Then, the dual problem of \eqref{eq:eot} is found without gap by 
  \begin{equation}\label{eq:dual}
    \inf_{\pi\in\Pi(\mu,\nu)} \mkern-1.2mu \KL(\pi\Vert\cR) = \!\sup_{\varphi \in L^1\mkern-.5mu(\mu),\psi\in L^1\mkern-.5mu(\nu)}\! \mu(\varphi) + \nu(\psi) - \textstyle\iint_{\cS\times\cS} e^{\varphi\oplus\psi}\,d\cR + 1,
  \end{equation}
  where $\mu(\varphi)\coloneqq\!\int_\cS\varphi\, d\mu$;\enspace$\nu(\psi) \coloneqq\!\int_\cS\psi\,d\nu$; the operator $\oplus$ indicates the direct sum of two potentials; the symbol $\cR$ denotes the reference of \Gibbs{} measure: $d\cR = e^{-c_\varepsilon} d (\mu \otimes \nu)$ with $c_\varepsilon\mspace{-1.1mu}(x,y) \coloneqq \tfrac{1}{2\varepsilon}\lVert x - y \rVert^2$.
\end{lemma}
By the duality, the \Lebesgue{} integrable functions $(\varphi^\ast\mspace{-3mu}, \psi^\ast)\in L^1(\mu)\!\times\!L^1(\nu)$ represent dual solution of \eqref{eq:dual}, associated with $\pi^\ast$ by $d\pi^\ast\mspace{-1mu} = e^{\varphi^\ast \oplus \psi^\ast - c_\varepsilon} d(\mu\otimes\nu)$, $(\mu\otimes\nu)$-almost surely. Apparently, the result is a general form of Legendre duality from convex optimization, and the realization that every element of the dual space can be split into two separate potential functions has inspired a \textit{halfway} projection method known as the Sinkhorn algorithm \citep{lightspeed}. The \Sinkhorn{} algorithm is defined as following alternating updates:
\begin{equation} \label{eq:sink}
  \psi_{2t+1}\mspace{-1mu}(y) = -\mspace{-1mu}\log\!\int_\cS\mspace{-1mu} e^{\varphi_{2t}\mspace{-1mu}(x) - c_\varepsilon\mspace{-1mu}(x,y)}\mu(d x), \mspace{49mu} \varphi_{2t+2}\mspace{-1mu}(x) = -\mspace{-1mu}\log\!\int_\cS\mspace{-1mu} e^{\psi_{2t+1}\mspace{-1mu}(y) - c_\varepsilon\mspace{-1mu}(x,y)}\nu(d y),
\end{equation}
where each update is called a iterative proportional fitting (IPF; \citealp{kullback}) procedure. Based on the geometric understanding of Bregman gradient descent and flow \citep{leger2021gradient, sf}, the static SB problem and \Sinkhorn{} reveal a profound connection to the information geometry incurred by the KL functional. Of course, this implies that that the \Sinkhorn{} algorithm is not the only way of finding the solution space. Akin to differential geometry, JKO discovered that each density can be analyzed using the standard geometry of the Wasserstein-2  distance, defined as \citep{jko, villani2009oldnew}  
\begin{equation}
  W_{\mkern-1.2mu2}(\mu, \nu) \coloneqq \inf_{\pi \in \Pi(\mu, \nu)} \Bigl(\bE_{(x,y) \sim \pi} \lVert x- y\rVert^2 \Bigr)^{\!\frac{1}{2}},
\end{equation}
where $\pi$ is a joint distribution that has $\mu$ and $\nu$ as its marginals. In essence, the distance measures the minimum cost to transform the distribution $\mu$ into $\nu$, and provides an $L^2$-like metric that quantifies the discrepancy between two marginal distributions by considering the distance that mass must travel. Our method utilizes the theory of \Wasserstein{} gradient flows whose technical machinery is noted in Appendix~\ref{sect:discuss}.

\textbf{The dynamic SB problem.}\hspace*{6pt}An alternative perspective on the \Schrodinger{} bridge is the interpretation in terms of fluid dynamics or diffusion. In this setup, one considers a \Wiener{} process $W^\varepsilon$ with volatility $\varepsilon\in\bR^+$ to be the reference for SB. The \textit{dynamic} SBP \citep{sbml} aims to find a process $\cT^\ast$ such that
\begin{equation} \label{eq:sbp}
  \cT^\ast \coloneqq \argmin_{\cT\in\cQ(\mu, \nu)}\KL(\mspace{2mu}\cT\mspace{2mu}\Vert W^\varepsilon),
\end{equation}
where $\cQ(\mu, \nu) \subset \bP([0,1], \cS)$ is the set of processes with marginals $\mu$ and $\nu$. The SB process $\cT^\ast$ is uniquely described by a stochastic differential equation (SDE): $d X_t = g^\ast(t, X_t) + d W^\varepsilon_t$ in $t \in [0,1)$, governed by a drift function $g^\ast$ along with noises generated by $W^\varepsilon_t$.
Addressing $\cT^\ast$ involves continuous-time log-\Schrodinger{} potentials $\{\varphi_t\}_{t\in[0,1]}$, $\{\psi_t\}_{t\in[0,1]}$, and the time-dependent density $\rho_t$, governed by the following differential equations for $\rho_0 =\mu$ and $\rho_1 = \nu$ \citep{bbrcd} 
\begin{equation}\label{eq:pde4}
  \begin{dcases}
    \mspace{17mu}\partial_t \varphi_t = \tfrac{1}{2}\lvert\nabla \varphi_t\rvert^2 \! + \tfrac{\varepsilon}{2} \Delta \varphi_t,\\
    \,-\partial_t \psi_t = \tfrac{1}{2}\lvert\nabla \psi_t\rvert^2 \! + \tfrac{\varepsilon}{2} \Delta \psi_t,
  \end{dcases}
  \mspace{90mu}
  \begin{dcases}
    \,-\partial_t \rho_t + \nabla\!\vcdot\mkern-0.5mu(\nabla \varphi_t\rho_t) = \tfrac{\varepsilon}{2}\Delta\rho_t,\\
    \mspace{17mu}\partial_t \rho_t + \nabla\!\vcdot\mkern-0.5mu(\nabla \psi_t\rho_t) = \tfrac{\varepsilon}{2}\Delta\rho_t,
  \end{dcases}
\end{equation}
where the operators $(\nabla,\nabla\vcdot,\Delta)$ denote gradient, divergence, and Laplacian, respectively. The pairs on the left and the right are commonly known as Hamilton-Jacobi-Bellman equations and Fokker-Planck equations, respectively. The fundamental equivalence between static and dynamic SBPs \citep{onfree, sbp2mkp} allows us to consider the optimal coupling $\pi^\ast$ when finding the SB process $\cT^\ast$\!, vice versa. 

\textbf{The Bregman divergence.}\hspace*{6pt}To advance the discussion, one needs an alternative notion of gradients in order to generalize the classical Bregman divergence, defined as $D_\Omega(x\Vert y) \coloneqq \Omega(x) - \Omega(y) - \bigl\langle\nabla\Omega(y),\;x-y\bigr\rangle$. However, since the domain of $\Omega$ considered in this paper (the measure space of SB solutions) has an empty interior \citep{mdsem}, \Gateaux{} differentiability is not ensured. Consequently, to discuss OMD here, we adopt the following definitions of \textit{directional derivatives} \citep{infinite} and \textit{first variations} \citep{mdsem} in order to discuss MD. We also refer the readers to Definition~7.12~of~\citet{santambrogio2015ot} for alternative description on this topic.
\begin{definition}[Directional derivative] \label{def:directional}
  Given a locally convex topological vector space $\cM$, the directional derivative of $F$ in the direction $\xi$ is defined as $\dsupplus\mspace{-1mu}F(x; \xi) = \lim_{h\to 0^+}\mspace{-2mu} \frac{F(x + h\xi) - F(x)}{h}$.
\end{definition}
\begin{definition}[First variation] \label{def:firstvar}
  Given a topological vector space $\cM$ and a convex constraint $\cC \subseteq \cM$ and a function $F: \cM\to \bR\cup\{\infty\}$, we define the first variation of $F$ over $\cC$ evaluated at $x \in \cC \cap \dom(F)$, as $\deltaC F(x) \in \cM^\ast$, where $\cM^\ast$ is the topological dual of $\cM$, such that it holds for all $y \in \cC \cap \dom(F)$ and $v = y - x \in \cM$: $\langle \deltaC F(x), v\rangle = \dsupplus\mspace{-1mu}F(x; v)$. $\langle\cdot, \cdot\rangle$ denotes the duality product of $\cM$ and $\cM^\ast$.
\end{definition}
Following \citet{sf}, this work considers the following generalized divergence defined with a weak notion of the directional derivative. We explicitly set the \Bregman{} potential $\Omega(\cdot) = \KL(\cdot\Vert e^{-c_\varepsilon}\mu\otimes\nu)$ in the SB problems for the rest of the paper, which enforces the {\Gibbs{}} parameterization for the OT couplings.
\begin{definition}[Generalized \Bregman{} divergence] \label{def:breg}
  Let a convex functional $\Omega: \cM \to \bR \cup \{+\infty\}$ be a \Bregman{} potential. Define the \Bregman{} divergence associated with $\Omega$ as
  \begin{equation}D_\Omega(x\Vert y) \mspace{-1mu} \coloneqq \Omega(x) - \Omega(y) - \dsupplus\Omega\mspace{1mu}(y; x - y),\end{equation}
  for every pair of elements in the domain $x,y\in\cM$.
\end{definition}
The above definition preserves the essential role of the Bregman divergence as the first-order Taylor expansion of $\Omega$, allowing us to utilize properties in optimization similar to those used in the classical setting. To account for smoothness- and convexity-like properties of a functional defined on metric spaces, we employ the notions of \textbf{relative smoothness and convexity} with respect to the generalized Bregman divergence. These notions were first introduced by \citet{birnbaum2011distributed} and later popularized by \citet{mdsem}, who applied them to spaces of measures to provide a theoretical foundation for the Sinkhorn and EM algorithms. The formal definitions are given in \cref{def:relsmooth}.
 
\textbf{Asymptotically log-concave distributions.}\hspace*{6pt}Lastly, our theoretical analysis works with a certain form of \textit{measure concentration} property, and we formally address this property with asymptotically strong log-concave (alc) distributions. This section rigorously states the resulting assumption which is later used in address the desired properties of OMD, analougous to strong convexity assumption in the classical literature.
Let us consider the following informal definition of asymptotically strong log-concave distributions
\begin{equation}
  \cP_\textrm{alc}(\bR^d) \coloneqq \Bigl\{\,\zeta(\dx) = \exp\bigl(-U(x)\bigr)\dx\,:\, U \in C_2\bigl(\bR^d\bigr),\ \textrm{$U$ is asymptotically strongly convex}\,\Bigr\},
\end{equation}
where \cref{sect:proof} contains a formal version on asymptotical convexity. Note that asymptotically log-concave functions satisfy a certain form of logarithmic \Sobolev{} inequality (LSI; \citealp{lsi}). The condition can be an extension of \Sobolev{} space \citep{adams2003sobolev} for informational geometric problems. The simplest case of such condition for the \Gaussian{} measure is represented as follows.
\begin{remark}[LSI for the standard \Gaussian{}] \label{rem:lsi_gau}
  Suppose that $f$ is a nonnegative function, integrable with respect to a measure $\gamma$, and that the entropy is defined as $\mathrm{Ent}_\gamma(f) = \int_{\bR^d} f \log f d \gamma - \bigl(\int_{\bR^d} f d \gamma \bigr) \log \bigl(\int_{\bR^d} f d \gamma \bigr)$. the log \Sobolev{} inequality when $\gamma$ is the standard \Gaussian{} measure reads $\mathrm{Ent}_\gamma(f) \le \frac{1}{2}\int_{\bR^d} \frac{\lvert f\rvert^2 }{f} d \gamma$. 
\end{remark}
Historically, the log \Sobolev{} inequality condition arises from the implication of satisfying the Talagrand's inequality for bounding the Wasserstein-2 distance, and is closely related to measure concentration \citep{otto2000generalization}. The important extension of asymptotically strong log-concave distributions for \Schrodinger{} bridge $d \pi = e^{\varphi \oplus \psi - c_\varepsilon}d(\mu\otimes\nu)$, $(\mu\otimes\nu)$-a.s. is that induced SB model also satisfies asymptotically strongly log-concaveness and the LSI condition \citep{conforti2024weak}. For a representative model related to our work, the \Gaussian{} mixture parameterization \citep{lsb} is a representative model that our theoretical analysis holds, because \Gaussian{} mixture weights does not alter the asymptotic characteristic.
\begin{remark}[\citealp{conforti2024weak}]
  Let $\mu, \nu\in\cP_\textrm{alc}(\bR^d)$ with finite entropy on \Lebesgue{} measures and $\pi\in\cC$ be a coupling in the static \Schrodinger{} bridge problem. Then, for a quadratic cost function, the coupling distribution is also asymptotically log-concave and satisfies a form of logarithmic \Sobolev{} inequality.
\end{remark}
Let us suppose that a parameterized SB model $d\pi_t = e^{\varphi_t \oplus \psi_t - c_\varepsilon}d(\mu\otimes\nu)$ obeys the following constraints for marginals and potentials: 
\begin{equation} \label{eq:c_def}
  \cC \coloneqq \bigl\{\,\pi : (\mu, \nu) \in \cP_2\mspace{-1mu}(\bR^d)\cap \cP_\textrm{alc}\mspace{-1mu}(\bR^d),\enspace (\varphi, \psi) \in L^1\mspace{-2mu}(\mu) \mspace{-2mu} \times \mspace{-2.5mu} L^1\mspace{-2mu}(\nu),\enspace\textrm{and}\quad\varphi, \psi \in C^2(\bR^d)\cap\Lip(\cK) \,\bigr\},
\end{equation}
where $\Lip(\cK)$ denotes a set of functions with $\cK$-\Lipschitz{} continuity.
Using the disintegration theorem for probability measures \citep{someppm}, we assume the boundedness of \Bregman{} divergence between two transport plans using derivatives of first variations with a positive constraint $\omega > 0$ by the following assumption. 
\begin{assumption}[LSI for couplings] \label{asm:lsi} 
  Let us suppose $\Omega(\cdot) = \KL(\cdot\Vert \cR)$ for a reference measure $\cR$. Suppose that arbitrary $\pi,\bar{\pi}\in\cC$ satisfy a type of logarithmic \Sobolev{} inequality for relative entropy (KL divergence) is upper bounded by the relative \Fisher{} information \citep{lsi} for some $\bar{\omega} \in \bR_+$ as 
  \begin{equation}\label{eq:lsi_kl}
    \KL(\pi \Vert \cR) \le \frac{1}{2\bar{\omega}} \iint_{\bR^d\times\bR^d} \biggl\lvert \nabla \log \frac{d\pi(x,y)}{d\cR(x,y)}  \biggr\rvert^2 \! \pi(\dx, \dy),
  \end{equation}
  By the equivalence of the first variations of Bregman divergences (detailed in \cref{lem:equivf}), convergence in $\KL(\pi \Vert \bar\pi)$ is equivalent to convergence in $D_\Omega\mspace{-1mu}(\pi\Vert \bar{\pi})$. Introducing the relative logarithm $\log_\cR(\pi) \coloneqq \log(\pi/\cR)$ and adopting the framework of \citet{conforti2024weak}, we can assume there exists a constant $\omega > 0$ such that
  \begin{equation}
    D_\Omega\mspace{-1mu}(\pi\Vert \bar{\pi}) \le \frac{1}{2\omega} \bigl\lVert \nabla (\deltaC\Omega(\pi) - \deltaC\Omega(\bar{\pi})) \bigr\rVert^2_{L^2(\pi)} 
  \end{equation}
  holds for $\Omega = \KL(\cdot \Vert e^{-c_\varepsilon}\mu\otimes\nu)$ and the first variation $\deltaC$. We call the condition as $\mathrm{LSI}(\omega)$.
\end{assumption}
Since, as illustrated in \cref{fig:schema}, $\varphi_t\oplus\psi_t = \deltaC\Omega(\pi_t)$ for arbitrary $\pi_t$, the assumption geometrically enforces a quadratic upper bound on the Bregman divergence in terms of dual space gradients. In general, the LSI condition also has often been used to analyze the convergence of partial differential equations \citep{malrieu2001logarithmic}. To make our analysis on improvement (\cref{lem:step2}) and a solid regret bound of OMD (\cref{lem:single}), this work finds that \cref{asm:lsi} is necessary to ensure a certain asymptotical type of measure concentration. 

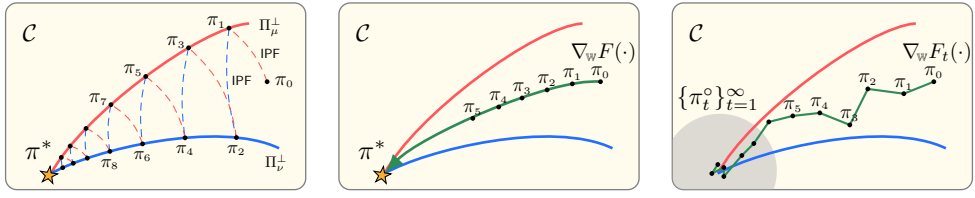
\begin{figure}[t]
  \centering  
  \tikzset{>=latex, inner sep=0pt, outer sep=0pt}
  \definecolor{ipfred}{RGB}{255, 87, 92}
  \definecolor{ipfblue}{RGB}{34, 110, 255}
  \definecolor{ipfpurple}{RGB}{186,85,211}
  \definecolor{ipfyellow}{RGB}{246, 175, 49}
  \definecolor{ipfgreen}{RGB}{46,139,87}
  \definecolor{ipfbg}{RGB}{255, 252, 245}
  \definecolor{verylightgray}{RGB}{185, 185, 185}
  \definecolor{sinkbg}{RGB}{255, 252, 237}
  \newcommand*{\MeasureSpace}{$\mathcal{C}$}
  \newcommand*{\RedLabel}{$\Pi^\perp_\mu$}
  \newcommand*{\BlueLabel}{$\Pi^\perp_\nu$}
  \newcommand*{\PointLabel}[1]{$\pi_{\scriptscriptstyle #1}$}
  \newcommand*{\Optimal}{$\pi^\ast$}
  \def\scaling{0.75}
  \newcommand*{\tstar}[5]{
    \pgfmathsetmacro{\starangle}{360/#3}
    \draw[#5] (#4:#1)
    \foreach \x in {1,...,#3}
    { -- (#4+\x*\starangle-\starangle/2:#2) -- (#4+\x*\starangle:#1)
    } -- cycle;
  }
  \newcommand{\ngram}[4]{ 
    \pgfmathsetmacro{\starangle}{360/#2}
    \pgfmathsetmacro{\innerradius}{#1*sin(90-\starangle)/sin(90+\starangle/2)}
    \tstar{\innerradius}{#1}{#2}{#3}{#4}
  }
  \centering 
  \begin{tikzpicture}[tight background, scale=1.26]
    \begin{scope}
      \fill[sinkbg, draw=gray, line width=.5pt, rounded corners=1.5mm]
        (-0.1,0.215) -- (3.051,0.215) -- (3.051,2.175) -- (-0.1,2.175) -- cycle;
      \begin{scope}[scale=0.63, yshift=-0.34cm]
        \coordinate (a0) at (3.9,3.45);
        \coordinate (ac1) at (3.1,3.4);
        \coordinate (ac2) at (1.1,1.9);
        \coordinate (b0) at (4.4,1.37);
        \coordinate (bc1) at (3.4,1.85);
        \coordinate (bc2) at (1.3,1.35);
        \coordinate (optimal) at (0.6,0.95);
        \node at (0.25,3.3) {\small\MeasureSpace{}};
        \node at ($(optimal) + (-0.2, 0.4)$) {\Optimal};
        \coordinate (c0) at (4.2, 2.48);
        \draw[ipfred, line width=1pt]
          (a0) .. controls (ac1) and (ac2) ..
            coordinate[pos=0.12] (a1)
            coordinate[pos=0.31] (a2)
            coordinate[pos=0.49] (a3)
            coordinate[pos=0.64] (a4)
            coordinate[pos=0.76] (a5)
            coordinate[pos=0.85] (a6)
            coordinate[pos=0.91] (a7)
          (optimal);
        \draw[ipfblue, line width=1pt]
          (b0) .. controls (bc1) and (bc2) ..
            coordinate[pos=0.2]   (b1)
            coordinate[pos=0.4]   (b2)
            coordinate[pos=0.56]  (b3)
            coordinate[pos=0.69]  (b4)
            coordinate[pos=0.79]  (b5)
            coordinate[pos=0.86]  (b6)
            coordinate[pos=0.92]  (b7)
          (optimal);
        \node[right=0.1cm, scale=\scaling] at (c0) {$\pi_{\scriptscriptstyle 0}$};
        \foreach \coor [evaluate=\coor as \idx using int(2*\coor-1)] in {1, 2, 3, 4}
          \node[above left =0.0cm and 0.04cm, scale=\scaling] at (a\coor) {\PointLabel{\idx}};
        \foreach \coor [evaluate=\coor as \idx using int(2*\coor)] in {1, 2, 3, 4}
          \node[below=0.1cm, scale=\scaling] at (b\coor) {\small \PointLabel{\idx}};
        \draw[densely dashed, ipfred, thin] (c0) to[bend right=15] (a1); 
        \foreach \coor [evaluate=\coor as \coornext using int(\coor+1)] in {1, 2, 3, 4, 5, 6, 7} {
          \draw[densely dashed, ipfblue, thin] (a\coor)to[bend right=10] (b\coor);
        }
        \foreach \coor [evaluate=\coor as \coornext using int(\coor+1)] in {1, 2, 3, 4, 5, 6} {
          \draw[densely dashed, ipfred, thin] (b\coor)to[bend right=20] (a\coornext);
        }
        \fill[black] (c0) circle (0.04); 
        \foreach \coor in {1,2,3,4,5,6,7} {
          \fill[black] (a\coor) circle (0.04);
        }        
        \foreach \coor in {1,2,3,4,5,6,7} {
          \fill[black] (b\coor) circle (0.04);
        }
        \begin{scope}[xshift=0.56cm, yshift=0.92cm]
          \ngram{0.15}{5}{45}{fill=ipfyellow}
        \end{scope}
        \begin{scope}
          \node[scale=0.58] at ($(a0)+(0.35,-0.05)$) {\RedLabel};
          \node[scale=0.58] at ($(b0)+(-0.05,-0.24)$) {\BlueLabel};
        \end{scope}
      \end{scope}
      \node[font=\sffamily, scale=0.5] at (2.68,1.65) {IPF};
      \node[font=\sffamily, scale=0.5] at (2.38,1.36) {IPF};
    \end{scope}
    \begin{scope}[xshift=3.5cm]
      \fill[sinkbg, draw=gray, line width=.5pt, rounded corners=1.5mm]
       (-0.1,0.215) -- (3.051,0.215) -- (3.051,2.175) -- (-0.1,2.175) -- cycle;
      \begin{scope}[scale=0.63, yshift=-0.34cm]
        \coordinate (a0) at (3.9,3.45);
        \coordinate (ac1) at (3.1,3.4);
        \coordinate (ac2) at (1.1,1.9);
        \coordinate (b0) at (4.4,1.37);
        \coordinate (bc1) at (3.4,1.85);
        \coordinate (bc2) at (1.3,1.35);
        \coordinate (c0) at (4.2, 2.48);
        \coordinate (cc1) at (3.33, 2.55);
        \coordinate (cc2) at (1.17,1.55);
        \coordinate (optimal) at (0.6,0.95);
        \node at (0.25,3.3) {\small\MeasureSpace{}};
        \node at ($(optimal) + (-0.2, 0.4)$) {\Optimal};
        \draw[ipfred, line width=1pt]
          (a0) .. controls (ac1) and (ac2) ..
            coordinate[pos=0.12] (a1)
            coordinate[pos=0.23] (a2)
            coordinate[pos=0.32] (a3)
            coordinate[pos=0.42] (a4)
            coordinate[pos=0.49] (a5)
            coordinate[pos=0.61] (a6)
            coordinate[pos=0.7] (a7)
            coordinate[pos=0.76] (a8)
            coordinate[pos=0.83] (a9)
          (optimal);
        \draw[ipfblue, line width=1pt]
          (b0) .. controls (bc1) and (bc2) ..
            coordinate[pos=0.12] (b1)
            coordinate[pos=0.23] (b2)
            coordinate[pos=0.32] (b3)
            coordinate[pos=0.41] (b4)
            coordinate[pos=0.47] (b5)
            coordinate[pos=0.55] (b6)
            coordinate[pos=0.66] (b7)
            coordinate[pos=0.72] (b8)
            coordinate[pos=0.79] (b9)
          (optimal);
        \draw[ipfgreen, line width=1pt, -{Latex}]
          (c0) .. controls (cc1) and (cc2) ..
            coordinate[pos=0.15] (c1)
            coordinate[pos=0.26] (c2)
            coordinate[pos=0.36] (c3)
            coordinate[pos=0.45] (c4)
            coordinate[pos=0.55] (c5)
            coordinate[pos=0.64] (c6)
            coordinate[pos=0.71] (c7)
            coordinate[pos=0.78] (c8)
            coordinate[pos=0.86] (c9)
          (optimal);
        \foreach \coor in {0,1,2,3,4,5}
          \node[above=0.05cm, scale=\scaling] at (c\coor) {\PointLabel{\coor}};
        \foreach \coor in {0,1,2,3,4,5} 
          \fill[black] (c\coor) circle (0.04);
        \begin{scope}[xshift=0.56cm, yshift=0.92cm]
          \ngram{0.15}{5}{45}{fill=ipfyellow}
        \end{scope}        
      \end{scope}
      \node[font=\sffamily, scale=0.8] at (2.68,1.65) {$\nablaW F(\cdot)$};
    \end{scope}
    \begin{scope}[xshift=7cm]
      \fill[sinkbg, draw=gray, line width=.5pt, rounded corners=1.5mm]
        (-0.1,0.215) -- (3.051,0.215) -- (3.051,2.175) -- (-0.1,2.175) -- cycle;
      \begin{scope}
        \clip[rounded corners=1.5mm]
          (-0.1,0.215) -- (3.051,0.215) -- (3.051,2.175) -- (-0.1,2.175) -- cycle;
        \fill[verylightgray, opacity=0.5] (0.4,0.4139) circle (17pt);
      \end{scope}
      \begin{scope}[scale=0.63, yshift=-0.34cm]
        \coordinate (a0) at (3.9,3.45);
        \coordinate (ac1) at (3.1,3.4);
        \coordinate (ac2) at (1.1,1.9);
        \coordinate (b0) at (4.4,1.37);
        \coordinate (bc1) at (3.4,1.85);
        \coordinate (bc2) at (1.3,1.35);
        \coordinate (c0) at (4.2, 2.48);
        \coordinate (cc1) at (3.33, 2.55);
        \coordinate (cc2) at (1.17,1.55);
        \coordinate (optimal) at (0.6,0.95);
        \node at (0.25,3.3) {\small\MeasureSpace{}};
        \node[scale=0.9] at ($(optimal) + (0, 1.24)$) {$\{\pct\}_{t=1}^\infty$};
        \draw[ipfred, line width=1pt]
          (a0) .. controls (ac1) and (ac2) ..
            coordinate[pos=0.12] (a1)
            coordinate[pos=0.23] (a2)
            coordinate[pos=0.32] (a3)
            coordinate[pos=0.42] (a4)
            coordinate[pos=0.49] (a5)
            coordinate[pos=0.61] (a6)
            coordinate[pos=0.7] (a7)
            coordinate[pos=0.76] (a8)
            coordinate[pos=0.83] (a9)
          (optimal);
        \draw[ipfblue, line width=1pt]
          (b0) .. controls (bc1) and (bc2) ..
            coordinate[pos=0.12] (b1)
            coordinate[pos=0.23] (b2)
            coordinate[pos=0.32] (b3)
            coordinate[pos=0.41] (b4)
            coordinate[pos=0.47] (b5)
            coordinate[pos=0.55] (b6)
            coordinate[pos=0.66] (b7)
            coordinate[pos=0.72] (b8)
            coordinate[pos=0.79] (b9)
          (optimal);
        \coordinate (c1) at ($(c0) - (0.5,0.2)$);
        \coordinate (c2) at ($(c1) - (0.6,-0.08)$);
        \coordinate (c3) at ($(c2) - (0.3,0.6)$);
        \coordinate (c4) at ($(c3) - (0.5,-0.2)$);
        \coordinate (c5) at ($(c4) - (0.45,0.05)$);
        \coordinate (c6) at ($(c5) - (0.4,0.1)$);
        \coordinate (c7) at ($(optimal) + (0.6,0.5)$);
        \coordinate (c8) at ($(optimal) + (0.4,0.3)$);
        \coordinate (c9) at ($(optimal) + (0.1,-0.05)$);
        \coordinate (c10) at ($(optimal) + (0.1,0.1)$);
        \coordinate (c11) at ($(optimal) + (-0.1, 0.0)$);
        \coordinate (c12) at ($(optimal) + (0.0, 0.15)$);
        \foreach \coor [evaluate=\coor as \coornext using int(\coor+1)] in {0,1,2,3,4,5,6,7,8, 9, 10, 11} {
          \draw[ipfgreen, thick] (c\coor) to (c\coornext);
        }
        \foreach \coor in {0,1,2,3,4,5}
          \node[above=0.05cm, scale=\scaling] at (c\coor) {\PointLabel{\coor}};
        \foreach \coor in {0,1,2,3,4,5,6,7,8,9,10, 11, 12} 
          \fill[black] (c\coor) circle (0.04);
      \end{scope}
      \node[font=\sffamily, scale=0.8] at (2.68,1.65) {$\nablaW F_t(\cdot)$};
    \end{scope}    
  \end{tikzpicture}
  \caption{Learning for an SB model $\{\pi_t\}_{t=1}^\infty$ in the primal space $\cC$ (see \cref{fig:schema} for the details). Left: \Sinkhorn{} (\cref{lem:sink}). Middle: \Wasserstein{} gradient descent in the distributional space $\cC$ for fixed $F$ (\cref{lem:wd}). Right: Variational online mirror descent with sequence of convex costs using uncertain estimates $\{\pct\}_{t=1}^\infty$.} \label{fig:mdvar} 
\end{figure}

\section{Learning \Schrodinger{} Bridge via Online Mirror Descent} \label{sect:md}

The goal in this section is to derive an OMD update rule for SB, and analyze its convergence. To accomplish this objective, we postulate on the existence of temporal estimates and an online learning problem. Our analysis suggests that applying an MD approach can reduce the uncertainty of these estimates.

\subsection{\Sinkhorn{} and \Wasserstein{} descent as mirror descent algorithms} \label{subsect:swmd}

We start with a novel characterization of \Sinkhorn{} as a heuristic form of online mirror descent as illustrated in the left side of \cref{fig:mdvar}, which will lead to a better understanding of our arguments. OMD updates are determined by the first order approximation of costs $F_t$ and proximity of previous iterate with respect to a \Bregman{} divergence \citep{md_nonlin}. Using the first variation $\deltaC$ in \cref{def:firstvar} instead of standard gradient $\nabla$, the proximal form of generalized OMD over $\mathcal{C}$ is derived as \citep{sf}
\begin{equation} \label{eq:omd}
  \pi_{t+1}\mspace{-1mu}= \mspace{-1mu} \argmin_{\pi\in\cC}\Bigl\{\mspace{-1mu}\bigl\langle \deltaC\mspace{-1mu} F_t(\pi_{t}), \pi - \pi_{t} \bigr\rangle\mspace{-1mu} + \tfrac{1}{\eta_t}\mspace{-1mu}D_\Omega(\pi\Vert\pi_{t})\mspace{-1mu}\Bigr\},
\end{equation}
where $F_t$ denotes a temporal cost function for SB models in $\cC$. In \cref{eq:omd}, the updates are determined by the first order approximation of $F_t$ and proximity of previous iterate $\pi_t$ with respect to the \Bregman{} divergence \citep{md_nonlin}. In contrast to the `half-bridge' interpretation provided by \citet{sf}, the online MD iteration \eqref{eq:omd} involves a temporal cost $F_{t}$, which offers more general reinterpretation of the \Sinkhorn{} algorithm. Using the feasible model space $\cC$ in \eqref{eq:c_def}, IPF projections \eqref{eq:sink} are reformulated as following subproblems of alternating \Bregman{} projections:
\begin{equation} \label{eq:canon_sink}
 \argmin_{\pi \in \Pi^\perp_\mu}\mspace{-.5mu}\bigl\{ \KL(\pi\Vert \pi_{2t}) : \pi \in \cC, \gamma_2 \pi = \nu\bigr\},\mspace{40mu} 
 \argmin_{\pi \in \Pi^\perp_\nu}\mspace{-.5mu}\bigl\{ \KL(\pi\Vert \pi_{2t+1}) : \pi \in \cC, \gamma_1 \pi = \mu\bigr\},
\end{equation}  
where $\gamma_1\pi(x) \coloneqq \smallint\pi(x,y)\,\dy$ and $\gamma_2\pi(y) \coloneqq \smallint\pi(x,y)\,\dx$ denotes the marginalization operations, and the symbols $(\Pi^\perp_\mu, \Pi^\perp_\nu)$ denote the \Sinkhorn{} projection spaces that preserve the property of marginals. As a generalized optimization problem in $\cC$, one can consider a temporal cost $\Ftildesmallt(\pi) \coloneqq a_t \KL(\mspace{-1mu}\gamma_1\pi \Vert \mu) + (1- a_t) \KL(\mspace{-1mu}\gamma_2\pi \Vert \nu)$ with sequence $\{a_t\}_{t=1}^\infty = \{0, 1, 0, 1, \dots \}$. By construction, we show that an online form of MD for $\Ftildesmallt$ with a constant step size $\eta_t\!\equiv\!1$ matches the \Sinkhorn{}.
\begin{lemma}[\Sinkhorn{}] \label{lem:sink}
  For $\Omega(\cdot) \mspace{-1mu}= \mspace{-1mu}\KL(\pi\Vert e^{-c_\varepsilon}\mu\otimes\nu)$, iterates from  $\pi_{t+1} = \argmin_{\pi\in\cC}\bigl\{\mspace{-1mu}\langle \deltaC\mspace{-1mu} \Ftildesmallt(\pi_{t}), \pi - \pi_{t} \rangle\mspace{-1mu} + D_\Omega(\pi\Vert\pi_t)\bigr\}$ are equivalent to estimates from $(\varphi_t, \psi_t)$ of \eqref{eq:sink}, for every update step $t\in\mathbb{N}_0$.
\end{lemma}
The proof is in \cref{sect:proof}. Consequently, we established that the \Sinkhorn{} algorithm corresponds to an instance OMD; however, its inherent structure limits flexibility on step sizes and other underlying assumptions, making it challenging to analyze directly using standard OMD theoretical arguments.  Instead, one can alternatively consider OMD by recovering a ``static'' objective, namely $F(\cdot) \coloneqq \KL(\cdot\Vert\pi^\ast)$, where the KL functional is originated from the formal definition of SBP \citep{sbml, likesb}. The following lemma shows that the MD updates for the static cost correspond to discretization of a \Wasserstein{} gradient flow for SB models, a \Riemannian{} steepest descent in the model space of $\cC$. 
\begin{lemma}[MD in the \Wasserstein{} space] \label{lem:wd}
  Suppose that $F(\pi) \coloneqq \KL(\pi\Vert \pi^\ast\mspace{-1mu})$ for $\pi\in\cC$. The MD formulation of $F$ corresponds to a discretization of a geodesic flow such that $\lim_{\eta_t\to0^+}\frac{\pi_{t+1} - \pi_t}{\eta_t} = - \nablaW F(\pi_t)$, where $\nablaW$ denotes the \Wasserstein{}-2 gradient operator.
\end{lemma}
According to \cref{lem:wd}, gradients of $F$ are tangential to the geodesic curve from $\pi_0$ to $\pi^\ast$ (green line in the middle of \cref{fig:mdvar}) in terms of the \Wasserstein{}-2 metric $W_2$. Hence, the geometric interpretation allows us to consider the static cost $F$ as the ground-truth cost for optimization in our variational OMD framework. However, the ideal case falls short in practice since $\pi^\ast$ is inherently unknown. Therefore, we postulate on an online learning problem that nonstationary estimates $\{\pct\}_{t=1}^\infty$ are offered instead of $\pi^\ast$ as learning signals, making an optimization process with $F_t(\cdot) \coloneqq \KL(\cdot\Vert \pct)$. As shown in in the right side of \cref{fig:mdvar}, we focus on the online learning setting where $\{\pct\}_{t=1}^\infty$ are fundamentally uncertain with perturbation, since the optimal coupling $\pi^\ast$ is not accessible during the training time.

\begin{figure*}
  \def\toytitletxt{0.85}
  \tikzset{inner sep=0pt, outer sep=0pt}
  \centering
  \begin{tikzpicture}[tight background]
    \clip (-3.83,-2.35) rectangle (11.72, 2.32);
    \begin{scope}[xshift=0cm]
      \node at (0,1.2) {\includegraphics[width=195pt]{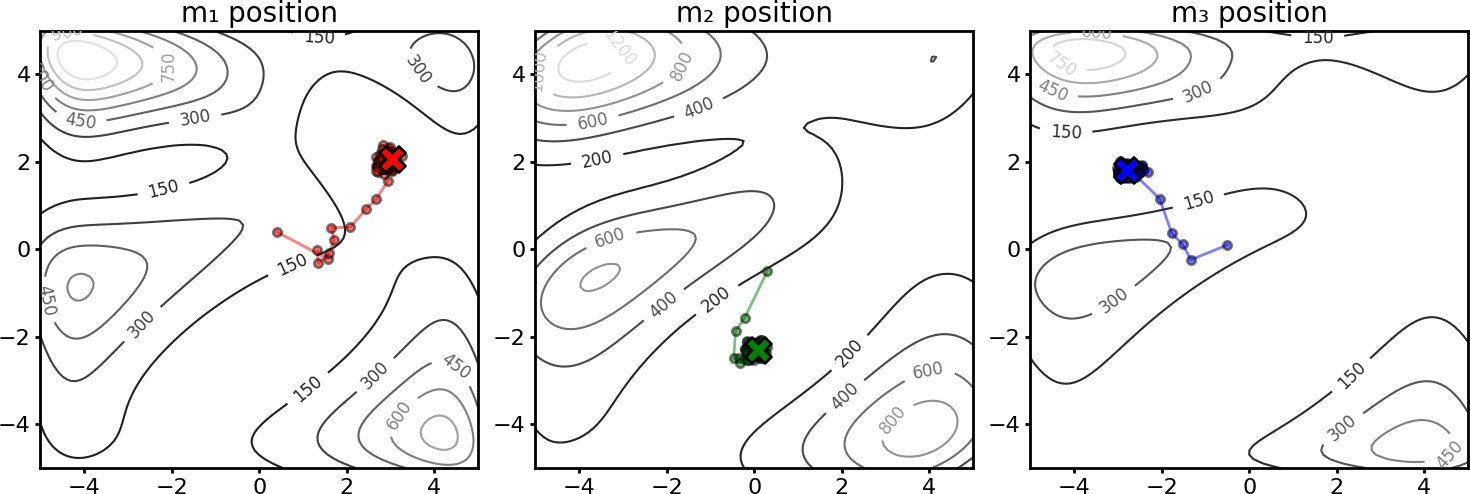}}; 
      \node at (0,-1.2) {\includegraphics[width=195pt]{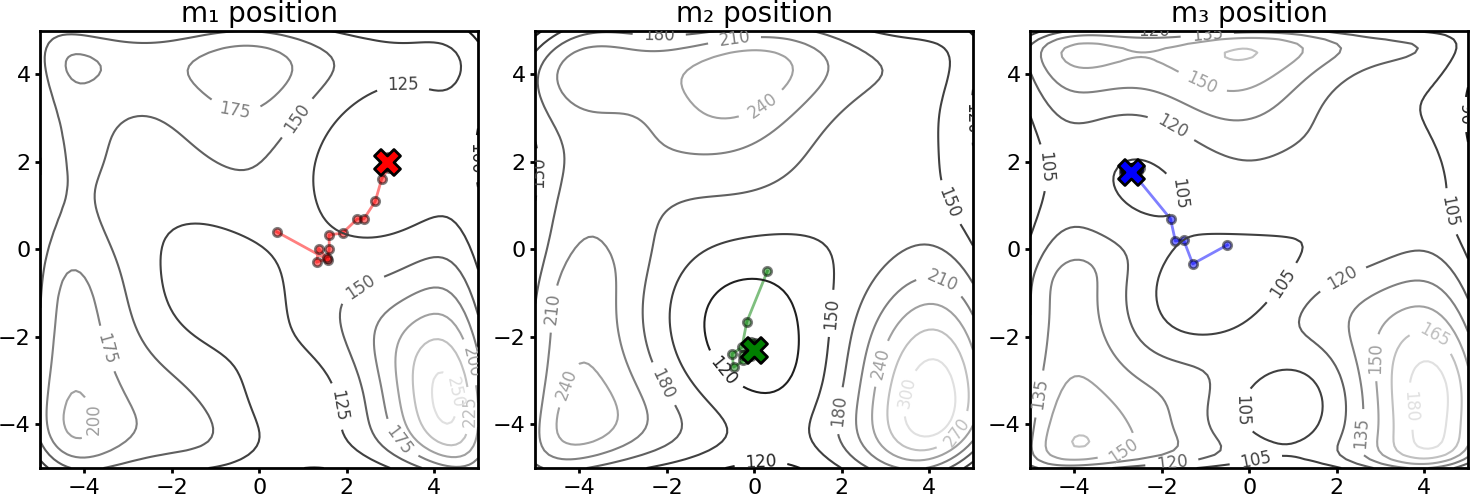}}; 
      \node[rotate=90, scale=0.75] at (-3.67,1.2) {\bfseries \textsf{LightSB} ($\pct$)};
      \node[rotate=90, scale=0.75] at (-3.67,-1.18) {\bfseries \textsf{Ours} ($\pi_t$)};
    \end{scope}
    \begin{scope}[xshift=7.9cm]
      \node at (0,1.2) {\includegraphics[width=195pt]{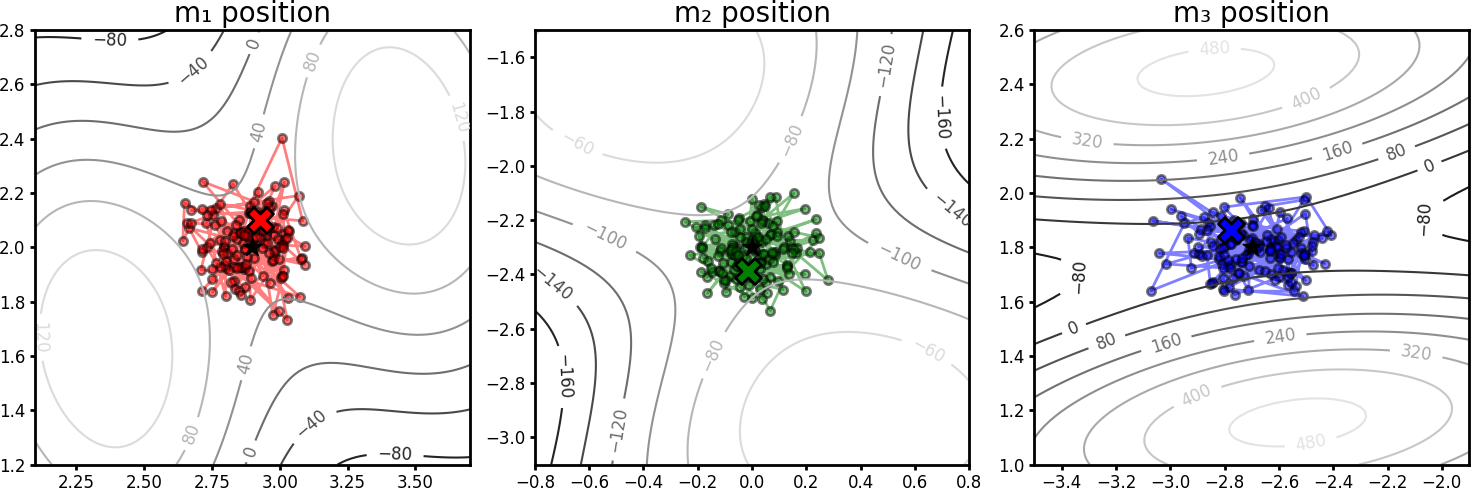}};
      \node at (0,-1.2) {\includegraphics[width=195pt]{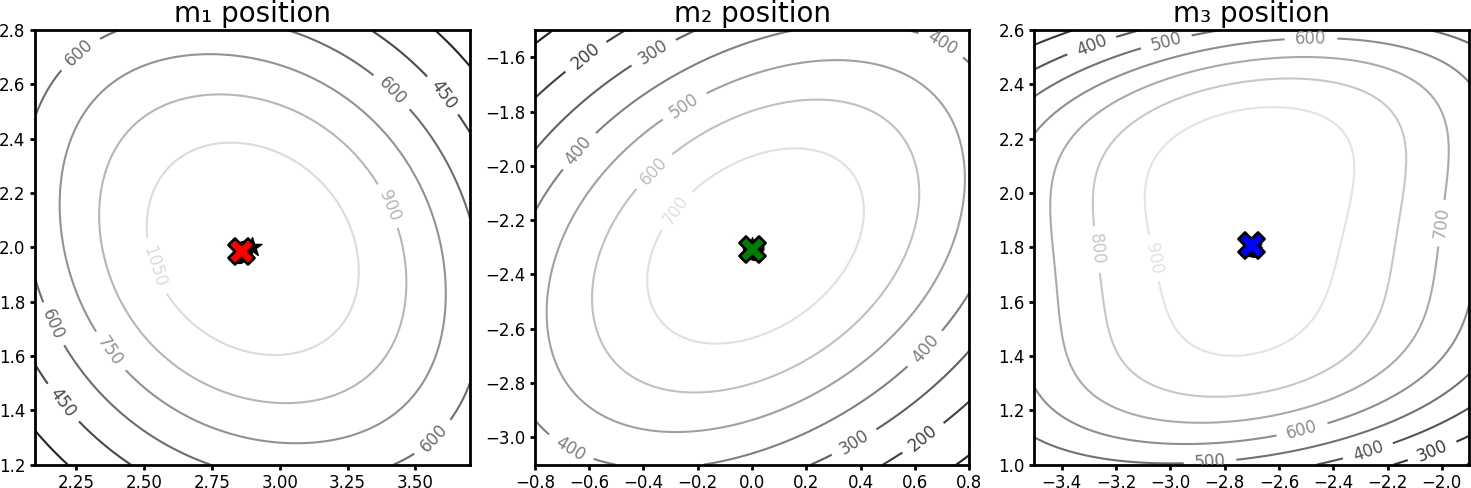}};
      \node[rotate=90, scale=0.75] at (-3.67,1.2) {\bfseries \textsf{LightSB}};
      \node[rotate=90, scale=0.75] at (-3.67,-1.18) {\bfseries \textsf{Ours}};
    \end{scope}
  \end{tikzpicture}
  \caption{Loss landscapes and gradient dynamics in a 2D problem. Left: In an early stage, parameters of three modalities $\{m_k\}_{k=1}^3$ (mean estimations) for both LightSB (top) and VMSB (bottom) methods approach the optimality with different costs. Right: When magnified the landscapes in the late stages (10 times), while LightSB is vibrant, whereas our method emits strictly convex landscape and stable dynamics.} \label{fig:loss_landscape}
\end{figure*}

\subsection{Online mirror descent for \Schrodinger{} bridges: theoretical analysis} \label{subsect:analysis}

In this section, we examine the robustness of our OMD algorithm by analyzing its convergence behavior under statistical uncertainty. From a learning theoretic standpoint, an apparent yet understated premise is that an SB algorithm does not retain the global target $\pi^\ast$ in practice. The global objective $F$ (also $\Ftildesmallt$) are fundamentally unknown, but are instead inferred, imposing innate uncertainty of optimization. Instead, we postulate on an online learning problem that nonstationary ergodic estimates $\{\pct\}_{t=1}^\infty$ are offered instead of $\pi^\ast$ (gray region in \cref{fig:schema}). Let $\Omega^\ast$ be the \Fenchel{} conjugate of $\Omega + \mspace{-1mu}i_{\mspace{.5mu}\cC}$ with the convex indicator\footnote{Defined as $i_{\mspace{.5mu}\cC}(x) = 0$ if $x \in \cC$ and $+\infty$ otherwise.} $i_{\mspace{.5mu}\cC}$. For the space $\cD \coloneqq \deltaC \Omega(\cC)$, a directional derivative $\deltaD$ of $\Omega^\ast$ exists by the \Danskin{}'s theorem \citep{danskin2012theory, bernhard1995theorem}, such that
\begin{equation}
  \deltaD \Omega^\ast(\varphi \oplus \psi) = \argmax_{\pi \in\cC} \bigl\{ \langle \varphi \oplus \psi, \pi \rangle - \Omega(\pi) \bigr\}.
\end{equation}
Note that $(\deltaC\Omega, \deltaD \Omega^\ast\mspace{-1mu})$ form bidirectional maps; a direct sum of potentials $\varphi\oplus\psi \in \cD$ represent an element of the generalized dual space. The key assumption is that the learning target $\pct$ is asymptotically mean stationary \citep{gray1980asymptotically} for the dual space, which have been used to analyze stochastic dynamics. Since iterates are updated through dual parameters in MD, we refer to the process as being dually stationary.
\begin{assumption}[Dually stationary process] \label{asm:bs}
  Suppose that $\pcD \in \cC$ exists, which is the primal representation of asymptotic dual average $\pcD \coloneqq \deltaD (\lim_{t\to\infty} \bE_{t}[\deltaC\Omega( \pct)])$, where the notation $\bE_{t}$ denotes the time-average.
\end{assumption}

\begin{wrapfigure}{r}{0.24\textwidth}
  \vskip-10.8pt
  \centering
  \begin{tikzpicture}[scale=0.42, transform shape,tight background, every node/.style={inner sep=0,outer sep=0},on grid]
    \clip (-0.8,-0.34) rectangle+(8.36,6.0);
    \begin{axis}[grid=both, mark=none, xmin=0.85, ymin=0, xmax=3.4, ymax=2.4,
      axis line style=ultra thick,
      axis lines=middle,
      enlargelimits=upper,
      clip=false,
      xtick={1,1.6,...,3.4},
      ytick={0,0.6,...,2.4},
      yticklabel=\empty,
      xticklabel=\empty,
      every axis plot/.append style={very thick},
      legend style={at={(0.65,0.83)},anchor=west}]
      \addplot[line width=0.65mm,domain=1:3.6,restrict y to domain=0:2.6, samples=100]{1/(x-0.5)};
      \addplot[line width=1mm,dashed, domain=1:3.6,restrict y to domain=0:2.6, samples=100]{1-1/(x-0.5))};
      \node[circle,fill,inner sep=3.5pt] at (axis cs:1,2) {};
      \legend{\LARGE$\eta_t$,\LARGE$1\!-\!\eta_t$};
    \end{axis}
    \node at (7.05,-0.1) {\huge $t$};
    \node at (-0.44,4.3) {\huge $\eta_1$};
  \end{tikzpicture}
  \vskip-3pt
  \caption{A sequence example of $\eta_t$ and $1- \eta_t$.} \label{fig:eta_eg}
  \vskip-10pt
\end{wrapfigure}
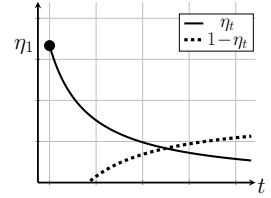
Plugging a temporal cost$F_t(\cdot) = \KL(\cdot\Vert\pct)$ to \eqref{eq:omd} from \S~\ref{subsect:swmd}, we achieve a distinct problem setup of online mirror descent. \cref{fig:loss_landscape} shows a demonstrative experiment regarding our online learning hypothesis. OMD decomposes the global problem into local convex problems, and prevented iterates from being vibrant by stopping at a single point $\pcD$. This verifies that OMD stabilizes learning of $\pi_t$, even when the reference $\pct$ tends to inherently have some perturbation. Additionally, we state two conditions for the sequence $\{\eta_t\}_{t=0}^\infty$, which will be justified in \cref{thm:step_size,prop:convergence}. \cref{fig:eta_eg} shows a plot of harmonic progression $\frac{1}{a+td}$ for $a\in\bR$ and $d\in\bR^+$ with respect to $t$ which satisfies both conditions. 
\begin{assumption}[Step sizes] \label{asm:eta} 
  Assume two conditions for step sizes $\{\eta_t\}_{t=0}^\infty$. (a) \textit{Convergent sequence \& divergent series:} $\lim_{t\to \infty }\eta_t=0$ and $\sum\nolimits_{t=1}^\infty \eta_t=\infty$. (b) \textit{Convergent series for squares:} $\sum\nolimits_{t=1}^\infty \eta^2_t < \infty$.
\end{assumption}
Under the above assumptions, we firstly argue that OMD for the temporal cost $\KL(\cdot\Vert\pct)$ with respect to the \Bregman{} potential $\Omega = \KL(\cdot\Vert e^{-c_\varepsilon}\mu\otimes\nu)$ requires step size scheduling for the sake of convergence. The following theorem states that convergence in the case of $\pi^\ast=\pcD$ is assured when $\eta_t$ follows Assumption~(\ref{asm:eta}a). In contrast to the well-known linear convergence guarantees for the \Sinkhorn{} algorithm under bounded costs with fixed marginals \citep{centered}, our OMD-based analysis establishes sublinear convergence rates, accommodating scenarios involving unbounded and non-stationary costs.
\begin{theorem}[Step size considerations] \label{thm:step_size}
  Suppose the idealized case of $\pi^\ast=\pcD$. Then, for $\{\pi_t\}_{t_1}^T\subset \cC$ we get $\lim_{T\to\infty}\bE_{1:T}[D_\Omega(\pcD \Vert \piT)] =0$ if and only if Assumption~(\ref{asm:eta}a) is satisfied. Furthermore, if the step size is in the form of $\eta_t = \frac{2}{t+1}$, then $\bE_{1:T}[D_\Omega(\pi^\ast \Vert \pi_t)] = \cO(1/T)$.
\end{theorem}
Therefore, we can guarantee the ideal convergence in the SB learning when the scheduling of $\eta_t$ follows the step size assumptions. Next, we argue that general convergence toward $\pcD$ is guaranteed under Assumption~(\ref{asm:eta}b). Given the convex nature of SB cost functionals, we argue that this convergence toward $\pcD$ is beneficial as long as $\pct$ is trained to approximate $\pi^\ast$ and remain bounded. Therefore, we argue that the convergence of SB is beneficial and address the following statement.
\begin{proposition}[Convergence] \label{prop:convergence}
  Suppose that $\pi^\ast\ne\pcD$, hence $\inf_{\pi\in\cC} \bE[F_t(\pi)] >0$. If the step sizes $\{\eta_t\}_{t=0}^\infty$ satisfies \cref{asm:eta}, then $\lim_{t\to\infty} \bE_{1:t}[D_\Omega(\pcD \Vert \pi_t)]$ converges to $0$ almost surely. 
\end{proposition}
Lastly, assume that the log \Sobolev{} inequality in \cref{asm:lsi} holds  with continuity of potentials. We establish an online learning regret bound of $\cO(\textrm{\small$\sqrt{T}$})$ for certain instance of step sizes, demonstrating that imposing specific measure-theoretic properties in SBPs generalizes classical OMD results \citep{nesterov2009primal, omd_univ, orabona2018scale, omd_converge}. The analysis on our general online learning setup is compatible with these results by using the \textit{dual} norm $\lVert\hat{g}_t\rVert$, which represents a distance between the OMD iterate $\pi_t$ and the empirical estimate $\pct$ in the dual space $\cD$, where the generalized notional of duality comes from the transformation $\deltaC\Omega(\cdot)$, defined by applying the first variation of $\Omega$ with respect to $\cC$.
\begin{proposition}[Regret bound] \label{thm:regret}
  Assume $\varphi,\psi\in C^2(\bR^d)\cap \Lip(\cK)$, \cref{asm:lsi} holds, and the given costs $\{F_t\}_{t=1}^T$ are bounded. For arbitrary $u\in\cC$ and a total step $T$, define $D^2 = \max_{1 \le t \le T} D_\Omega(u\Vert \pi_t)$. (a) When the number of time step is  known a priori, the regret is bounded to {\small $2D\sqrt{2 \omega^{-1} \cK T}$} for a constant step size  {\small $\eta \equiv \frac{D\sqrt{\omega}}{\sqrt{2 \cK T}}$}. (b) For an adaptive scheduling {\small $\eta_t = D\sqrt{\omega}/ \sqrt{2\sum_{i=1}^{t} \lVert \hat{g}_i\rVert^2}$} the regret is bounded to  {\small $2D\sqrt{2\omega^{-1}\sum_{t=1}^{T} \lVert \hat{g}_t \rVert^2 }$} where $\hat{g}_t \coloneqq \deltaC\Omega(\pi_t) - \deltaC\Omega(\pct)$.
\end{proposition}
Note that although our analysis establishes a rigorous connection between SB and OMD, it inherits certain limitations from classical regret analyses. For instance, sublinear regrets in Proposition~\ref{thm:regret} relies on an additional boundedness assumption on costs, and there exist some cases of \cref{asm:eta} that may yield asymptotically linear regret \citep{orabona2018scale}. Addressing these limitations may involve advanced hybrid OMD methodologies which are actively being studied, such as dual averaging \citep{fang2022online} or FTRL \citep{chen2023generalized}. As exploring (as well as computing) such extensions for SBPs falls beyond our scope, from now on we focus on practical implementations of OMD that demonstrate theoretical convergence (\cref{prop:convergence}) over a relatively long horizon of several hundred steps. Hence, we adopt \cref{asm:eta} and provide corresponding experimental evidence to substantiate its validity in practical scenarios.

\begin{figure*}
  \definecolor{schedulelinecolor}{RGB}{25,119,178}
  \tikzset{inner sep=0pt,outer sep=0pt}
  \centering
  \begin{tikzpicture}
    \clip (-8.68,-1.753) rectangle (7.65, 1.771);
    \begin{scope}[xshift=-4.4cm]
      \node at (-2.62,0) {\includegraphics[width=89.7pt] {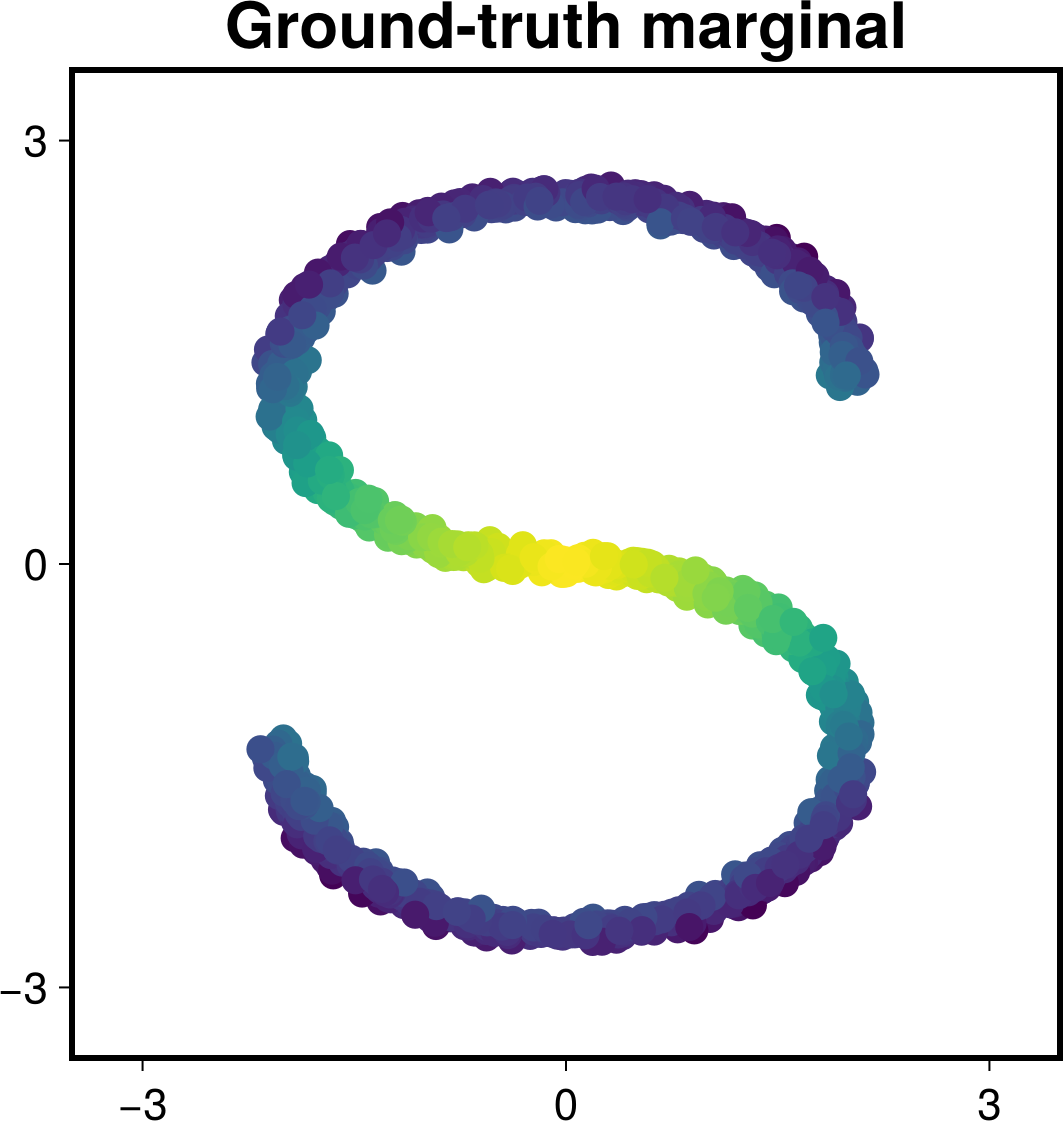}};
      \begin{scope}[xshift=-0.05cm, yshift=-0.17cm]  
        \node[scale=0.6] at (0.7,1.8){\textsf{\textbf{Trained distribution}}};
        \node[scale=0.55] at (2.74,1.81){\textsf{\textbf{Zoom 4X}}};
        \node[scale=0.55, rotate=90] at (-0.53,0.85){\textsf{\textbf{Monte Carlo}}};
        \node[scale=0.55, rotate=90] at (-0.53,-0.85){\textsf{\textbf{Variational MD}}};
        \begin{scope}
          \begin{scope}[yshift=0.85cm] 
            \node at (0.7,0) {\includegraphics[width=62.5pt]{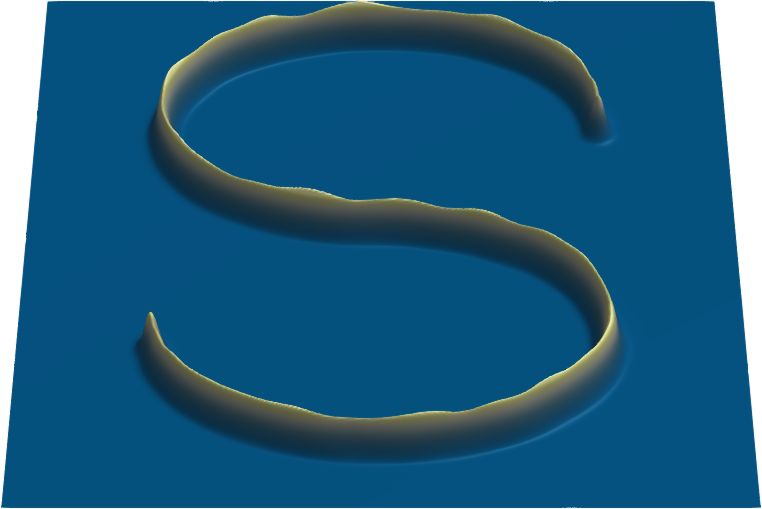}};    
            \coordinate (Zoom) at (2.75, 0);
            \coordinate (ZoomTL) at ($(Zoom) + (-0.725, 0.725)$);
            \coordinate (ZoomBL) at ($(Zoom) + (-0.725, -0.725)$);
            \coordinate (ZoomTR) at ($(Zoom) + (0.725, 0.725)$);
            \coordinate (ZoomL) at ($(Zoom) - (0.725, 0)$);
            \coordinate (ZoomBR) at ($(Zoom) + (0.725, -0.725)$);
            \coordinate (Zoomed) at (0.68, 0.4);
            \coordinate (ZoomedTL) at ($(Zoomed) + (-0.3625, 0.3625)$);
            \coordinate (ZoomedTR) at ($(Zoomed) + (0.3625, 0.3625)$);
            \coordinate (ZoomedR) at ($(Zoomed) + (0.3625, 0)$);
            \coordinate (ZoomedBR) at ($(Zoomed) + (0.3625, -0.3625)$);
            \coordinate (ZoomedBL) at ($(Zoomed) + (-0.3625, -0.3625)$);
            \path (Zoomed) -- (ZoomTL) coordinate[pos=0.4] (TopSplit);
            \path (Zoomed) -- (ZoomBL) coordinate[pos=0.47] (BotSplit);
            \begin{scope}[xshift=2.75cm, yshift=-0.7cm]
              \clip (-0.725,0.1) rectangle (0.725, 1.55);
              \node[draw, thick] at (0,0) {\includegraphics[width=125pt]{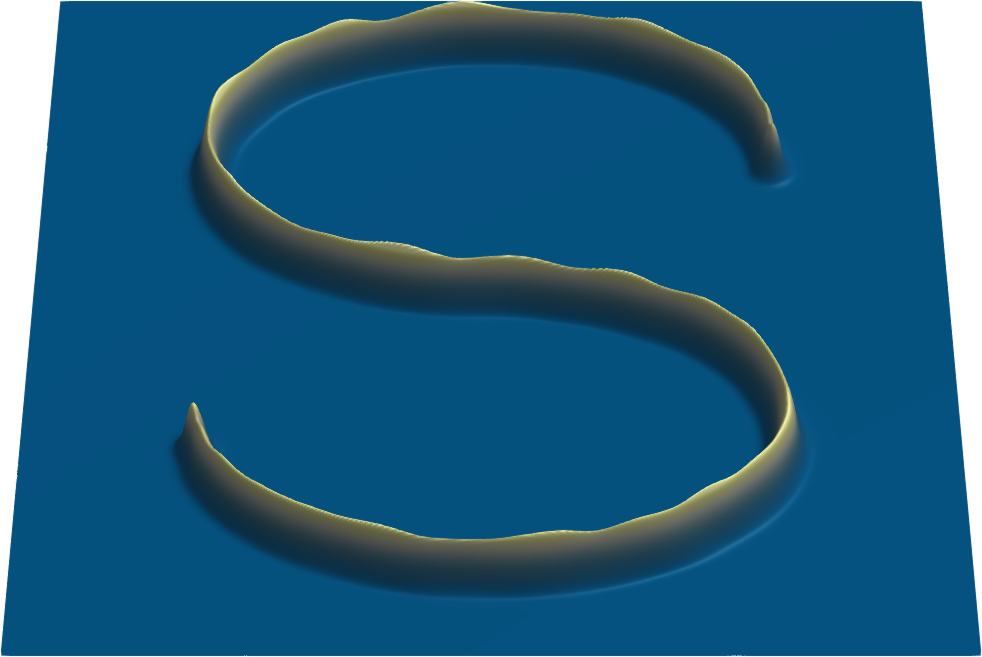}}; 
              \draw[line width=2.4pt, orange] (-0.725,0.1) rectangle (0.725, 1.55);
            \end{scope}
            \draw[orange, thick] (ZoomedBL) rectangle (ZoomedTR);
            \draw[orange, thick] (ZoomL) -- (ZoomedR);
          \end{scope} 
          \begin{scope}[yshift=-0.85cm]
            \node at (0.7,0) {\includegraphics[width=62.5pt]{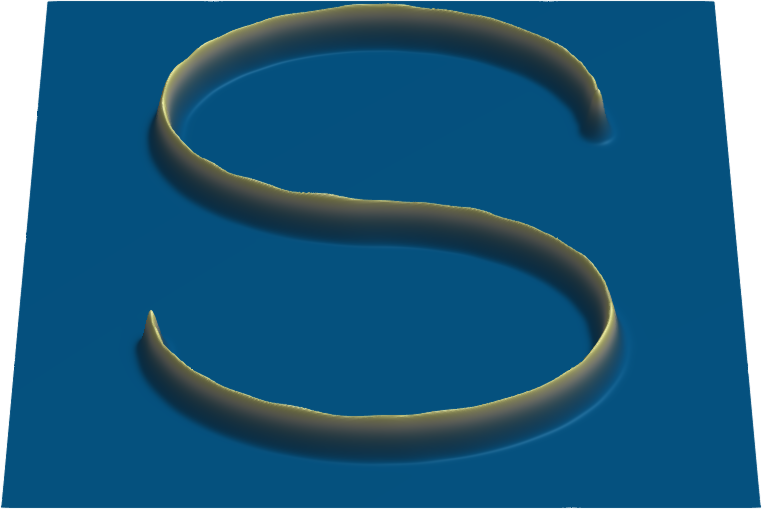}};  
            \coordinate (Zoom) at (2.75, 0);
            \coordinate (ZoomTL) at ($(Zoom) + (-0.725, 0.725)$);
            \coordinate (ZoomBL) at ($(Zoom) + (-0.725, -0.725)$);
            \coordinate (ZoomTR) at ($(Zoom) + (0.725, 0.725)$);
            \coordinate (ZoomL) at ($(Zoom) - (0.725, 0)$);
            \coordinate (ZoomBR) at ($(Zoom) + (0.725, -0.725)$);
            \coordinate (Zoomed) at (0.68, 0.4 );
            \coordinate (ZoomedTL) at ($(Zoomed) + (-0.3625, 0.3625)$);
            \coordinate (ZoomedTR) at ($(Zoomed) + (0.3625, 0.3625)$);
            \coordinate (ZoomedR) at ($(Zoomed) + (0.3625, 0)$);
            \coordinate (ZoomedBR) at ($(Zoomed) + (0.3625, -0.3625)$);
            \coordinate (ZoomedBL) at ($(Zoomed) + (-0.3625, -0.3625)$);
            \path (Zoomed) -- (ZoomTL) coordinate[pos=0.4] (TopSplit);
            \path (Zoomed) -- (ZoomBL) coordinate[pos=0.47] (BotSplit);
            \begin{scope}[xshift=2.75cm, yshift=-0.7cm]
              \clip (-0.725,0.1) rectangle (0.725, 1.55);
              \node[draw, thick] at (0,0) {\includegraphics[width=125pt]{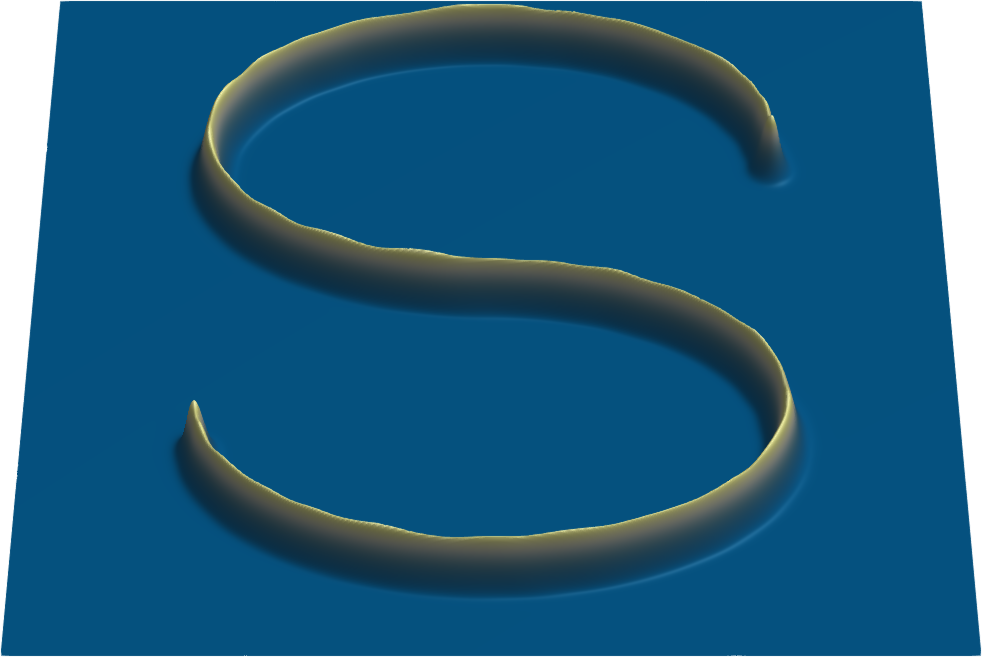}}; 
              \draw[line width=2.4pt, orange] (-0.725,0.1) rectangle (0.725, 1.55);
            \end{scope}
            \draw[orange, thick] (ZoomedBL) rectangle (ZoomedTR);
            \draw[orange, thick] (ZoomL) -- (ZoomedR);
          \end{scope} 
        \end{scope}
      \end{scope}
    \end{scope}
    \node at (3.45,0) {\includegraphics[width=224pt]{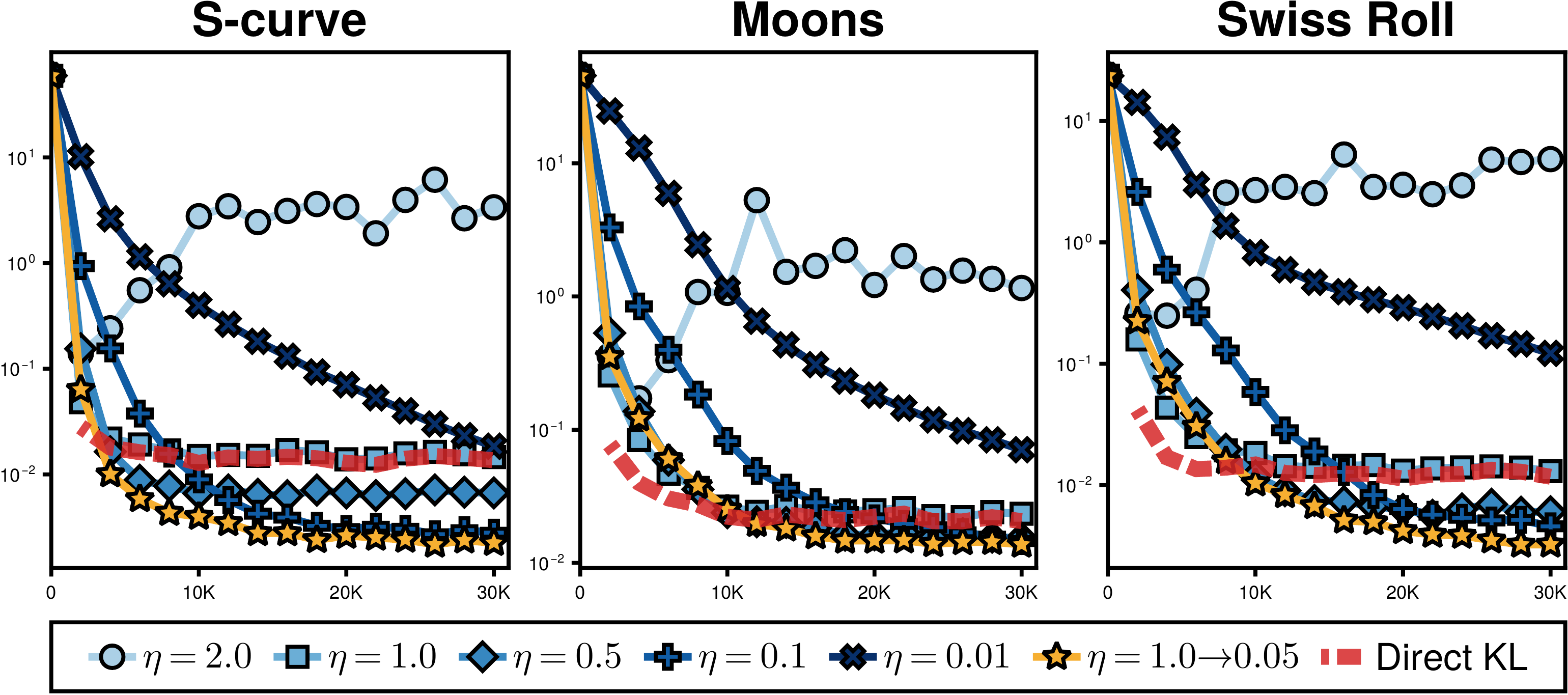}};
    \node[scale=0.48, rotate=90] at (-0.56,0.15) {\textsf{\textbf{KL estimates \small(100K)}}};
    \node[scale=1.39] at (-8.55,1.58) {\textsf{a}};
    \node[scale=1.39] at (-4.97,1.585) {\textsf{b}};
    \node[scale=1.39] at (-0.55,1.58) {\textsf{c}};
  \end{tikzpicture}
  \caption{Variational MD with synthetic datasets. \subfigA{} A distribution is accessible by finite batch data. \subfigB{} 3D surfaces of $(\vec{\pi}^\circ_{\mspace{-1mu}\scriptscriptstyle T}, \vec{\pi}_{\mspace{-1mu}\scriptscriptstyle T})$ trained by \MonteCarlo{} method for KL (top) and variational MD (bottom) show that the MD results in more stable outcomes. \subfigC{} The plots show the estimated $\KL(\vec{\pi}_{\mspace{-1mu}t} \Vert \vec{\pi}^\ast)$ with different step size scheduling (5 runs), with red dashed baselines $\KL(\vecpct \Vert \vec{\pi}^\ast)$.} \label{fig:md_demo}
\end{figure*}

\subsection{Online mirror descent updates using \Wasserstein{} gradient flows} \label{subsect:md_wgf}

The remaining challenge is practicality of our OMD theory, namely the accurate computability of each critical point for \eqref{eq:omd}. To resolve this issue, this work newly presents an approximation method using \Wasserstein{} gradient flows \citep{jko} through an equivalence property of first variations. Suppose we expand a subinterval $[t, t+1)$ for each OMD step \eqref{eq:omd} into continuous dynamics of $\rho(\tau) \in \cC$ for a $\tau \in [0, \infty)$. By \Otto{}'s calculus on the \Wasserstein{} space, known as the \Otto{} calculus (\citealp{otto}; see Appendix~\ref{subsect:otto_intro}), one can describe the dynamics of $\rho_\tau$ for minimizing a strictly convex functional $\cE_t:\cC \to \bR$ as the PDE  
\begin{equation} \label{eq:main_pde}
  {\partial_\tau \rho_\tau = -\nablaW\mspace{1mu} \cE_t(\rho)},
\end{equation}
where $\nablaW$ denotes the \Wasserstein{}-2 gradient operator $\nablaW \coloneqq \nabla\!\vcdot\!\bigl(\rho\,\nabla \frac{\delta}{\delta \rho}\bigr)$. In this work, we adopt the \Wasserstein{} gradient flow theory \citep{jko} to efficiently perform OMD where the equilibrium indicates the subsequent iterate $\pi_{t+1}$. Note that \Wasserstein{} gradient flows are asymptotically stable by the \LaSalle{}'s invariance principle \citep{ipgf}. We present a simple and exact closed-form expression for the VOMD update. Note that the cost $F_t(\cdot) = \KL(\cdot\Vert\pct)$ satisfies the 1-relative-smoothness and 1-strong-convexity relative to $\Omega$ (see \cref{def:relsmooth}; \citealp{mdsem}). Then, a first variation of the OMD problem can be decomposed into multiple variations of another problem with similar characteristics (\eg{}, equilibrium, smoothness, and convexity). We present the following theorem for the computation of OMD.
\begin{theorem}[Dynamics equivalence in first variation] \label{thm:vmdkl}
  Consider the \Wasserstein{} gradient dynamics of \eqref{eq:main_pde} which solves a local update of \eqref{eq:omd}. The gradient dynamics of updates are equivalent to that of a linear combination of KL functionals such that for any $\rho_\tau\in\cC$
  \begin{equation} \label{eq:vmdkl}
    \eta_t \deltaC \cE_t(\rho_\tau) = \deltaC\mspace{-1mu} \bigl\{\eta_t\mspace{1mu}\KL(\rho_\tau \Vert \pct)  +  \mspace{-1mu}(1\mspace{-1mu}-\mspace{-1mu}\eta_t)\mspace{1mu}\KL(\rho_\tau \Vert \pi_t)\mspace{-1mu}\bigr\} \quad \forall \rho_\tau\in\cC,
  \end{equation}
  and the PDE \eqref{eq:vmdkl} converges to a unique critical point of subsequent OMD iterate \eqref{eq:omd} as $\tau\to \infty$.
\end{theorem}
\begin{proof}[Sketch of Proof] 
  We identify $\delta\cE_t$ as a dynamics that reaches an equilibrium solution for 
  \begin{equation} \label{eq:omd_update}
  \begin{aligned}
    &\minimize_{\pi\in\cC}\, \bigl\langle \deltaC F_t(\pi_t), \pi - \pi_t \bigr\rangle + \tfrac{1}{\eta_t} D_\Omega(\pi\Vert\pi_t) \\
    &\mspace{60mu}\iff\mspace{15mu} \minimize_{\pi\in\cC}\, \eta_t\mspace{-5mu} \underbrace{D_\Omega(\pi \Vert \pct)}_\textrm{empirical estimates} +\ (1 - \eta_t) \underbrace{D_\Omega(\pi \Vert \pi_t)}_\textrm{proximity}, 
  \end{aligned}
  \end{equation}
  and then the equivalence of first variation for recursively defined \Bregman{} divergences is applied (\cref{lem:equivf}). At a glance, \cref{eq:omd_update} appears analogous to the interpolation search between two points, where the influence of $\pct$ is controlled by $\eta_t$. We leave the full version of proof in \cref{subsect:proof_grad_equiv}.
\end{proof}
\cref{thm:vmdkl} holds practical importance for OMD computation, since following the argument allows us to perform gradient-based updates without directly constructing a desired \Bregman{} divergence. That is, updates can be drawn based on a linear combination of gradient flows $\eta_t\nablaW \KL(\rho_\tau\Vert\pct)+(1-\eta_t)\nablaW \KL(\rho_\tau\Vert\pi_t)$. Note that, just like ordinary gradients, Wasserstein gradient operators on a measure space allow for this direct translation of \cref{eq:vmdkl}, where such expression has been extensively studied both theoretically and computationally \citep{ipgf, viwgf}. Therefore, we can utilize interpolation of \Wasserstein{} gradient flows for performing updates and utilize a certain variational class for reducing the computational cost. \cref{fig:md_demo} shows our actual experiments using GMMs. Let a reference estimation be fitted using a \MonteCarlo{} method, and our model be trained through a variational OMD method which is explained in the subsequent section. We initially observed that the VOMD method provides stability improvement when $\eta < 1$. In contrast, the condition of $\eta > 1$ performed worse than the Monte Carlo method and $\eta = 1$ showed almost equivalent performance. Furthermore, the performance of VOMD was greatly improved by choosing a harmonic step size scheduling in the interval $[1.0, 0.05]$. All of these results on variational approximation precisely matches our analysis.

\section{Variational Mirrored \Schrodinger{} Bridge} \label{sect:vmsb}

In this section, we propose variational mirrored \Schrodinger{} bridge, a simulation-free method that offers iterative MD updates for parameterized SB models with mixture models, using  the \WassersteinFisherRao{} geometry. We provide a tractable and exact VOMD-based update rule for LightSB models and draw a practical VOMD updates algorithm that closely resembles ordinary machine learning methods.

\subsection{\Gaussian{} mixture parameterization for the \Schrodinger{} bridge problem} \label{subsect:lsb}

In order to translate our theoretical arguments on VOMD into practical algorithm implementation, this section focuses on a computational implementation of our theory. Recently, \citet{lsb} proposed the GMM parameterization, which provides theoretically and computationally desirable models for our variational OMD approach.\footnote{Adapting the GMM parameterization to our theory is straightforward, achieved by specifying the cost $c_\varepsilon(x,y) = \nicefrac{1}{2\varepsilon}\lVert x - y\rVert^2$.} The parameterization considers the \textit{adjusted} \Schrodinger{} potential $u^\ast\mspace{-1mu}(x) \coloneqq \exp\mspace{-1mu}(\varphi^\ast\mspace{-1mu}(x) - \nicefrac{\lVert x \rVert^2}{2\varepsilon}\mspace{-.5mu})$ and $v^\ast\mspace{-1mu}(y) \coloneqq \exp\mspace{-.5mu}(\psi^\ast\mspace{-1mu}(y) - \nicefrac{\lVert y \rVert^2}{2\varepsilon}\mspace{-.5mu})$ such that we have a proportional property $\pi^\ast(y \vert x) \propto \exp(\nicefrac{\langle x, y\rangle}{\varepsilon}) v^\ast(y)$. With a finite set of parameters $\theta\triangleq \{\alpha_k, m_k, \Sigma_k\}_{k=1}^K$ for weights $\alpha_k > 0$, means $m_k \in \bR^d$ and covariances $\Sigma_k \in \vSppd$, \citet{lsb} proposed to approximate the adjusted \Schrodinger{} potential $v_\theta$ and conditional probability density $\vec{\pi}_\theta$
\begin{equation} \label{eq:gmm_param}
  v_\theta(y) \coloneqq \sum_{k=1}^K \alpha_k\mspace{2mu}\cN\mspace{-1.7mu}(\mspace{1mu}y\mspace{1.2mu}\vert m_k\mspace{-1mu}, \varepsilon \Sigma_k),\mspace{70mu} \vec{\pi}^x_\theta(y) \coloneqq \frac{1}{z^x_\theta}\sum_{k=1}^K \alpha^x_k\mspace{2mu} \cN(\mspace{1mu}y\mspace{1.2mu}\vert \mspace{1mu}m^x_k, \varepsilon \Sigma_k),
\end{equation}
where GMM component for $\vec{\pi}^x_\theta$ is conditioned by an input $x$: $m^x_k \coloneqq m_k + \Sigma_k x$, $\alpha^x_k \coloneqq \mspace{-1mu}\alpha_k \exp\mspace{-1mu}\bigl(\mspace{-1mu}\tfrac{x^{\sT} \Sigma_k x + \langle m_k, x\rangle}{2\varepsilon}\mspace{-1mu}\bigr)$, $z^x_\theta \coloneqq \sum_{k=1}^K \alpha^x_k$ (see Proposition~3.2 of \citeauthor{lsb}). For this parameterization, the closed-from expression of SB process $\cT_\theta$ is given as the following SDE for $t\in[0, 1)$:
\begin{equation} \label{eq:lsb_drift}
  \begin{gathered}
    \cT_\theta: \rd X_t = g_\theta(t, X_t)\,\dt + \sqrt{\varepsilon}\,\dWt,\\
    g_\theta\mspace{-1mu}(t, x) \coloneqq \varepsilon\mspace{1mu} \nabla\mspace{-1mu} \log\mspace{-1mu}\cN\mspace{-1mu}(\mspace{1mu}x\mspace{1.5mu}\vert\mspace{1.5mu} 0, \varepsilon(1-t)I_d) \sum_{k=1}^K \alpha_k\, \cN\mspace{-1mu}(m_k\mspace{1.5mu}\vert\mspace{1.5mu} 0, \varepsilon \Sigma_k)\mspace{1.5mu}\cN\bigl(m_k\mspace{-1mu}(t,x) \mspace{1.5mu}\big\vert\mspace{1.5mu} 0, A_k\mspace{-1mu}(t)\bigr),
  \end{gathered}
\end{equation}
where $m_k(t,x) \triangleq \frac{x}{\varepsilon(1-t)} + \frac{1}{\varepsilon} \Sigma_k^{-1} m_k$ and $A_k(t) \triangleq \frac{t}{\varepsilon(1-t)} I_d + \frac{1}{\varepsilon} \Sigma_k^{-1}$. Therefore, the LightSB parameterization represent both static and dynamic SB models and arbitrary SB solvers can be applied without restrictions. We utilize the GMM parameterization for our computational algorithm for three key reasons. Firstly, the parameterization induce the universal approximation property for both $\vec{\pi}_\theta$ and $\cT_\theta$ \citep{lsb}. Secondly, GMMs are asymptotically log-concave (see \cref{lem:gmmalc}), which is a fundamental assumption for our theory. Lastly, the GMM parameterization makes the computation of \Wasserstein{} gradient flows with respect to the KL divergence tractable, which in turn enables us to apply a canonical treatment, akin to solving ordinary convex optimization problems endowed with a Riemannian geometry. 

\subsection{Computation of VOMD in the \WassersteinFisherRao{} geometry} \label{subsect:wfr}
For tractable computation, we formally derive a particular variant of Wasserstein gradient flow for the GMM parameterization. The space of \Gaussian{} parameters $\bR^d \times \vSppd$ ($d$-dimensinoal mean vectors and positive definite, symmetric convariance matrices), endowed with the Wasserstein-2 metric $W_2$, is formally recognized as the \BuresWasserstein{} (BW) geometry \citep{bures, bhatia2019bures, viwgf} $\BW(\bR^d) \subseteq \cP_2(\bR^d)$. The \WassersteinFisherRao{} (WFR) geometry, equivalently characterized by the spherical \HellingerKantorovich{} distance, extends this setting by considering liftings of positive, complete, and separable measures while preserving total mass \citep{liero2018, chizat2018, lu2019accelerating}.  Building upon the BW space, the \WassersteinFisherRao{} geometry of GMMs, namely $\cP_2(\BW(\bR^d))$, naturally provides liftings of \Gaussian{} particles satisfying distributional consistency. In this work, we introduce the following proposition, which refines and extends the results from \citet[\S~6]{viwgf}  specifically enhancing their framework through the introduction of freely trainable GMM weights $\alpha_k$. 
\begin{proposition}[WFR gradient dynamics] \label{prop:wfr}
  Suppose a time-varying GMM model $\rho_{\theta_\tau}$ with the parameter $\theta_\tau = \{\alpha_{k,\tau}, m_{k,\tau}, \Sigma_{k,\tau}\}_{k=1}^K$ at time $\tau$. Let $y_{k,\tau} \sim \cN(m_{k,\tau}, \Sigma_{k,\tau})$ denote a sample from the $k$-th \Gaussian{} particle of $\rho_{\theta_\tau}$. Then, the WFR dynamics $\nabla_{\mspace{-3mu}\textup{\texttt{WFR}}} \KL(\rho_{\theta_\tau}\Vert \rho^\ast\mspace{-1mu})$ wrt $\dot{\theta}_\tau = \{\dot{\alpha}_{k,\tau}, \dot{m}_{k,\tau}, \dot{\Sigma}_{k,\tau}\}_{k=1}^K$ are given as
  \begin{equation} \label{eq:wfr}
  \begin{gathered}
    \dot{\alpha}_{k,\tau}\mspace{-1.5mu}= -\biggl(\bE\mspace{-1mu}\biggl[\log\frac{\rho_{\theta_\tau}}{\mspace{1.5mu}\rho^\ast}(y_{k,\tau})\mspace{-1mu}\biggr] - \frac{1}{z_\tau}\sum_{\ell=1}^K \mspace{-1mu}\alpha_\ell \mspace{1mu} \bE\mspace{-1mu}\biggl[ \log\frac{\rho_{\theta_\tau}}{\mspace{1.5mu}\rho^\ast}(y_{\ell,\tau})\mspace{-1mu}\biggr] \mspace{-1mu} \biggr) \alpha_{k,\tau}, \\ 
    \dot{m}_{k,\tau}\mspace{-1.5mu}= - \bE\mspace{-1mu}\biggl[\nabla\log\frac{\rho_{\theta_\tau}}{\mspace{1.5mu}\rho^\ast}(y_{k,\tau})\biggr]\mspace{-1.5mu}, 
    \quad \dot{\Sigma}_{k,\tau}\mspace{-1.5mu}= - \bE\mspace{-1mu}\biggl[\nabla^2\mspace{-1mu} \log\frac{\rho_{\theta_\tau}}{\mspace{1.5mu}\rho^\ast}(y_{k,\tau})\mspace{-1mu}\biggr]\mspace{-1mu}\Sigma_{k,\tau}\mspace{-1mu} - \Sigma_{k,\tau}\bE\mspace{-1mu}\biggl[\nabla^2\mspace{-1mu} \log\frac{\rho_{\theta_\tau}}{\mspace{1.5mu}\rho^\ast}(y_{k,\tau})\mspace{-1mu}\biggr]\mspace{-1.5mu},
  \end{gathered} 
  \end{equation}
  for $\tau\in[0,\infty)$, where $z_\tau\coloneqq\sum_{k=1}^K \alpha_k$; $\nabla$ and $\nabla^2$ denote gradient and \Hessian{} with respect to $y_{k,\tau}$.
\end{proposition}
Appendices~\ref{subsect:proof_wfr}~and~\ref{sect:discuss} contain the complete theory. \cref{prop:wfr} argues that the one parameter family $\theta_\tau$ predicts a gradient-based algorithm of $\nabla_{\mspace{-3mu}\textup{\texttt{WFR}}} \KL(\rho_{\theta_\tau}\Vert \rho^\ast\mspace{-1mu})$, and thus \cref{eq:wfr} can be directly used for training GMM models. Recall that GMMs have a closed form expression of likelihoods, which means each log likelihood difference can be calculated without errors. Given that the target has the identical number of \Gaussian{} particles, both \cref{eq:wfr} and its approximation using finite samples strictly induce zero gradients at the equilibrium. Hence, we argue that the simulation-free algorithm VMSB will result in more robust and stable outcomes than standard data-driven SB learning.  Our VOMD framework can be implemented in modern deep learning libraries and, when coupled with advanced optimizers that provide inherently more stable gradient estimates (e.g., adaptive learning rate schedules), the proposed method empirically find convergence even faster than the conservative rates predicted by theory.

\subsection{Algorithmic considerations}

\begin{algorithm}[t]
  \renewcommand{\algorithmicrequire}{\textbf{Input:}}
  \renewcommand{\algorithmicensure}{\textbf{Output:}}
  \caption{Variational Mirrored SB (VMSB).} \label{alg:vmsb}
  \begin{algorithmic}[1] 
    \Require SB models $(\vec{\pi}_\theta, \vec{\pi}_\phi)$ parameterized by \Gaussian{} mixtures, step sizes $(\eta_{\scriptscriptstyle1}, \eta_{\scriptscriptstyle T})$, $n_y,B\!\in\!\bN$.
    \For{$t\gets1$ \textbf{to} $T$} 
      \State Acquire $\phi_t$ with an external data-driven SB solver.
      \State $\theta_{t}\gets\theta$, $\eta_t \gets { 1/\bigl(\eta^{\shortminus 1}_1 + (\eta^{\shortminus 1}_{\scriptscriptstyle T} - \eta^{\shortminus 1}_1) (\nicefrac{t-1}{T-1})\bigr)}$ \label{line:eta}    
      \For{$n\gets1$ \textbf{to} $N$}
        \State $\{x_i\}_{i=1}^B \gets$ sample mini batch data from $\mu$.
        \State $\tfrac{\partial \cL}{\partial \theta} \gets \tfrac{1}{B}\mspace{-2mu}\sum_{i=1}^B\mspace{-2mu} \eta_t\mspace{1mu}\mathtt{WFRgrad}(\theta; \phi_t, x_i, n_y)\mspace{1mu}+(1 - \eta_t)\mspace{1mu}\mathtt{WFRgrad}(\theta; {\theta_{t}}, x_i, n_y)$ \label{line:wfr} 
        \State Update $\theta$ with the gradient $\tfrac{\partial \cL}{\partial \theta}$.
      \EndFor
    \EndFor
    \Ensure Trained SB model $\vec{\pi}_\theta$. 
  \end{algorithmic}
\end{algorithm}

\cref{alg:vmsb} outlines the overall procedure. The proposed VMSB algorithm requires SB parameters $\theta$ and $\phi$, which represents $\vec{\pi}_t$ and $\vec{\pi}^\circ_t$ from the theoretical framework in \S~\ref{subsect:analysis}. The target model $\vec{\pi}_\phi$ is independently fitted using an arbitrary SB solver. By the results of analysis, one can schedule of the step size $\eta_t$ with a harmonic progression satisfying \cref{asm:eta}; thus, we propose to schedule by the series for $1 \ge \eta_1 \ge \etaT > 0$ as in Line~\ref{line:eta} of the algorithm. In our settings, the hyperparameters are set $\eta_1 = 1$ and $\eta_T \in \{0.05, 0.01\}$ which varies depending on each length of training. The algorithm can also put ``warm up'' steps leveraged by a existing solver, and start from $\theta=\phi_t$ enforcing $\eta_t \equiv 1$ for a certain period of the early stage. 

The VMSB algorithm is essentially designed to perform the following approximation of the WFR gradient operation \eqref{eq:wfr} in \cref{prop:wfr}, approximated by the following equation with finite data samples $\{x_i\}_{i=1}^B\sim\mu$ 
\begin{equation*}
 \textstyle\frac{1}{B}\sum_{i=1}^B\WFRgrad(\theta; \phi, x_i, n_y) \approx \nablaWFR\,\KL(\vec{\pi}_{\theta} \Vert \vec{\pi}_\phi), 
\end{equation*}
where each expectation is estimated using $n_y$ samples from each \Gaussian{} particle. Following \cref{thm:vmdkl}, we propose to update the SB model $\vec{\pi}_\theta$ with $\eta_t\WFRgrad(\theta; \phi_t, x_i, n_y) + (1-\eta_t)\WFRgrad(\theta; \theta_{t-1}, x_i, n_y)$ at each VOMD iteration $t$ (see Line~\ref{line:wfr}). When $\mu$ is a zero-centered distribution, we set $B=1$ and $x=0$ for the fast training time. This trick is equivalent to training the adjusted \Schrodinger{} potential \citep{lsb} $v_\theta \coloneqq \sum_{k=1}^K \alpha_k\mspace{2mu}\cN\mspace{-1.7mu}(\mspace{1mu}y\mspace{1.2mu}\vert m_k\mspace{-1mu}, \varepsilon \Sigma_k) \propto \pi_\theta(\cdot|x=0)$ directly, which enables the VMSB algorithm to run efficiently for certain tasks. We argue that our design provides a simple yet faithful realization of OMD updates, yielding a procedure that closely resembles classical gradient descent in machine learning. Although our methodology and computational strategy build on well-established ideas \citep{viwgf, mdsem, sf}, we deliberately integrate them into a unified framework to verify our novel online learning theory for the SB problem. As a result, VMSB emerges as a robust solver that embeds OMD within a variational formulation, offering both rigorous theoretical guarantees and clear computational advantages. 

\begin{figure}[b]
  \tikzset{inner sep=0pt, outer sep=0pt}
  \definecolor{plotblue}{RGB}{25,119,178}
  \definecolor{plotred}{RGB}{214, 39, 40}
  \definecolor{plotgreen}{RGB}{46,139,87}
  \definecolor{plotpurple}{RGB}{186,85,211}
  \definecolor{plotorange}{RGB}{255, 127, 14}
  \definecolor{colorone}{RGB}{224, 146, 24}
  \definecolor{colortwo}{RGB}{82, 173, 228}
  \definecolor{colorthree}{RGB}{20, 148, 104}
  \definecolor{colorfour}{RGB}{236, 222, 66}
  \definecolor{colorfive}{RGB}{22, 105, 167}
  \definecolor{colorsix}{RGB}{205, 79, 28}
  \definecolor{colorseven}{RGB}{196, 109, 157}
  \def\figw{8.2cm}
  \def\figh{9.45cm}
  \def\pltlw{0.8mm}
  \def\pltms{4pt}
  \def\errbarlw{1mm}
  \tikzset{inner sep=0pt, outer sep=0pt}    
  \centering
  \begin{tikzpicture}[transform shape,tight background]
    \clip  (-4.78,-0.29) rectangle (11.59,2.98);
    \begin{scope}[xshift=-2.82cm, yshift=1.33cm]
      \node at (0,0) {\includegraphics[width=95pt]{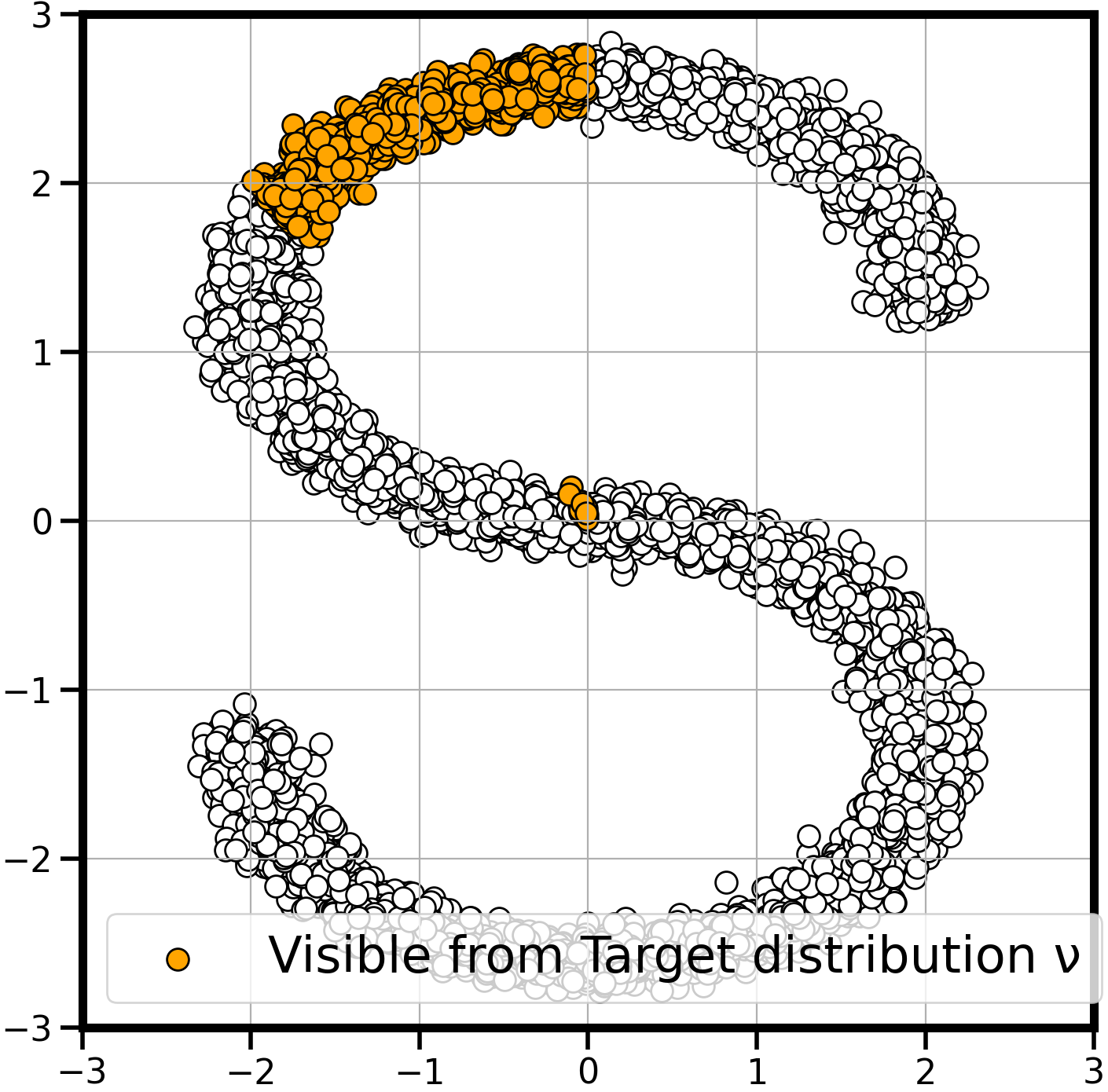}};
      \draw[->, plotred, line width=0.5mm] (0.36,0.12) arc[start angle=0, end angle=300, radius=0.3cm]; 
      \node[scale=0.7] at (0.42,0.52) {\color{plotred} \textbf{\textsf{filter\,{\scriptsize (45\textdegree)}}}};
    \end{scope}
    \begin{scope}[scale=0.35]
      \begin{scope}
        \begin{axis}[
            grid=both, mark=none, xmin=4, ymin=1.7, xmax=54, ymax=3.15,
            axis line style={ultra thick},
            xtick={0, 8, 20, 50}, ytick={0, 1, 2,3},
            yticklabels = {$0$, $0.01$, $0.02$, $0.03$},
            yticklabel style={font=\Large\sffamily, xshift=-4pt},
            xticklabel style={font=\large\sffamily, yshift=-4pt},
            title={8 gaussian\,$\bm{\to}$\,Swiss Roll},
            xlabel={\Gaussian{} particles},
            ylabel={Energy distance},
            title style={font=\Large\bfseries\sffamily, yshift=-2pt},
            xlabel style={font=\large\sffamily, yshift=-1.5pt},
            ylabel style={font=\Large\sffamily},
            width=\figw, height=\figh]
          \addplot[
            line width=\pltlw,
            color=colorone,
            mark=*,
            mark size=\pltms,
            error bars/.cd,
            y dir=both,
            y explicit,
            error bar style={line width=\errbarlw},
            ] coordinates {
              (8, 2.685416851) +- (0,0.1394441222)
              (20, 2.696946447) +- (0,0.1581699982)
              (50, 2.724788823) +- (0,0.1670845176)
            };
          \addplot[
            line width=\pltlw,
            color=colorseven,
            mark=*,
            mark size=\pltms,
            error bars/.cd,
            y dir=both,
            y explicit,
            error bar style={line width=\errbarlw},
            ] coordinates {
              (8, 2.674450701) +- (0,0.3783114392)
              (20, 2.62375619) +- (0,0.2474319895)
              (50, 2.758492105) +- (0,0.341031204)
            };
          \addplot[
            line width=\pltlw,
            color=colortwo,
            mark=*,
            mark size=\pltms,
            error bars/.cd,
            y dir=both,
            y explicit,
            error bar style={line width=\errbarlw},
            ] coordinates {
              (8, 2.070668729) +- (0,0.06411765085)
              (20, 2.032815039) +- (0,0.09105797441)
              (50, 2.030671606) +- (0,0.03251723978)
            };
          \addplot[
            line width=\pltlw,
            color=colorthree,
            mark=*,
            mark size=\pltms,
            error bars/.cd,
            y dir=both,
            y explicit,
            error bar style={line width=\errbarlw},
            ] coordinates {
              (8, 2.0817168) +- (0,0.06659024547)
              (20, 2.069781638) +- (0,0.06820182444)
              (50, 1.982159899) +- (0,0.08201169112)
            };
        \end{axis}
      \end{scope}
      \begin{scope}[xshift=8.28cm]
        \begin{axis}[
           grid=both, mark=none, xmin=4, ymin=0.2, xmax=54, ymax=2.8,
           axis line style={ultra thick},
           xtick={0, 8, 20, 50},
           ytick={0, 1, 2, 3},
           yticklabels = {$0$, $0.01$, $0.02$, $0.03$},
           yticklabel style={font=\Large\sffamily, xshift=-4pt},
           xticklabel style={font=\large\sffamily, yshift=-4pt},
           title={Swiss Roll\,$\bm{\to}$\,8 gaussian},
           xlabel={\Gaussian{} particles},
           title style={font=\Large\bfseries\sffamily, yshift=-2pt},
           xlabel style={font=\large\sffamily, yshift=-1.5pt},
           width=\figw,
           height=\figh]
          \addplot[
            line width=\pltlw,
            color=colorone,
            mark=*,
            mark size=\pltms,
            error bars/.cd,
            y dir=both,
            y explicit,
            error bar style={line width=\errbarlw},
            ] coordinates {
              (8, 2.185881728) +- (0,0.173838597)
              (20, 2.136568137) +- (0,0.5661596855)
              (50, 1.70831089) +- (0,0.1806868487)
            };
          \addplot[
            line width=\pltlw,
            color=colorseven,
            mark=*,
            mark size=\pltms,
            error bars/.cd,
            y dir=both,
            y explicit,
            error bar style={line width=\errbarlw},
            ] coordinates {
              (8, 2.180818299) +- (0,0.07332761836)
              (20, 1.367886579) +- (0,0.1277935528)
              (50, 1.074704952) +- (0,0.4357037214)
            };
          \addplot[
            line width=\pltlw,
            color=colortwo,
            mark=*,
            mark size=\pltms,
            error bars/.cd,
            y dir=both,
            y explicit,
            error bar style={line width=\errbarlw},
            ] coordinates {
              (8, 1.843582596) +- (0,0.2308115405)
              (20, 1.120157388) +- (0,0.3258815513)
              (50, 0.8210918099) +- (0,0.2323436778)
            };
          \addplot[
            line width=\pltlw,
            color=colorthree,
            mark=*,
            mark size=\pltms,
            error bars/.cd,
            y dir=both,
            y explicit,
            error bar style={line width=\errbarlw},
            ] coordinates {
              (8, 1.731504202) +- (0,0.4044850959)
              (20, 1.026545347) +- (0,0.1682403357)
              (50, 0.5949850185) +- (0,0.1533958277)
            };
        \end{axis}
      \end{scope}  
      \begin{scope}[xshift=16.56cm]
        \begin{axis}[
           grid=both, mark=none, xmin=4, ymin=0.9, xmax=54, ymax=3.1,
           axis line style={ultra thick},
           xtick={0, 8, 20, 50},
           ytick={0, 1, 2, 3},
           yticklabels = {$0$, $0.01$, $0.02$, $0.03$},
           yticklabel style={font=\Large\sffamily, xshift=-4pt},
           xticklabel style={font=\large\sffamily, yshift=-4pt},
           title={Moons\,$\bm{\to}$\,S-curve},
           xlabel={\Gaussian{} particles},
           title style={font=\Large\bfseries\sffamily, yshift=-2pt},
           xlabel style={font=\large\sffamily, yshift=-1.5pt},
           width=\figw,
           height=\figh]
          \addplot[
            line width=\pltlw,
            color=colorone,
            mark=*,
            mark size=\pltms,
            error bars/.cd,
            y dir=both,
            y explicit,
            error bar style={line width=\errbarlw},
            ] coordinates {
              (8, 2.761366359) +- (0,0.119303427)
              (20, 2.59794369) +- (0,0.07937971184)
              (50, 2.55781511) +- (0,0.087194847)
            };
          \addplot[
            line width=\pltlw,
            color=colorseven,
            mark=*,
            mark size=\pltms,
            error bars/.cd,
            y dir=both,
            y explicit,
            error bar style={line width=\errbarlw},
            ] coordinates {
              (8, 1.774686684) +- (0,0.2283845975)
              (20, 1.619390317) +- (0,0.2182340672)
              (50, 1.48363352) +- (0,0.08473560905)
            };
          \addplot[
            line width=\pltlw,
            color=colortwo,
            mark=*,
            mark size=\pltms,
            error bars/.cd,
            y dir=both,
            y explicit,
            error bar style={line width=\errbarlw},
            ] coordinates {
              (8, 1.936359132) +- (0,0.1751042479)
              (20, 1.847294976) +- (0,0.1459069063)
              (50, 1.828126779) +- (0,0.1104989364)
            };
          \addplot[
            line width=\pltlw,
            color=colorthree,
            mark=*,
            mark size=\pltms,
            error bars/.cd,
            y dir=both,
            y explicit,
            error bar style={line width=\errbarlw},
            ] coordinates {
              (8, 1.413202536) +- (0,0.06947994509)
              (20, 1.361967089) +- (0,0.2031265124)
              (50, 1.194983179) +- (0,0.1567007307)
            };
        \end{axis}
      \end{scope}
      \begin{scope}[xshift=24.84cm]
        \begin{axis}[
           grid=both, mark=none, xmin=4, ymin=0.8, xmax=54, ymax=3.15,
           axis line style={ultra thick},
           xtick={0, 8, 20, 50},
           ytick={0, 1, 2, 3},
           yticklabels = {$0$, $0.01$, $0.02$, $0.03$},
           yticklabel style={font=\Large\sffamily, xshift=-4pt},
           xticklabel style={font=\large\sffamily, yshift=-4pt},
           title={S-curve\,$\bm{\to}$\,Moon},
           xlabel={\Gaussian{} particles},
           title style={font=\Large\bfseries\sffamily, yshift=-2pt},
           xlabel style={font=\large\sffamily, yshift=-1.5pt},
           legend entries={LightSB, LightSB-M, VMSB, VMSB-M},
             legend style={font=\small, at={(0.565, 0.695)}, draw=none, anchor= south west 
           },
           width=\figw,
           height=\figh]
          \addplot[
            line width=\pltlw,
            color=colorone,
            mark=*,
            mark size=\pltms,
            error bars/.cd,
            y dir=both,
            y explicit,
            error bar style={line width=\errbarlw},
            ] coordinates {
              (8, 2.327995253) +- (0,0.1425853481)
              (20, 2.438335231) +- (0,0.06529486516)
              (50, 2.559868725) +- (0,0.3067355946)
            };
          \addplot[
            line width=\pltlw,
            color=colorseven,
            mark=*,
            mark size=\pltms,
            error bars/.cd,
            y dir=both,
            y explicit,
            error bar style={line width=\errbarlw},
            ] coordinates {
              (8, 2.411090799) +- (0,0.3775710934)
              (20, 1.886241183) +- (0,0.3725341518)
              (50, 1.590254167) +- (0,0.1844505505)
            };
          \addplot[
            line width=\pltlw,
            color=colortwo,
            mark=*,
            mark size=\pltms,
            error bars/.cd,
            y dir=both,
            y explicit,
            error bar style={line width=\errbarlw},
            ] coordinates {
              (8, 2.258539859) +- (0,0.294511892)
              (20, 2.267123233) +- (0,0.20089427191)
              (50, 2.163867585) +- (0,0.11225540266)
            };
          \addplot[
            line width=\pltlw,
            color=colorthree,
            mark=*,
            mark size=\pltms,
            error bars/.cd,
            y dir=both,
            y explicit,
            error bar style={line width=\errbarlw},
            ] coordinates {
              (8, 2.231518043) +- (0,0.3023465231)
              (20, 1.464797827) +- (0,0.3869413849)
              (50, 1.128870973) +- (0,0.1085810034)
            };
        \end{axis}
      \end{scope}            
    \end{scope}
    \node[scale=1.25] at (-4.58,2.796) {\textsf{a}};
    \node[scale=1.25] at (-0.45,2.811) {\textsf{b}};
  \end{tikzpicture}
  \caption{Online SBPs for synthetic dataset streams.  
  \subfigA{} We designed an online learning problem with a rotating filter where an algorithm is allowed to observe the data in $y\sim\nu$ only 12.5\% at a time. \subfigB{} The plots show that our VMSB and VMSB-M show consistent improvements from their references regarding the ED metric with 95\% confidence intervals for 5 runs with different seeds.} \label{fig:2d}
\end{figure}

\section{Experimental Results} \label{sect:expr}

\textbf{Experiment goals.}\hspace*{6pt} We aimed to test our online learning hypothesis and verify that the VMSB effectively induces OMD updates. Since our theoretical claims are intended to be highly versatile, consistent performance improvements for each setting coincides with the generality of the proposed VOMD method. We delineate our objectives as follows: \one{} We aimed to affirm our online learning hypothesis by demonstrating consistent improvements. \two{} We sought to corroborate our theoretical results, aiming for stable performance that consistently exceeds that of benchmarks. \three{} We aimed to verify that our algorithm effectively induces OMD by the \Wasserstein{} gradient flow. To achieve these goals, we validate our algorithm across diverse SB problem settings, including online learning scenarios, classical OT benchmarks, and image translation tasks. 

\textbf{Baselines and VMSB variants.}\hspace*{6pt}  
\citet{lsb} proposed a streamlined, simulation-free solver referred to as \textbf{LightSB} that optimizes $\phi$ through \MonteCarlo{} approximation of $\KL(\vec{\pi}^\ast\Vert\vec{\pi}_\phi)$. As an alternative, \textbf{LightSB-M} \citep{lsbm} reformulated the reciprocal projection from DSBM \citep{dsbm} to a projection method termed \textit{optimal projection}, establishing approximated bridge matching for the path measure $\cT_\phi$. Applying \cref{alg:vmsb}, we derived two distinct methods called \textbf{VMSB} and \textbf{VMSB-M} ($\vec{\pi}_\theta$), trained upon LightSB and LightSB-M solvers ($\vec{\pi}_\phi$), respectively. Since the theoretical arguments imply that the algorithm is agnostic to targets, the performance benefits of VMSB variants from their references support the generality of our claims. Additionally, we adopted VMSB on \textit{hybrid} settings, leveraging networks or embeddings for complex problems. We refer readers to Appendix~\ref{sect:details} for additional experimental setups.

\subsection{Online SB learning for synthetic data streams}

To validate our online learning hypothesis, we first considered two-dimensional synthetic SB problems for data streams depicted in \cref{fig:2d}~\subfigA{}. We applied an angle-based rotating filter, making the marginal as a data stream where only 12.5\% (or 45-degree angle) of the total data is accessible for each step $t$. We trained conditional models $\vec{\pi}_{\theta}$ with VMSB and  $\vec{\pi}_{\phi}$ with other baseline SB solvers for the 2D problem, respectively. \cref{fig:2d}~\subfigB{} shows the plots of squared energy distance (ED), which is a special instance of squared maximum mean discrepancy (MMD), approximating the $L^2$ distance between distributions: $\mathrm{ED}(P, Q) = \int (P(x) - Q(x))^2 \rd x$ \citep{rizzo2016energy}. In our ED evaluation, the VMSB algorithm achieved a strictly lower divergence than the LightSB and LightSB-M solvers for various numbers of \Gaussian{} particles $K$. Therefore, we conclude that these results aligned with our hypothesis and theory of online mirror descent.

\begin{table*}[t]
  \tikzset{inner sep=0pt, outer sep=0pt}    
  \caption{EOT benchmark scores with \cBWUVP{} $\downarrow$ (\%) between the optimal coupling $\pi^\ast$ and the learned model $\pi_\theta$ (5 runs). Results of classical EOT solvers marked with $\dagger$ are taken from \citep{lsb}, and $\ddagger$ from \citep{lsbm}. Additionally, LightSB-EMA indicates a hybrid approach using the exponential moving average techniques (EMA; \citealp{ema}) for LightSB parameters ($decay=0.99$). Our VMSB and VMSB-M results are highlighted in bold when VMSB methods exceed their reference algorithm.} \label{tab:eot}
  \centering
  \begin{adjustbox}{width=0.995\textwidth, center}
  \begin{tabular}{c c c c c c c c c c c c c c}
    \toprule
    \multirow{2}{*}{\raisebox{-5pt}{\textbf{Type}}}& \multirow{2}{*}{\raisebox{-5pt}{\textbf{Solver}}} & \multicolumn{4}{c}{$\varepsilon=0.1$} & \multicolumn{4}{c}{$\varepsilon=1$} & \multicolumn{4}{c}{$\varepsilon=10$} \\
    \cmidrule(lr){3-6}\cmidrule(lr){7-10}\cmidrule(lr){11-14}  
    & & $d=2$ & $d=16$ & $d=64$ & $d=128$ & $d=2$ & $d=16$ & $d=64$ & $d=128$ & $d=2$ & $d=16$ & $d=64$ & $d=128$ \\
    \midrule
      \multicolumn{2}{c}{Classical solvers (best; \citeauthor{lsb})$^\dagger$}& 1.94 & 13.67 & 11.74 & 11.4 & 1.04 & 9.08 & 18.05 & 15.23 & 1.40 & 1.27 & 2.36 & 1.31 \\ 
            {\small Bridge-M}&DSBM (\citeauthor{dsbm})$^\ddagger$& 5.2 & 10.8 & 37.3 & 35 & 0.3 & 1.1 & 9.7 & 31 & 3.7 & 105 & 3557 & 15000 \\
      {\small Bridge-M}&SF$^2$M-Sink (\citeauthor{tong2023simulation})$^\ddagger$& 0.54 & 3.7 & 9.5 & 10.9 & 0.2 & 1.1 & 9 & 23 & 0.31 & 4.9 & 319 & 819 \\
      \midrule
      rev. KL& LightSB (\citeauthor{lsb}) & 
        $0.007$ & $0.040$ & $0.100$ & $0.140$ & 
        $0.014$ & $0.026$ & $0.060$ & $0.140$ & 
        $0.019$ & $0.027$ & $0.052$ & $0.092$ \\
      {\small Bridge-M}&LightSB-M (\citeauthor{lsbm}) & 
        $0.017$ & $0.088$ & $0.204$ & $0.346$ & 
        $0.020$ & $0.069$ & $0.134$ & $0.294$ & 
        $0.014$ & $0.029$ & $0.207$ & $0.747$ \\
      EMA& LightSB-EMA & 
        $0.005$ & $0.040$ & $0.078$ & $0.149$ & 
        $0.012$ & $0.022$ & $0.051$ & $0.127$ & 
        $0.017$ & $0.021$ & $0.025$ & $0.042$ \\
    \midrule
      Var-MD &VMSB {\small(ours)}& 
        $\mathbf{0.004}$ & $\mathbf{0.012}$ & $\mathbf{0.038}$ & $\mathbf{0.101}$ & 
        $\mathbf{0.010}$ & $\mathbf{0.018}$ & $\mathbf{0.044}$ & $\mathbf{0.114}$ & 
        $\mathbf{0.013}$ & $\mathbf{0.019}$ & $\mathbf{0.021}$ & $\mathbf{0.040}$ \\
      Var-MD &VMSB-M {\small(ours)}& 
        $\mathbf{0.015}$ & $\mathbf{0.067}$ & $\mathbf{0.108}$ & $\mathbf{0.253}$ & 
        $\mathbf{0.010}$ & $\mathbf{0.019}$ & $\mathbf{0.094}$ & $\mathbf{0.222}$ &
        $\mathbf{0.013}$ & $\mathbf{0.029}$ & $\mathbf{0.193}$ & $0.748$ \\
    \bottomrule
  \end{tabular}
  \end{adjustbox}
\end{table*} 

\subsection{Quantitative Evaluation} \label{subsect:eot_bench}

\textbf{EOT benchmark.}\hspace*{6pt} Next, we considered the EOT benchmark proposed by \citet{eotbench}, which contains 12 entropic OT problems with different volatility and dimensionality settings. \cref{tab:eot} shows that LightSB and VMSB methods outperforms other EOT methods in terms of the \cBWUVP{} metric, or conditional Bures--Wasserstein unexplained variance percentage, soldifying previous reports by \citet{lsb} and \citet{lsbm}. We also observed that a hybrid approach combining LightSB and the exponential moving average (EMA; \citealp{ema}) named as LightSB-EMA was effective for improving stability. Among 24 different settings, our MD approach exceeded the reference model and the EMA method in 23 settings in terms of the \cBWUVP{} metric \citep{eotbench}. Our replication of LightSB/LightSB-M achieved better performance than originally reported results, and our method accordingly reached the state-of-the-art performance in this benchmark with stability, which represents strong evidence of \cref{prop:convergence}. Among all cases, the only exception was LightSB-M, which had the highest dimension and volatility. We suspected that the drift form \cref{eq:lsb_drift}, which is proportional to $\varepsilon$, may have violated our assumptions \cref{asm:bs} and the boundedness assumption during the training. Thus, we conclude that our variational MD training is effective in various EOT setups. Tables~\ref{tab:app_bw}~and~\ref{tab:app_cbw} in the appendix show comprehensive statistics on this benchmark with a wider range of SB solvers.

\begin{wraptable}{r}{0.56\textwidth}
  \vskip-15pt
  \centering
  \caption{Energy distance on the MSCI dataset (95\% confidence interval, ten trials with instances of two setups). Results marked with $\ddagger$ are from \citep{lsbm}.} \label{tab:msci}
  \begin{adjustbox}{width=0.55\textwidth}
  \begin{tabular}{c|c c c c c}
    \toprule
    \textbf{Type}& \textbf{Solver} & $d=50$ & $d=100$ & $d=1000$ \\
    \midrule
      \Sinkhorn{}&\citet{sbml}$^\dagger$& 2.34 & 2.24 & 1.864 \\
      {\small Bridge-M}&DSBM (\citeauthor{dsbm})$^\ddagger$& $2.46\pm0.1$ & $2.35\pm0.1$ & $1.36\pm0.04$ \\
      {\small Bridge-M}&SF$^2$M-Sink (\citeauthor{tong2023simulation})$^\ddagger$& $2.66\pm0.18$ & $2.52\pm0.17$ & $1.38\pm0.05$ \\
    \midrule
      rev. KL&LightSB& $2.31\pm0.08$ & $2.15\pm0.09$ & $1.264\pm0.06$ \\
      {\small Bridge-M}&LightSB-M& $2.30\pm0.08$ & $2.15\pm0.08$ & $1.267\pm0.06$ \\
    \midrule
      Var-MD &VMSB (ours)& $\mathbf{2.28\pm0.09}$ & $\mathbf{2.13\pm0.09}$& ${\mathbf{1.260\pm0.06}}$ \\
      Var-MD &VMSB-M (ours)& $\mathbf{2.26\pm0.10}$ & $\mathbf{2.12\pm0.09}$ &${\mathbf{1.265\pm0.05}}$ \\
    \bottomrule
  \end{tabular}
  \end{adjustbox}
  \vskip-8.5pt
\end{wraptable}

\textbf{SB on single cell dynamics.}\hspace*{6pt}We evaluated VMSB on unpaired single-cell data problems in the high-dimensional single cell dynamics experiment \citep{tong2024tmlr}. The dataset provided single cell data from four donors on days 2, 3, 4, and 7, describing the gene expression levels of distinct cells. Given  samples collected on two different dates, the task involves performing inference on temporal evolution, such as interpolation and extrapolation of PCA projections with $\{50, 100, 1000\}$ dimensions. We evaluated the energy distance over ten trials, which were divided into two distinct settings with five runs each \citep{tong2024tmlr}. The first setting spanned from Day 2 to Day 4 (evaluated at Day 3), while the second setting considers duration from Day 3 to Day 7 (evaluated at Day 4). The quantitative results in \cref{tab:msci} show that our VMSB method exhibited competitive results even for large dimensionalities, demonstrating its competitiveness as a practical SB solver for real-world problems.

\subsection{Unpaired image-to-image transfer}

\begin{figure}[t]
  \centering
  \begin{minipage}[b]{0.49\textwidth}
    \centering
    \tikzset{inner sep=0pt, outer sep=0pt}    
    \begin{tikzpicture}
      \clip (-3.96, -3.516) rectangle (3.96,1.28);
      \begin{scope}
        \node[scale=0.78, rotate=90] at (-3.75,0) {Pixel Space};
        \begin{scope}[xshift=-2.45cm]
          \node[scale=0.51] at (0, 1.18) {\bfseries\textsf{VMSB-adv (MNIST)}};      
          \node at (0,0) {\includegraphics[width=62pt]{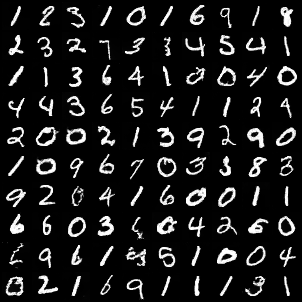}};
        \end{scope}
        \begin{scope}
          \node[scale=0.51] at (0, 1.18) {\bfseries\textsf{LightSB-adv (EMNIST)}};
          \node at (0,0) {\includegraphics[width=62pt]{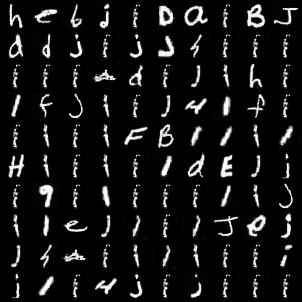}};          
        \end{scope}   
        \begin{scope}[xshift=2.45cm]
          \node[scale=0.51] at (0, 1.18) {\bfseries\textsf{VMSB-adv (EMNIST)}};
          \node at (0,0) {\includegraphics[width=62pt]{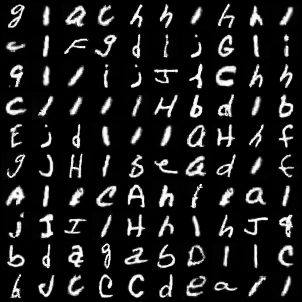}};
        \end{scope}
      \end{scope}
      \begin{scope}[yshift=-2.43cm]
        \node[scale=0.78, rotate=90] at (-3.75,0) {Latent Space};
        \begin{scope}[xshift=-2.45cm]
          \node[scale=0.51] at (0, 1.18) {\bfseries\textsf{VMSB (MNIST)}};      
          \node at (0,0) {\includegraphics[width=62pt]{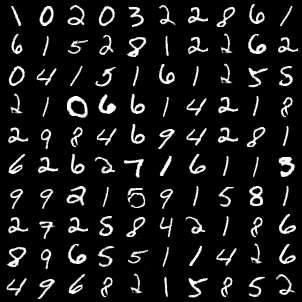}};
        \end{scope}
        \begin{scope} 
          \node[scale=0.51] at (0, 1.18) {\bfseries\textsf{LightSB (EMNIST)}};
          \node at (0,0) {\includegraphics[width=62pt]{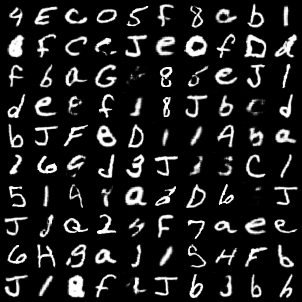}};
        \end{scope}
        \begin{scope}[xshift=2.45cm] 
          \node[scale=0.51] at (0, 1.18) {\bfseries\textsf{VMSB (EMNIST)}};
          \node at (0,0) {\includegraphics[width=62pt]{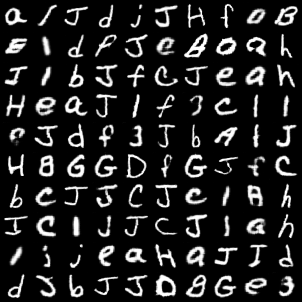}};
        \end{scope}
      \end{scope}
    \end{tikzpicture}
    \caption{Generated MNIST/EMNIST samples. Top: Raw pixel SB results. Bottom: Latent SB results.} \label{fig:mnist_gen}
  \end{minipage}
  \hfill  
  \begin{minipage}[b]{0.49\textwidth} 
    \centering
    \captionsetup{type=table, position=above}
    \caption{FID and MSD scores in EMNIST-to-MNIST translation tasks. Hyperparameters between LightSB and VMSB are shared. We examined the scores with five runs for the ALAE case.} \label{tab:mnist_fid}
    \begin{adjustbox}{width=0.99 \textwidth,center}
      \begin{tabular}{c| l c c}
        \toprule 
        \multicolumn{2}{c}{\normalsize\textbf{Method}} & \textbf{FID} & \textbf{MSD} \\
        \midrule
        \multirow{3}{*}{U-net} 
        &\SFMSink{} & $23.215$ & $0.456$\\ 
        &DSBM-IPF & $15.211$ & $0.352$\\
        &DSBM-IMF & $11.429$ & $0.373$ \\ 
        \midrule
        \multirow{2}{*}{Pixel} 
        &LightSB-adv & $20.017$ & $0.362$ \\
        &VMSB-adv (ours) & $15.471$ & $0.356$ \\
        \midrule
        \multirow{2}{*}{ALAE}
        & LightSB & $9.183_{\pm 0.569}$ & $0.371_{\pm 0.018}$ \\
        & VMSB  (ours) & $\mathbf{8.774_{\pm 0.065}}$ & $0.365_{\pm 0.002}$ \\
        \bottomrule 
      \end{tabular}
    \end{adjustbox}
    \vskip6pt
    \null
  \end{minipage}
\end{figure}

\textbf{MNIST-EMNIST.}\hspace*{6pt} We applied VMSB to unpaired image translation tasks for MNIST and EMNIST datasets. In these tasks, LightSB methods struggled to generate raw pixels due to the limited scalability of the loss function. To solve this issue, we opted to find a viable alternative to LightSB the raw pixel space, and we discovered that the capabilities of GMM parameterization can be extended by incorporating the adversarial learning technique (\citealp{gan}; see \cref{subsect:mnist_detail}) was effective in providing rich learning signals for $\pi_\phi$. Therefore, we named the adversarial method and the VMSB adaptation \textbf{LightSB-adv} and \textbf{VMSB-adv}. Also, we pretrained encoder networks using the \textit{Adversarial Latent AutoEncoder} (ALAE; \citealp{alae}) technique, and applied the LightSB and VMSB algorithms on the 128-dimensional latent space that represent the both of data. \cref{fig:mnist_gen} shows that VMSB/VMSB-adv outperformed Light/LightSB-adv (with identical architecture) in the fidelity of samples and semantics of letters for latent and pixel spaces. In \cref{tab:mnist_fid}, the VMSB method on the ALAE embedding space was able to surpass deep SB models with a fewer number of parameters of $K=256$. Even for raw pixels, our algorithm also achieved competitive FID and input/output MSD similarity scores for $K=4096$. The consistent performance gains from the LightSB and LightSB-adv algorithms strongly supports our theoretical claims on online learning.

\textbf{FFHQ.}\hspace*{6pt} Following the latent SB setting of \citet{lsb}, we assessed our method by utilizing a pretrained ALAE model for generating $1024\times1024$ images of the FFHQ dataset \citep{ffhq}. With the predefined 512-dimensional embedding space, we trained our SB models on the latent space to solve four distinct tasks: \textit{Adult\,$\to$\;Child}, \textit{Child\,$\to$\;Adult}, \textit{Female\,$\to$\;Male}, and \textit{Male\,$\to$\;Female}. \cref{fig:i2i} illustrates that our method delivered high-quality translation results. We also conducted a quantitative analysis using the ED on the ALAE embedding as a metric for evaluation, and the corresponding quantitative results are reported in \cref{tab:emmd}. The result also verifies that our VMSB and VMSB-M algorithms consistently achieved lower ED scores than other baselines, demonstrating its applicability for the high dimensional embedding space. Consequently, the image-to-image transfer results showed that the generality of our online learning hypothesis and that the proposed algorithm is highly capable of interacting with neural networks of complex learning dynamics. Considering the significantly higher dimensionality of image domains relative to the batch sizes used in VOMD, the consistent and stable performance improvements demonstrated in our experiments strongly validate our theoretical claims regarding the robustness of our approach in online learning scenarios. 

\begin{figure}[t]
  \tikzset{inner sep=0pt, outer sep=0pt}
  \centering
  \def\gaph{1.9cm}
  \def\gapw{2.595cm}
  \def\imgapw{1.585}
  \def\labelh{0.898cm}
  \def\imw{45pt}
  \begin{tikzpicture}[tight background]
    \clip (-5.2,-0.799) rectangle (10.8, 2.87);
    \begin{scope}[yshift=\gaph]
      \begin{scope}[yshift=\labelh]
        \begin{scope}[xshift=-\gapw, every node/.style={font=\sffamily, scale=0.51}]
          \node at (-\imgapw, 0) {\textit{\textbf{Adult $\bm{\to}$ Child}}};
          \node at (0, 0) {VMSB};
          \node at (\imgapw, 0) {VMSB-M};
        \end{scope}  
        \begin{scope}[xshift=\gapw, every node/.style={font=\sffamily, scale=0.51}]
          \node at (-\imgapw, 0) {\textit{\textbf{Male $\bm{\to}$ Female}}};
          \node at (0, 0) {VMSB};
          \node at (\imgapw, 0) {VMSB-M};
        \end{scope}        
      \end{scope}
      \begin{scope}[xshift=-\gapw]
        \node at (-\imgapw, 0) {\includegraphics[width=\imw]{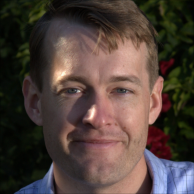}};
        \node at (0, 0) {\includegraphics[width=\imw]{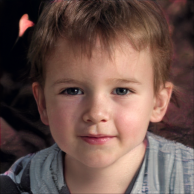}};
        \node at (\imgapw, 0) {\includegraphics[width=\imw]{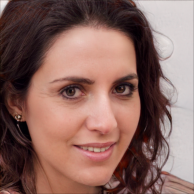}};
      \end{scope}   
      \begin{scope}[xshift=\gapw]
        \node at (-\imgapw, 0) {\includegraphics[width=\imw]{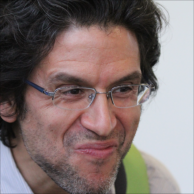}};
        \node at (0, 0) {\includegraphics[width=\imw]{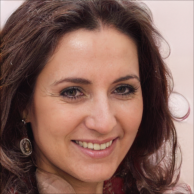}};
        \node at (\imgapw, 0) {\includegraphics[width=\imw]{fig/ffhq_male-female_md_eps_0.5.png}};
      \end{scope}  
    \end{scope}
    \begin{scope}
      \begin{scope}[yshift=\labelh]
        \begin{scope}[xshift=-\gapw, every node/.style={font=\sffamily, scale=0.51}]
          \node at (-\imgapw, 0) {\textit{\textbf{Child $\bm{\to}$ Adult}}};
          \node at (0, 0) {VMSB};
          \node at (\imgapw, 0) {VMSB-M};
        \end{scope}  
        \begin{scope}[xshift=\gapw, every node/.style={font=\sffamily, scale=0.51}]
          \node at (-\imgapw, 0) {\textit{\textbf{Female $\bm{\to}$ Male}}};
          \node at (0, 0) {VMSB};
          \node at (\imgapw, 0) {VMSB-M};
        \end{scope}        
      \end{scope}
      \begin{scope}[xshift=-\gapw]
        \node at (-\imgapw, 0) {\includegraphics[width=\imw]{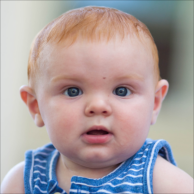}};
        \node at (0, 0) {\includegraphics[width=\imw]{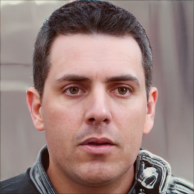}};
        \node at (\imgapw, 0) {\includegraphics[width=\imw]{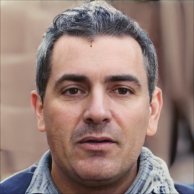}};
      \end{scope}   
      \begin{scope}[xshift=\gapw]
        \node at (-\imgapw, 0) {\includegraphics[width=\imw]{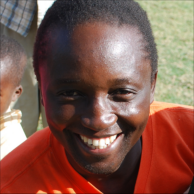}};
        \node at (0, 0) {\includegraphics[width=\imw]{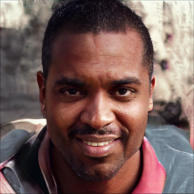}};
        \node at (\imgapw, 0) {\includegraphics[width=\imw]{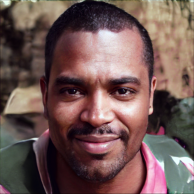}};
      \end{scope}  
    \end{scope}
    \begin{scope}[xshift=8.15cm, yshift=0.914cm]
      \node at (0,0) {\includegraphics[width=143pt]{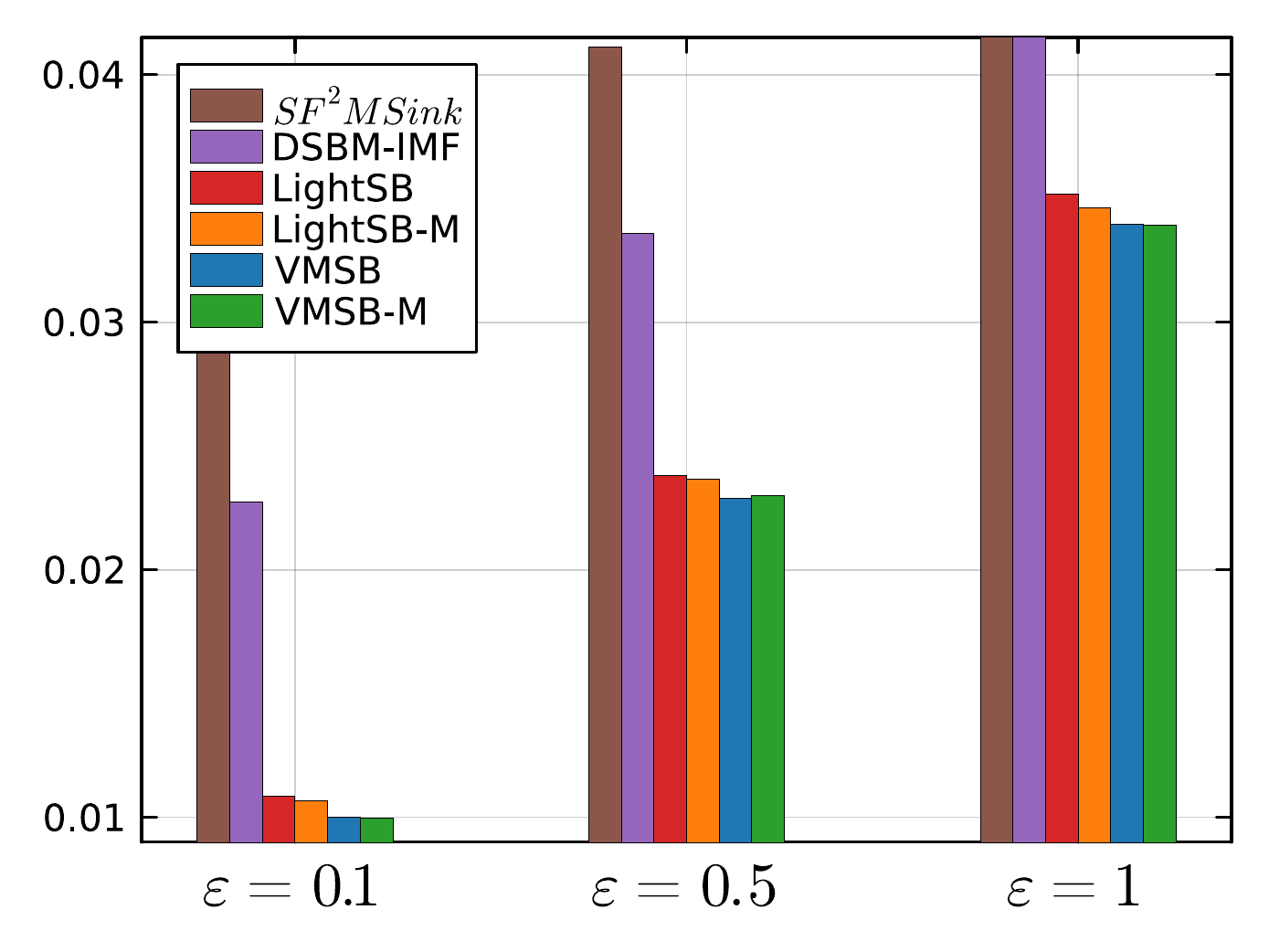}};
      \node[scale=0.61] at (0.17, 1.85) {\bfseries \textsf{Embedding energy distance}};
    \end{scope}
  \end{tikzpicture}
  \caption{Image-to-Image translation on a latent space. Left: Generation results for the FFHQ dataset ($1024\times1024$) using our two SB variants. Right: Quantitative results of ED metrics for ALAE embeddings.} \label{fig:i2i}
\end{figure} 

\section{Conclusion} \label{sect:concl}

In this paper, we introduced VMSB, a practical simulation-free \Schrodinger{} bridge algorithm based on the online learning theory, designed for effectively solving SB problems encountered in real-world scenarios. We proposed a robust theoretical learning framework applicable to general SB solvers, leveraging a dual geometric interpretation of convex optimization to construct a robust OMD algorithm with rigorous guarantees on convergence and regret bounds. Furthermore, we proposed the computational algorithm for our OMD framework by employing the \WassersteinFisherRao{} geometry. Through extensive empirical evaluation, we validated the effectiveness of VMSB across diverse settings, including high-dimensional spaces, limited-sample regimes, and online learning environments. The experimental results consistently demonstrated stable and superior benchmark performance, highlighting the enhanced robustness of our approach. Consequently, we argue that the proposed VMSB algorithm and our theoretical arguments regarding VOMD provide a promising and robust methodology for probabilistic generative modeling within learning-theoretic contexts.

\textbf{Limitations.}\hspace*{6pt} In this work, we significantly reduced the computational complexity inherent in the MD framework by adopting the \WassersteinFisherRao{} geometry. GMM-based models, due to the lack of deep structural processing, tend to focus on \textit{instance-level} associations of images in EOT couplings rather than the \textit{subinstance-} or \textit{feature-level} associations that are intrinsic to deep generative models. As a result, while VMSB produces statistically valid representations of optimal transportation within the given architectural constraints, these outcomes may be perceived as somewhat synthetic compared to large neural networks. Nevertheless, GMM-based models still hold an irreplaceable role in numerous problems such as latent diffusion and variational methods, due to their simplicity and distinctive properties \cite{lsb}. As we successfully demonstrated in two distinct ways of interacting with neural networks for solving unpaired image transfer, we hope our theoretical and empirical findings help novel neural architecture studies. While VMSB strictly outperforms existing SB solvers across standard numerical benchmarks, the performance gains are marginal in some scenarios. For instance, the single-cell dynamics experiment utilizes a PCA-based preprocessing pipeline \citep{tong2024tmlr}, and the transformation into this lower-dimensional space compromises the biological relevance of the SB objective, consequently narrowing the discernible performance gap.

\textbf{Future research.}\hspace*{6pt} One line of future studies in SB is a general understanding of learning in diffusion models with various regularizations. This includes diffusion models in various problem-specific constraints, and geometric constraints from manifolds. Another direction is the extension of the theoretical results into network architecture design. From \S~\ref{subsect:analysis}, a pair of \Schrodinger{} potentials represent a dual representation of SB in a statistical manifold. In \cite{bbrcd}, such potentials satisfy the Hamilton-Jacobi-Bellman (HJB) equations and, this can be trained with forward-backward SDE (SB-FBSDE) as presented by \citet{deepgsb}. However, this requires many simulation samples from SDEs, and the requirements for applying VMSB contain a tractable way of estimating gradient flows, and a guarantee of measure concentration. Therefore, we expect there will be a new studies of energy-based neural architecture for efficiently representing SB, which will advance various subfields of machine learning. Lastly, a theoretical generalization of our work can be done by considering the Orlicz space for EOT studied by \citet{lorenz2022orlicz}. Since we essentially devised our theoretical framework to be compatible with arbitrary \Bregman{} potentials, we believe controlling regularity of Young functionals can find more generalized learning algorithms for a wider range of OT problems.

{
  \small
  \bibliography{main}

@article{dro,
  title={Distributionally robust optimization},
  author={Kuhn, Daniel and Shafiee, Soroosh and Wiesemann, Wolfram},
  journal={Acta Numerica},
  volume={34},
  pages={579--804},
  year={2025},
  publisher={Cambridge University Press}
}

@article{drro,
  title={Wasserstein Distributionally Robust Regret Optimization},
  author={Fiechtner, Lukas-Benedikt and Blanchet, Jose},
  journal={arXiv preprint arXiv:2504.10796},
  year={2025}
}

@book{royden1988real,
  title={Real analysis},
  author={Royden, Halsey Lawrence and Fitzpatrick, Patrick},
  volume={32},
  year={1988},
  publisher={Macmillan New York}
}

@inproceedings{ghai2020exponentiated,
  title={Exponentiated gradient meets gradient descent},
  author={Ghai, Udaya and Hazan, Elad and Singer, Yoram},
  booktitle={Algorithmic learning theory},
  pages={386--407},
  year={2020},
  organization={PMLR}
}

@article{amid2020reparameterizing,
  title={Reparameterizing mirror descent as gradient descent},
  author={Amid, Ehsan and Warmuth, Manfred K.},
  journal={Advances in Neural Information Processing Systems},
  volume={33},
  pages={8430--8439},
  year={2020}
}

@inproceedings{amid2020winnowing,
  title={Winnowing with gradient descent},
  author={Amid, Ehsan and Warmuth, Manfred K},
  booktitle={Conference on Learning Theory},
  pages={163--182},
  year={2020},
  organization={PMLR}
}

@article{shalev2012online,
  title={Online learning and online convex optimization},
  author={Shalev-Shwartz, Shai and others},
  journal={Foundations and Trends in Machine Learning},
  volume={4},
  number={2},
  pages={107--194},
  year={2012},
  publisher={Now Publishers, Inc.}
}

@article{radford2015unsupervised,
  title={Unsupervised representation learning with deep convolutional generative adversarial networks},
  author={Radford, Alec},
  journal={arXiv preprint arXiv:1511.06434},
  year={2015}
}

@article{milnor1964microbundles,
  title={Microbundles: Part I},
  author={Milnor, John},
  journal={Topology},
  volume={3},
  pages={53--80},
  year={1964},
  publisher={Elsevier}
}

@book{danskin2012theory,
  title={The theory of max-min and its application to weapons allocation problems},
  author={Danskin, John M},
  volume={5},
  year={1967},
  publisher={Springer}
}

@article{ema,
  title={Exponential moving average of weights in deep learning: Dynamics and benefits},
  author={Morales-Brotons, Daniel and Vogels, Thijs and Hendrikx, Hadrien},
  journal={Transactions on Machine Learning Research},
  year={2024}
}

@article{gray1980asymptotically,
  title={Asymptotically mean stationary measures},
  author={Gray, Robert M and Kieffer, John C},
  journal={The Annals of Probability},
  pages={962--973},
  year={1980},
  publisher={JSTOR}
}

@article{smd,
  title={Mirror descent and nonlinear projected subgradient methods for convex optimization},
  journal={Operations Research Letters},
  volume={31},
  number={3},
  pages={167-175},
  year={2003},
  issn={0167-6377},
  author={Amir Beck and Marc Teboulle}
}

@inproceedings{jang2016categorical,
  title={Categorical reparameterization with gumbel-softmax},
  author={Jang, Eric and Gu, Shixiang and Poole, Ben},
  journal={arXiv preprint arXiv:1611.01144},
  booktitle={5th International Conference on Learning Representations},
  year={2016}
}

@inproceedings{gan,
  title={Generative adversarial nets},
  author={Goodfellow, Ian and Pouget-Abadie, Jean and Mirza, Mehdi and Xu, Bing and Warde-Farley, David and Ozair, Sherjil and Courville, Aaron and Bengio, Yoshua},
  booktitle={Advances in Neural Information Processing Systems},
  pages={2672--2680},
  year={2014}
}

@article{malrieu2001logarithmic,
  title={Logarithmic Sobolev inequalities for some nonlinear PDE's},
  author={Malrieu, Florient},
  journal={Stochastic processes and their applications},
  volume={95},
  number={1},
  pages={109--132},
  year={2001},
  publisher={Elsevier}
}

@inproceedings{sf,
  title={{S}inkhorn flow as mirror flow: a continuous-time framework for generalizing the sinkhorn algorithm},
  author={Karimi, Mohammad Reza and Hsieh, Ya-Ping and Krause, Andreas},
  booktitle={International Conference on Artificial Intelligence and Statistics},
  pages={4186--4194},
  year={2024},
  organization={PMLR}
}

@article{jax2018github,
  author={James Bradbury and Roy Frostig and Peter Hawkins and Matthew James Johnson and Chris Leary and Dougal Maclaurin and George Necula and Adam Paszke and Jake Vander{P}las and Skye Wanderman-{M}ilne and Qiao Zhang},
  title={{JAX}: composable transformations of {P}ython+{N}um{P}y programs},
  url={http://github.com/google/jax},  
  year={2018},
}

@article{lsi,
  title={Logarithmic {S}obolev inequalities},
  author={Gross, Leonard},
  journal={American Journal of Mathematics},
  volume={97},
  number={4},
  pages={1061--1083},
  year={1975},
  publisher={JSTOR}
}

@inproceedings{jayasumana2024rethinking,
  title={Rethinking fid: Towards a better evaluation metric for image generation},
  author={Jayasumana, Sadeep and Ramalingam, Srikumar and Veit, Andreas and Glasner, Daniel and Chakrabarti, Ayan and Kumar, Sanjiv},
  booktitle={Proceedings of the IEEE/CVF Conference on Computer Vision and Pattern Recognition},
  pages={9307--9315},
  year={2024}
}

@article{sscs,
  title={Stability of {S}chr{\"o}dinger potentials and convergence of Sinkhorn’s algorithm},
  author={Nutz, Marcel and Wiesel, Johannes},
  journal={The Annals of Probability},
  volume={51},
  number={2},
  pages={699--722},
  year={2023},
  publisher={Institute of Mathematical Statistics}
}

@book{ergodic,
  title={Ergodic theory},
  author={Cornfeld, Isaac P and Fomin, Sergej V and Sinai, Yakov Grigorevich},
  volume={245},
  year={2012},
  publisher={Springer Science \& Business Media}
}

@article{otto2000generalization,
  title={Generalization of an inequality by Talagrand and links with the logarithmic Sobolev inequality},
  author={Otto, Felix and Villani, C{\'e}dric},
  journal={Journal of Functional Analysis},
  volume={173},
  number={2},
  pages={361--400},
  year={2000},
  publisher={Elsevier}
}

@article{conforti2024weak,
  title={Weak semiconvexity estimates for Schr{\"o}dinger potentials and logarithmic Sobolev inequality for Schr{\"o}dinger bridges},
  author={Conforti, Giovanni},
  journal={Probability Theory and Related Fields},
  pages={1--27},
  year={2024},
  publisher={Springer}
}

@book{infinite,
  author={Aliprantis, Charalambos D. and Border, Kim C.},
  publisher={Springer},
  title={Infinite Dimensional Analysis: a Hitchhiker's Guide},
  year={2006}
}

@article{conforti2023quantitative,
  title={Quantitative contraction rates for Sinkhorn algorithm: beyond bounded costs and compact marginals},
  author={Conforti, Giovanni and Durmus, Alain and Greco, Giacomo},
  journal={arXiv preprint arXiv:2304.04451},
  year={2023}
}

@article{lu2018relatively,
  title={Relatively smooth convex optimization by first-order methods, and applications},
  author={Lu, Haihao and Freund, Robert M and Nesterov, Yurii},
  journal={SIAM Journal on Optimization},
  volume={28},
  number={1},
  pages={333--354},
  year={2018},
  publisher={SIAM}
}

@article{chen1993convergence,
  title={Convergence analysis of a proximal-like minimization algorithm using Bregman functions},
  author={Chen, Gong and Teboulle, Marc},
  journal={SIAM Journal on Optimization},
  volume={3},
  number={3},
  pages={538--543},
  year={1993},
  publisher={SIAM}
}

@article{altschuler2021averaging,
  title={Averaging on the {B}ures-{W}asserstein manifold: dimension-free convergence of gradient descent},
  author={Altschuler, Jason and Chewi, Sinho and Gerber, Patrik R and Stromme, Austin},
  journal={Advances in Neural Information Processing Systems},
  volume={34},
  pages={22132--22145},
  year={2021}
}

@article{lu2019accelerating,
  title={Accelerating langevin sampling with birth-death},
  author={Lu, Yulong and Lu, Jianfeng and Nolen, James},
  journal={arXiv preprint arXiv:1905.09863},
  year={2019}
}

@article{brenier,
  title={A computational fluid mechanics solution to the {M}onge-{K}antorovich mass transfer problem},
  author={Benamou, Jean-David and Brenier, Yann},
  journal={Numerische Mathematik},
  volume={84},
  number={3},
  pages={375--393},
  year={2000},
  publisher={Springer-Verlag Berlin/Heidelberg}
}

@book{nelson,
  title={Dynamical theories of {B}rownian motion},
  author={Nelson, Edward},
  year={2001},
  publisher={Princeton University Press}
}

@article{lott2008some,
  author={John Lott},
  title={Some Geometric Calculations on {W}asserstein Space},
  journal={Communications in Mathematical Physics},
  year={2008},
  volume={277},
  number={2},
  pages={423-437}
}

@article{eotbench,
  title={Building the Bridge of {S}chr{\"o}dinger: A Continuous Entropic Optimal Transport Benchmark},
  author={Gushchin, Nikita and Kolesov, Alexander and Mokrov, Petr and Karpikova, Polina and Spiridonov, Andrei and Burnaev, Evgeny and Korotin, Alexander},
  journal={Advances in Neural Information Processing Systems},
  volume={36},
  year={2024}
}

@inproceedings{ffhq,
  title={A style-based generator architecture for generative adversarial networks},
  author={Karras, Tero and Laine, Samuli and Aila, Timo},
  booktitle={Proceedings of the IEEE/CVF conference on computer vision and pattern recognition},
  pages={4401--4410},
  year={2019}
}

@inproceedings{alae,
  title={Adversarial latent autoencoders},
  author={Pidhorskyi, Stanislav and Adjeroh, Donald A and Doretto, Gianfranco},
  booktitle={Proceedings of the IEEE/CVF Conference on Computer Vision and Pattern Recognition},
  pages={14104--14113},
  year={2020}
}

@article{rizzo2016energy,
  title={Energy distance},
  author={Rizzo, Maria L and Sz{\'e}kely, G{\'a}bor J},
  journal={wiley interdisciplinary reviews: Computational statistics},
  volume={8},
  number={1},
  pages={27--38},
  year={2016},
  publisher={Wiley Online Library}
}

@article{bures,
  title={An extension of {K}akutani’s theorem on infinite product measures to the tensor product of semifinite $\omega^\ast$-algebras},
  author={Bures, Donald},
  journal={Transactions of the American Mathematical Society},
  volume={135},
  pages={199--212},
  year={1969}
}

@article{bhatia2019bures,
  title={On the {B}ures--{W}asserstein distance between positive definite matrices},
  author={Bhatia, Rajendra and Jain, Tanvi and Lim, Yongdo},
  journal={Expositiones Mathematicae},
  volume={37},
  number={2},
  pages={165--191},
  year={2019},
  publisher={Elsevier}
}

@article{someppm,
  title={Some properties of path measures},
  author={L{\'e}onard, Christian},
  journal={S{\'e}minaire de Probabilit{\'e}s XLVI},
  pages={207--230},
  year={2014},
  publisher={Springer}
}

@inproceedings{chen2023generalized,
  title={Generalized implicit follow-the-regularized-leader},
  author={Chen, Keyi and Orabona, Francesco},
  booktitle={International Conference on Machine Learning},
  pages={4826--4838},
  year={2023},
  organization={PMLR}
}

@inproceedings{mcmahan2011follow,
  title={Follow-the-regularized-leader and mirror descent: Equivalence theorems and l1 regularization},
  author={McMahan, Brendan},
  booktitle={Proceedings of the Fourteenth International Conference on Artificial Intelligence and Statistics},
  pages={525--533},
  year={2011},
  organization={JMLR Workshop and Conference Proceedings}
}

@article{md_info,
  author={G. {R}askutti and S. {M}ukherjee},
  journal={IEEE Transactions on Information Theory},
  title={The Information Geometry of Mirror Descent},
  year={2015},
  volume={61},
  number={3},
  pages={1451-1457}
}

@book{infogeo,
  title={Information geometry and its applications},
  author={Amari, Shun-ichi},
  volume={194},
  year={2016},
  publisher={Springer},
}

@book{online,
  title={Online algorithms: The state of the art},
  author={Fiat, Amos and Woeginger, Gerhard J},
  volume={1442},
  year={1998},
  publisher={Springer},
}

@article{rankin2023bregman,
  title={{B}regman-{W}asserstein divergence: geometry and applications},
  author={Rankin, Cale and Wong, Ting-Kam Leonard},
  journal={arXiv preprint arXiv:2302.05833},
  year={2023}
}

@article{hsieh2018mirrored,
  title={Mirrored {L}angevin dynamics},
  author={Hsieh, Ya-Ping and Kavis, Ali and Rolland, Paul and Cevher, Volkan},
  journal={Advances in Neural Information Processing Systems},
  volume={31},
  year={2018}
}

@article{bernhard1995theorem,
  title={On a theorem of {D}anskin with an application to a theorem of Von Neumann-Sion},
  author={Bernhard, Pierre and Rapaport, Alain},
  journal={Nonlinear Analysis: Theory, Methods \& Applications},
  volume={24},
  number={8},
  pages={1163--1181},
  year={1995},
  publisher={Pergamon}
}

@article{mirrorless,
  author={Suriya Gunasekar and Blake E. Woodworth and Nathan Srebro},
  title={Mirrorless Mirror Descent: {A} More Natural Discretization of Riemannian Gradient Flow},
  journal={CoRR},
  volume={abs/2004.01025},
  year={2020},
  eprinttype={arXiv},
  eprint={2004.01025},
}

@inproceedings{bregdist,
  title={{B}regman distances, totally convex functions, and a method for solving operator equations in {B}anach spaces},
  author={Butnariu, Dan and Resmerita, Elena},
  booktitle={Abstract and Applied Analysis},
  volume={2006},
  year={2006},
}

@article{geohess,
  title={Geometry of {H}essian manifolds},
  journal={Differential Geometry and its Applications},
  volume={7},
  number={3},
  pages={277-290},
  year={1997},
  issn={0926-2245},
  author={Hirohiko Shima and Katsumi Yagi},
}

@book{fundamentals_cvx,
  title={Fundamentals of convex analysis},
  author={Hiriart-Urruty, Jean-Baptiste and Lemar{\'e}chal, Claude},
  year={2004},
  publisher={Springer Science \& Business Media},
}

@book{cvx,
  title={Convex optimization},
  author={Boyd, Stephen and Vandenberghe, Lieven},
  year={2004},
  publisher={Cambridge university press}
}

@article{md_nonlin,
  title={Mirror descent and nonlinear projected subgradient methods for convex optimization},
  author={Beck, Amir and Teboulle, Marc},
  journal={Operations Research Letters},
  volume={31},
  number={3},
  pages={167--175},
  year={2003},
  publisher={Elsevier},
}

@inproceedings{diao2023forward,
  title={Forward-backward {G}aussian variational inference via {JKO} in the {B}ures-{W}asserstein Space},
  author={Diao, Michael Ziyang and Balasubramanian, Krishna and Chewi, Sinho and Salim, Adil},
  booktitle={International Conference on Machine Learning},
  pages={7960--7991},
  year={2023},
  organization={PMLR}
}

@article{omd_converge,
  title = {Convergence of online mirror descent},
  journal = {Applied and Computational Harmonic Analysis},
  volume = {48},
  number = {1},
  pages = {343--373},
  year = {2020},
  author = {Yunwen Lei and Ding-Xuan Zhou}
}

@inproceedings{omd_univ,
  title={On the universality of online mirror descent},
  author={Srebro, Nati and Sridharan, Karthik and Tewari, Ambuj},
  booktitle={Advances in Neural Information Processing Systems},
  pages={2645--2653},
  year={2011}
}

@inproceedings{tong2024aistat,
  title={Simulation-Free {S}chrödinger Bridges via Score and Flow Matching},
  author={Tong, Alexander Y. and Malkin, Nikolay and Fatras, Kilian and Atanackovic, Lazar and Zhang, Yanlei and Huguet, Guillaume and Wolf, Guy and Bengio, Yoshua},
  booktitle={Proceedings of The 27th International Conference on Artificial Intelligence and Statistics},
  pages={1279--1287},
  year={2024},
  volume={238},
  series={Proceedings of Machine Learning Research},
  publisher={PMLR},
}

@article{tong2024tmlr,
  title={Improving and generalizing flow-based generative models with minibatch optimal transport},
  author={Alexander Tong and Kilian Fatras and Nikolay Malkin and Guillaume Huguet and Yanlei Zhang and Jarrid Rector-Brooks and Guy Wolf and Yoshua Bengio},
  journal={Transactions on Machine Learning Research},
  issn={2835-8856},
  year={2024},
  note={Expert Certification}
}

@article{tong2023simulation,
  title={Simulation-free {S}chr{\"o}dinger bridges via score and flow matching},
  author={Tong, Alexander and Malkin, Nikolay and Fatras, Kilian and Atanackovic, Lazar and Zhang, Yanlei and Huguet, Guillaume and Wolf, Guy and Bengio, Yoshua},
  journal={arXiv preprint arXiv:2307.03672},
  year={2023}
}

@article{chen2018optimal,
  title={Optimal transport for {G}aussian mixture models},
  author={Chen, Yongxin and Georgiou, Tryphon T and Tannenbaum, Allen},
  journal={IEEE Access},
  volume={7},
  pages={6269--6278},
  year={2018},
  publisher={IEEE}
}

@article{daudel2021mixture,
  title={Mixture weights optimisation for alpha-divergence variational inference},
  author={Daudel, Kam{\'e}lia and others},
  journal={Advances in Neural Information Processing Systems},
  volume={34},
  pages={4397--4408},
  year={2021}
}

@inproceedings{deepgsb,
  author={Liu, Guan-Horng and Chen, Tianrong and So, Oswin and Theodorou, Evangelos},
  title={Deep Generalized {S}chr\"{o}dinger Bridge},
  booktitle={Advances in Neural Information Processing Systems},
  pages={9374--9388},
  year={2022}
}

@article{bregman,
  title={The relaxation method of finding the common point of convex sets and its application to the solution of problems in convex programming},
  author={Bregman, Lev M.},
  journal={USSR computational mathematics and mathematical physics},
  volume={7},
  number={3},
  pages={200--217},
  year={1967},
  publisher={Elsevier}
}

@article{sbp2mkp,
  title={From the {S}chr{\"o}dinger problem to the {M}onge--{K}antorovich problem},
  author={L{\'e}onard, Christian},
  journal={Journal of Functional Analysis},
  volume={262},
  number={4},
  pages={1879--1920},
  year={2012},
  publisher={Elsevier}
}

@article{leger2021gradient,
  title={A gradient descent perspective on {S}inkhorn},
  author={L{\'e}ger, Flavien},
  journal={Applied Mathematics \& Optimization},
  volume={84},
  number={2},
  pages={1843--1855},
  year={2021},
  publisher={Springer},
}

@article{vistat,
  title={Variational inference: A review for statisticians},
  author={Blei, David M and Kucukelbir, Alp and McAuliffe, Jon D},
  journal={Journal of the American statistical Association},
  volume={112},
  number={518},
  pages={859--877},
  year={2017},
  publisher={Taylor \& Francis},
}

@book{villani2021topics,
  title={Topics in optimal transportation},
  author={Villani, C{\'e}dric},
  volume={58},
  year={2021},
  publisher={American Mathematical Soc.},
}

@book{villani2009oldnew,
  title={Optimal transport: old and new},
  author={Villani, C{\'e}dric},
  volume={338},
  year={2009},
  publisher={Springer},
}

@book{ambrosio2005gf,
  title={Gradient flows: in metric spaces and in the space of probability measures},
  author={Ambrosio, Luigi and Gigli, Nicola and Savar{\'e}, Giuseppe},
  year={2005},
  publisher={Springer Science \& Business Media},
}

@article{santambrogio2017euclidean,
  title={\{{E}uclidean, metric, and {W}asserstein\} gradient flows: an overview},
  author={Santambrogio, Filippo},
  journal={Bulletin of Mathematical Sciences},
  volume={7},
  pages={87--154},
  year={2017},
  publisher={Springer}
}

@article{santambrogio2015ot,
  title={Optimal transport for applied mathematicians},
  author={Santambrogio, Filippo},
  journal={Birk{\"a}user, NY},
  volume={55},
  number={58-63},
  pages={94},
  year={2015},
  publisher={Springer},
}

@article{jko,
  title={The variational formulation of the {F}okker--{P}lanck equation},
  author={Jordan, Richard and Kinderlehrer, David and Otto, Felix},
  journal={SIAM journal on mathematical analysis},
  volume={29},
  number={1},
  pages={1--17},
  year={1998},
  publisher={SIAM},
}

@article{centered,
  title={On the linear convergence of the multimarginal Sinkhorn algorithm},
  author={Carlier, Guillaume},
  journal={SIAM Journal on Optimization},
  volume={32},
  number={2},
  pages={786--794},
  year={2022},
  publisher={SIAM},
}

@article{introeot,
  title={Introduction to entropic optimal transport},
  author={Nutz, Marcel},
  journal={Lecture notes, Columbia University},
  year={2021},
}

@inproceedings{i2sb,
  title={I{$^2$}SB: Image-to-Image {S}chr{\"o}dinger Bridge},
  author={Liu, Guan-Horng and Vahdat, Arash and Huang, De-An and Theodorou, Evangelos A and Nie, Weili and Anandkumar, Anima},
  booktitle={International Conference on Machine Learning},
  pages={22042--22062},
  year={2023},
  organization={PMLR},
}

@inproceedings{likesb,
  author={Tianrong Chen and Guan{-}Horng Liu and Evangelos A. Theodorou},
  title={Likelihood Training of {S}chr{\"{o}}dinger Bridge using Forward-Backward SDEs Theory},
  booktitle={10th International Conference on Learning Representations},
  year={2022}
}

@article{wassword,
  title={{G}romov--{W}asserstein alignment of word embedding spaces},
  author={Alvarez-Melis, David and Jaakkola, Tommi S},
  journal={arXiv preprint arXiv:1809.00013},
  year={2018}
}

@article{swav,
  title={Unsupervised learning of visual features by contrasting cluster assignments},
  author={Caron, Mathilde and Misra, Ishan and Mairal, Julien and Goyal, Priya and Bojanowski, Piotr and Joulin, Armand},
  journal={Advances in neural information processing systems},
  volume={33},
  pages={9912--9924},
  year={2020}
}

@article{cot,
  title={Computational optimal transport: With applications to data science},
  author={Peyr{\'e}, Gabriel and Cuturi, Marco and others},
  journal={Foundations and Trends{\textregistered} in Machine Learning},
  volume={11},
  number={5-6},
  pages={355--607},
  year={2019},
  publisher={Now Publishers, Inc.},
}

@article{vbi,
  title={Variational {B}ayesian inference with stochastic search},
  author={Paisley, John and Blei, David and Jordan, Michael},
  journal={arXiv preprint arXiv:1206.6430},
  year={2012},
}

@book{methodinfogeo,
  title={Methods of information geometry},
  author={Amari, Shun-ichi and Nagaoka, Hiroshi},
  volume={191},
  year={2000},
  publisher={American Mathematical Soc.},
}

@article{chizat2018,
  title={An interpolating distance between optimal transport and {F}isher--{R}ao metrics},
  author={Chizat, Lenaic and Peyr{\'e}, Gabriel and Schmitzer, Bernhard and Vialard, Fran{\c{c}}ois-Xavier},
  journal={Foundations of Computational Mathematics},
  volume={18},
  pages={1--44},
  year={2018},
  publisher={Springer},
}

@article{liero2016,
  title={Optimal transport in competition with reaction: The {H}ellinger--{K}antorovich distance and geodesic curves},
  author={Liero, Matthias and Mielke, Alexander and Savar{\'e}, Giuseppe},
  journal={SIAM Journal on Mathematical Analysis},
  volume={48},
  number={4},
  pages={2869--2911},
  year={2016},
  publisher={SIAM},
}

@article{liero2018,
  title={Optimal entropy-transport problems and a new {H}ellinger--{K}antorovich distance between positive measures},
  author={Liero, Matthias and Mielke, Alexander and Savar{\'e}, Giuseppe},
  journal={Inventiones mathematicae},
  volume={211},
  number={3},
  pages={969--1117},
  year={2018},
  publisher={Springer},
}

@article{viwgf,
  title={Variational inference via {W}asserstein gradient flows},
  author={Lambert, Marc and Chewi, Sinho and Bach, Francis and Bonnabel, Silv{
            \`e}re and Rigollet, Philippe},
  journal={Advances in Neural Information Processing Systems},
  volume={35},
  pages={14434--14447},
  year={2022},
}

@article{mdirl,
  title={Robust imitation via mirror descent inverse reinforcement learning},
  author={Han, Dong-Sig and Kim, Hyunseo and Lee, Hyundo and Ryu, JeHwan and Zhang, Byoung-Tak},
  journal={Advances in Neural Information Processing Systems},
  volume={35},
  pages={30031--30043},
  year={2022},
}

@article{bbrcd,
  title={{B}enamou-{B}renier and duality formulas for the entropic cost on {$RCD^\ast(K, N)$} spaces},
  author={Gigli, Nicola and Tamanini, Luca},
  journal={Probability Theory and Related Fields},
  volume={176},
  number={1-2},
  pages={1--34},
  year={2020},
  publisher={Springer},
}

@book{md,
  title={Problem complexity and method efficiency in optimization},
  author={Nemirovsky, Arkadi{\u\i} Semenovich and Yudin, David Borisovich},
  series={A Wiley-Interscience publication},
  year={1983},
  publisher={Wiley},
}

@article{kullback,
  title={Probability densities with given marginals},
  author={Kullback, Solomon},
  journal={The Annals of Mathematical Statistics},
  volume={39},
  number={4},
  pages={1236--1243},
  year={1968},
  publisher={JSTOR},
}

@inproceedings{lightspeed,
  author={Cuturi, Marco},
  booktitle={Advances in Neural Information Processing Systems},
  title={{S}inkhorn distances: lightspeed computation of optimal transport},
  volume={26},
  year={2013},
}

@article{mdsem,
  title={Mirror descent with relative smoothness in measure spaces, with application to {S}inkhorn and {EM}},
  author={Aubin-Frankowski, Pierre-Cyril and Korba, Anna and L{\'e}ger, Flavien},
  journal={Advances in Neural Information Processing Systems},
  volume={35},
  pages={17263--17275},
  year={2022},
}

@article{wmf,
  title={{W}asserstein mirror gradient flow as the limit of the {S}inkhorn algorithm},
  author={Deb, Nabarun and Kim, Young-Heon and Pal, Soumik and Schiebinger, Geoffrey},
  journal={arXiv preprint arXiv:2307.16421},
  year={2023},
}

@article{otto,
  title={The geometry of dissipative evolution equations: the porous medium equation},
  author={Otto, Felix},
  journal={Communications in Partial Differential Equations},
  volume={26},
  number={1-2},
  pages={101-174},
  year={2001},
  publisher={Taylor \& Francis},
}

@book{gf,
  title={Gradient flows: in metric spaces and in the space of probability measures},
  author={Ambrosio, Luigi and Gigli, Nicola and Savar{\'e}, Giuseppe},
  year={2005},
  publisher={Springer Science \& Business Media},
}

@article{ipgf,
  title={An invariance principle for gradient flows in the space of probability measures},
  author={Carrillo, Jos{\'e} A and Gvalani, Rishabh S and Wu, Jeremy S-H},
  journal={Journal of Differential Equations},
  volume={345},
  pages={233--284},
  year={2023},
  publisher={Elsevier},
}

@inproceedings{onfree,
  title={On free energy, stochastic control, and {S}chr{\"o}dinger processes},
  author={Pavon, Michele and Wakolbinger, Anton},
  booktitle={Modeling, Estimation and Control of Systems with Uncertainty: Proceedings of a Conference held in Sopron, Hungary, September 1990},
  pages={334--348},
  year={1991},
  organization={Springer},
}

@inproceedings{lsb,
  title={Light {S}chr{\"o}dinger bridge},
  author={Alexander Korotin and Nikita Gushchin and Evgeny Burnaev},
  booktitle={13th International Conference on Learning Representations},
  year={2024},
}

@inproceedings{lsbm,
  title={Light and Optimal Schr{\"o}dinger Bridge Matching},
  author={Gushchin, Nikita and Kholkin, Sergei and Burnaev, Evgeny and Korotin, Alexander},
  booktitle={International Conference on Machine Learnin},
  year={2024},
  organization={PMLR}
}

@article{dsb,
  title={Diffusion {S}chr{\"o}dinger bridge with applications to score-based generative modeling},
  author={De Bortoli, Valentin and Thornton, James and Heng, Jeremy and Doucet, Arnaud},
  journal={Advances in Neural Information Processing Systems},
  volume={34},
  pages={17695--17709},
  year={2021},
}

@inproceedings{dsbm,
  title={Diffusion {S}chr{\"o}dinger Bridge Matching},
  author={Shi, Yuyang and De Bortoli, Valentin and Campbell, Andrew and Doucet, Arnaud},
  booktitle={Advances in Neural Information Processing Systems},
  pages={62183--62223},
  volume={36},
  year={2023},
}

@article{sbml,
  title={Solving {S}chr{\"o}dinger bridges via maximum likelihood},
  author={Vargas, Francisco and Thodoroff, Pierre and Lamacraft, Austen and Lawrence, Neil},
  journal={Entropy},
  volume={23},
  number={9},
  pages={1134},
  year={2021},
  publisher={MDPI},
}

@article{sbsurvey,
  title={A survey of the {S}chr{\"o}dinger problem and some of its connections with optimal transport},
  author={L{\'e}onard, Christian},
  journal={arXiv preprint arXiv:1308.0215},
  year={2013},
}

@inproceedings{schrodinger2,
  title={Sur la th{\'e}orie relativiste de l'{\'e}lectron et l'interpr{\'e} tation de la m{\'e}canique quantique},
  author={{S}chr{\"o}dinger, Erwin},
  booktitle={Annales de l'institut Henri Poincar{\'e}},
  volume={2},
  pages={269--310},
  year={1932},
}

@inproceedings{birnbaum2011distributed,
  title={Distributed algorithms via gradient descent for Fisher markets},
  author={Birnbaum, Benjamin and Devanur, Nikhil R and Xiao, Lin},
  booktitle={Proceedings of the 12th ACM conference on Electronic commerce},
  pages={127--136},
  year={2011}
}

@article{xu2008robust,
  title={Robust regression and lasso},
  author={Xu, Huan and Caramanis, Constantine and Mannor, Shie},
  journal={Advances in neural information processing systems},
  volume={21},
  year={2008}
}

@article{duchi2021learning,
  title={Learning models with uniform performance via distributionally robust optimization},
  author={Duchi, John C and Namkoong, Hongseok},
  journal={The Annals of Statistics},
  volume={49},
  number={3},
  pages={1378--1406},
  year={2021},
  publisher={Institute of Mathematical Statistics}
}

@article{madry2017towards,
  title={Towards deep learning models resistant to adversarial attacks},
  author={Madry, Aleksander and Makelov, Aleksandar and Schmidt, Ludwig and Tsipras, Dimitris and Vladu, Adrian},
  journal={arXiv preprint arXiv:1706.06083},
  year={2017}
}

@inproceedings{zinkevich2003online,
  title={Online convex programming and generalized infinitesimal gradient ascent},
  author={Zinkevich, Martin},
  booktitle={Proceedings of the 20th international conference on machine learning (icml-03)},
  pages={928--936},
  year={2003}
}

@book{adams2003sobolev,
  title={Sobolev spaces},
  author={Adams, Robert A and Fournier, John JF},
  volume={140},
  year={2003},
  publisher={Elsevier}
}

@article{lorenz2022orlicz,
  title={Orlicz space regularization of continuous optimal transport problems},
  author={Lorenz, Dirk and Mahler, Hinrich},
  journal={Applied Mathematics \& Optimization},
  volume={85},
  number={2},
  pages={14},
  year={2022},
  publisher={Springer}
}

@article{auer2002adaptive,
  title={Adaptive and self-confident on-line learning algorithms},
  author={Auer, Peter and Cesa-Bianchi, Nicolo and Gentile, Claudio},
  journal={Journal of Computer and System Sciences},
  volume={64},
  number={1},
  pages={48--75},
  year={2002},
  publisher={Elsevier}
}

@article{orabona2018scale,
  title={Scale-free online learning},
  author={Orabona, Francesco and P{\'a}l, D{\'a}vid},
  journal={Theoretical Computer Science},
  volume={716},
  pages={50--69},
  year={2018},
  publisher={Elsevier}
}

@article{nesterov2009primal,
  title={Primal-dual subgradient methods for convex problems},
  author={Nesterov, Yurii},
  journal={Mathematical programming},
  volume={120},
  number={1},
  pages={221--259},
  year={2009},
  publisher={Springer}
}

@article{fang2022online,
  title={Online mirror descent and dual averaging: keeping pace in the dynamic case},
  author={Fang, Huang and Harvey, Nicholas JA and Portella, Victor S and Friedlander, Michael P},
  journal={Journal of Machine Learning Research},
  volume={23},
  number={121},
  pages={1--38},
  year={2022}
}

@book{sde,
  author = {{\O}ksendal, Bernt},
  title = {Stochastic Differential Equations: An Introduction with Applications},
  year = {2003},
  publisher = {Springer}
}

@article{fa2011solution,
  title={Solution of Fokker-Planck equation for a broad class of drift and diffusion coefficients},
  author={Fa, Kwok Sau},
  journal={Physical Review E},
  volume={84},
  number={1},
  pages={012102},
  year={2011},
  publisher={APS}
}
  \bibliographystyle{tmlr}
}

\clearpage
\appendix

\vbox{
  \vskip-3pt
  \hsize\textwidth
  \linewidth\hsize
  \vskip 0.08in
  \hrule height 1pt
  \vskip 0.118in
  \vskip -\parskip
    \centering
    {\Large \normalsize\bf\sffamily Appendices for\vskip-2pt{\large \bf\sffamily \titleonelinetext{}}\par}
  \vskip 0.2in
  \vskip -\parskip
  \hrule height 1pt
}

\section*{Abbreviation and Notation}

\begin{center}
  \tikzset{inner sep=0, outer sep=0}
  \begin{tikzpicture}[tight background]
    \begin{scope}[xshift=-.85cm]
      \node[scale=0.85] at (0, -0.57) {
        \begin{tabular}{ll}
          \toprule
            Abbreviation&Expansion\\
          \midrule
            SB&\Schrodinger{} Bridge\\
            SBP&\Schrodinger{} Bridge Problem\\
            EOT&Entropy-regularized Optimal Transport\\
            MD&Mirror Descent\\
            OMD&Online Mirror Descent\\
            KL&\KullbackLeibler{}\\
            IPF&Iterative Proportional Fitting\\
            BW&\BuresWasserstein{}\\
            WFR&\WassersteinFisherRao{}\\
            SDE&Stochastic Differential Equation\\
            PDE&Partial Differential Equation\\
            FP&\FokkerPlanck{}\\
            GMM&\Gaussian{} mixture model\\
          \bottomrule
        \end{tabular}};
    \end{scope}
    \begin{scope}[xshift=6.85cm, yshift=0.17cm]
      \node[scale=0.85] at (0, -1.1) {
        \begin{tabular}{ll}
          \toprule
            Notation&Usage\\
          \midrule
            $\mu,\nu$& marginal distributions\\
            $\varepsilon$&volatility of reference measure\\
            $c_\varepsilon$&cost $c_\varepsilon(x,y) \coloneqq \frac{1}{2\varepsilon}\lVert x - y\rVert^2$\\
            $\pi$ & a coupling of $\mu$ and $\nu$\\
            $\vec{\pi},\cev{\pi}$&conditional distributions \\
            $\gamma_n$ & $n$-th marginal\\
            $\varphi, \psi$& log-\Schrodinger{} potential\\
            $\Omega, D_\Omega$& \Bregman{} potential/divergence\\
            $d^+$&directional derivative\\
            $\deltaC, \deltaD$&First variations\\
            $\nablaW$&\Wasserstein{}-2 gradient operator\\
            $\cT$&dynamic stochastic process in SB\\
            $g$&drift function\\
            $i_\cC$&indicator function\\
            $\vSppd$&positive definite and symmetric matrices\\
          \bottomrule
        \end{tabular}};
    \end{scope}
  \end{tikzpicture}
\end{center}

\section{Theoretical Details and Proofs} \label{sect:proof}

\textbf{Background on first variation operators.}\hspace*{6pt} In this paper, we utilize the notation of first variation operators $\deltaC$ and $\deltaD$ to identify the generalized primal and dual spaces in \Schrodinger{} bridge. Since the problems are classified as an infinite-dimensional optimization \citep{infinite}, we introduce the essential background supporting the necessity of these operators. We introduce \Gateaux{} and \Frechet{} differentiablility \citep{mdsem, sf}.
\begin{definition}[\Gateaux{} \& \Frechet{} differentiablility] \label{def:gatfre} 
  Let $\cM$ be a topological vector space of measures on a space $\cX$. Define the \Gateaux{} differentiablity of a functional $F: \cM\to\bR$, if there exists a gradient operator $\nablaGat$ such that for an arbitrary direction $v \in\cM$, defined as the limit
  \[
    \nablaGat F(x)[v] = \lim_{h\to0} \frac{F(x+hv) - F(x)}{h},\quad x\in\cM
  \]
  If the limit exists in the unit ball in $\cM$, the function $F$ is called \Frechet{} differentiable with $\nablaFre F(x)$.
\end{definition}
The problem of the \Gateaux{} and \Frechet{} differentiability in the context of SB is that the limit must be given in \textit{all} directions, implying that every neighboring point must be within the domain of the topological space $\cM$. For the case of functionals such as the KL divergence functional $F(\cdot) = \KL(\cdot\vert\pi^\ast)$, the domain of $F$ and has an empty interior \citep{mdsem}. To resolve this issue, we use \textit{directional derivatives} and \textit{first variations}, defined in Definitions~\ref{def:directional}~and~\ref{def:firstvar}.

\textbf{First variations of KL.}\hspace*{6pt}
Suppose that we have two distributions $\rho, \rho^\prime \in \cP_2(\cX), \cX\subseteq\bR^d$. Let us consider the log likelihood of $\rho^\prime$: $\ell^\prime(x) \coloneqq \log\mspace{-1mu}\rho^\prime(x)$, and an element of a (topological) tangent space $v \in T_\rho\cP_2(\cX)$ \citep{milnor1964microbundles}. Then, we can achieve the followings:
\begin{gather}
  \KL(\rho\Vert \rho^\prime) = \int_\cX\mspace{-1mu}\log\rho(x)\ \rd \rho(x) - \int_\cX\mspace{-1mu}\ell^\prime(x)\ \rd \rho(x) \label{eq:freederive1} \\
  \int\mspace{-1mu}\ell^\prime(x)\bigl[\rho(x) + h v(x)\bigr]\ \dx = \int\mspace{-1mu}\ell^\prime(x)\rho(x)\ \rd x+h\!\int\mspace{-1mu}\ell^\prime(x) v(x)\,\dx \label{eq:freederive2}    
\end{gather}
Given that $\log(z+\varepsilon)(z+\varepsilon) = \log(z) z+ [\mspace{1mu}\log(z) + 1] \varepsilon + o(\varepsilon)$, and $\int_\cX\mspace{-1mu} v(x)\,\dx = 0$, we achieve
\begin{equation} \label{eq:freederive3}
\begin{aligned}
  \int_\cX\mspace{-1mu}\log\bigl(\rho(x) + h v(x)\mspace{-1mu}\bigr)\bigl(\rho(x) + h v(x)\mspace{-1mu}\bigr)\,\dx  &=\int_\cX\mspace{-1mu}\log\rho(x)\rho(x)\,\dx + [\mspace{1mu}\log\rho(x) + 1]h v(x) + o(h)\ \dx\\
  &=\int_\cX\mspace{-1mu}\log\rho(x)\rho(x)\,\dx + h\!\int_\cS\mspace{-1mu}\log\rho(x)v(x)\dx + h\!\int\!v(x)\dx + o(h) 
\end{aligned}
\end{equation}
Combining Eqs.~(\ref{eq:freederive1}\,-\ref{eq:freederive3}), we achieve
\begin{equation} \label{eq:Fcombine}
  F(\rho + h v) = F(\rho) + h \biggl\langle \mspace{-1mu}\log\biggl(\frac{\rho}{\rho^\prime}\biggr)\mspace{-1mu},\, v\biggr\rangle + o(h).
\end{equation}
By \cref{eq:Fcombine} and \cref{def:firstvar}, the first variation $\delta F_2 \in T^\ast\cP(\cX)$ exists for infinitesimal $h > 0$. Therefore, the first variation of KL is derived as $\delta \KL(\rho\Vert\rho^\prime) = \log\frac{\rho}{\rho^\prime}$. In machine learning, log likelihoods of probabilistic models are often given in a closed-form expression, incentivizing development of computational continuous EOT/SB methods. Generally, identical arguments generally apply to all KL functionals with respect to distributions ($\pi$, $\vec{\pi}$, and marginals) in our setup.

\textbf{Asymptotically log-concave distributions.}\hspace*{6pt} For convergence analysis, we assume each marginal distribution is in log-concave distribution, particularly satisfying the log \Sobolev{} inequality of measures, motivated by relevant literature \citep{otto2000generalization, conforti2024weak}. This assumption works a wider range of costs and marginals beyond popular choices with boundedness and compactness \citep{sscs, conforti2023quantitative}. Suppose that marginals admit densities of the form
\begin{equation} \label{eq:marginal_alc}
  \mu(\dx) = \exp\bigl(-\Umu\mspace{-1mu}(x)\bigr) \rd x \mspace{50mu}\textrm{and}\mspace{50mu}\nu(\dy) = \exp\bigl(-\Unu\mspace{-1mu}(y)\bigr) \rd y.
\end{equation}
We exploit the following definition from \citep{conforti2023quantitative} in order to describe asymptotically log-concaveness.
\begin{definition}[Asymptotically strongly log-concavity; \citealp{conforti2023quantitative}]
  Suppose that marginals $\mu$ and $\nu$ admit a positive density against the \Lebesgue{} measure, which can be written in the form (\ref{eq:marginal_alc}). In particular, consider a collection of functions $\cG \coloneqq \{g\in C^2((0, +\infty), \bR_+)\vert r\mapsto r^{1/2} g(r^{1/2}) \textrm{is non-increasing and concave,}\enspace \lim_{r\to0} r g(r) = 0 \}$. Accordingly, define a set
  \[
    \tilde{\cG} \coloneqq \{g\in\cG\ \textrm{bounded and s.t.}\enspace \lim_{r\to0^+}g(r) = 0, \enspace g^\prime \ge  0 \quad \textrm{and} \quad 2g^{\prime\prime} + g g^\prime \le 0  \} \subset \cG. 
  \]
  and \textit{convexity profile} $\kappa_U:\bR_+ \to \bR$ of a differentiable function $U$ as the following
  \[
    \kappa_U(r) \coloneqq \inf \biggl\{\frac{\langle \nabla U(x) - \nabla U(y), x - y \rangle}{\lvert x - y \rvert^2} \ :\ \lvert x - y \rvert = r \biggr\}.
  \]
  We say a potential is asymptotically strongly convex if there exists $\alpha_U \in \bR_+$ and $\tilde{g}_U\in\tilde{\cG}$ such that
  \begin{equation} \label{eq:cvx_prof}
    \kappa_U(r) \ge \alpha_U - r^{-1} \tilde{g}_U(r)
  \end{equation}
  holds for all $r> 0$. We consider the set of asymptotically strongly log-concave probability measures
  \[
    \cP_\textrm{alc}(\bR^d) \coloneqq \{ \zeta(\dx) = \exp(-U(x))\dx : U \in C_2(\bR^d),\enspace \textrm{$U$ is asymptotically strongly convex} \}.
  \]
\end{definition}
It is essential to note that a mixture of asymptotically log concave is also asymptotically log concave.
\begin{lemma} \label{lem:gmmalc}
  For positive weights $\bm{\beta} = \{\beta_k\}_{k=1}^K$ with $\sum_{k=1}^K \beta_k = 1$ and asymptotically log concave distributions $\{\rho_k\}_{k=1}^K$, $\pi = \sum_{k=1}^K\beta_k \rho_k$.    
\end{lemma}
\begin{proof}
  Let us reformulate the mixture as $\log\pi(x) = \log\sum_{k}\beta_k\exp(-U_k(x))$ for asymptotically strongly convex functions $\mathbf{U} = \{U_k\}_{k=1}^K$. The gradient is
  \[
    \nabla\log\pi = \mathfrak{J}^T \mathbf{p}, \quad \mathfrak{J} = - \begin{bmatrix}
      \nabla U_1 \\ 
      \vdots\\
      \nabla U_K
    \end{bmatrix}, \quad \mathbf{p} =  \textrm{softmax}(\log\bm{\beta} - \mathbf{U})
  \]
  If each  $U_k$ of mixture satisfy \cref{eq:cvx_prof} with $\alpha_{U_k}$ and $\tilde{g}_{U_k}$, there exist $\alpha_U = \min_{1 \le k \le K }\alpha_{U_k}$ and $\tilde{g}_U(r) = -r \log\sum_{k=1}^K \exp(-r^{-1}\tilde{g}_{U_k})$ that satisfies the condition \eqref{eq:cvx_prof} for $U = -\log\pi$. By direct calculation, one can easily see that soft min-like property of $r^{-1}\tilde{g}_U$ from $\{r^{-1}\tilde{g}_{U_k}\}_{k=1}^K$ does not change the conditions of $\tilde{\cG}$.
\end{proof}

\textbf{General assumptions and justifications.}\hspace*{6pt} We additionally need the following general assumptions for our OMD framework. \one{} (Existence) The sequence of MD from \cref{eq:omd} exists $\{\pi_t\}_{t\in \bN}\subset \cC$, and are unique, \two{} (Relative smoothness/convexity) For some $l, L \ge 0$, the functional $F_t$ is $L$-smooth and $l$-strongly-convex relative to $\Omega$. \three{} (Existence of first variations) For each $t\ge0$, the first variation $\deltaC \Omega(\pi_t)$ exists. \four{} (Boundedness of estimations) The asymptotic dual mean $\pcD$ is almost surely bounded $\Pr(D_\Omega(\pi_t \Vert \pcD) \le R) = 1$ for some $R>0$. \five{} (Ergodicity) The estimation process of $\{\pct\}_{t=1}^\infty$ is governed by a measure-preserving transformation on a measure space $(\cY, \Sigma, \varsigma)$ with $\varsigma(\cY)=1$; for every event $E\in\Sigma$, $\varsigma(T^{-1}(E) \triangle E) = 0$ (that is, $E$ is invariant), either $\varsigma(E) = 0$ or $\varsigma(E) = 1$ \citep{ergodic}.\footnote{Here, $\triangle$ denotes the symmetric difference, equivalent to the exclusive-or with respect to set membership.}  For \one{}, the temporal cost $F_t(\cdot) = \KL(\cdot\vert \pct)$ is well defined since KL is a strong \Bregman{} divergence with lower semicontinuity, where the existence of a primal solution in guaranteed as discussed in \citet{mdsem}. For \two{}-\three{}, we can identify $l=L=1$ and close-form expression of the first variation that is shown in \cref{def:relsmooth} and \cref{prop:wfr}. For the assumptions \four{}-\five{}, we postulate the existence of estimates produced from a Monte-Carlo method, using a fixed amount of updates on topological vector space. Hence, it is natural to consider that these estimates will be bounded in a probabilistic sense and yield \Markovian{} transitions, which are aperiodic and irreducible.

\subsection{Proofs of Lemmas~\ref{lem:sink}~and~\ref{lem:wd}}

The EOT in \cref{eq:eot} can be reformulated as a divergence minimization problem with respective to a reference measure. If a \Gibbs{} parameterization is enforced with the quadratic cost functional $c_\varepsilon(x,y) = \frac{1}{2\varepsilon}\lVert x - y\rVert^2$ for $\varepsilon>0$, it is well-known that the problem has the equivalence with the entropy regularized optimal transport problem \citep{introeot}
\begin{equation} \label{eq:sbp2}
  \mathrm{OT}_{\mspace{-.5mu}\varepsilon}(\mu, \nu)= \! \inf_{\pi\in\Pi(\mu, \nu)} \KL\bigl(\pi \Vert e^{-c_\varepsilon}\mu\otimes\nu\bigr).
\end{equation}
Note that the above equation corresponds to the constrained minimization of $\KL(\cT\Vert W^\varepsilon)$ in \cref{eq:sbp} by the disintegration theorem of \Schrodinger{} bridge (Appendix~A of \citealp{sbml}). While the \Bregman{} projection formulation of \Sinkhorn{} \cref{eq:canon_sink} are described by the spaces $(\Pi^\perp_\mu, \Pi^\perp_\nu)$, it is (equally) natural to think that considering the problem as convex problem with the distributional constraint $\cC$ (see the primal space in illustrated in \cref{fig:schema}). As a problem in the constraint $\cC$, one can consider a temporal cost functional $\Ftildesmallt(\pi) \coloneqq a_t \KL(\mspace{-1mu}\gamma_1\pi \Vert \mu) + (1-a_t) \KL(\mspace{-1mu}\gamma_2\pi \Vert \nu)$ with sequences $\{a_t\}_{t=1}^\infty = \{0, 1, 0, 1, \dots \}$  for $\gamma_1\pi(x) \coloneqq \smallint\pi(x,y) \dy$ and $\gamma_2\pi(y) \coloneqq \smallint\pi(x,y) \dx$. By construction, we have the following MD update:
\begin{equation} \label{eq:omd_canon_sink}
  \minimize_{\pi\in\cC} \bigl\langle \deltaC \widetilde{F}_t(\pi_t), \pi - \pi_t \bigr\rangle + D_\Omega(\pi\Vert \pi_t).
\end{equation}
The optimization problem \eqref{eq:omd_canon_sink} is equivalent to having the property for subsequent $\pi_{t+1}$:
\begin{equation}
\begin{aligned}
  &\dsupplus\mspace{-1.2mu} \Ftildesmallt(\pi_t; \pi - \pi_t) + D_\Omega(\pi \Vert \pi_t) \ge \dsupplus\mspace{-1.2mu} \Ftildesmallt(\pi_t; \pi_{t+1}\! - \pi_t) + D_\Omega(\pi_{t+1} \vert \pi_t)\\
  &\mspace{60mu} \iff \bigl\langle\mspace{-1.2mu} \deltaC \Ftildesmallt(\pi_t) - \deltaC \Omega(\pi_t),\, \pi - \pi_{t+1} \mspace{-1mu}\bigr\rangle + \bigl(\Omega(\pi) - \Omega(\pi_{t+1})\bigr) \ge 0,\quad \forall \pi \in \cC.
\end{aligned}
\end{equation}
Setting the free parameter $\pi = \pi_{t+1} + h(\pi - \pi_{t+1})$ and taking the limit $h \to 0^+$ yields described the time evolution of the log-\Schrodinger{} potentials for $\pi_t = e^{\varphi_t \oplus \psi_t - c_\varepsilon}\rd(\mu\otimes\nu)$:
\begin{subequations} \label{eq:sink_dyna}
\begin{align}
  \dot{\varphi}_t &= -\log\frac{\rd (\gamma_1 \pi_t)}{\rd\nu_\ast} = -\alpha \biggl(\mspace{-1mu}\varphi_t - \varphi^\ast + \log\int_{\bR^d} e^{\psi_t - \psi^\ast} \nu(\dy)\!\biggr),\\
  \dot{\psi}_t &= -\log\frac{\rd (\gamma_2\pi_t)}{\rd\mu_\ast} = -\beta \biggl(\psi_t - \psi^\ast  + \log\int_{\bR^d} e^{\varphi_t - \varphi^\ast} \mu(\dx)\!\biggr),
\end{align}    
\end{subequations}
for $\alpha = a_t$ and $\beta = 1 - a_t$.\footnote{More precisely, one needs to apply \cref{lem:equivf} for KL, and the disintegration theorem to get \cref{eq:sink_dyna}.} Setting a discrete approximation of dynamics \cref{eq:sink_dyna}: $\varphi_{t+1} = \varphi_t + \dot{\varphi}_t$ and $\psi_{t+1} = \psi_t + \dot{\psi}_t$ yields the following alternating updates: 
\begin{equation*}
  \psi_{2t+1}\mspace{-1mu}(y) = -\mspace{-1mu}\log\!\int_{\bR^d}\mspace{-1mu} e^{\varphi_{2t}\mspace{-1mu}(x) - c_\varepsilon\mspace{-1mu}(x,y)}\mu(\rd x), \quad \varphi_{2t+2}\mspace{-1mu}(x) = -\mspace{-1mu}\log\!\int_{\bR^d}\mspace{-1mu} e^{\psi_{2t+1}\mspace{-1mu}(x) - c_\varepsilon\mspace{-1mu}(x,y)}\nu(\rd y).
\end{equation*}
Therefore, the proof of \cref{lem:sink} is complete.

From the dual iteration of KL stated in \cref{eq:duali_kl} \citep{mdsem}, the static, idealized MD cost $F(\cdot) = \KL(\cdot\Vert\pi^\ast)$ yield the following closed-form expression for the first variation:
\[
  \deltaC \Omega(\pi_t) - \deltaC \Omega(\pi_{t+1}) = \eta_t \bigl(\deltaC \Omega(\pi_t) - \deltaC \Omega(\pi^\ast) \bigr),
\]
where the equation implies that setting $\eta_t \equiv 1$ for MD yields one-step optimality $\pi^\ast$ in this idealized condition. Utilizing the equivalence of first variation stated in \cref{lem:equivf} and the disintegration theorem for the \RadonNikodym{} derivatives, we get the first variation of $F$ with respect to $\pi$ for all $x$ as
\begin{equation} \label{eq:fvar}
  \delta F(\pi_t) = \log\frac{d\pi^\ast}{d\pi}.
\end{equation}
And by the disintegration theorem \citep{someppm}, we also achieve the first variation of $f$ with respect to $\vec{\pi}$ for all $x$ as
\begin{equation}
  \delta f(\vec{\pi}^x_t) = \log\frac{d (\vec{\pi}_t^\ast)^x}{d \vec{\pi}^x},
\end{equation}
where $f(\vec{\pi}^x) = \KL(\vec{\pi}^x\Vert (\vec{\pi}^\ast)^x\mspace{-1mu})$. Since this disintegration theorem always hold for every directional derivative, we can use expression for $\vec{\pi}^x$ and $\pi$ interchangeably. It is well-known that MD is a discretization of natural gradient descent \citep{mirrorless}, and our setting for $\Omega$ generates the geometry governed by the (generalized) Fisher information. In this particular case, one can use \Otto{}'s formalization of \Riemannian{} calculus (\citealp{otto}; \S~3.2), and the probability space equipped with the \Wasserstein{}-2 metric $(\cP_2(\bR^d), W_2)$, is generally represented as a \Wasserstein{} gradient flow
\begin{equation} \label{eq:jko}
  \partial_t \pi_t  = - \nablaW\mspace{1mu} F\mspace{-.3mu}(\pi_t),\qquad \forall \pi_t \in \cC,
\end{equation}
where $\nablaW$ denotes the \Wasserstein{}-2 gradient operator $\nablaW \coloneqq \nabla\!\vcdot\!\bigl(\rho\,\nabla \frac{\delta}{\delta \rho}\bigr)$. In particular, plugging \cref{eq:fvar} yields
\begin{equation} \label{eq:fokker}
  \partial_t \pi_t  =  -\nabla\mspace{-2mu}\vcdot\mspace{-1mu}(\pi \nabla \log \pi^\ast) + \Delta\pi,    
\end{equation}
where $\Delta$ denotes the Laplace operator. The foundational results concerning \Wasserstein{} gradients were initially established by JKO \citep{jko}, who demonstrated that the formulation in \cref{eq:jko} corresponds precisely to the \FokkerPlanck{} equation \eqref{eq:fokker}. Consequently, it follows that \Wasserstein{} gradients characterize the tangential direction of flows on a manifold constrained by distributional properties and endowed with the $W_2$ metric. \hfill\qedsymbol

\subsection{Proof of \texorpdfstring{\cref{thm:step_size}}{Theorem \ref{thm:step_size}}}

We start with introducing basic properties of the \Bregman{} divergence in \cref{def:breg}. First, the \textit{idempotence} property states that a \Bregman{} divergence associated with another \Bregman{} divergence $D_\Omega(\cdot \vert y)$ remains as the identical divergence with the original. Note that the (global or universal) idempotence initially stated by \citet{mdsem}, but we apply some changes to the statement and only work with localized version of idempotence for the purpose of this paper.
\begin{lemma}[Idempotence] \label{lem:idem}
  Suppose a convex potential $\Omega: \cM(\cX) \to \bR \cup \{+\infty\}$, where $\cM(\cX)$ denotes a topological vector space for $\cX$. Assume that for all $z \in \dom(\Omega)$, $\deltaC \Omega(z)$ exists. Then, $\forall x, y \in \cC \cap \dom(\Omega)$: $D_{D_\Omega(\cdot | y)}(x | y) = D_\Omega(x | y)$.
\end{lemma}
\begin{proof}[Proof of \cref{lem:idem}]
  Both \Bregman{} divergences and \Bregman{} potentials are convex functionals. By definition, we have $D_{D_\Omega(\cdot \vert z)}(x | y) = D_\Omega(x \Vert z) - D_\Omega(y \Vert z) - \langle \deltaC \Omega(y) - \deltaC \Omega(z),\, x - y\mspace{.5mu}\rangle$ for arbitrary $z$, and setting $z = y$ completes the proof. Another (informal) point of view is considering the \Bregman{} divergence as a first-order approximation of a \Hessian{} structure, and $D_{D_\Omega(\cdot \vert z)}$ converges to $D_\Omega(\cdot \vert z)$ by taking a limit, knowing that $D_\Omega(y|y) = 0$.
\end{proof}

We then proceed to the equivalence property of the family of recursive \Bregman{} divergences. The property is important for proving the theorem and representing the dual representation of MD. Moreover, it is also used in \cref{thm:vmdkl} as a key ingredient which constructs our VOMD framework.
\begin{lemma}[Equivalence of first variations] \label{lem:equivf}
  Suppose $\Omega: \cM(\cX) \to \bR \cup \{+\infty\}$ Assume that for all $z\in\dom(\Omega)$, the first variation $\deltaC\Omega(z)$ exists, then, for all $x,y,y_1, y_2 \in\dom(\Omega)$, the first variation taken for the first argument $x$ of the following \Bregman{} divergences are equivalent: $\deltaC D_\Omega(x\vert y) = \deltaC D_{D_\Omega(\cdot \vert y_1)}(x|y) = \deltaC D_{D_\Omega(\cdot \vert y_2)}(x|y)$.
\end{lemma}
\begin{proof}[Proof of \cref{lem:equivf}] First, it can be analytically driven $\deltaC D_\Omega(x\vert y) = \deltaC \Omega(x) - \deltaC \Omega(y)$. Next, by definition, taking the first variation of $D_{D_\Omega(\cdot|z)}(x|y)$ with respect to $x$ for arbitrary $z\in\dom(\Omega)$ yields $\deltaC D_\Omega(x\Vert z) - \deltaC \langle \Omega(y) - \Omega(z), x -y \rangle$. Knowing that the second term $\deltaC \langle \Omega(y) - \Omega(z), x -y \rangle$ is linear, we achieve $\deltaC D_{D_\Omega(\cdot|z)}(x|y) = \deltaC \Omega(x) - \deltaC \Omega(z) - (\deltaC\Omega(y) - \deltaC \Omega(z)) = \deltaC \Omega(x) - \deltaC \Omega(y)$, which completes the proof.
\end{proof}
By an inductive reasoning, we arrive at the basic property of family of \Bregman{} divergences, that all divergence recursively defined by the \Bregman{} potential $\Omega$, has the (local) idempotence and the (global) equivalence of first variation. To address characteristics for particular \Bregman{} potential $\Omega$, we apply the notions of relative smoothness and convexity with respect to $\Omega$, which was first introduced by \citet{birnbaum2011distributed}.
\begin{definition}[Relative smoothness and convexity] \label{def:relsmooth}
  Let $G: \cM(\cX) \to \bR \cup \{+ \infty \}$ be a proper convex functional. Given scalar $l, L \ge 0$, we define that $G$ is $L$-smooth and $l$-strongly-convex relative to $\Omega$ over $\cC$ if for every $x,y \in \dom(G) \cap \dom(\Omega)\cap\cC$, we have
  \[
    D_{\mspace{-1mu}G}(x \Vert y) \le L D_\Omega(x \Vert y),\mspace{50mu}D_{\mspace{-1mu}G}(x \Vert y) \ge l D_\Omega(x \Vert y),
  \]
  respectively, where $D_{\mspace{-1mu}G}$ and $D_{\mspace{-1mu}G}$ are \Bregman{} divergences associated with $G$ defined in \cref{def:breg}.
\end{definition}
Applying the idempotence lemma \cref{lem:idem}, we immediately recognize that the \Bregman{} divergence $D_\Omega$ is relatively $1$-smooth and $1$-strongly-convex for $\Omega$. To start our analysis, we reintroduce the well-known three-point identity for a \Bregman{} divergence.
\begin{lemma}[Three-point identity] \label{lem:threeid}
  For all $\pi_a, \pi_b, \pi_c \in \cC \cap\dom(\Omega)$, we have the following identity
  \[ 
    \bigl\langle\deltaC\Omega(\pi_a)-\deltaC\Omega(\pi_b),\pi_c-\pi_b\bigr\rangle = D_\Omega(\pi_c\Vert\pi_b)-D_\Omega(\pi_c\Vert\pi_a)+D_\Omega(\pi_b\Vert\pi_a)
  \]
  when $D_\Omega$ is the \Bregman{} divergence defined in \cref{def:breg}.
\end{lemma}
\begin{proof}[Proof of \cref{lem:threeid}] By the definition of \Bregman{} divergence, we have
  \begin{equation*}
  \begin{aligned}
    D_\Omega(\pi_c\Vert\pi_b)-D_\Omega(\pi_c\Vert\pi_a)+D_\Omega(\pi_b\Vert\pi_a)
    &=\Omega(\pi_c)-\Omega(\pi_b)-\bigl\langle \deltaC\Omega(\pi_b),\pi_c-\pi_b\bigr\rangle\\
    &-\Omega(\pi_c)+\Omega(\pi_a)+\bigl\langle \deltaC\Omega(\pi_a),\pi_c-\pi_a\bigr\rangle\\
    &+\Omega(\pi_b)-\Omega(\pi_a)-\bigl\langle \deltaC\Omega(\pi_a),\pi_b-\pi_a\bigr\rangle\\
    &=\bigl\langle\deltaC\Omega(\pi_a)-\deltaC\Omega(\pi_b),\pi_c-\pi_b\bigr\rangle.
  \end{aligned}
  \end{equation*}
  Therefore, the proof is complete.
\end{proof}
Utilizing the three-point identity, we present the following useful lemmas for dealing inequalities regarding improvements by \citet{mdirl}, which we call left and right \Bregman{} differences.
\begin{lemma}[Left \Bregman{} difference] \label{lem:left_breg_diff}
  For all $\pi_a, \pi_b, \pi_c \in \cC \cap\dom(\Omega)$, the following identity holds. 
  \begin{equation}
    D_\Omega(\pi_b\Vert\pi_a\bigr) - D_\Omega(\pi_c\Vert\pi_a) =  - \bigl\langle \deltaC \Omega(\pi_c)-\deltaC \Omega(\pi_a),\pi_c-\pi_b\bigr\rangle + D_\Omega(\pi_{b}\Vert\pi_c).
  \end{equation}
\end{lemma}
\begin{proof}[Proof of \cref{lem:left_breg_diff}]
  Using Lemma~\ref{lem:threeid}, we have
  \begin{equation*}
    D_\Omega(\pi_b\Vert\pi_a) - D_\Omega(\pi_c\Vert\pi_a) = - D_\Omega(\pi_c\Vert\pi_b) + \bigl\langle \deltaC\Omega(\pi_a)-\deltaC \Omega(\pi_b), \pi_c-\pi_b\bigr\rangle.
  \end{equation*}
  Utilizing an identity of two \Bregman{} divergences for arbitrary $(\rho,\bar{\rho})$:
  \begin{equation} \label{eq:breg_two}
    D_\Omega(\rho\Vert\bar{\rho})+
    D_\Omega(\bar{\rho}\Vert\rho)=\bigl\langle \deltaC\Omega(\rho) - \deltaC\Omega(\bar{\rho}), \rho-\bar{\rho}\bigr\rangle.
  \end{equation}
  We separate $\deltaC\Omega(\pi_a)-\deltaC\Omega(\pi_b)$ into $\deltaC\Omega(\pi_a) - \deltaC\Omega(\pi_c)$ and $\deltaC\Omega(\pi_c)-\deltaC \Omega(\pi_b)$ and write the rest of the derivation as follows.
  \begin{equation*}
  \begin{aligned}
    &D_\Omega(\pi_b\Vert\pi_a) - D_\Omega(\pi_c\Vert\pi_a)\\
    &\qquad =\underbrace{-D_\Omega(\pi_c\Vert\pi_b) + \bigl\langle \deltaC \Omega(\pi_c) - \deltaC \Omega(\pi_b), \pi_c-\pi_b\bigr\rangle}_\textrm{\cref{eq:breg_two}} + \bigl\langle \deltaC \Omega(\pi_a) - \deltaC \Omega(\pi_c),\pi_c-\pi_b\bigr\rangle\\
    &\qquad = D_\Omega(\pi_b\Vert\pi_c) + \bigl\langle \deltaC\Omega(\pi_a)- \deltaC \Omega(\pi_c),\pi_c-\pi_b\bigr\rangle
  \end{aligned}
  \end{equation*}
  Therefore, we achieve the desired identity.
\end{proof}
\begin{lemma}[Right \Bregman{} difference] \label{lem:right_breg_diff}
  For all $\pi_a, \pi_b, \pi_c$, the following identity holds.
  \begin{equation} \label{eq:identity1}
    D_\Omega(\pi_c\Vert\pi_b)-D_\Omega(\pi_c\Vert\pi_a)=D_\Omega(\pi_a\Vert\pi_b)+\bigl\langle \deltaC \Omega(\pi_a)- \deltaC\Omega(\pi_b), \pi_c-\pi_a\bigr\rangle
  \end{equation}
\end{lemma}
\begin{proof}[Proof of Lemma~\ref{lem:right_breg_diff}]
  By Lemma~\ref{lem:threeid}, we have
  \begin{equation*}
    D_\Omega(\pi_c\Vert\pi_b)-D_\Omega(\pi_c\Vert\pi_a)=-D_\Omega(\pi_b\Vert\pi_a) + \bigl\langle \deltaC \Omega(\pi_a) - \deltaC\Omega(\pi_b),\pi_c-\pi_b\bigr\rangle.
  \end{equation*}
  We separate $\pi_c\!-\pi_b$ into $\pi_c\!-\pi_a$ and $\pi_a\!-\pi_b$ and write the rest of the derivation as follows.
  \begin{equation*}
  \begin{aligned}
    &D_\Omega(\pi_c\Vert\pi_b)-D_\Omega(\pi_c\Vert\pi_a)\\
    &\qquad=\underbrace{-D_\Omega(\pi_b\Vert\pi_a) + \bigl\langle \deltaC \Omega(\pi_a) -\deltaC\Omega(\pi_b),\,\pi_a\!-\pi_b\bigr\rangle}_\textrm{\cref{eq:breg_two}}+\bigl\langle \deltaC\Omega(\pi_a)- \deltaC\Omega(\pi_b),\pi_c-\pi_a\bigr\rangle\\
    &\qquad=D_\Omega(\pi_a\Vert\pi_b)+\bigl\langle \deltaC\Omega(\pi_a)-\deltaC\Omega(\pi_b),\pi_c -\pi_a\big\rangle
  \end{aligned}
  \end{equation*}
  Therefore, we achieve the desired identity.
\end{proof}

Additionally, we introduce the three-point inequality \citep{chen1993convergence}, which has been a key statement for proving MD convergence for a static cost functional \citep{mdsem}, and OMD improvement for temporal costs. The proof mostly follows \citet{mdsem} with a slight change of notation.
\begin{lemma}[Three-point inequality] \label{lem:threeineq}
  Given $\pi \in \cM(\cX)$ and some proper convex functional $\Psi: \cM(\cX) \to \bR\cup\{+\infty\}$, if $\deltaC \Omega$ exists, as well as $\bar{\rho} = \argmin_{\rho\in\cC}\{\Psi(\rho) + D_\Omega(\rho \Vert \pi)\}$, then for all $\rho\in\cC\cap\dom(\Omega)\cap\dom(\Psi)$: $\Psi(\rho) + \mspace{-2mu} D_\Omega(\rho\Vert\pi) \ge \Psi(\bar{\rho})  + \mspace{-2mu}  D_\Omega(\bar{\rho} \Vert \pi) + \mspace{-2mu} D_\Omega(\rho \Vert \bar{\rho}).$
\end{lemma}
\begin{proof}[Proof of \cref{lem:threeineq}]
  The existence of $\deltaC \Omega$ implies $\cC \cap \dom(D_\Omega(\cdot | y)) = \cC \cap \dom(\Omega) \cap \dom(\Psi)$. Set $G(\cdot) = \Psi(\cdot) + D_\Omega(\cdot \Vert y)$. By linearity and idempotence, we have for any $\rho\in\cC \cap \dom(\Omega) \cap \dom(\Psi)$
  \begin{equation} \label{eq:glower}
    D_{\mspace{-1mu}G}(\rho\Vert\bar{\rho})  = D_\Psi\mspace{-1mu}(\rho \Vert \bar{\rho}) + D_\Omega(\rho \Vert \bar{\rho}) \ge D_\Omega(\rho \Vert \bar{\rho}).
  \end{equation}
  By $\bar{\rho}$ being the optimality for $G$, for all $x\in\cC$,
  \[
   \dsupplus\mspace{-1.8mu}G(\bar{\rho}; \mspace{1mu}\rho - \bar{\rho}) = \!\lim_{h\to0^+} \frac{G((1-h) \bar{\rho} + h\rho) - G(\bar{\rho})}{h} \ge 0,
  \]
  which suggests $G(\rho) \ge G(\bar{\rho}) + D_{\mspace{-1mu}G}(\rho \Vert \bar{\rho})$. Applying \eqref{eq:glower} to this inequality complete the proof.
\end{proof} 
The following argument is from the convergence rate of mirror descent for relatively smooth and convex pairs of functionals, and extend to infinite dimensional convergence results of \citet{lu2018relatively} and  \citet{mdsem}. We aim to reformulate the statements in online learning, addressing one-step improvement of OMD.
\begin{lemma}[OMD improvement] \label{lem:improve}
  Suppose a temporal cost $F_t: \cM(\cX) \to \bR$ which is $L$-smooth and $l$-strongly-convex relative to $\Omega$ and $\eta_t \le \frac{1}{L}$. Then, OMD improves for current cost $F_t(\pi_{t+1}) \le F_t(\pi_t)$.
\end{lemma}
\begin{proof}[Proof of \cref{lem:improve}]
  Since $F$ is $L$ relatively smooth, we initially have the inequality
  \begin{equation} \label{eq:updateineq}
    F_t(\pi_{t+1}) \le F_t(\pi_t) + \dsupplus\mspace{-2mu} F(\pi_t; \pi_{t+1} - \pi_t) + L D_\Omega(\pi_{t+1} \vert \pi_t)       
  \end{equation}
  Applying the three-point inequality (\cref{lem:threeineq}) to \cref{eq:updateineq}, and setting a linear functional $\Psi(\rho) = \eta_t \dsupplus\mspace{-1mu}F_t(\pi_t;\pi - \pi_t)$, $\rho = \pi_t$ and $\bar{\rho}=\pi_{t+1}$ yields
  \[
    \dsupplus\mspace{-1.8mu}F_t(\pi_{t}; \pi_{t+1} - \pi_t) + \tfrac{1}{\eta_t} D_\Omega(\pi_{t+1} \vert \pi_t) \le \dsupplus\mspace{-1.8mu} F_t(\pi_t; \rho -\pi_t) + \tfrac{1}{\eta_t} D_\Omega(\rho \vert \pi_t) - \tfrac{1}{\eta_t} D_\Omega(\rho\Vert \pi_{t+1}).
  \]
  Since $F_t$ is assumed to be $l$-strongly convex relative to $\Omega$, we also have
  \begin{equation} \label{eq:omd_improve_strong}
    \dsupplus\mspace{-1.8mu}F(\pi_t; \rho - \pi_t) \le F_t(\rho) - F_t(\pi_t) - l D_\Omega(\rho\vert \pi_{t}),
  \end{equation}
  Then, by using \eqref{eq:omd_improve_strong}, \cref{eq:updateineq} becomes
  \begin{equation}
    F_t(\pi_{t+1}) \le F_t(\rho) + (\tfrac{1}{\eta_t} - l) D_\Omega(\rho \vert \pi_t) - \tfrac{1}{\eta_t} D_\Omega(\rho\vert \pi_{t+1})  + (L - \tfrac{1}{\eta_t}) D_\Omega(\pi_{t+1} \Vert \pi_t).
  \end{equation}
  By substituting $\rho = \pi_t$, since $D_\Omega(\rho \vert \pi_{t+1}) \ge 0$ and $L - \frac{1}{\eta_t} \le 0$, this shows $F_t(\pi_{t+1}) \le F_t(\pi_t)$, \ie{}, $F_t$ is decreasing at each iteration. This completes the proof.
\end{proof}

A fundamental property with the dual space $\cD$ induced by the first variation $\deltaC$ holds in our OMD setting. The existence of such learning sequence--particularly in \Sinkhorn{}--is well discussed by \citet{introeot} and \cite{mdsem}. Focusing on the dual geometry, we explicitly call this relationship with arbitrary step size $\eta_t$ as ``dual iteration.''
\begin{lemma}[Dual iteration] \label{lem:duali}
  Suppose that first variations $\deltaC F_t(\pi_t)$ and $\deltaC \Omega(\pi_t)$ exists for $t\ge0$. Then, online mirror descent updates \cref{eq:omd} is equivalent to $\deltaC \Omega (\pi_{t+1}) - \deltaC\Omega (\pi_t) = - \eta_t \deltaC F_t(\pi_t)$, for all $\pi_t \in \cC, t \in \bN$.
\end{lemma}
\begin{proof}[Proof of \cref{lem:duali}]
  The optimization \eqref{eq:omd} is equivalent to having the property for subsequent $\pi_{t+1}$:
  \begin{equation}
  \begin{aligned}
    &\dsupplus\mspace{-1.2mu} F_t(\pi_t; \pi - \pi_t) + \tfrac{1}{\eta_t} D_\Omega(\pi \Vert \pi_t) \ge \dsupplus\mspace{-1.2mu} F_t(\pi_t; \pi_{t+1}\! - \pi_t) + \tfrac{1}{\eta_t} D_\Omega(\pi_{t+1} \vert \pi_t)\\
    &\mspace{60mu} \iff \bigl\langle\mspace{-1.2mu} \deltaC F_t(\pi_t) - \tfrac{1}{\eta_t}\deltaC \Omega(\pi_t),\, \pi - \pi_{t+1} \mspace{-1mu}\bigr\rangle  + \tfrac{1}{\eta_t}  \bigl(\Omega(\pi) - \Omega(\pi_{t+1})\bigr) \ge 0,\quad \forall \pi \in \cC.
  \end{aligned}
  \end{equation}
  Setting the free parameter $\pi = \pi_{t+1} + h(\pi - \pi_{t+1})$ and taking the limit $h \to 0^+$ yields the result.
\end{proof}
\begin{remark} \label{rem:duali}
  With applications of \cref{lem:duali} and \cref{lem:equivf}, we can achieve a concise form of iteration in the dual  using our temporal cost as:
  \begin{equation} \label{eq:duali_kl}
  \begin{aligned}
    \deltaC \Omega(\pi_t) - \deltaC \Omega(\pi_{t+1}) &= \eta_t \bigl(\deltaC (-H)(\pi_t) - \deltaC (-H)(\pct)\bigr)\\
    &= \eta_t \bigl(\deltaC \Omega(\pi_t) - \deltaC \Omega(\pct) \bigr),
  \end{aligned}
  \end{equation}
  where $H$ denotes the  entropy, \ie{}, the minus KL divergence with the \Lebesgue{} measure.
\end{remark}

Leveraging the aforementioned lemmas, we have systematically introduced and rigorously formalized the essential concepts necessary to progress with our analysis within the OMD framework. Finally, we are ready to describe a suitable step size scheduling by the following arguments.
\begin{lemma}[Step size I] \label{lem:step1}
  Suppose that $F_t = \KL(\pi \Vert \pct)$ and $\Omega = \KL(\pi\Vert e^{-c_\varepsilon}\mu\otimes\nu)$. If \one{} $\lim_{t\to\infty} \eta_t = 0^+$ and \two{} $\sum_{t=1}^\infty \eta_t = +\infty$ \three{} $\eta \le \frac{1}{L}$, the OMD algorithm converges to a certain $\pcD$
\end{lemma} 
\begin{proof}[Proof of \cref{lem:step1}]
  From \cref{lem:improve}, we have
  \begin{equation} 
    \eta_t( F_t(\pi_{t+1}) - F_t(\pi_t))  \le - D_\Omega(\pi_t\Vert \pi_{t+1})  + (\eta_t L - 1) D_\Omega(\pi_{t+1} \Vert \pi_t).
  \end{equation}
  Taking $\lim_{t\to\infty} \eta_t = 0$ ensures improvements; this means for any $\varepsilon > 0$ there exists some $0 < \delta \le 1$ such that $D_\Omega(\pi_t\Vert \pi_{t+1}) + D_\Omega(\pi_{t+1} \Vert \pi_t) < \varepsilon$ whenever $\eta_t < \delta$. Since convexity and the lower semicontinuity of the \Bregman{} divergence $D_\Omega$ induced by KL, we conclude that OMD to a certain point upon the assumed step size scheduling.
\end{proof}

\begin{lemma}[Step size II] \label{lem:step2} 
  Assume that $\inf_{\pi\in\cC}\bE_t[D_\Omega(\pi_t, \pct)] > 0$ for all $t \in [1, \infty)$. Suppose that $\eta_t \to 0$ and $\lim_{T\to\infty} \bE[\frac{1}{T}\sum_{t=1}^T D_\Omega(\pi_t \Vert \pct)] = 0$ if and only if  $\sum_{t=1}^\infty \eta_t = +\infty$.
\end{lemma}
\begin{proof}[Proof of \cref{lem:step2}]
  We firstly argue that due to dual iteration \eqref{eq:duali_kl}, any improvements on KL in \cref{lem:improve} are also indicates corresponding improvements in the \Bregman{} divergence, \ie{} $D_\Omega(\pi_{t+1}\Vert \pct) \le D_\Omega(\pi_t \Vert \pct)$, and if $\eta_t \to 0$, then the process $\{\pi_t\}_{t=1}^\infty$ is convergent. By the dominated convergence theorem \citep{royden1988real} with the ergodicity assumtion of nonstationary solutions $\{\pct\}_{t=1}^\infty$ (\citealp{ergodic}; Appendix~\ref{sect:proof}), there exists a constant $\varepsilon$ that satisfies \begin{equation}\label{eq:eventual_convergence}
    \bE_{1:t+1}[D_\Omega(\pi_{t+1}\Vert \pctplusone)] \ge \bE_{1:t+1}[D_\Omega(\pi_{t+1} \Vert \pct)] + \varepsilon
  \end{equation} for $t > n$ for some $n$ as $\eta_t \to 0$, where an expectation subscripted by the range of ``$1\!:\!t$'' indicates a notation of time-averaging from the time step $1$ to $t$. Consequently, we achieve the following inequality
  \begin{align}
    &\bE_{1:t+1}[D_\Omega(\pi_{t+1} \Vert \pctplusone) ] \nonumber\\
    &\qquad\ge \bE_{1:t+1}[D_\Omega(\pi_{t+1} \Vert \pct)] + \varepsilon \mspace{441mu}\textrm{\cref{eq:eventual_convergence}}\nonumber\\
    &\qquad\ge \bE_{1:t}\bigl[D_\Omega(\pi_t\Vert \pct)-\!\langle \deltaC \Omega(\pi_{t+1}) - \deltaC \Omega(\pi_t), \pct\! -  \pi_t  \rangle\bigr]\!+\bE_{1:t+1}\bigl[D_\Omega(\pi_{t+1}\Vert \pi_t)\bigr] + \varepsilon\mspace{35mu}\textrm{Lemma~\ref{lem:left_breg_diff}}\nonumber\\
    &\qquad= \bE_{1:t}\bigl[D_\Omega(\pi_t\Vert \pct)- \eta_t D_\Omega(\pi_t\Vert \pct) + \eta_t D_\Omega(\pct\Vert \pi_t)\bigr]\!+\bE_{1:t+1}\bigl[D_\Omega(\pi_{t+1}\Vert \pi_t)\bigr] + \varepsilon\mspace{45mu}\textrm{\cref{eq:duali_kl}}\nonumber\\
    &\qquad= (1 - \eta_t) \bE_{1:t}\bigl[D_\Omega(\pi_t \Vert \pct)\bigr] + \bE_{1:t+1}\bigl[D_\Omega(\pi_{t+1}\Vert \pi_t) + \eta_t D_\Omega(\pct \Vert \pi_t)\bigr] + \varepsilon \nonumber\\
    &\qquad\ge  (1 - \eta_t) \bE_{1:t}\bigl[D_\Omega(\pi_t\Vert \pct)\bigr] + \varepsilon^{\prime} \label{eq:left_ineq}
  \end{align}
  for some $t$ and $0 < \varepsilon < \varepsilon^{\prime}$, where \cref{eq:eventual_convergence}, \cref{lem:left_breg_diff}, and \cref{eq:duali_kl} are used in order.

  \textit{Necessity.}\enspace For big enough $t \ge n$ where $n \in \bN$, we can achieve the inequality in \cref{eq:left_ineq} as
  \begin{equation} \label{eq:left_rewrite}
    \bE_{1:t+1}\bigl[D_\Omega(\pi_{t+1}\Vert \pctplusone)\bigr] \ge (1 - \eta_t) \bE_{1:t}\bigl[D_\Omega(\pi_t \Vert \pct)\bigr],
  \end{equation}
  Since we have assumed that $\eta_t$ converges to $0$, consider a step size sequence $0 < \eta_t \le \frac{2}{2+k}$ for $k > 0$. Denote a constant $a = \frac{2+k}{2} \log \frac{2+k}{k}$ and apply the elementary inequality \citep{omd_converge}
  \[ 
    1 - x \ge \exp(-a x),\quad \textrm{such that}\quad  0<x\le\frac{2}{2+ k}.
  \]
  From \cref{eq:left_rewrite}, we achieve
  \[
    \bE_{1:t+1}\bigl[D_\Omega(\pi_{t+1}\Vert \pctplusone)\bigr] \ge  \exp(-a\eta_t) \bE_{1:t}\bigl[D_\Omega(\pi_t \Vert \pct)\bigr]. 
  \]
  for all $t \ge n$. Iteratively applying this inequality iterative for $t=n,n+1,\cdots, T-1$ gives
  \begin{equation} \label{eq:iter_left}
  \begin{aligned}
    \bE_{1:T}[D_\Omega(\piT\Vert\pcT)] &\ge \bE_{1:n}[D_\Omega(\pi_n \Vert \pcn)] \prod_{t=n}^{T-1} \exp(-a\eta_t)\\ 
    &= \exp\biggl\{-a\sum_{t=n}^{T-1}\eta_t\biggr\} \bE_{1:n}[D_\Omega(\pi_n\Vert \pcn)]. 
  \end{aligned}
  \end{equation}
  From the assumption $\pi^\ast \ne \pi_n$, $D_\Omega(\pi_n \Vert \pcn) > 0$ by the property of divergence. Therefore, by \cref{eq:iter_left}, the convergence $\lim_{t\to\infty}\bE_{1:t}[D_\Omega(\pi_t \Vert \pct)] = 0$ implies the series $\sum_{t=1}^\infty \eta_t$ diverges to $+\infty$ so that $\exp(-a\sum_{t=n}^{T-1}\eta_t)$ converges to $0$.

  \textit{Sufficiency.} Consider a static \Schrodinger{} bridge problem with couplings $\pi\in\Pi(\mu,\nu)$, which is in a constraint set
  \[
  \cC = \bigl\{ \pi \vert (\mu, \nu) \in \cP_2(\bR^d)\cap \cP_\textrm{alc}(\bR^d), (\varphi, \psi) \in L^1(\mu) \times L^1(\nu),\,\textrm{and}\,\varphi, \psi \in C^2(\bR^d)\cap\Lip(\cK) \bigr\}.
  \]
  For $\rho, \bar{\rho} \in \cC$, we can see
  \begin{equation*}
    D_\Omega(\bar{\rho}\Vert\rho) = \Omega(\bar{\rho}) - \Omega(\rho) - \langle\deltaC\Omega(\rho), \bar{\rho}-\rho\rangle \ge 0 \iff  -\langle\deltaC\Omega(\rho), \bar{\rho}-\rho\rangle \ge \Omega(\rho) - \Omega(\bar{\rho}).
  \end{equation*}
  By adding $\langle\deltaC\Omega(\bar{\rho}), \bar{\rho} - \rho\rangle$, we achieve a property:
  \begin{equation} \label{eq:diffprod}
    \langle \deltaC\Omega(\rho) - \deltaC\Omega(\bar{\rho}), \rho - \bar{\rho} \rangle \ge D_\Omega(\rho\Vert\bar{\rho}). 
  \end{equation} 
  Then, suppose that we have the asymptotic dual mean $\pcD$. Using the right \Bregman{} difference \cref{lem:right_breg_diff}, the one-step progress from the perspective of dual mean writes as
  \begin{equation*}
  \begin{aligned}
    D_\Omega(\pcD \Vert \pi_{t+1}) - D_\Omega(\pcD \Vert \pi_t) &= \bigl\langle \deltaC\Omega(\pi_t) - \deltaC \Omega(\pi_{t+1}), \pcD - \pi_t \bigr\rangle + D_\Omega(\pi_t \Vert \pi_{t+1}).\\
    &= \eta_t \bigl\langle \deltaC\Omega(\pi_t) - \deltaC\Omega(\pct), \pcD - \pi_t\bigr\rangle + D_\Omega(\pi_t \Vert \pi_{t+1}) \\
    &= \eta_t \bigl\langle \deltaC\Omega(\pi_t) - \deltaC\Omega(\pcD), \pcD - \pi_t\bigr\rangle + \eta_t\bigl\langle \deltaC\Omega(\pcD) - \deltaC\Omega(\pct), \pct - \pi_t\bigr\rangle + D_\Omega(\pi_t \Vert \pi_{t+1}) \\
    &\le - \eta_t D(\pcD \Vert \pi_t) + \eta_t\bigl\langle \deltaC\Omega(\pcD) - \deltaC\Omega(\pct), \pct - \pi_t\bigr\rangle + D_\Omega(\pi_t \Vert \pi_{t+1})
  \end{aligned}
  \end{equation*}
  where the inequality is from \cref{eq:diffprod}. By applying the definition of $\pcD$ and ergodicity of $\{\pct\}_{t=1}^\infty$, we can bound the expectation by finding some $t > n$ such that
  \begin{align}
    \bE_{1:t+1}\bigl[D_\Omega(\pcD\Vert \pi_{t+1})\bigr] &\le \bE_{1:t}\bigl[(1-\eta_t)D_\Omega(\pcD \Vert \pi_t)\bigr] + \bE_{1:t+1}\bigl[D_\Omega(\pi_t \Vert \pi_{t+1})\bigr] \nonumber \\
    &\le \bE_{1:t}[(1 - \eta_t) D_\Omega(\pcD \Vert \pi_t)] +\frac{1}{2\omega} \bE_{1:t+1}\bigl[\bigl\lVert \nabla (\deltaC\Omega(\pi_t) - \deltaC\Omega(\pi_{t+1})) \bigr\rVert^2_{L^2(\pi_t)}\bigr] \nonumber \\
    & \le \bE_{1:t}[(1 - \eta_t) D_\Omega(\pcD \Vert \pi_t)] + \frac{\eta^2_t}{2\omega} \bE_{1:t}\bigl[\lVert \nabla (\deltaC\Omega(\pi_t) - \deltaC\Omega(\pct))  \rVert_{L^2(\pi_t)}\bigr] \nonumber \\
    &\le \bE_{1:t}[(1-\eta_t) D_\Omega(\pcD \Vert \pi_t)] + 2\eta^2_t \omega^{-1} \cK, \label{eq:right_ineq}
  \end{align}
  where $\cK$ is the \Lipschitz{} constant for each log \Schrodinger{} potential in $\cC$. For the first inequality, we use \cref{asm:bs}, and we use the log \Sobolev{} inequality $\mathrm{LSI}(\omega)$ from \cref{asm:lsi} in the second inequality. 
  Let $\{A_t\}^{\infty}_{t=1}$, denote a sequence of $A_t = \bE_{1:t}[D_\Omega(\pcD \Vert \pi_t)]$. As a result, we have
  \begin{equation} \label{eq:seqa}
    A_{t+1} \le (1 - \eta_t) A_t + z\eta_t^2,\quad \forall t > n,
  \end{equation}
  where $z \coloneqq 2\omega^{-1} \cK$. For a constant $h > 0$, we argue that $A_{t_1} < h$ for some $t_1 > n^\prime$. Suppose that this statement is \textit{not} true; we find some $t \ge t_1$ such that $A_t > h$, $\forall t \ge t_2$. Since $\lim_{t\to\infty} \eta_t = 0$, there are some $t > t_3 > t_2$ that $\eta_t \le \frac{h}{4}$. However, \cref{eq:seqa} tells us that for $t \ge t_3$, for $t \ge t_3$, 
  \[
    A_{t+1} \le (1 - \eta_t) A_t + z\eta^2_t \le A_{t_3} - \frac{h}{4} \sum_{k=t_3}^T \eta_k \to -\infty\quad (\textrm{as}\ t\to\infty).
  \]
  This results to a contradiction, which verifies $A_{t} < h $ for $t > n^\prime$. Since $\lim_{t\to\infty} \eta_t = 0$, we can find some $\eta_t$ which makes $A_t$ monotonically decreasing. Therefore, we conclude the nonnegative sequence $\{A_t\}_{t=1}^\infty$ finds convergence by iteratively applying the upper bound in \cref{eq:seqa}.
\end{proof}
  
  We now prove \cref{thm:step_size} under consideration of the particular case of $\eta_t = \frac{2}{t+1}$. Then, \cref{eq:seqa} becomes
  \[
    A_{t+1} \le \biggl(1 - \frac{2}{t+1}\biggr) A_t + \frac{4z}{(t+1)^2},\quad \forall t \ge n. 
  \]
  It follows that recursive relation writes as
  \[
    t(t+1) A_{t+1} \le (t-1)t A_t + 4z,\quad \forall t \ge n.
  \]
  Iterative applying the relation, we achieve the following inequality:
  \[
    (T-1) T A_T \le (n-1) n A_n + 4z(T- n),\quad \forall T \ge n.
  \]
  Therefore, we finally achieve inequality as follows:
  \begin{equation} \label{eq:thm1_final_ineq}
    \bE_{1:T}[D_\Omega(\pcD \Vert \piT)] \le \frac{(n-1)n \bE_{1:n}[D_\Omega(\pcD \Vert \pi_n)]}{(T-1)T} + \frac{4z}{T},\quad\forall T \ge n. 
  \end{equation}
  Since we assumed $\pi^\ast = \pcD$, $\bE_{1:T}[D_\Omega(\pi^\ast \Vert \piT)] = \cO(1/T)$. Combining this with Lemmas~\ref{lem:step1}~and~\ref{lem:step2}, the proof of \cref{thm:step_size} is complete.

\subsection{Proof of \texorpdfstring{\cref{prop:convergence}}{Proposition \ref{prop:convergence}}}

The proof is based on the \Doob{}'s forward convergence theorem.
\begin{theorem}[\Doob{}'s forward convergence theorem] \label{thm:doobfwd}
  Let $\{X_t\}_{t\in\bN}$ be a sequence of nonnegative random variables and let $\{\cF_t\}_t$ be a random variable and let $\{\cF_t\}_{t\in\bN}$ be a filtration with $\cF_t \subset \cF_{t+1}$ for every $t\in \bN$. Assume that $\bE[X_{t+1}\vert \cF_t] \le X_t$ almost surely for every $t\in \bN$. Then, the sequence $\{X_t\}$ converges to a nonnegative random variable $X_\infty$ almost surely.
\end{theorem}

We follow the derivation of \cref{eq:right_ineq}: there exists $n\in \bN$ which satisfies
\[
  \bE_t[D_\Omega(\pcD \Vert \pi_{t+1})] \le D_\Omega(\pcD \Vert \pi_t) + 2\eta^2_t\omega^{-1} \cK, \quad \forall t \ge n
\]
and since the step size is scheduled as $\lim_{t\to\infty}\eta_t= 0$, the condition $\sum_{t=1}^\infty \eta^2_t < \infty$ enables us to define a stochastic process $\{X_t\}_{t\in\bN}$:
\begin{equation}
  X_t = D_\Omega(\pcD \Vert \pi_{t}) + 2\omega^{-1}\cK\mspace{-1.5mu}\sum_{i=t}^\infty \eta_i^2. 
\end{equation}
It is straightforward that the defined random variable satisfies $\bE_t[X_{t+1}] \le X_t$ for $t \ge n$. Since $X_t \ge 0$, the process is a sub martingale. By \cref{thm:doobfwd}, the sequence $\{X_t\}_{t\in\bN}$ converges to a nonnegative random variable $X_\infty$ almost surely. Therefore $D_\Omega(\pcD \Vert \pi_t)$ converges to $0$ almost surely.\hfill\qedsymbol

\subsection{Proof of \texorpdfstring{\cref{thm:regret}}{Proposition \ref{thm:regret}}}

To achieve a meaningful regret bound for our problem setup, we first demonstrate the following. 

\begin{lemma} \label{lem:subgrad}  
  For all $w = \argmin_y\{ \langle \hat{g}, y \rangle+ \tfrac{1}{\eta}D_\Omega(y \Vert z)\}$ with $\eta>0$, the following equation.  
  \begin{equation}
    \forall u. \langle \eta \hat{g}, w - u \rangle \le D_\Omega(u \Vert z) - D_\Omega(u \Vert w) - D_\Omega(w \Vert z)
  \end{equation}
\end{lemma}
\begin{proof}[Proof of \cref{lem:subgrad}]
  By the first order optimality of $\{ \langle g, y \rangle+ D_\Omega(y \Vert z)\}$ as a function of $w$, we have
  \[
  \begin{aligned}
    &\langle \hat{g} + \tfrac{1}{\eta}\deltaC D_\Omega(w \Vert z), u - w \rangle \ge 0 \\
    &\qquad \implies \langle \hat{g}, w - u \rangle \le \tfrac{1}{\eta}\langle - \deltaC D_\Omega(w \Vert z), w - u \rangle = \tfrac{1}{\eta}(D_\Omega(u \Vert z) - D_\Omega(u \Vert w) - D_\Omega(w \Vert z)).
  \end{aligned}
  \]
  where used \cref{lem:left_breg_diff} in the derivation. This completes the proof.
\end{proof}

Next, we derive the one-step relationship for OMD. The result entails that the regret at each step is related to a quadratic expression of $\eta_t$, which is a key aspect of sublinear total regret. From a technical standpoint, we can see that the assumption for log \Sobolev{} inequality generally works as a premise for \Lipschitz{} continuity of gradient, \ie{}, $\nabla\Omega$ in classical MD analyses.
\begin{lemma}[Single step regret] \label{lem:single}
  Suppose a static \Schrodinger{} bridge problem with the aforementioned constraint $\cC$.  Let $D_\Omega$ be the \Bregman{} divergence wrt $\Omega: \cP(\cX) \to \bR + \{+\infty\}$. Then,
  \begin{equation}
    \eta_t (F_t(\pi_t) - F_t(u)) \le D_\Omega(u \Vert \pi_t ) - D_\Omega(u \Vert \pi_{t+1}) + \frac{\eta^2_t}{2\omega}\lVert \hat{g}_t  \rVert^2_{L^2(\pi_t)},\quad  \forall u \in \cC
  \end{equation}
  holds, where $\hat{g}_t \coloneqq \deltaC F_t(\pi_t) =  \frac{1}{\eta_t}(\deltaC\Omega(\pi_t) - \deltaC\Omega(\pi_{t+1}))$ in an MD iteration for the dual space for a step size $\eta_t$, and $\omega > 0$ is drawn from a type of log \Sobolev{} inequality in \cref{asm:lsi}.
\end{lemma}
\begin{proof}[Proof of \cref{lem:single}]
  Consider single step regrets by the adversary plays of a linearization for $\hat{g}_t$: 
  \[
    F_t(\pi_{t}) - F_t(u) \le \langle \hat{g}_t, \pi_t - u\rangle.
  \]
  Therefore, we derive a inequality for $\langle \hat{g}_t, \pi_t - u\rangle$ as follows.
  \[
  \begin{aligned}
    \langle \eta_t \hat{g}_t, \pi_t - u \rangle 
    &= \langle \eta_t \hat{g}_t, \pi_{t+1} - u \rangle + \langle \eta_t \hat{g}_t, \pi_t - \pi_{t+1} \rangle \\
    &\le D_\Omega(u \Vert \pi_t) - D_\Omega(u \Vert \pi_{t+1}) - D_\Omega(\pi_{t+1}\Vert \pi_t) +\langle \eta_t \hat{g}_t, \pi_t - \pi_{t+1} \rangle\\
    &= D_\Omega(u \Vert \pi_t) - D_\Omega(u \Vert \pi_{t+1}) - D_\Omega(\pi_{t+1}\Vert \pi_t) +\langle  \deltaC\Omega(\pi_{t+1}) - \deltaC \Omega(\pi), \pi_t - \pi_{t+1} \rangle\\
    &= D_\Omega(u \Vert \pi_t) - D_\Omega(u \Vert \pi_{t+1}) + D_\Omega(\pi_t\Vert \pi_{t+1}). 
  \end{aligned}
  \]
  Since we assumed that $\hat{g}_t =  \frac{1}{\eta_t}(\deltaC\Omega(\pi_t) - \deltaC\Omega(\pi_{t+1}))$ by the dual iteration and that \cref{asm:lsi} holds, we can achieve the upperbound $D_\Omega(\pi_t\Vert \pi_{t+1}) \le \frac{\eta^2_t}{2\omega}\lVert \hat{g}_t  \rVert^2_{L^2(\pi_t)}$ by direct calculation.
\end{proof}

We now show our upper bound of total regret by utilizing \cref{lem:single}.
\begin{lemma} \label{lem:totr}
  Assume $\eta_{t+1} \le \eta_t$. Then, $u\in\cC$, the following regret bounds for fixed $u\in\cC$ hold
  \begin{equation}
    \sum_{t=1}^T F_t(\pi_t) - F_t(u) \le \max_{1 \le t \le T} \frac{D_\Omega(u\Vert \pi_t)}{\etaT} + \frac{1}{2\omega} \sum_{t=1}^T \eta_t \lVert \tilde{g}_t\rVert^2_{L^2(\pi_t)}  
  \end{equation}
  where $\hat{g}_t = \frac{1}{\eta_t}(\deltaC\Omega(\pi_t) - \deltaC\Omega(\pi_{t+1})).$
\end{lemma}
\begin{proof}[Proof of \cref{lem:totr}]
  Define $D^2 = \max_{1 \le t \le T} D_\Omega(u\Vert \pi_t)$. We get
  \[
  \begin{aligned}
    \mathrm{Regret}(u) &= \sum_{t=1}^T (F_t(\pi_t) - F_t(u)) \\&\le \sum_{t=1}^T \biggl(\frac{1}{\eta_t} D_\Omega(u\Vert \pi_t) - \frac{1}{\eta_t} D_\Omega(u\Vert \pi_{t+1})\biggr) + \sum_{t=1}^T \frac{\eta_t}{2\omega} \lVert \hat{g}_t \rVert^2_{L^2(\pi_t)}  \\
    &= \frac{1}{\eta_1} D_\Omega(u\Vert \pi_1) - \frac{1}{\etaT} D_\Omega(u\Vert \piTplusone) + \sum_{t=1}^{T-1}\biggl(\frac{1}{\eta_{t+1}} - \frac{1}{\eta_t}\biggr) D_\Omega(u\Vert \pi_{t+1}) + \sum_{t=1}^T \frac{\eta_t}{2\omega}\lVert \hat{g}_t \rVert^2_{L^2(\pi_t)}\\
    &\le \frac{1}{\eta_1}D^2 + D^2 \sum_{t=1}^{T-1}\biggl(\frac{1}{\eta_{t+1}} - \frac{1}{\eta_t}\biggr) + \sum_{t=1}^T \frac{\eta_t}{2\omega}\lVert \hat{g}_t\rVert^2_{L^2(\pi_t)} = \frac{D^2}{\etaT} + \sum_{t=1}^T \frac{\eta_t}{2\omega} \lVert \hat{g}_t\rVert^2_{L^2(\pi_t)}.
  \end{aligned}
  \]
  Therefore, the proof is complete.
\end{proof}
Following \cref{lem:totr} and \cref{asm:lsi}, we can have the inequality
\[
  \sum_{t=1}^T F_t(\pi_t) - F_t(u) \le \frac{D^2}{\etaT} + \sum_{t=1}^T \frac{\eta_t}{2\omega} \lVert \hat{g}_t\rVert^2_{L^2(\pi_t)} \le \frac{D^2}{\eta_T} + 2\eta_1\omega^{-1} \cK T.
\]
where $D^2 = \max_{1 \le t \le T} D_\Omega(u\Vert \pi_t)$. Setting a constant step size $\eta_t \equiv \frac{D\sqrt{\omega}}{\sqrt{2 \cK T}}$ yields an upper bound of $2D\sqrt{2 \omega^{-1} \cK T}$ which proves the regret bound of $\cO(\sqrt{T})$. Also, recall that the following lemma.
\begin{lemma}[Lemma~3.5~of~\citealp{auer2002adaptive}]
  Let a sequence $a_1, a_2, \dots, a_T$ be non-negative real numbers. If $a_1 > 0$, then
  \begin{equation}
    \sum_{t=1}^T \frac{a_t}{\sqrt{\sum_{i=1}^t} a_i} \le 2\sqrt{\sum_{t=1}^T a_t}. 
  \end{equation}
\end{lemma}
Setting a adaptive scheduling $\eta_t = \frac{D\sqrt{\omega}}{\sqrt{2\sum_{i=1}^{t} \lVert \hat{g}_i\rVert^2}}$ yields  $2D\sqrt{2\omega^{-1}\sum_{t=1}^{T} \lVert \hat{g}_t \rVert^2 }$ which has a possibility to be lower than $\cO(\sqrt{T})$ depending on $\{\pct\}_{t=1}^T$. Therefore, we have formally expanded the convergence results of OMD \citep{omd_converge, orabona2018scale, omd_univ} to SBPs. \hfill\qedsymbol

\subsection{Proof of \texorpdfstring{\cref{thm:vmdkl}}{Theorem \ref{thm:vmdkl}}} \label{subsect:proof_grad_equiv}

Since $D_\Omega(\cdot \Vert \cdot) \coloneqq D_{\KL(\cdot\Vert \cR)}(\cdot \Vert \cdot) $ for a reference measure $\cR \in \cC$, we can apply \cref{lem:equivf} and achieve \cref{eq:vmdkl}. We write the following equivalent convex problems, using the equivalence of first variation for recursively defined \Bregman{} divergences.
\[
\begin{aligned}
  \bigl\langle \deltaC F_t(\pi_t), \pi - \pi_t \bigr\rangle + \tfrac{1}{\eta_t} D_\Omega(\pi\Vert\pi_t) &=  \bigl\langle \deltaC D_\Omega(\pi_t \Vert \pct), \pi - \pi_t \bigr\rangle + \tfrac{1}{\eta_t} D_\Omega(\pi\Vert\pi_t) \\
  &= \bigl\langle \deltaC \Omega(\pi_t) - \deltaC \Omega(\pct), \pi - \pi_t \bigr\rangle + \tfrac{1}{\eta_t} D_\Omega(\pi\Vert\pi_t)\\
  &= D_\Omega(\pi\Vert \pct) - D_\Omega(\pi\Vert \pi_t)  + \tfrac{1}{\eta_t} D_\Omega(\pi\Vert\pi_t)\\
  &= \biggl(\frac{1}{\eta_t}\biggr) D_\Omega(\pi \Vert \pct) + \biggl(\frac{1-\eta_t}{\eta_t}\biggr) D_\Omega(\pi \Vert \pi_t)
\end{aligned}
\]
We refer to \cref{sect:discuss} for the stability of \Wasserstein{} gradient flows according to the \LaSalle{}'s invariance principle. We can now interpret $\deltaC\cE_t$ as a dynamics that reaches an equilibrium solution
\[
  \minimize_{\pi\in\cC}\, \bigl\langle \deltaC F_t(\pi_t), \pi - \pi_t \bigr\rangle + \tfrac{1}{\eta_t} D_\Omega(\pi\Vert\pi_t)\enspace \Leftrightarrow\enspace \minimize_{\pi\in\cC}\,\eta_t\! \underbrace{D_\Omega(\pi \Vert \pct)}_\textrm{empirical estimates}\!+ (1 - \eta_t) \underbrace{D_\Omega(\pi \Vert \pi_t)}_\textrm{proximity}, 
\]
At a glance, the above equation appears analogous to the interpolation search between two points, where the influence of $\pct$ is controlled by $\eta_t$. \hfill\qedsymbol

\subsection{Proof of \texorpdfstring{\cref{prop:wfr}}{Proposition~\ref{prop:wfr}}} \label{subsect:proof_wfr}

The proof is closely related to the work of \citet{viwgf} where the difference lies in we correct the \Wasserstein{} gradient term $\dot{\alpha}_{k,\tau}$ for suitable for generally unbalanced weight. Suppose take parameterization $\theta\in (\cP_2(\BW(\bR^d)), \WFR)$, the space of \Gaussian{} mixtures equipped with the \WassersteinFisherRao{} metric, over the measure space of \Gaussian{} particles. Following the arguments from \cref{subsect:backgeo} and the studies for this particular GMM problem \citep{lu2019accelerating, viwgf} of the \WassersteinFisherRao{} of the KL functional is derived as 
\begin{equation} \label{eq:wfr_derive}
  \nabla_{\mspace{-2mu}\WFR}\mspace{1mu}\KL(\rho_\theta\Vert \rho^\ast) = \Bigl( \nabla_{\mspace{-2mu}\BW}\mspace{1mu} \delta \KL(\rho \Vert \rho^\ast), \frac{1}{2}\biggl(\delta \KL(\rho_\theta\Vert\rho^\ast)  - \int \delta \KL(\rho \Vert \rho^\ast) \rd \rho\biggr) \Bigr),
\end{equation}
where we can consider the WFR gradient is taken with respect to $\theta$ of its first argument. By \cref{eq:wfr_derive}, we separately consider \Wasserstein{} gradient in the \BuresWasserstein{} space and the space of lighting that controls the amount of each \Gaussian{} particle.

Given a functional $F: \cP_2(\cX) \to \bR \cup \{+\infty\}$, the \Wasserstein{}  gradient $\nablaW F \cap T_\rho \cP_2(\cX)$ such that all $\{\rho_t\}_{t\in\bR^+}$ satisfy the continuity eqatuion starting from $\rho_0$ \citep{jko, villani2021topics}. If the functional is the KL divergence $\KL(\rho \Vert \pi)$ we can compute the \BuresWasserstein{} gradient for the \Gaussian{} distribution with respect to $(m,\Sigma)$ using \cref{eq:grad_bw}
\[
\begin{aligned}
   \nabla_\BW F(m, \Sigma) &= (\nabla_m F(m, \Sigma), 2\nabla_\Sigma F(m, \Sigma))\\ 
   &= \biggl( \int \nabla_m \rho_{m, \Sigma}\log\frac{\rho_{m,\Sigma}}{\pi}, 2 \int\nabla_{\Sigma} \rho_{m, \Sigma} \log \frac{\rho_{m,\Sigma}}{\pi} \biggr),
\end{aligned}
\]
with some abuse of notation for $\rho$. Using the following closed-form identities for the \Gaussian{} distributions 
\[
  \forall x.\quad \nabla_m \rho_{m, \Sigma}(x) = -\nabla_x \rho_{m, \Sigma}(x) \quad\textrm{and}\quad\nabla_{\Sigma}\rho_{m,\Sigma}(x) = \frac{1}{2} \nabla^2_x \rho_{m, \Sigma}(x).
\]
and the equivalence between the \Hessian{} and \Fisher{} information, we achieve the following form:
\[
  \nabla_\BW F(m, \Sigma) = \biggl(\bE_\rho\Bigl[\nabla \frac{\rho}{\pi}\Bigr], \bE_\rho\Bigl[\nabla^2\log\frac{\rho}{\pi}\Bigr]\biggr).
\]
Define $r_{k,\tau} = \sqrt{\alpha_{k, \tau}}$. Since $r_t$ follows the \FisherRao{} metric in \cref{def:fr}, by the Proposition A.1 from \citet{lu2019accelerating} and specialization of \citet{viwgf}, we can think of dynamics of $K$ \Gaussian{} particles $\{\alpha_{k,\tau}, m_{k,\tau}, \Sigma_{k,\tau}\}_{k=1}^K$ such that 
\begin{gather*}
  \dot{r}_{k,\tau}\mspace{-1.5mu}= -\frac{1}{2}\biggl(\bE\mspace{-1mu}\biggl[\log\frac{\rho_{\theta_\tau}}{\mspace{1.5mu}\rho^\ast}(y_{k,\tau})\mspace{-1mu}\biggr] - \frac{1}{z_\tau}\sum_{\ell=1}^K \mspace{-1mu}\alpha_\ell \mspace{1mu} \bE\mspace{-1mu}\biggl[ \log\frac{\rho_{\theta_\tau}}{\mspace{1.5mu}\rho^\ast}(y_{\ell,\tau})\mspace{-1mu}\biggr] \mspace{-1mu} \biggr) r_{k,\tau},\\ 
  \dot{m}_{k,\tau}\mspace{-1.5mu}= - \bE\mspace{-1mu}\biggl[\nabla\log\frac{\rho_{\theta_\tau}}{\mspace{1.5mu}\rho^\ast}(y_{k,\tau})\biggr]\mspace{-1.5mu}, 
  \ \dot{\Sigma}_{k,\tau}\mspace{-1.5mu}= - \bE\mspace{-1mu}\biggl[\nabla^2\mspace{-1mu} \log\frac{\rho_{\theta_\tau}}{\mspace{1.5mu}\rho^\ast}(y_{k,\tau})\mspace{-1mu}\biggr]\mspace{-1mu}\Sigma_{k,\tau}\mspace{-1mu} - \Sigma_{k,\tau}\bE\mspace{-1mu}\biggl[\nabla^2\mspace{-1mu} \log\frac{\rho_{\theta_\tau}}{\mspace{1.5mu}\rho^\ast}(y_{k,\tau})\mspace{-1mu}\biggr]\mspace{-1.5mu},
  \end{gather*}
Since $\alpha_{k,\tau} = \sqrt{r_{k,\tau}}$ by previous definition, it is straightforward that  
\[
  \dot{\alpha}_{k,\tau}\mspace{-1.5mu}= -\biggl(\bE\mspace{-1mu}\biggl[\log\frac{\rho_{\theta_\tau}}{\mspace{1.5mu}\rho^\ast}(y_{k,\tau})\mspace{-1mu}\biggr] - \frac{1}{z_\tau}\sum_{\ell=1}^K \mspace{-1mu}\alpha_\ell \mspace{1mu} \bE\mspace{-1mu}\biggl[ \log\frac{\rho_{\theta_\tau}}{\mspace{1.5mu}\rho^\ast}(y_{\ell,\tau})\mspace{-1mu}\biggr] \mspace{-1mu} \biggr) \alpha_{k,\tau}.
\]
For $\alpha_k > 0$. This completes the proof. \hfill\qedsymbol

\section{A \Riemannian{} Perspective on \Wasserstein{} Geometries} \label{sect:discuss}

\subsection{An introduction to \Otto{} calculus and the \LaSalle{} invariance principle}\label{subsect:otto_intro}

In this appendix, we introduce a basic notion of \Wasserstein{} gradient flows in the space of continuous probability measures. We focus on describing the particular example, the KL cost, initially studied by JKO \citep{jko} and formally generalized by \citet{otto} in the context of \Riemannian{} geometry. For more details and mathematical rigor, we refer the reader to \citep{gf, ipgf}. For $\cX \subset \bR^d$, and functions $U: \bR_{\ge0}\to\bR$; $V,W:\cX\to\bR$. We first consider an energy function $\cE\!: \cP_2(\cX) \to \bR$:
\begin{equation} \label{eq:energy}
  \cE(\rho) = \mspace{-3mu}\underbrace{\int_\cX\mspace{-1mu}U\bigl(\rho(x)\bigr)\,\rd x}_\textrm{internal potential $\cU$} + \mspace{-1mu}\underbrace{\int_\cX\mspace{-1mu}V(x)\; \rd\rho (x)}_\textrm{external potential $\cE_V$} + \underbrace{\frac{1}{2}\mspace{-1mu}\int_\cX\mspace{-1mu}(W \mspace{-1.6mu}\ast \rho) (x)\;\rd \rho (x)}_\textrm{interaction energy $\cW$}, \enspace \rho\in\cP_2(\cX).
\end{equation}
For this function, we refer to the solution of the following PDE:
\begin{equation} \label{eq:pde}
  \partial_t \rho_t = \nabla\mspace{-1mu}\vcdot\mspace{-1mu}\bigl[\mspace{1mu}\rho\, \nabla (U^\prime + V + W\mspace{-1.6mu}\ast\rho)\mspace{1mu}\bigr],\qquad t \ge 0
\end{equation}
as the \Wasserstein{} gradient flow of $\cE$. Following \Otto{}'s formalization of \Riemannian{} calculus on the continuous probability space equipped with the \Wasserstein{} metric $(\cP_2(\cX), W_2)$, the PDE \eqref{eq:pde} can be interpreted close to an ODE of \Riemannian{} gradient flow:
\begin{equation} \label{eq:pde_gf}
  \partial_t \rho_t  = - \nablaW\mspace{1mu}\cE\mspace{-.3mu}(\rho), 
\end{equation}
where $\nablaW$ denotes the \Wasserstein{}-2 gradient operator $\nablaW \coloneqq \nabla\!\vcdot\!\bigl(\rho\,\nabla \frac{\delta}{\delta \rho}\bigr)$. Considering the \Otto{}'s \Wasserstein{}-2 \Riemannian{} metric $\mathfrak{g}$ \citep{otto, lott2008some}, under the absolute continuity, we see that
\begin{equation}
  \frac{\partial}{\partial t}\mspace{1mu}\cE(\rho_t) = -\mathfrak{g}_{\mspace{-1mu}\rho}\biggl(\frac{\partial \rho}{\partial t}, \frac{\partial \rho}{\partial t}  \biggr) = - \int_\cX \bigl\lvert \nabla (U^\prime +  V + W \ast \rho) \bigr\rvert^{\mspace{-1mu}2}\mspace{1mu}\rd \rho(x) \le 0,
\end{equation}
which is closely related to the strict \Lyapunov{} condition. As a result, dynamical systems following the PDE are guaranteed to reach an equilibrium solution, under the \LaSalle{} invariance principle for probability measures \citep{ipgf}.

For a representative example, we identify \cref{eq:energy} for the relative entropy (the KL functional) for a target density $\rho^\ast \in \cP_2(\cX)$ writes
\[
   \cE(\rho) = \KL(\rho\Vert \rho^{\mspace{-1mu}\ast}) = \mspace{-3mu}\underbrace{\int_\cX \mspace{-1mu}U\bigl(\rho(x)\bigr)\,\rd x}_{\cU} + \mspace{-1mu}\underbrace{\int_\cX V(x)\ \rd\rho(x)}_{\cE_V} -\mspace{3mu} C,
\]
where $U(s) = s\log s$, $V(x) = -\log \mspace{-0.5mu}\rho^{\mspace{-1mu}\ast}\mspace{-1mu}(x)$, and $C = \cU(\rho^{\mspace{-1mu}\ast}) + \cE_V\mspace{-.5mu}(\rho^{\mspace{-1mu}\ast})$. Recall that $\delta \cE(\rho) = \log \frac{\rho(x)}{\rho^\ast}$, then we have
\begin{equation} \label{eq:kl_wgf}
  \nablaW \cE(\rho) = \mathfrak{G}^{-1}_\rho \delta E(\rho) = -\nabla\mspace{-1mu}\vcdot\mspace{-1mu}[\rho \nabla \delta E(\rho)] = \nabla\mspace{-1mu}\vcdot\mspace{-2mu}\biggl[\rho\nabla \log \frac{\rho}{\rho_\ast} \biggr]
\end{equation}
where $\mathfrak{G}$ denotes the metric tensor in matrix form. We can derive the the \FokkerPlanck{} equation
\[
  \partial_t \rho_t =  -\nabla\mspace{-2mu}\vcdot\mspace{-1mu}(\rho \nabla \log \rho^\ast) + \Delta\rho_t,
\]
describing the time evolution of the probability density. Combining the convexity of KL and the \LaSalle{} invariance principle \Wasserstein{} gradient flows, the PDE reaches a unique stationary solution of $\frac{e^{-V(x)}}{\int_\cX e^{-V(y)} \rd y}$.

\subsection{Background on \WassersteinFisherRao{} and other related geometries} \label{subsect:backgeo}

The \WassersteinFisherRao{} geometry is also known as \textit{Hellinger--Kantorovich} in some of papers \citep{liero2016, liero2018}. In this section, we provide an overview of the geometry tailored to meet our technical needs. Along the way, we also briefly describe various metrics and geometries related to the \Wasserstein{} space.

\textbf{The \Wasserstein{} space.}\hspace*{6pt} Let $\mu, \nu \in \cP_2(\bR^d)$ be marginal probability densities with respect to the \Lebesgue{} measure. We define the squared \Wasserstein{} distance by a problem of couplings \citep{villani2009oldnew}
\begin{equation}
  \tfrac{1}{2}W^2_2(\mu, \nu) = \inf_{\pi \in \Pi(\mu, \nu)} \int_{\bR^d\times\bR^d} \frac{1}{2}\lVert x- y\rVert^2 d \pi(x, y).
\end{equation}
The Wasserstein distance offers a principled metric for quantifying the discrepancy between the probability distributions of random variables $X$ and $Y$. Moreover, the \Wasserstein{} space admits a fluid-dynamical formulation, where optimal transport is represented by space-time velocity fields that satisfy the continuity equation. The \Brenier{} theorem \citep{villani2021topics} states that there exists an optimal mapping function that pushes forward $\mu$ to $\nu$, \ie{} $\nu = \nabla\zeta_\# \mu$,  where $\zeta: \bR^d \to \bR^d \cup \{+\infty\}$ is a convex and lower semicontinuous function. The property is formally referred to as an instance of the Monge--Amp\`ere equation \citep{villani2009oldnew}
\begin{equation}
  g\bigl(\mkern-0.5mu\nabla\mkern-0.6mu\zeta(x)\mkern-0.5mu\bigr)\det\bigl(\mkern-0.5mu\nabla^2\zeta(x)\mkern-0.5mu\bigr) = f(x)\quad x\in\bR^d,
\end{equation}
after identifying the source and target densities with $\mu(x) = f(x)dx$ and $\nu(y) = g(y)dy$, respectively. In terms of fluid dynamics, the \Brenier{} map $T = \nabla\zeta$ internally yields a constant-speed geodesic $\{\rho_t\}_{t\in[0,1]}$, time-dependent density evolving from $\rho_0$ to $\rho_1$, described by the following differential equation
\begin{equation} \label{eq:dynabernier}
  \rho_t = (\nabla\zeta_t)_\# \rho_0,\qquad \nabla \zeta_t \coloneqq (1 - t) \mathrm{id} + t \nabla \zeta, 
\end{equation}
where $\rho_0 = \mu$ and $\rho_1 = \nu$. Assuming the existence of such geodesic, we can also understand finding the optimality of $\{\rho_t\}_{t\in[0,1]}$ with the \BenamouBrenier{} formulation \citep{brenier}, which involves a velocity field $v_t$ for minimizing the total $L^2$ cost of transportation
\begin{equation} \label{eq:bb}
  W_2^2(\mu, \nu) = \min_{\rho, v} \biggl\{\int_{0}^1\mspace{-10mu}\int_{\bR^d} \frac{1}{2}\lVert v_t(x)  \rVert^2  \rd \rho_t(x)  \, d t \mspace{5mu}\Big\vert\mspace{5mu} \rho_0 = \mu,\ \rho_1 = \nu,\ \partial_t\rho_t = - \nabla\vcdot(v_t\rho_t) \biggr\}.
\end{equation}
The equation dictates \textit{how} the fluid should be transported (which shall be controlled by speed $v_t$) while satisfying the continuity equation of path measure on the right hand side. In the \Otto{} calculus \citep{otto}, we can understand the \BenamouBrenier{} formula \eqref{eq:bb} as a \Riemannian{} geometry with respect to the $W_2$ metric. In this geometric interpretation, the tangent space at $\rho\in\cP_2(\cX)$ are measures of the form $\delta \rho = - \nabla\vcdot (v\rho) $ with a velocity field $v\in L^2(\rho, \bR^d)$ and the metric is given by
\begin{equation}
 \lVert \rho \rVert^2_\rho  = \inf_{v\in L^2(\rho, \bR^d)}\biggl\{\int \lVert v \rVert^2 \,d \rho \mspace{5mu}\Big\vert\mspace{5mu} \delta \rho = -\nabla \vcdot (v\rho) \biggr\}.
\end{equation}
The \BenamouBrenier{} formula exhibits dynamics in the \Wasserstein{} space of probability densities where the metric generally governed by dynamical mass transportation costs with the continuity equation, implying the mass of probability is preserved. 

\textbf{\FisherRao{} metric.}\hspace*{6pt} The \FisherRao{} metric is a metric on the space of positive measures $\cP_+$ with possibly different total masses. We are interested in the simple case where such measure are represented with a fininte number of parameters such as exponential families. We use the following definition throughout the paper.
\begin{definition}[\FisherRao{} metric] \label{def:fr}
  The \FisherRao{} distance between measures $\rho_0,\rho_1 \in \cM_+$ is given by
  \[
    d^2_\FR(\rho_0, \rho_1) \coloneqq\! \inf_{\rho, v \in \cA[\rho_0, \rho_1]} \int_{0}^1\mspace{-10mu}\int_{\bR^d} \frac{1}{2}\omega^2_t(x)    \,d \rho_t(x)   d t   =  2 \int_{\bR^d} \biggl\lvert \sqrt{\frac{\rd \rho_0}{\rd \lambda}}  - \sqrt{\frac{\rd \rho_1}{\rd \lambda }}\biggr\rvert^2 d \lambda 
  \]
  where $\cA$ is an admissible set for a scalar field on positive measures; $\lambda$ is any reference measure such that $\rho$ and $\rho^\prime$ are both absolutely continuous with respect to $\lambda$, with \RadonNikodym{} derivatives $\frac{\rd \rho_i}{\rd \lambda}$.  
\end{definition}
The equivalence between the square \FisherRao{} distance and squared \Hellinger{} distance \citep{liero2016, liero2018} quantifies the similarity between two probability distributions ranging from 0 to 1. The total variation bounds the squared form and is well-studied in the information geometry \citep{infogeo}. The partial differential equations of the form $\partial_t\rho_t = \alpha_t \rho_t$ are called reaction equations of $\alpha_t$, which describes dynamics regarding concentration.

\textbf{\WassersteinFisherRao{}.}\hspace*{6pt} The \WassersteinFisherRao{} geometry, or equivalently, spherical \HellingerKantorovich{} distance, considers liftings of positive, complete, and separable measures while preserving the total mass. This can be expresses as  combining the \FisherRao{} and \Wasserstein{} geometries characterized by PDE such as  \citep{liero2016}:
\begin{equation} \label{eq:wfrpde1}
  \partial_t \rho_t + \nabla\vcdot(v_t\rho_t) = \frac{\omega_t}{2} \rho_t.
\end{equation}
One problem, is that the PDE \eqref{eq:wfrpde1} In order to stay the dynamics on the space of probability measures, which is our interest, we adopt the definition from \citep{lu2019accelerating, viwgf} the equation becomes
\begin{equation} \label{eq:wfrpde2}
  \partial_t \rho_t + \nabla\vcdot(\rho_t  v_t) = \frac{1}{2} \biggl(\beta_t - \!\int_{\bR^d}\beta_t \,d \rho_t\biggr)\rho_t,
\end{equation}
which satisfies mass conservation. For the geometry, the norm on tangent space is given by 
\begin{equation}
  \lVert (\beta_t, \rho)\rVert^2_\rho \coloneqq \mspace{-1.5mu} \int\biggl\{\Bigl(\omega - \!\int_{\bR^d}\beta_t\,\rd \rho \Bigr)^{\mspace{-4mu}2} + \lVert v\rVert^2 \biggr\} \rd \rho.
\end{equation}
and we define the WFR distance as 
\begin{equation}
  d^2_\WFR(\rho_0, \rho_1) \coloneqq\! \inf_{\rho,\beta_t,v} \biggl\{ \int_0^1 \lVert (\beta_t, v_t) \rVert^2_{\rho_t} \rd t \mspace{5mu}\Big\vert\,\{\rho_t, \beta_t, v_t\}_{t\in[0,1]}\enspace\text{satisfies \eqref{eq:wfrpde2}}\biggr\}.
\end{equation}
\citet{lu2019accelerating} demonstrated that \WassersteinFisherRao{} gradient dynamics over the \BuresWasserstein{} space can be analytically derived with closed form expressions. In this work, we were able to design a computational method for OMD iterates in the WFR geometry. Using \cref{prop:wfr}, this geometry allowed the VMSB algorithm to perform tractable gradient computation within \Wasserstein{} space. 

\subsection{The \BuresWasserstein{} space and a mixture of \Gaussian{}s}
The space of \Gaussian{} distribution in the \Wasserstein{} space is known as \BuresWasserstein{} space, denoted as $\BW(\bR^d)$. Given $\theta_0, \theta_1 \in \BW(\bR^d)$, we can identify the space with the manifold $\bR^d \times \vSppd$, where $\vSppd$ denotes the space of symmetric positive definite matrices. For $\theta_0 = (m_0, \Sigma_0)$ and $\theta_1 = (m_1, \Sigma_1)$ an affine map from $p_{\theta_0}$ to $p_{\theta_1}$ is given as a closed-form expression:
\[
  \nabla \zeta (x) = m_1 + \Sigma_0^{-1/2}\bigl(\Sigma_0^{1/2} \Sigma_1 \Sigma_0^{1/2}\bigr)^{1/2}\Sigma^{-1/2} (x - m_0).
\]
Note that the constant-speed geodesic also lies in $\BW(\bR^d)$, as pushforward of a \Gaussian{} with an affine map is also a \Gaussian{}. Therefore, it can be said that $\BW(\bR^d)$ is a geodesically convex subset of $\cP_2(\bR^d)$. For the \Brenier{} map, a constant-speed geodesic in $\BW(\bR^d)$, for the tangent vector to the geodesic $(r, S)$
\begin{equation}
  p_{\theta_t} = \exp_{p_{\theta_0}}\mspace{-5mu}\bigl(t\cdot(r, S)\bigr) = \cN\bigl(m_0 + tr, (tS+I_d)\Sigma_0 (tS + I_d) \bigr), 
\end{equation}
and the dynamics at its current position at time $t=0$ is represented as 
\begin{align}
  \dot{m}_0 &= r,\\
  \dot{\Sigma}_0 &= S\Sigma_0 + \Sigma_0 S. 
\end{align}
Generalizing this geodesic dynamics, the \BuresWasserstein{} gradient $\nabla_\BW\,f$ of a function $f:\bR^d\times\vSppd \to \bR$ for a tangent vector $(r, S)$ at time $0$ \citet{altschuler2021averaging}
\[
  \bigl\langle\nabla_\BW f(m_0, \Sigma_0),  (r, S)\bigr\rangle_\BW = \partial_t f(m_t, \Sigma_t)\bigg\vert_{t=0}
\]
Identifying each component, we achieve the following result of \Wasserstein{} gradient flow in \BuresWasserstein{} space as
\begin{equation} \label{eq:grad_bw}
  \nabla_\BW f = (\nabla_m f, 2\nabla_\Sigma f), 
\end{equation}
where $\nabla_m$ and $\nabla_\Sigma$ denote \Euclidean{} gradient. We refer readers to Appendix~A of \citet{altschuler2021averaging} and Appendix~B of \citet{viwgf} for further useful geometric properties of \Wasserstein{} spaces and dedicated discussions for the \BuresWasserstein{} space.

\section{Background on the Dynamic \Schrodinger{} Bridge Problem}\label{sect:back_dsbp}

The equivalence between static and dynamic SBPs \citep{onfree, sbp2mkp} has been studied, allowing us to consider the both problems interchangeably. This appendix introduces a general control dynamic formulation  for describing SB and SB-FBSDE theory \citep{bbrcd, likesb}. These formulations establish fundamental links to optimal control theory and diffusion models. 

For variable drift and diffusion coefficient, a stochastic process, and its time reversal respectively follow the forward and backward \Kolmogorov{} (or \FokkerPlanck{}) equations \citep{fa2011solution, bbrcd}:
\begin{subequations}\label{eq:fp3}
\begin{align}
  -\frac{\partial \rho}{\partial t} + \nabla\vcdot\bigl[ (f_t + \nabla \varphi_t)\rho_t\bigr] = \frac{1}{2}\,\divsq\bigl(g_t\gT_t \rho_t\bigr),\label{eq:fp3a}\\
  \frac{\partial \rho_t}{\partial t} + \nabla\vcdot\bigl[(-f_t + \nabla \psi_t)\rho_t)\bigr]  = \frac{1}{2}\,\divsq\bigl(g_t\gT_t\rho_t\bigr),\label{eq:fp3b}
\end{align}
\end{subequations}
where $f_t$ and $g_t$ are time-varying base drift and diffusion coefficients which together determine a reference measure $R$;\footnote{In \eqref{eq:pde4} of \S~\ref{sect:prelim}, base drift is zero and diffusion is $\sqrt{\varepsilon}$.} $\div$ denotes divergence; $\divsq$ denotes squared divergence. By subtracting \cref{eq:fp3b} from \cref{eq:fp3a}, the continuity equation
\begin{equation}\label{eq:derivecontinuity}
  \frac{\partial\rho}{\partial t} + \nabla\vcdot(v\rho) = 0
\end{equation}
is achieved by $v = f + \frac{1}{2}[\nabla\varphi - \nabla\psi] $.    Adding and scaling of \cref{eq:fp3}, also derives another identity
\begin{equation}\label{eq:fwdbwddiff}
  (\nabla\varphi+\nabla\psi)\rho = \div\bigl(g\gT\rho\bigr),
\end{equation}
where we can derive an explicit form of the score function
\begin{equation}\label{eq:fluxx}
\nabla\varphi+\nabla\psi = g\gT \nabla \log \rho + \underbrace{\div g\gT}_{\textrm{heat flux}}.\\
\end{equation}
The rightmost term of \eqref{eq:fluxx} is a heat flux, indicating the amount of diffusion per unit time.

The time-symmetry relation can be considered as a multivariate derivation of the  the stochastic mechanics. Generalizing the Nelson's notation \citep{nelson}, let us define the SB drifts and \textit{current} drifts for continuous-time path measures $P$ and $Q$ receptively satisfying \cref{eq:fp3} with the shared diffusion coefficient:
\begin{gather}\label{eq:nelsongeneral}  
  \fsupplus_{P,t} \coloneqq f_{P,t} + \nabla \varphi_{P,t},\quad  \fsupminus_{P,t} \coloneqq -f_{P,t} + \nabla \psi_{P,t}, \quad   \fsupplus_{Q, t} \coloneqq f_{Q,t} + \nabla \varphi_{Q,t},\quad  \fsupminus_{Q,t} \coloneqq -f_{Q,t} + \nabla \psi_{Q,t}  \\
  v_{P} = \frac{\fsupplus_{P} - \fsupminus_{P}}{2},\qquad
  v_{Q} = \frac{\fsupplus_{Q} - \fsupminus_{Q}}{2}.
\end{gather}
Then, we consider the \Girsanov{} theorem \citep{sde} for the drifts. The theorem formulate \RadonNikodym{} derivatives between path measures of stochastic processes by the fact that potential gradients reconstruct the score function $\nabla\varphi_t + \nabla\psi_t = \varepsilon\nabla\log\rho_t$. 
\begin{lemma}[\Girsanov{} theorem; Theorem 8.6.3 of \citealp{sde}]\label{lem:girsanov}
For adapted processes with a given time interval $[0, T]$, let $\widehat{B}_s$ be an It\^o process solving SDE
\begin{equation}
 \rd \widehat{B}_s = -\alpha(\omega, s) + \rd B^\prime_s
\end{equation}
for $\omega\in\cX$, $0\le s\le T$ and $\widehat{B}_{\scriptscriptstyle0} = 0$, where $\alpha$ satisfies the Novikov’s condition. Then, $\widehat{B}_s$ is a \Brownian{} motion with respect to the path measure $Q$, satisfying the \RadonNikodym{} derivative
\begin{equation}
 \frac{\rd P}{\rd Q}(\omega) \coloneqq \exp\biggl(\int_0^T\alpha(\omega, s)\rd B^\prime_s -\int_0^T \frac{1}{2} \lVert \alpha(\omega, s) \rVert^2 \rd s \biggr)
\end{equation}

\end{lemma}

Next, we present the disintegration theorem in the context of probability measures \citep{someppm, sbml}, which extends the product rule to measures that do not admit the traditional product rule. Similar to the product rule, these theorems are essential for decomposing and manipulating path measures for dynamic SBPs, eventually connecting various formulation of \Schrodinger{} bridge problems.
\begin{lemma}[Disintegration for continuous probability measures]\label{lem:disint}
  For a probability space $(\cZ, \cF, \cP)$ where $\cZ$ is a product space: $\cZ=\cX\times\cY$ and \one{} $\cX,\cY\in\bR^d$ and $\phi_i: \cZ\to\cZ_i$ is a measurable function known as the canonical projection operator (\ie{}, $\phi_1(x, \cdot) = x$ and $\phi^{-1}_1(x) = \{(x,y) \vert \phi_1(x, y) = x\}$), There exists a measure $P_{y\vert x }(\cdot \vert x)$, such that
  \begin{equation}
    \iint_{\cX\times\cY} f(x,y) \rd P(y) = \iint_{\cX\times\cY} f(x,y) \rd P_{y\vert x }(y \vert x) \rd P(\phi^{-1}_1(x))
  \end{equation}
  where $P_x(\cdot) = P(\phi^{-1}_1(\cdot))$ is a probability measure, referred to as a push-forward measure, and corresponds to the marginal distribution.
\end{lemma}
The theorem suggests that one way to achieve KL projection between path measures is by matching drifts with the time reversal drifts of $(\gammasupplus, \gammasupminus)$.
Under the \Girsanov{} theorem, we observe that
\begin{equation}\label{eq:gir1}
  \begin{aligned}
  \KL(P \Vert Q) &= \KL(P_{\scriptscriptstyle 0} \Vert Q_{\scriptscriptstyle 0}) + \bE_{P}\biggl[\int_0^T \frac{1}{2}\lVert\fsupplus_{P,t} -  \fsupplus_{Q,t}  \rVert^2  \dt  \biggr] \\&= \KL(P_{\scriptscriptstyle T} \Vert Q_{\scriptscriptstyle T}) + \bE_{P}\biggl[\int_0^T \frac{1}{2} \lVert \fsupminus_{P,t} -  \fsupminus_{Q,t}  \rVert^2  \dt \biggr]
  \end{aligned}
\end{equation}

Let us consider a reference measure $Q = \cR$ with based drift and diffusion \ie{} $\fsupplus_Q = f$. Knowing the boundary conditions $P_0 = \cR_0 = \mu$ and $\cR_T = Q_T = \nu$, \cref{eq:gir1} can be reduce to the following problems
\begin{subequations}\label{eq:bb2}
\begin{align}
  \KL(P\Vert R) &= \inf_{( \vsupplus, \rhosupplus )}\mspace{-2mu}\int_{\scriptscriptstyle0}^{\scriptscriptstyle T}\mspace{-13mu}\int_{\bR^d}\frac{1}{2}\lvert\vsupplus_t(x)\rvert^2\rhosupplus_t(x)\ \rd x\rd t,\\
  &=\inf_{(\vsupminus, \rhosupminus)}\int_{\scriptscriptstyle0}^{\scriptscriptstyle T}\mspace{-13mu}\int_{\bR^d} \frac{1}{2}\lvert\vsupminus_t(x)\rvert^2 \rhosupminus_t(x)\ \rd x\rd t,
\end{align}
\end{subequations}
and Eqs.~\eqref{eq:sbp}~and~\eqref{eq:pde4} in the main text represent a simplified case when $R = W^\varepsilon$. Furthermore, by combining \cref{eq:gir1}, we get a symmetrized version of relative entropy (KL functional)
\begin{equation}\label{eq:gir2}
  \KL(P\Vert Q) = \frac{1}{2}\KL(P_0 \Vert Q_0)  + \frac{1}{2} \KL(P_T \Vert Q_T) + \bE_P \biggl[\int_{0}^T \frac{1}{4}\lVert \fsupplus_P - \fsupplus_Q \rVert^2 + \frac{1}{4}\lVert \fsupminus_P - \fsupminus_Q \rVert^2 \dt   \biggr]
\end{equation}
Finally, using \cref{eq:fwdbwddiff} we get the constrained problem:
\begin{equation*}
\begin{aligned}
  &\inf_{(\tilde{\rho}, \tilde{v})} \int_{\scriptscriptstyle0}^{\scriptscriptstyle T}\mspace{-13mu}\int_{\bR^d}\!\biggl(\,\frac{1}{2} \lVert \tilde{v}(t, x)  \rVert^2 + \frac{1}{8}\lVert  \nabla\log\!\tilde{\rho}(t, x)\rVert^2_{\ggTscript} + \frac{1}{8} \lVert \nabla\vcdot g\gT \rVert^2 \biggr) \tilde{\rho}(x, t)\ \dt\,\dx,  \\
  &\textrm{such that }\quad  \frac{\partial\tilde{\rho}}{\partial t} + \div [ (f + \tilde{v}) \tilde{\rho}]=0,\enspace \tilde{\rho}(0, \cdot)\equiv \mu,\enspace \tilde{\rho}(T,\cdot) \equiv \nu.
\end{aligned}
\end{equation*}
To solve this problem, we convert the problem to the Lagrangian function:
\begin{equation*}
\begin{aligned}
  \cL(\rho, v) = \int_{\scriptscriptstyle0}^{\scriptscriptstyle T}\mspace{-13mu}\int_{\bR^d}\! &\frac{1}{2} \lVert \tilde{v}(t,x) \rVert^2 \tilde{\rho}(t, x) + \frac{1}{8}\lVert\nabla\log\tilde{\rho}(t, x)\rVert^2_{\ggTscript}\tilde{\rho}(t,x) + \frac{1}{8}\lVert \nabla\vcdot g\gT(t,x) \rVert^2\tilde{\rho}(t,x)\\
  &+  \lambda(t, x)  \biggl(\frac{\partial \tilde{\rho}}{\partial t} + \div((f+\tilde{v})\tilde{\rho})\biggr) \dx\,\dt
\end{aligned}
\end{equation*}
where $\lambda$ is $C^{1,2}$-Lagrangian multiplier. After integration by part, assuming that limits for $x\to\infty$ are zero, and observing that the boundary values are constant over $\Pi(\mu, \nu)$, we resort to the following problem:
\begin{equation}\label{eq:sbprob}
  \inf_{(\tilde{\rho}, \tilde{v}) \in \cP \times \cV} \int_{\bR^d}\int_{0}^T \biggl[\frac{1}{2} \lVert  \tilde{v}(t, x)  \rVert^2 + \frac{1}{8}\lVert \nabla\log\tilde{\rho}(t,x)\rVert^2_{\ggTscript} + \frac{1}{8}\lVert \nabla\vcdot g\gT(t,x) \rVert^2  + \biggl( - \frac{\partial \lambda}{\partial t} - \nabla\lambda \,\vcdot ( f+ \tilde{v} ) \biggr)\biggr] \tilde{\rho}(t, x)\dt \dx .
\end{equation}
Pointwise minimization with respect to $\tilde{v}$ for each fixed flow of probability densities $\tilde{\rho}$ gives $v^\ast(x, t) = \nabla \lambda (x, t)$. Plugging this form of the optimal control into \cref{eq:sbprob}, we get the functional of $\tilde{\rho}\in\cP$:
\begin{equation*}
  \cJ(\tilde{\rho}) = - \int_{\bR^d}\int_{0}^T \biggl[\frac{\partial \lambda}{\partial t} + (f+v) \vcdot \nabla\lambda + \frac{1}{2} \lVert \nabla \lambda \rVert^2 + \frac{1}{8}\lVert \nabla\log\tilde{\rho}(t,x) \rVert^2_\ggTscript  +  \frac{1}{8}\lVert \nabla\vcdot g\gT(t,x) \rVert^2 \biggr]\tilde{\rho}(t, x)  \dt \dx
\end{equation*}
Utilizing the existence and uniqueness of SDE solutions \citep{sde}, the optimality of the problem is uniquely identified by $\nabla\lambda$ being the mean current of SB-FBSDE which is introduced in \citep{likesb, deepgsb}. Therefore, both static and dynamic \Schrodinger{} bridge problems solve identical optimization problem of KL minimization, with different appearence by the problem-specific origin of objectives, entropies, and geometries.

\clearpage
\section{Experimental Details} \label{sect:details}

\subsection{Rationales of the GMM parameterization for VMSB}

Our parameterization choice follows LightSB \citep{lsb} because of the following two key reasons. First, GMMs ensure that the model space satisfies certain measure concentration, which is suitable for analyzing theoretical properties of SB models \citep{conforti2023quantitative}. Firstly, we analyzed the regret under the log \Sobolev{} inequality in \cref{thm:regret}. Enforcing the LightSB parameterization will automatically satisfy \cref{asm:lsi}. Secondly, VMSB requires tractable gradient computation of \Wasserstein{} gradient flow in \S~\ref{subsect:md_wgf}. As shown in \cref{prop:wfr}, we can perform VMSB using the variational inference in the WFR geometry of the GMM parameterization.

\begin{table}[t]
    \centering
    \caption{Hyperparameters.} \label{tab:hyper}
  \begin{adjustbox}{width=0.95\linewidth}
    \begin{tabular}{c | c c c c c c}
      \toprule
        &\textbf{2D}& \textbf{EOT} & \textbf{MSCI} & \textbf{MNIST (Pixel)} & \textbf{MNIST (Latent)} & \textbf{FFHQ} \\
      \midrule
     Dimension $d$ & 2 & $\{2, 16, 64, 128\}$ & $\{50, 100, 1000\}$ & $784$ & $128$ & $512$\\
     Modality $K$ & $\{8, 20, 50\}$ & $[5, 100]$ & 50 & \{256, 1024, 4096\} & \{256, 1024\} & 10 \\
     Volatility $\varepsilon$ & $0.1$ & $\{0.1, 1, 10\}$ & $0.1$ & $10^{-4}$ & $10^{-3}$ & $\{0.1, 0.5, 1.0, 10.0\}$ \\
     Total steps ($\tau$) & 20,000 & 30,000 & 10,000 & 100,000 & 30,000 & 20,000 \\
     OMD steps ($t$) & 400 & 600 & 200 & 1000 & 375 & 400\\ \bottomrule
    \end{tabular}
  \end{adjustbox}
\end{table}

\subsection{Hyperparameters.}

The hyperparameters are displayed in \cref{tab:hyper}. For step size scheduling, we followed the theoretical result in \cref{thm:step_size} and \cref{prop:convergence}, and chose $\eta_1 = 1$ and $\eta_T \in \{ 0.05, 0.1 \}$ with harmonic sequences, as illustrated in \cref{fig:eta_eg}. For high dimensional tasks in MSCI (1000d), MNIST-EMNIST (784d), and latent FFHQ Image-to-Image transfer tasks (512d), the initial \textit{warm up} steps for 10\% of the total learning helped starting a training sequence from a reasonable starting point as this set $\eta_t = 1$ as verified in \cref{fig:md_demo}~\subfigC{}.

\begin{figure}[h]
  \definecolor{plotred}{RGB}{214, 39, 40}
  \tikzset{inner sep=0, outer sep=0}
    \centering
    \begin{tikzpicture}[tight background]
      \begin{scope}[xshift=-3.88cm]
        \node[scale=0.75] at (0.1,1.73) {$\mathsf{x\sim\mu},\enspace\mathsf{y\sim\nu}$};
        \node at (0,0) {\includegraphics[width=92pt]{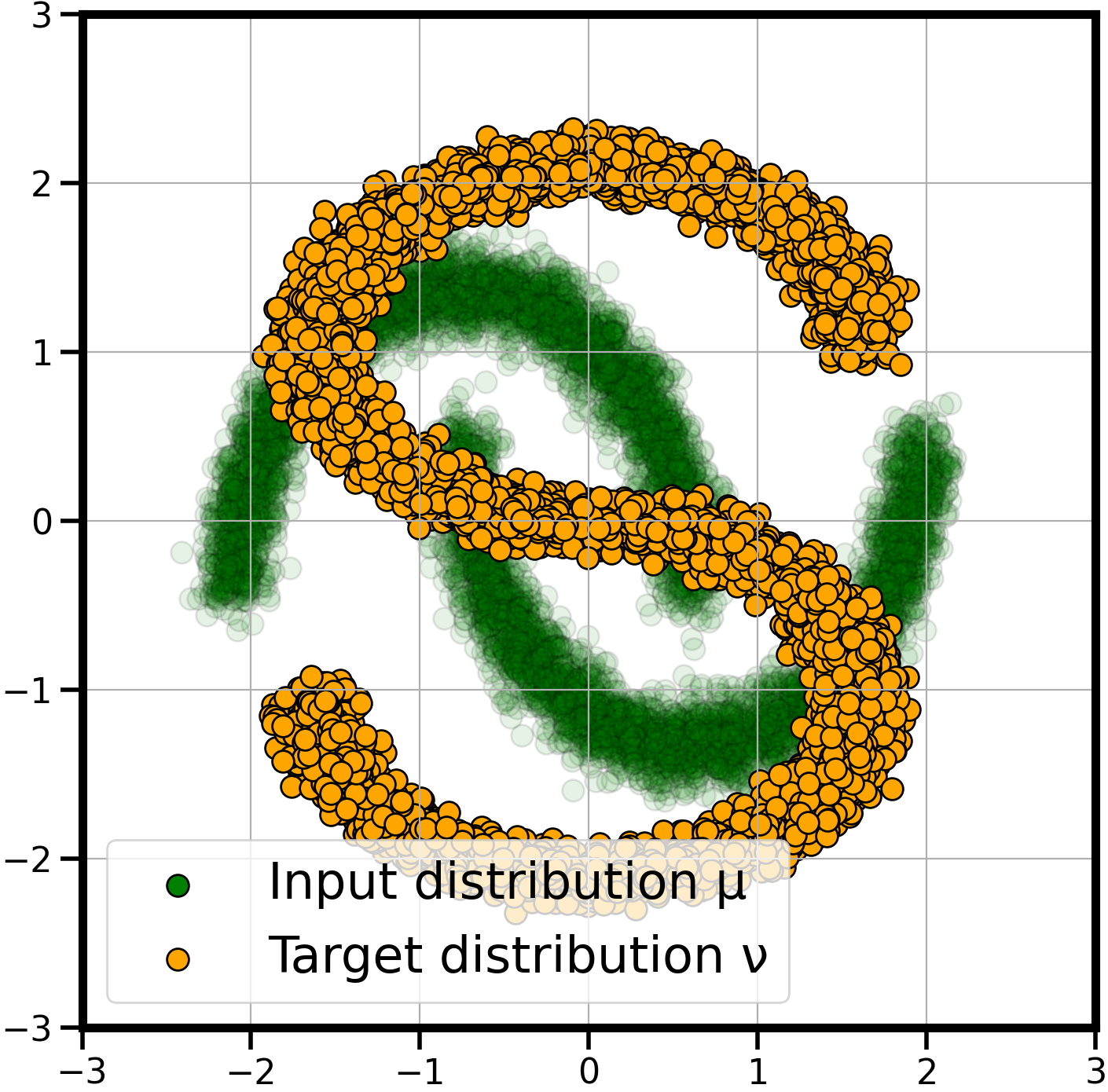}};
      \end{scope}
      \begin{scope}
        \node[scale=0.75] at (0.1,1.73) {$\mathsf\varepsilon = {\scriptstyle\textsf{0.05}}$};
        \node at (0,0) {\includegraphics[width=92pt]{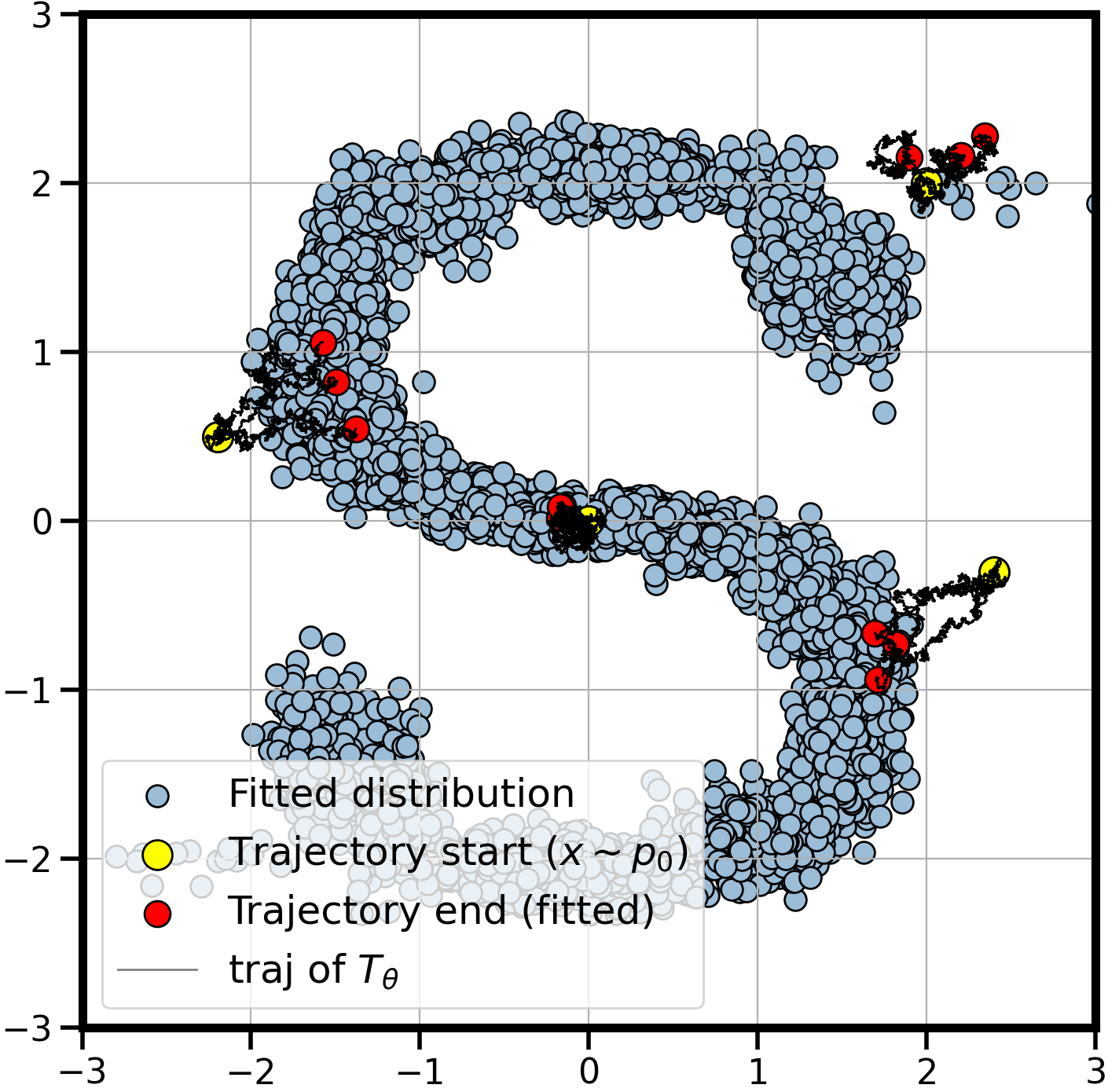}};
        \draw[gray, thick] (-1.263, -1.235) -- (-1.02, -1.235);
      \end{scope}
      \begin{scope}[xshift=3.88cm]
        \node[scale=0.75] at (0.1,1.73) {$\mathsf\varepsilon = {\scriptstyle\textsf{1.0}}$};
        \node at (0,0) {\includegraphics[width=92pt]{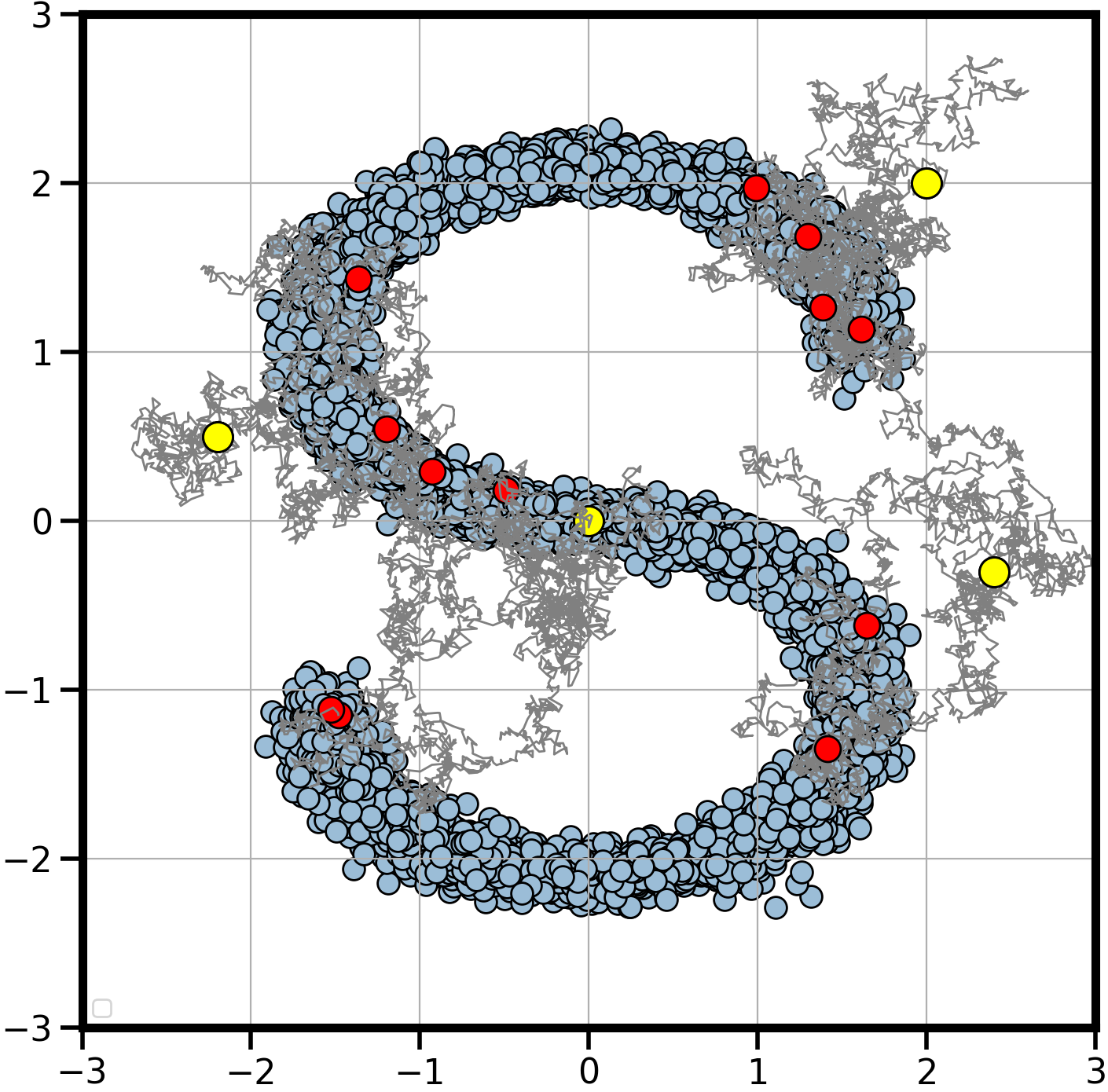}};
      \end{scope}
    \end{tikzpicture}
    \caption{SB in 2D synthetic datasets. SB processes $\cT_\theta$ with different volatility $\varepsilon$.} \label{fig:2d_eps}
\end{figure}

\subsection{2D Synthetic datasets and the online learning setup}

\cref{fig:2d_eps} demonstrates that our method achieved the SB model for the various volatility $\varepsilon$. For various configurations, most of baseline SB algorithms are capable of learning in the 2D space (\ref{fig:2d_eps}). In order to align our theoretical arguments for online learning, we selectively offered with a rotating filter that only $12.5\%$ of the samples to the SB solvers based on the angles measured from the origin. For instance, we provided data for angle of $[0, \pi/4]$ for first $t\in[0, 25)$ steps, and so on. This partial observability is periodically rotated through the data stream, thereby testing the algorithm’s ability to learn robustly under sparse and shifting information. Since this requires $200$ batches for the full rotation of the filter, the problem became substantially more challenging, and LightSB and LightSB-M algorithms oftentimes failed on this online learning setting. 

\subsection{Entropic optimal transport benchmark}

Our hyperparameter for the EOT benchmarks choices mostly follow the official repositories of the LightSB\footnote{\href{https://github.com/ngushchin/LightSB}{https://github.com/ngushchin/LightSB}} and LightSB-M\footnote{\href{https://github.com/SKholkin/LightSB-Matching}{https://github.com/SKholkin/LightSB-Matching}}. Since it is known that initial distribution $\mu$ is the standard \Gaussian{} distribution \citep{eotbench}, we only trained $v_\theta$ using the variational MD algorithm. Due to the huge number of configurations, some hyperparameter settings were not clearly reported. Thus, we conducted our own examination on these cases; we replicated better performance than the reported numbers by carefully dealing each benchmark configuration.

\subsection{SB learning with adversarial networks} \label{subsect:mnist_detail}

Suppose a discriminator network, denoted as $D$, is equipped with useful architectural properties for discriminating images. The discriminator outputs a binary classification regarding authenticity through sigmoidal outputs, \ie{}, $D(x)\in[0,1]\enspace \forall x \in \bR^{28\times 28\times 1}$. For image samples $\mathbf{x} = \lbrace x^1,\dots,x^N\rbrace\sim\mu$, we trained the discriminator $D$ with the logistic regression: 
\begin{equation}
  \maximize_D  \frac{1}{N}\sum_{n=1}^N \log D(y^n) + \frac{1}{B} \sum_{m=1}^M\log (1-D(\hat{y}^m_\phi)),
\end{equation}
where $\hat{y}^m_\phi$ in the right-hand side denotes a sample from an SB model parameterized by $\phi$, generated using an input $x^m$. From our experiment setting, we use the SB distribution $\rho_\phi$ which is generated by $\vec{\pi}_\phi$ from samples of the marginal $\mu$. This makes the objective of adversarial learning of training the law of SB process at time $t=1$. For a completely separable metric space, it is well known that the discriminator converges at $D(x) = \frac{\nu(x)}{\nu(x) + \rho_\phi(x)}$ \citep{gan}.

\begin{wraptable}{r}{0.36\textwidth}
  \vskip-30pt
  \centering 
  \caption{Model hyperparameters for $D$.} \label{tab:disc_arch}
  \vskip-8pt
  \begin{tabular}{ | c | c | }
    \hline
    \textbf{Layer Type} & \textbf{Shape} \\ 
    \hline
    Input Layer & (-1, 28, 28, 1) \\
    \hline
    Conv Layer 1 & (-1, 14, 14, 64) \\
    \hline
    Conv Layer 2 & (-1, 7, 7, 128) \\
    \hline
    Batch Norm & (-1, 7, 7, 128) \\
    \hline
    Flatten & (-1, 6272) \\
    \hline
    Dense & (-1, 1024) \\
    \hline
    Dense & (-1, 1) \\
    \hline
  \end{tabular}
  \vskip-25.8pt
\end{wraptable}
In the adversarial learning technique, retaining a fully differentiable computation path from the input pixels to the discriminator outputs is essential. Therefore, we implemented a differentiable inference function using the categorical reparameterization trick with Gumbel-softmax \citep{jang2016categorical}, as well as the Gaussian reparameterization trick. These reparameterization tricks enabled learning with samples generated through LightSB-adv-$K$, directly by maximizing
\[
 \tilde{\mathcal{J}}(\phi) = \frac{1}{M} \sum_{m=1}^M \log D(y^m_\phi) - \log (1 - D(y^m_\phi)), 
\]
where the term essentially represents the \textit{logit} function $\mathrm{logit}(D(y)) = \log \frac{D(y)}{1-D(y)}$. When $D$ approaches the equilibrium, we can approximate the following KL learning
\[
  \tilde{\mathcal{J}}(\phi) \approx \int \log\frac{\nu(y)}{\rho_\phi(y)} \rho_\phi(y) \mathrm{d} y = \mathrm{KL}(\rho_\phi \Vert \nu),
\]
where the KL functional directly corresponds to the divergence minimization of the SB problems \eqref{eq:sbp}~and~\eqref{eq:sbp2}, under the disintegration theorem of \Schrodinger{} bridge \citep{someppm}.

In the MNIST-EMNIST image transfer tasks, we set one of the baseline as the aforementioned adversarial learning as the baseline for training the SB model for the pixel space. Among our attempts, while the LightSB-adv method successfully generated learning signals to train GMM-based models, the losses proposed by LightSB \citep{lsb} and LightSB-M \citep{lsbm} failed to generate relevant images with high fidelity. For the discriminator, we used the DCGAN \citep{radford2015unsupervised} architecture shown in \cref{tab:disc_arch}, and this can be replaced with more complex architecture for more realistic images with high fidelity. We fixed the covariance after warm-ups in 10,000 steps, and we used the entropy coefficient $\varepsilon=10^{-4}$ based on our hyperparameter search.

\subsection{Latent diffusion experiments}

For the latent space, we pretrained ALAE \citep{alae} model using the both MNIST and EMNIST (first ten letters) datasets. The ALAE is a high-fidelity autoencoder internally use an adversarial learning to generate high-fidelity images. For the encoder network, as well as decoder network, we mostly adopt the DCGAN architecture. Therefore, the encoder is  mostly identical to \cref{tab:disc_arch} except the point the final layer is $128$ dimension instead of $1$, and the decoder is a convolutional neural network with four convolutional layers.

Following the latent SB setting \citep{lsb}, we assessed our method by utilizing the ALAE model \citep{alae} for generating $1024\times1024$ images of the FFHQ dataset \citep{ffhq}. The base generative model has a latent embedding layer which represent 512-dimensional embedding space. The goal is to transport a point latent space to another, performing unpaired image-to-image translation tasks for four distinct cases: \textit{Adult\,$\to$\;Child}, \textit{Child\,$\to$\;Adult}, \textit{Female\,$\to$\;Male}, and \textit{Male\,$\to$\;Female}. We conducted a quantitative analysis using the ED on the predefined ALAE embedding as a metric for evaluation.

\begin{table}[t]
  \centering
  \caption{Training time for the 100-dimension single-cell data problem.} \label{tab:training_time1}
    \begin{adjustbox}{width=0.52\textwidth, center}
  \begin{tabular}{c|c|c}
    \toprule
      {Sinkhorn (IPF)} & {LightSB} & {VMSB} \\ \midrule
      8m (GPU) & 66s (CPU) & 32s (GPU) / 22m (CPU) \\
    \bottomrule
  \end{tabular}
    \end{adjustbox}
\end{table}

\begin{table}
  \centering
  \caption{Generation time for the 784-dimension MNIST pixel data.} \label{tab:training_time2}
    \begin{adjustbox}{width=0.6\textwidth, center}
  \begin{tabular}{c|cccc|c}
    \toprule
      & $K=64$ & $K=256$ & $K=1024$ & $K=4096$ & NN (SDE) \\
    \midrule
      {GPU} & 721$\mu$s & 726$\mu$s & 739$\mu$s & 740$\mu$s & 1.372s \\
      {CPU} & 60.140ms & 133.333ms & 428.433ms & 1.527s & $-$\\
    \bottomrule
  \end{tabular}
    \end{adjustbox}
\end{table}

\section{Discussion on Implementation of VMSB} \label{sect:limit}

\textbf{Limitations.}\hspace*{6pt} GMM-based SB models, due to the lack of deep structural processing, tend to focus on \textit{instance-level} associations of images in EOT couplings rather than the \textit{subinstance-} or \textit{feature-level} associations that are intrinsic to deep generative models. As a result, while VMSB produces statistically valid representations of optimal transportation within the given architectural constraints, these outcomes may be perceived as somewhat ``synthetic.'' Nevertheless, GMM-based models still hold an irreplaceable role in numerous problems such as latent diffusion and variational methods, due to their simplicity and distinctive properties \citep{lsb}. As we successfully demonstrated in two distinct ways of interacting with neural networks for solving unpaired image transfer, we hope our theoretical and empirical findings help novel neural architecture studies. 

\textbf{Computation.}\hspace*{6pt} For fast computation, we utilized the JAX automatic differentiation library \citep{jax2018github} for computing gradients and Hessians in \cref{prop:wfr}. For each input, the computational of VMSB requires quadratic time for computing the \Wasserstein{} gradient flow (asymptotically $\cO(K^2 n_y)$) and the memory footprint for estimating with internal \Gaussian{} particles is linear (asymptotically $\cO(K n_y)$). There are inherent trade-offs between accuracy and computational efficiency when choosing between LightSB and VMSB; nevertheless, VMSB remains significantly more manageable and computationally tractable compared to deep learning methods for moderate settings. For instance, we have presented performance regarding efficiency and scalability up to 1,000 dimensions in the experiments. Driven by parallel nature of Gaussian particles, we observed that the computation of \cref{prop:wfr} favors vectorized instructions, and the expected speed enhancement from using GPUs is much more evident in neural network cases. In \cref{tab:training_time1}, we report the wall-clock time for a 100-dimensional single-cell data problem \cite{sbml,lsb}, where the performance is reported in \cref{tab:msci}. Additionally, training time in the MNIST-EMNIST translation is reported in \cref{tab:mnist_stat} in the ablation study. This property also holds for generation, allowing practitioners to deploy the model much faster on GPUs. In \cref{tab:training_time2}, we also report that generating 100 MNIST samples from 4096 Gaussian particles, equipped with competitive performance, can be done 1,854 times faster under the same hardware. Since VMSB a simulation-free, the GMM generation process does not suffer from discretization errors of SDE.

\textbf{Reproducibility statement.}\hspace*{6pt} Comprehensive justification and theoretical background are presented in Appendices~\ref{sect:proof}~and~\ref{sect:discuss}.  Since the primary contributions of this paper pertain to the learning methodology, we ensured that all architectures and hyperparameters remained consistent across the LightSB variants. All datasets utilized in this study are available for download alongside the training scripts. Please refer to \cref{sect:details} for more information on the experimental setups.

\begin{table}[h]
  \tikzset{inner sep=0pt, outer sep=0pt}    
  \caption{EOT Benchmark scores of \BWUVP{} $\downarrow$ (\%).} \label{tab:app_bw}
  \centering
  \begin{adjustbox}{width=1.1\textwidth, center}
    \begin{tabular}{c c c c c c c c c c c c c c}
      \toprule
      \multirow{2}{*}{\raisebox{-5pt}{\textbf{Type}}}& \multirow{2}{*}{\raisebox{-5pt}{\textbf{Solver}}} & \multicolumn{4}{c}{$\varepsilon=0.1$} & \multicolumn{4}{c}{$\varepsilon=1$} & \multicolumn{4}{c}{$\varepsilon=10$} \\
      \cmidrule(lr){3-6}\cmidrule(lr){7-10}\cmidrule(lr){11-14}  
      & & $d=2$ & $d=16$ & $d=64$ & $d=128$ & $d=2$ & $d=16$ & $d=64$ & $d=128$ & $d=2$ & $d=16$ & $d=64$ & $d=128$ \\
      \midrule
        \multicolumn{2}{c}{Classical solvers (best)$^\dagger$}& 0.016 & 0.05 & 0.25 & 0.22 & 0.005 & 0.09 & 0.56 & 0.12 & 0.01 & 0.02 & 0.15 & 0.23 \\
        {\small Bridge-M}&DSBM (\citeauthor{dsbm})$^\ddagger$& 0.03 & 0.18 & 0.7 & 2.26 & 0.04 & 0.09 & 1.9 & 7.3 & 0.26 &  102 & 3563 & 15000 \\
        {\small Bridge-M}&SF$^2$M-Sink (\citeauthor{tong2023simulation})$^\ddagger$& 0.04 & 0.18 & 0.39 & 1.1 & 0.07 & 0.3 & 4.5 & 17.7 & 0.17 & 4.7 & 316 & 812 \\
      \midrule
        rev. KL& LightSB (\citeauthor{lsb}) & 
          $0.004\pm0.004$ & $0.009\pm0.004$ & $0.023\pm0.003$ & $0.036\pm0.003$ & 
          $0.004\pm0.005$ & $0.009\pm0.003$ & $0.016\pm0.002$ & $0.035\pm0.003$ & 
          $0.009\pm0.004$ & $0.013\pm0.007$ & $0.034\pm0.004$ & $0.066\pm0.008$ \\
        {\small Bridge-M}&LightSB-M (\citeauthor{lsbm}) & 
          $0.005\pm0.003$ & $0.012\pm0.004$ & $0.034\pm0.003$ & $0.063\pm0.002$ & 
          $0.005\pm0.001$ & $0.027\pm0.007$ & $0.057\pm0.010$ & $0.108\pm0.004$ & 
          $0.004\pm0.002$ & $0.017\pm0.007$ & $0.133\pm0.010$ & $0.409\pm0.042$ \\
        EMA& LightSB-EMA & 
          $0.004\pm0.002$ & $0.014\pm0.003$ & $0.021\pm0.003$ & $0.044\pm0.001$ & 
          $0.004\pm0.003$ & $0.009\pm0.004$ & $0.013\pm0.001$ & $0.032\pm0.004$ & 
          $0.004\pm0.001$ & $0.008\pm0.003$ & $0.023\pm0.013$ & $0.010\pm0.002$ \\  
      \midrule
        Var-MD &VMSB (ours) & 
          $\mathbf{0.003\pm0.001}$ & $\mathbf{0.007\pm0.003}$ & $\mathbf{0.018\pm0.002}$ & $\mathbf{0.039\pm0.001}$ & 
          $\mathbf{0.002\pm0.002}$ & $\mathbf{0.004\pm0.001}$ & $\mathbf{0.009\pm0.001}$ & $\mathbf{0.023\pm0.003}$ & 
          $\mathbf{0.005\pm0.007}$ & $\mathbf{0.006\pm0.004}$ & $\mathbf{0.011\pm0.010}$ & $\mathbf{0.011\pm0.004}$ \\
        Var-MD &VMSB-M (ours) & 
          $\mathbf{0.002\pm0.001}$ & $\mathbf{0.010\pm0.067}$ & $\mathbf{0.031\pm0.004}$ & $\mathbf{0.056\pm0.005}$ & 
          $\mathbf{0.003\pm0.004}$ & $\mathbf{0.005\pm0.002}$ & $\mathbf{0.032\pm0.006}$ & $\mathbf{0.077\pm0.018}$ & 
          $\mathbf{0.003\pm0.003}$ & $\mathbf{0.011\pm0.004}$ & $\mathbf{0.117\pm0.012}$ & $0.429\pm0.748$ \\
      \bottomrule
    \end{tabular}
  \end{adjustbox}
\end{table} 

\begin{table}[h]
  \tikzset{inner sep=0pt, outer sep=0pt}    
  \caption{EOT scores of \cBWUVP{} $\downarrow$ (\%), the fully extended version of \cref{tab:eot}.} \label{tab:app_cbw}
  \centering
  \begin{adjustbox}{width=1.1\textwidth, center}
    \begin{tabular}{c c c c c c c c c c c c c c}
      \toprule
      \multirow{2}{*}{\raisebox{-5pt}{\textbf{Type}}}& \multirow{2}{*}{\raisebox{-5pt}{\textbf{Solver}}} & \multicolumn{4}{c}{$\varepsilon=0.1$} & \multicolumn{4}{c}{$\varepsilon=1$} & \multicolumn{4}{c}{$\varepsilon=10$} \\
      \cmidrule(lr){3-6}\cmidrule(lr){7-10}\cmidrule(lr){11-14}
      & & $d=2$ & $d=16$ & $d=64$ & $d=128$ & $d=2$ & $d=16$ & $d=64$ & $d=128$ & $d=2$ & $d=16$ & $d=64$ & $d=128$ \\
      \midrule
        \multicolumn{2}{c}{Classical solvers (best)$^\dagger$}& 1.94 & 13.67 & 11.74 & 11.4 & 1.04 & 9.08 & 18.05 & 15.23 & 1.40 & 1.27 & 2.36 & 1.31 \\ 
        {\small Bridge-M}&DSBM (\citeauthor{dsbm})$^\ddagger$& 5.2 & 10.8 & 37.3 & 35 & 0.3 & 1.1 & 9.7 & 31 & 3.7 & 105 & 3557 & 15000 \\
        {\small Bridge-M}&SF$^2$M-Sink (\citeauthor{tong2023simulation})$^\ddagger$& 0.54 & 3.7 & 9.5 & 10.9 & 0.2 & 1.1 & 9 & 23 & 0.31 & 4.9 & 319 & 819 \\
      \midrule
        rev. KL& LightSB (\citeauthor{lsb}) & 
          $0.007\pm0.005$ & $0.040\pm0.023$ & $0.100\pm0.013$ & $0.140\pm0.003$ &
          $0.014\pm0.003$ & $0.026\pm0.002$ & $0.060\pm0.004$ & $0.140\pm0.003$ &
          $0.019\pm0.005$ & $0.027\pm0.005$ & $0.052\pm0.002$ & $0.092\pm0.001$ \\
        {\small Bridge-M}&LightSB-M (\citeauthor{lsbm}) & 
          $0.017\pm0.004$ & $0.088\pm0.014$ & $0.204\pm0.036$ & $0.346\pm0.036$ &
          $0.020\pm0.007$ & $0.069\pm0.016$ & $0.134\pm0.014$ & $0.294\pm0.017$ &
          $0.014\pm0.001$ & $0.029\pm0.004$ & $0.207\pm0.005$ & $0.747\pm0.028$ \\
        EMA& LightSB-EMA & 
          $0.005\pm0.002$ & $0.040\pm0.014$ & $0.078\pm0.007$ & $0.149\pm0.006$ &
          $0.012\pm0.002$ & $0.022\pm0.003$ & $0.051\pm0.001$ & $0.127\pm0.002$ &
          $0.017\pm0.003$ & $0.021\pm0.003$ & $0.025\pm0.002$ & $0.042\pm0.002$ \\
      \midrule
        Var-MD &VMSB (ours)& 
          $\mathbf{0.004\pm0.001}$ & $\mathbf{0.012\pm0.002}$ & $\mathbf{0.038\pm0.002}$ & $\mathbf{0.101\pm0.002}$ &
          $\mathbf{0.010\pm0.001}$ & $\mathbf{0.018\pm0.001}$ & $\mathbf{0.044\pm0.001}$ & $\mathbf{0.114\pm0.001}$ &
          $\mathbf{0.013\pm0.001}$ & $\mathbf{0.019\pm0.001}$ & $\mathbf{0.021\pm0.008}$ & $\mathbf{0.040\pm0.001}$ \\
        Var-MD &VMSB-M (ours)& 
          $\mathbf{0.015\pm0.016}$ & $\mathbf{0.067\pm0.036}$ & $\mathbf{0.108\pm0.020}$ & $\mathbf{0.253\pm0.107}$ &
          $\mathbf{0.010\pm0.001}$ & $\mathbf{0.019\pm0.001}$ & $\mathbf{0.094\pm0.010}$ & $\mathbf{0.222\pm0.033}$ &
          $\mathbf{0.013\pm0.001}$ & $\mathbf{0.029\pm0.003}$ & $\mathbf{0.193\pm0.015}$ & $0.748\pm0.036$ \\
      \bottomrule
    \end{tabular}
  \end{adjustbox}
\end{table} 

\section{Additional Experimental Results} \label{sect:extra}
 
\subsection{Additional results on the EOT benchmark}
We present the full results of EOT benchmark experiments. Tables~\ref{tab:app_bw}~and~\ref{tab:app_cbw} show comprehensive statistics on the EOT benchmark with more SB solvers. As mentioned in \S~\ref{subsect:eot_bench}, the VMSB and VMSB-M solvers consistently brought better performance with low standard deviations of scores for \cBWUVP{} and \BWUVP{} measures. We note that the experiment was conducted in a highly controlled setting with identical model configurations; with all other aspects controlled and outcomes differing only by learning methods, the consistent performance gains of our work were a well-anticipated result from our theoretical analysis.

\begin{wrapfigure}{r}{0.268\textwidth}
  \definecolor{plotblue}{RGB}{31, 119, 180}
  \definecolor{plotred}{RGB}{214, 39, 40}
  \centering
  \tikzset{inner sep=0pt, outer sep=0pt}
    \begin{tikzpicture}
      \clip (-0.706,-0.806) rectangle (3.265,2.81);
      \begin{scope}[scale=0.5]
        \begin{scope}[xshift=0.2cm,yshift=-0.65cm]
        \begin{axis}[
             grid=both, mark=none, xmin=0, ymin=0, xmax=14, ymax=131.1,
             axis line style={ultra thick},
             xtick={1,5,9,13},
             ytick={0,15.211,100},
             xticklabels={64,256,1024, 4096},
             yticklabels={0,DSBM,100},
             yticklabel style={font=\Large, xshift=-4pt},
             xticklabel style={font=\Large, yshift=-4pt},
             width=7cm,
             height=7.8cm]
          \addlegendentry{VMSB-adv}
          \addplot[
              line width=1mm, color=plotblue, mark=*, mark size=4pt,
              mark options={solid, fill=plotblue}] coordinates {
            (1, 129.4994659)
            (5, 52.63378143)
            (9, 24.02162552)
            (13, 15.471185684204102)
           };
        \end{axis}
        \node[scale=0.4] at (2.8,-0.8) {{\Large\textsf{Number of Gaussian modalities $K$}}};
        \node[scale=0.4, rotate=90] at (-0.7, 3.3) {\huge FID};
        \end{scope}
      \end{scope}
    \end{tikzpicture}
    \caption{FID vs. modality.} \label{fig:fvm}
    \vskip-11pt
\end{wrapfigure}
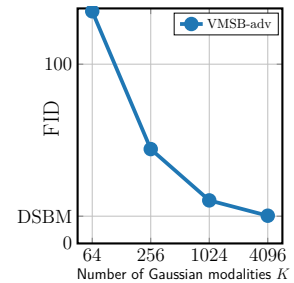

\subsection{Additional image generation results}

In the unpaired EMNIST-to-MNIST translation task for the raw 784 pixel, we measured FID scores for various $K$ for the SB parameterization. We considered $K\in\{64,256,1024,4096\}$ with $\varepsilon=10^{-4}$ for our VMSB algorithm. Our observations, both qualitative and quantitative, indicate that higher modalities yield higher-quality samples. In every case of $K$, VMSB-adv outperformed its counterpart. For instance, \cref{fig:fvm} demonstrates that VMSB generates more diverse samples with high fidelity. Notably, we achieved the competitive FID score of 15.471 using a standard neural network discriminator with relatively low MSD similarity scores. As the latent VMSB model for 128-dimensional embeddings also achieved the considerably low FID score of 9.558 (\cref{tab:mnist_fid}), we concluded that VMSB showed promising quality improvements for the both case, and this supports the generality of our theory.

\begin{figure}[t]
  \def\genmdltxt{0.7}
  \def\gentsktxt{0.7}
  \def\mnistw{66.4pt}
  \tikzset{inner sep=0pt, outer sep=0pt}    
  \centering 
  \begin{tikzpicture}[tight background, scale=0.8]
    \clip (-9.9,-1.463) rectangle (9.9,5.22);
    \begin{scope}[yshift=3.4cm]
      \node[scale=\gentsktxt, rotate=90] at (-9.76,0) {EMNIST-to-MNIST};
      \begin{scope}[xshift=-6.5cm, every node/.style={font=\sffamily}]
        \node[scale=\genmdltxt] at (-1.53, 1.65) {LightSB-adv-256};
        \node[scale=\genmdltxt] at (1.53, 1.65) {\bfseries Ours-256};
        \node at (-1.53, 0) {\includegraphics[width=\mnistw]{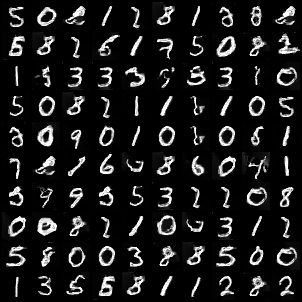}};
        \node at (1.53, 0) {\includegraphics[width=\mnistw]{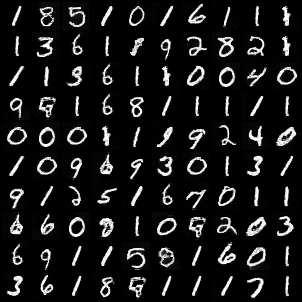}};
      \end{scope}
      \begin{scope}[every node/.style={font=\sffamily}]
        \node[scale=\genmdltxt] at (-1.53, 1.65) {LightSB-adv-1024};
        \node[scale=\genmdltxt] at (1.53, 1.65) {\bfseries Ours-1024};
        \node at (-1.53, 0) {\includegraphics[width=\mnistw]{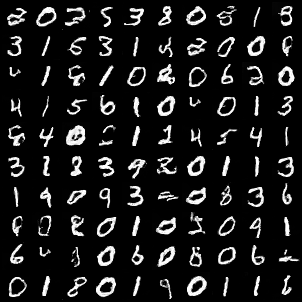}};
        \node at (1.53, 0) {\includegraphics[width=\mnistw]{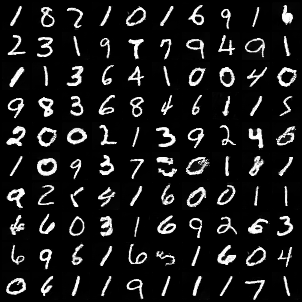}};
      \end{scope}
      \begin{scope}[xshift=6.5cm, every node/.style={font=\sffamily}]
        \node[scale=\genmdltxt] at (-1.53, 1.65) {LightSB-adv-4096};
        \node[scale=\genmdltxt] at (1.53, 1.65) {\bfseries Ours-4096};
        \node at (-1.53, 0) {\includegraphics[width=\mnistw]{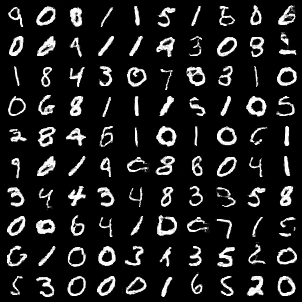}};
        \node at (1.53, 0) {\includegraphics[width=\mnistw]{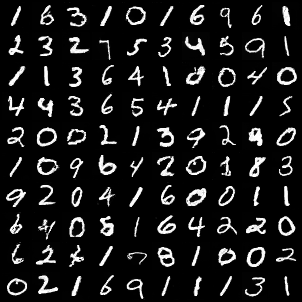}};
      \end{scope}
    \end{scope}
    \begin{scope}
      \node[scale=\gentsktxt, rotate=90] at (-9.76,0) {MNIST-to-EMNIST};
      \begin{scope}[xshift=-6.5cm, every node/.style={font=\sffamily}]
        \node[scale=\genmdltxt] at (-1.53, 1.65) {LightSB-adv-256};
        \node[scale=\genmdltxt] at (1.53, 1.65) {\bfseries Ours-256};
        \node at (-1.53, 0) {\includegraphics[width=\mnistw]{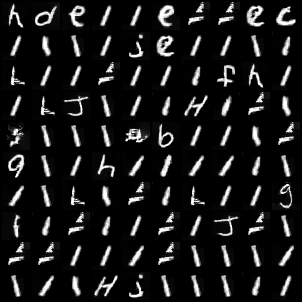}};
        \node at (1.53, 0) {\includegraphics[width=\mnistw]{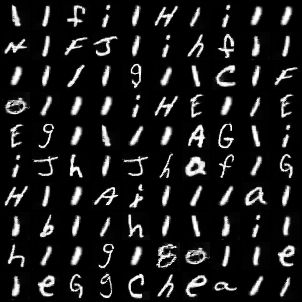}};
      \end{scope}
      \begin{scope}[every node/.style={font=\sffamily}]
        \node[scale=\genmdltxt] at (-1.53, 1.65) {LightSB-adv-1024};
        \node[scale=\genmdltxt] at (1.53, 1.65) {\bfseries Ours-1024};
        \node at (-1.53, 0) {\includegraphics[width=\mnistw]{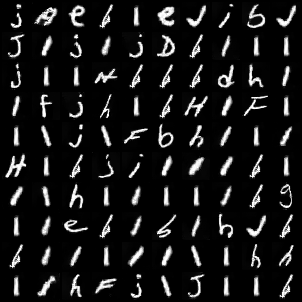}};
        \node at (1.53, 0) {\includegraphics[width=\mnistw]{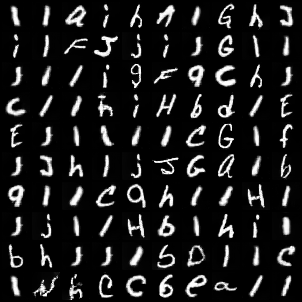}};
      \end{scope}
      \begin{scope}[xshift=6.5cm, every node/.style={font=\sffamily}]
        \node[scale=\genmdltxt] at (-1.53, 1.65) {LightSB-adv-4096};
        \node[scale=\genmdltxt] at (1.53, 1.65) {\bfseries Ours-4096};
        \node at (-1.53, 0) {\includegraphics[width=\mnistw]{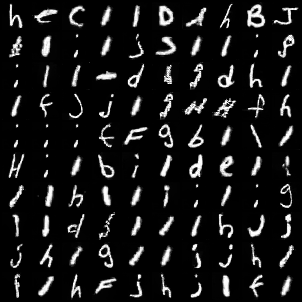}};
        \node at (1.53, 0) {\includegraphics[width=\mnistw]{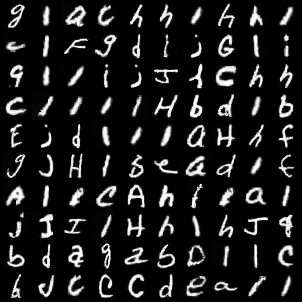}};
      \end{scope}
    \end{scope}  
  \end{tikzpicture}
  \caption{Generation results of unpaired image-to-image translation in the raw pixel space. We considered image data from MNIST and EMNIST (containing the first ten letters), sized as $28\!\times\!28$ pixels. For comparison, we trained GMM-based models with adversarial learning using a simple logistic discriminator (\cref{tab:disc_arch}). This was used as both a benchmark and a tractable target SB model (LightSB-adv-$K$). VMSB in the raw pixel domain demonstrate qualitative improvements in terms of diversity and clarity 
  of image samples.} \label{fig:mnist_qual}
\end{figure}

\begin{table}[b]
  \begin{minipage}[b]{0.36\textwidth}
    \centering 
    \caption{MNIST transfer statistics.} \label{tab:mnist_stat}
    \vskip-4.5pt
    \begin{adjustbox}{width=0.99\textwidth,center}
    \begin{tabular}{ l c c c }
      \toprule
        & \textbf{FID} & \textbf{Time} & \textbf{Parameters} \\
      \midrule
      LightSB-256 & $61.257$ & 30m & 0.4M \\
      LightSB-1024 & $26.487$ & 53m & 1.6M \\
      LightSB-4096 & $20.017$ & 135m & 6.4M \\
      \midrule
      VMSB-256 & $52.634$ & 76m & 0.4M \\
      VMSB-1024 & $24.022$ & 203m &1.6M \\
      VMSB-4096 & $15.471$ & 44h & 6.4M \\
      \midrule
      DSBM-IMF & $11.429$ & 42h & 6.6M \\
      \bottomrule
    \end{tabular}
    \end{adjustbox}
  \end{minipage}
  \hfill
  \begin{minipage}[b]{0.6\textwidth}
    \centering
    \caption{FID scores and differences for generated MNIST.} \label{tab:mnist_fid_diff}
    \vskip-4pt
    \begin{adjustbox}{width=0.99\textwidth,center}
    \begin{tabular}{lccc}
      \toprule
        & \textbf{FID (Train)} & \textbf{FID (Test)} & \textbf{Diff. (test $-$ train)}. \\
      \midrule
      LightSB-adv-256 & $60.746$ & $61.604$ & $0.858$ \\
      LightSB-adv-1024 & $25.934$ & $26.569$ & $0.635$ \\
      LightSB-adv-4096 & $19.960$ & $20.196$ & $0.237$ \\
      \midrule
      VMSB-adv-256 & $51.684$ & $52.283$ & $0.599$ \\
      VMSB-adv-1024 & $23.853$ & $24.053$ & $0.200$ \\
      VMSB-adv-4096 & $15.508$ & $15.496$ & $-0.012$ \\
      \bottomrule
    \end{tabular}
    \end{adjustbox}
  \end{minipage}
\end{table}

\cref{fig:mnist_qual} demonstrates that VMSB generated more diverse samples with high fidelity. Note that the proposed method suffers less from mode collapse than LightSB method (especially on the transfer MNIST-to-EMNIST), with the same Gaussian mixture setting. This result is especially a good point where the difference only lies in the learning methodology, which aligns with our theory. Tables~\ref{tab:mnist_stat}~and~\ref{tab:mnist_fid_diff} effectively show the statistics and FID scores on both the train and the test datasets. The quantitative results highlight that the VMSB solver is more performant with less overfitting than its counterpart. Consequently, our claim regarding the stability of SB solution acquisition is verified by additional experiments involving pixel spaces.

\begin{figure}[t]
  \tikzset{inner sep=0pt, outer sep=0pt}
  \centering
  \begin{tikzpicture}[tight background]
    \begin{scope}[yshift=1.8cm]
      \begin{scope}[yshift=0.85cm]
        \begin{scope}[xshift=-2.4cm, every node/.style={font=\sffamily, scale=0.51}]
          \node at (-1.4602, 0) {\textit{\textbf{Adult $\bm{\to}$ Child}}};
          \node at (0, 0) {VMSB};
          \node at (1.4602, 0) {VMSB-M};
        \end{scope}  
        \begin{scope}[xshift=2.4cm, every node/.style={font=\sffamily, scale=0.51}]
          \node at (-1.4602, 0) {\textit{\textbf{Male $\bm{\to}$ Female}}};
          \node at (0, 0) {VMSB};
          \node at (1.4602, 0) {VMSB-M};
        \end{scope}        
      \end{scope}
      \begin{scope}[xshift=-2.4cm]
        \node at (-1.4588, 0) {\includegraphics[width=41.7pt]{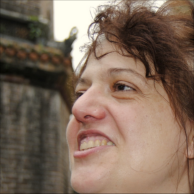}};
        \node at (0, 0) {\includegraphics[width=41.7pt]{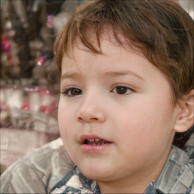}};
        \node at (1.4588, 0) {\includegraphics[width=41.7pt]{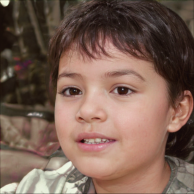}};
      \end{scope}   
      \begin{scope}[xshift=2.4cm]
        \node at (-1.4588, 0) {\includegraphics[width=41.7pt]{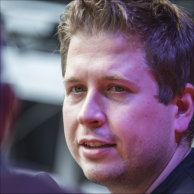}};
        \node at (0, 0) {\includegraphics[width=41.7pt]{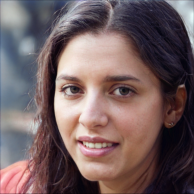}};
        \node at (1.4588, 0) {\includegraphics[width=41.7pt]{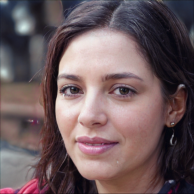}};
      \end{scope}
    \end{scope}
    \begin{scope}
      \begin{scope}[yshift=0.85cm]
        \begin{scope}[xshift=-2.4cm, every node/.style={font=\sffamily, scale=0.51}]
          \node at (-1.4602, 0) {\textit{\textbf{Child $\bm{\to}$ Adult}}};
          \node at (0, 0) {VMSB};
          \node at (1.4602, 0) {VMSB-M};
        \end{scope}  
        \begin{scope}[xshift=2.4cm, every node/.style={font=\sffamily, scale=0.51}]
          \node at (-1.4602, 0) {\textit{\textbf{Female $\bm{\to}$ Male}}};
          \node at (0, 0) {VMSB};
          \node at (1.4602, 0) {VMSB-M};
        \end{scope}        
      \end{scope}
      \begin{scope}[xshift=-2.4cm]
        \node at (-1.4588, 0) {\includegraphics[width=41.7pt]{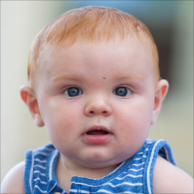}};
        \node at (0, 0) {\includegraphics[width=41.7pt]{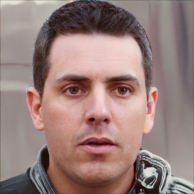}};
        \node at (1.4588, 0) {\includegraphics[width=41.7pt]{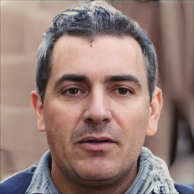}};
      \end{scope}   
      \begin{scope}[xshift=2.4cm]
        \node at (-1.4602, 0) {\includegraphics[width=41.7pt]{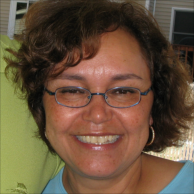}};
        \node at (0, 0) {\includegraphics[width=41.7pt]{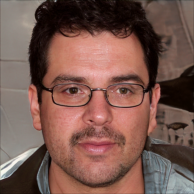}};
        \node at (1.4602, 0) {\includegraphics[width=41.7pt]{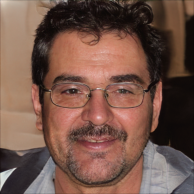}};
      \end{scope}
    \end{scope}
  \end{tikzpicture}
  \caption{Image-to-Image translation on a latent space for the  VMSB and VMSB-M algorithms.} \label{fig:i2i2}
\end{figure} 

We present Embedding-ED scores \citep{jayasumana2024rethinking} in \cref{tab:emmd}, and some qualitative generation results which is visualized in \cref{fig:i2i}. For quantitative results, we calculated statistics from ED scores on embeddings of the ALAE model \citep{alae}, for the four different unpaired image-to-image translation tasks. The results show that VMSB is capable of translating an arbitrary representation, which is closer to target domain than baselines. In \cref{fig:i2i2}, as well as \cref{fig:i2i}, we can see that VMSB and VMSB-M algorithms generate FFHQ data with a given translation task. To qualitatively verify these generation results, we generated images using LightSB and VMSB in Figures~\ref{fig:app_alae2}~and~\ref{fig:app_alae1}. Since these improvements are purely based on information geometry and learning theory, we anticipate that following works on the variational principle application across various fields such as image processing, natural language processing, and control systems \citep{swav, i2sb, wassword, likesb}.

\begin{table}[h]
  \centering
  \caption{ALAE Embedding-ED scores. To evaluate the performance, we computed averages and standard deviations of the ED scores across four different transfer tasks.}\label{tab:emmd}
  \vskip-4pt
  \begin{adjustbox}{width=0.8\textwidth}
    \begin{tabular}{c c c c c}
      \toprule 
      & $\varepsilon=0.1$ & $\varepsilon=0.5$ & $\varepsilon=1.0$ & $\varepsilon=10.0$ \\
      \midrule
      SF$^2$M-Sink & $0.02916 \pm 0.00145$ & $0.04112 \pm 0.00191$ & $0.05670 \pm 0.00249$ & $0.06641 \pm 0.00441$ \\
      DSBM-IMF & $0.02275 \pm 0.00101$ & $0.03358 \pm 0.00142$ & $0.04866 \pm 0.00168$ & $0.06474 \pm 0.00381$ \\
      \midrule
      LightSB & $0.01086 \pm 0.00045$ & $0.02382 \pm 0.00093$ & $0.03462 \pm 0.00148$ & $0.05376 \pm 0.00273$\\
      LightSB-M & $0.01066 \pm 0.00055$ & $0.02366 \pm 0.00107$ & $0.03519 \pm 0.00153$ & $0.05975 \pm 0.00298$\\
      \midrule
      VMSB & $\mathbf{0.01002 \pm 0.00055}$ & $\mathbf{0.02288 \pm 0.00101}$ & $\mathbf{0.03396 \pm 0.00174}$ & $\mathbf{0.05315 \pm 0.00307}$\\
      VMSB-M & $\mathbf{0.00997 \pm 0.00054}$ & $\mathbf{0.02298 \pm 0.00106}$ & $\mathbf{0.03391 \pm 0.00140}$ & $\mathbf{0.05351 \pm 0.00241}$\\
      \bottomrule
    \end{tabular}
  \end{adjustbox}
\end{table}

\begin{figure}[h]
  \tikzset{inner sep=0pt, outer sep=0pt}    
  \centering
  \begin{tikzpicture}[tight background]
    \clip (-5, -1.637) rectangle (5, 11.84);  
    \begin{scope}[yshift=10.2cm]
      \node[rotate=90, scale=0.84] at (-4.6,0) {VMSB,\enspace$\varepsilon = 0.1$};
      \node at (0,0) {\includegraphics[width=250pt]{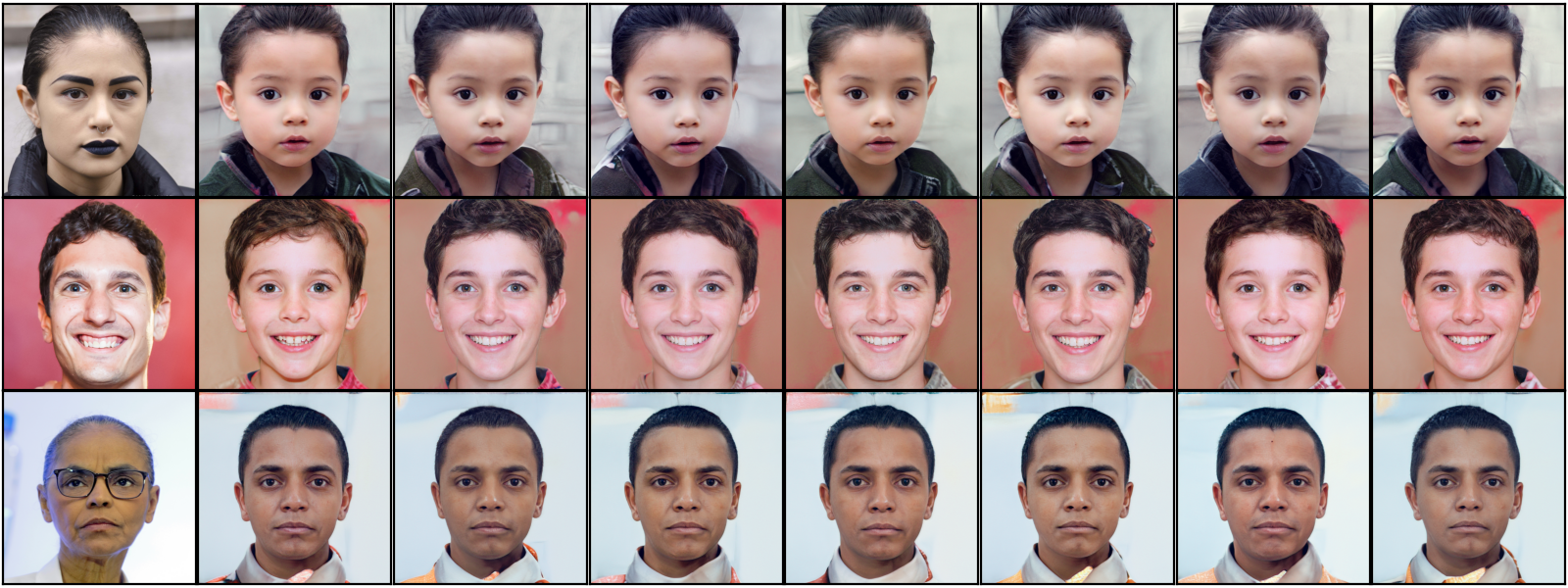}};
    \end{scope}
    \begin{scope}[yshift=6.8cm]
      \node[rotate=90, scale=0.84] at (-4.6,0) {VMSB,\enspace$\varepsilon = 0.5$};
      \node at (0,0) {\includegraphics[width=250pt]{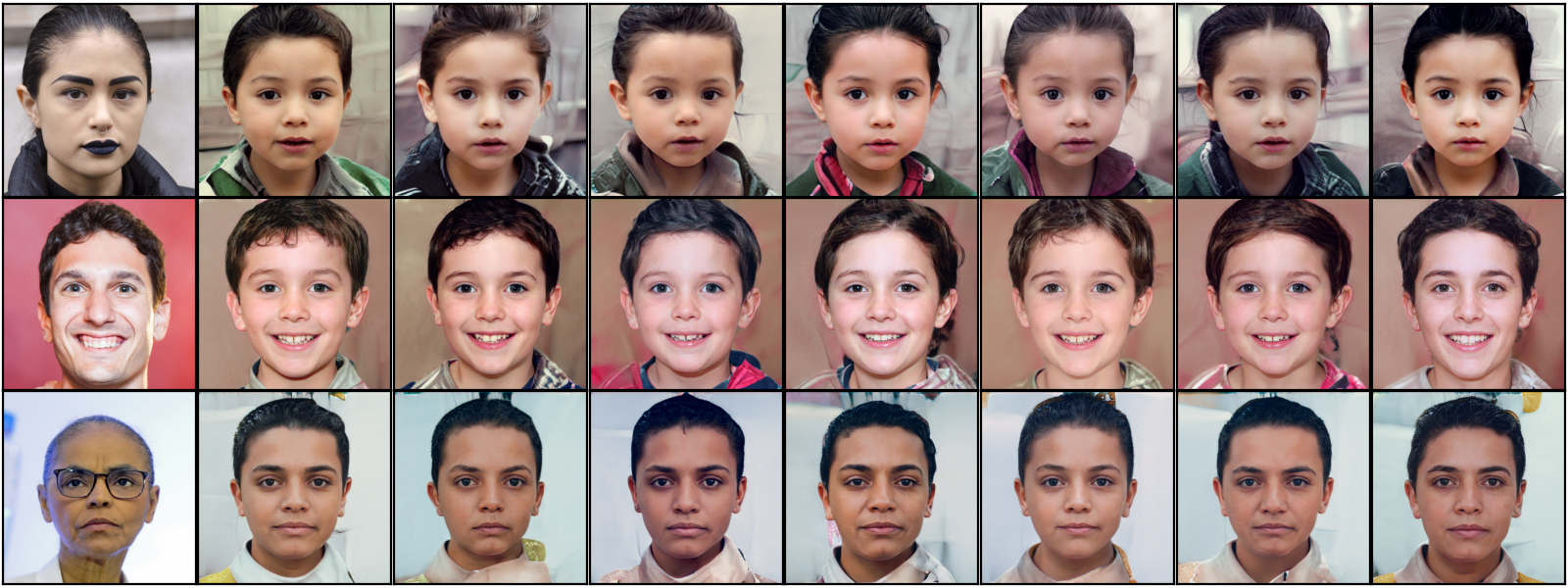}};
    \end{scope}
    \begin{scope}[yshift=3.4cm]
      \node[rotate=90, scale=0.84] at (-4.6,0) {VMSB,\enspace$\varepsilon = 1.0$};
      \node at (0,0) {\includegraphics[width=250pt]{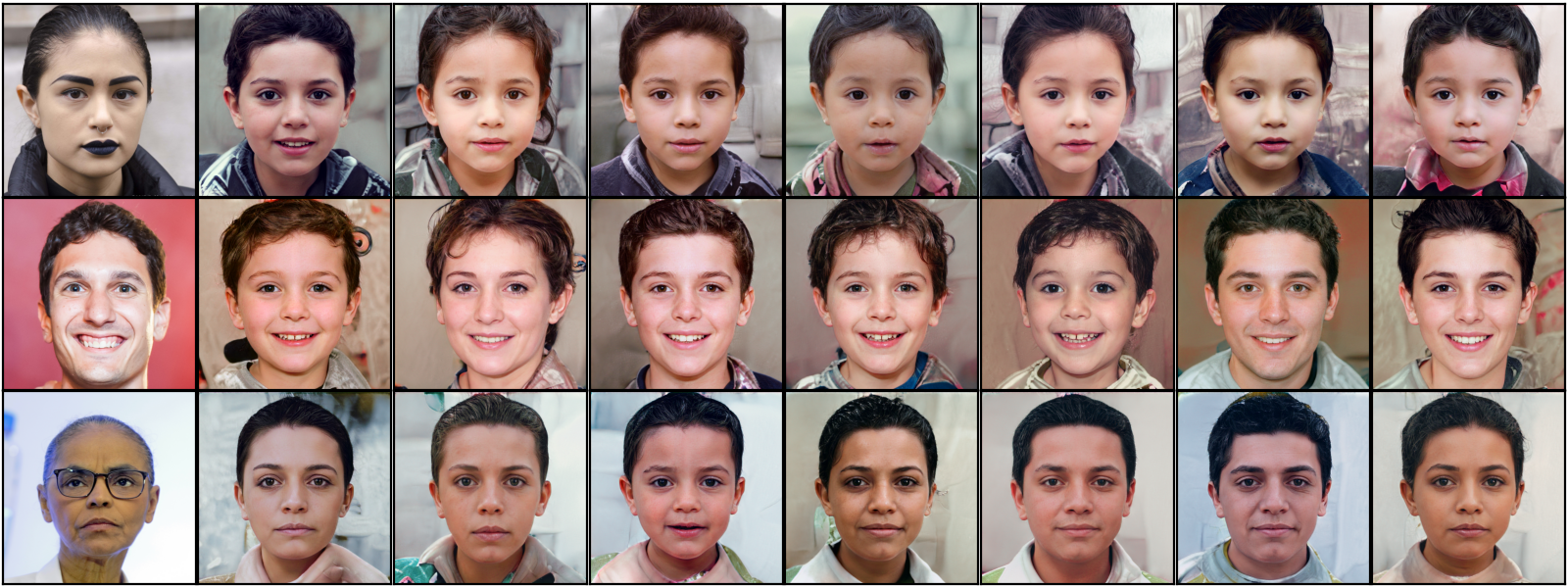}};
    \end{scope}
    \begin{scope}
      \node[rotate=90, scale=0.84] at (-4.6,0) {VMSB,\enspace$\varepsilon = 10.0$};
      \node at (0,0) {\includegraphics[width=250pt]{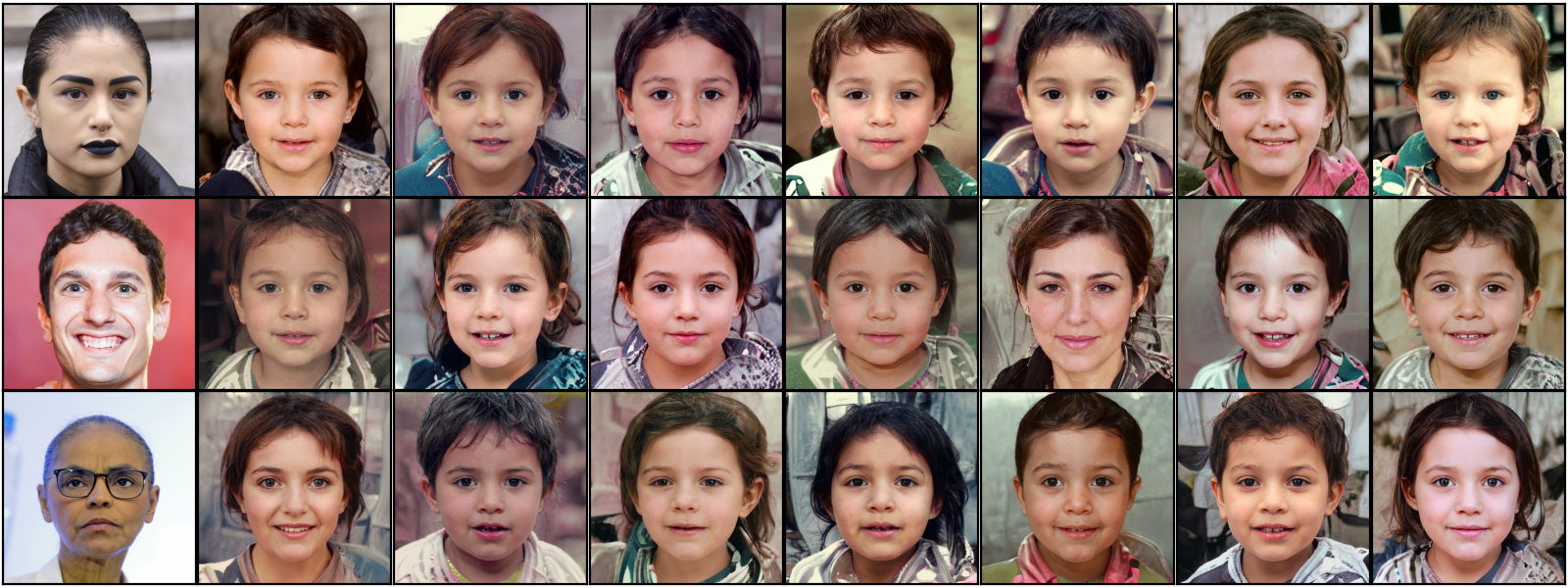}};
    \end{scope}
  \end{tikzpicture}
  \caption{Generation results of VMSB (\textit{Adult}\,$\bm{\to}$\,\textit{Child}) with different volatility settings} \label{fig:app_alae2}
\end{figure} 

\clearpage

\begin{figure}[p]
  \vspace*{\fill}
  \tikzset{inner sep=0pt, outer sep=0pt}
  \centering
  \begin{tikzpicture}[tight background]
    \clip (-6, -1.97) rectangle (6, 13.97);
    \begin{scope}[yshift=12cm]
      \node[rotate=90] at (-5.5,0) {LightSB,\enspace$\varepsilon = 1.0$};
      \node at (0,0) {\includegraphics[width=300pt]{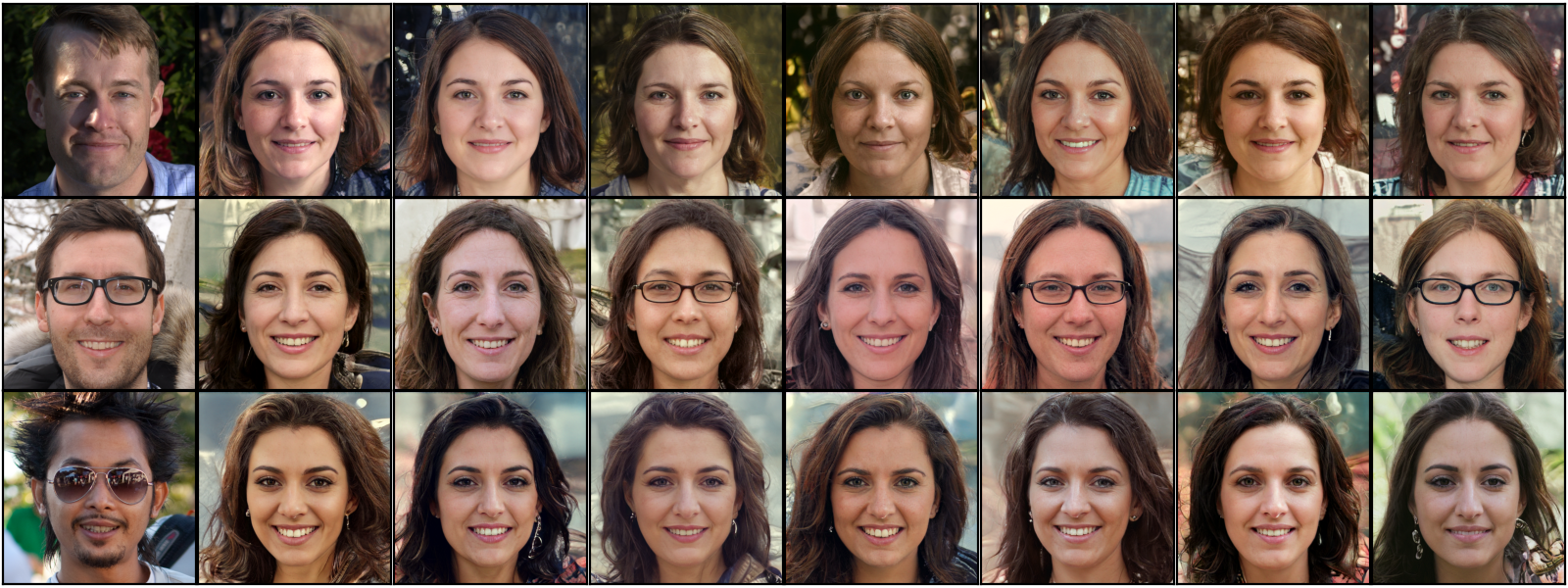}};
    \end{scope}
    \begin{scope}[yshift=8cm]
      \node[rotate=90] at (-5.5,0) {VMSB,\enspace$\varepsilon = 1.0$};
      \node at (0,0) {\includegraphics[width=300pt]{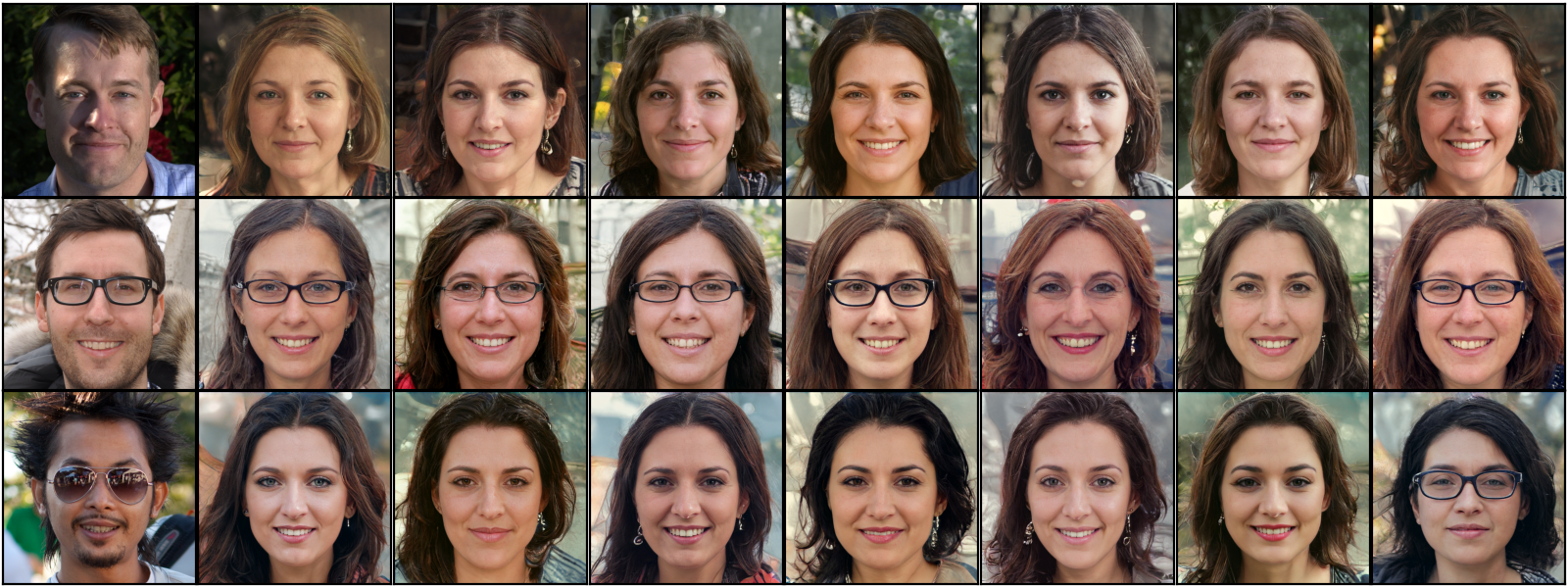}};
    \end{scope}
    \begin{scope}[yshift=4cm]
      \node[rotate=90] at (-5.5,0) {LightSB,\enspace$\varepsilon = 1.0$};
      \node at (0,0) {\includegraphics[width=300pt]{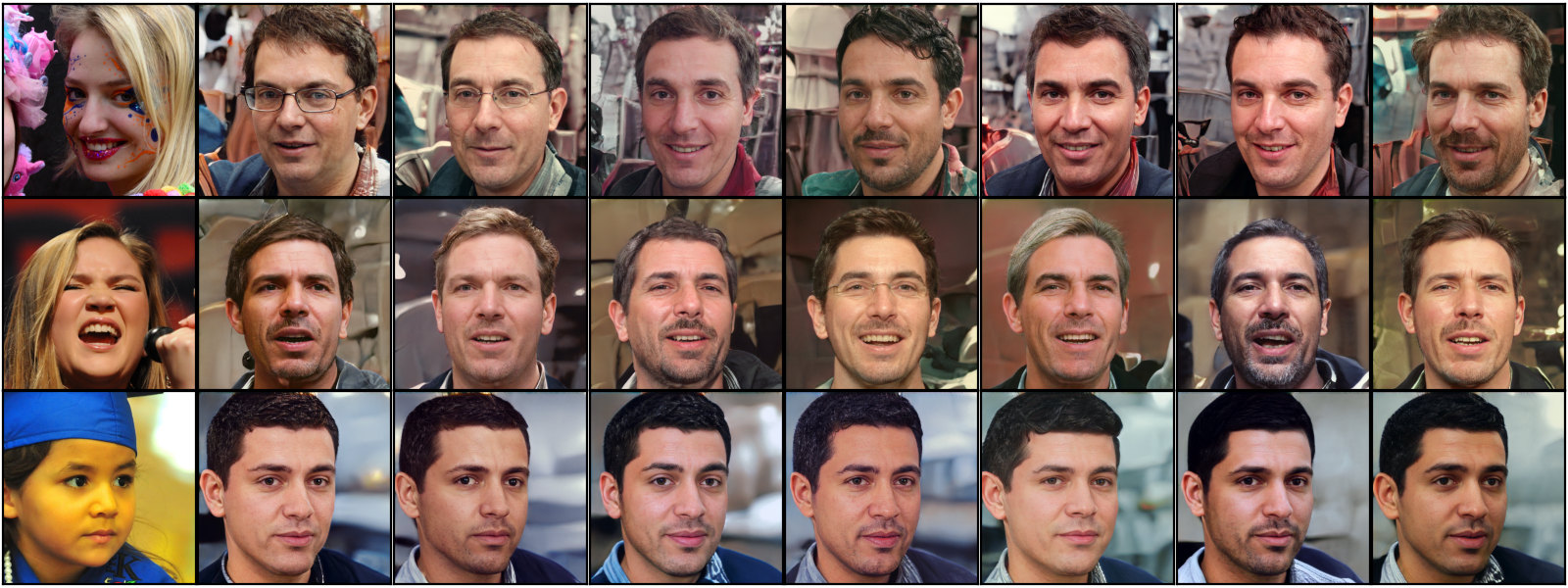}};
    \end{scope}
    \begin{scope}
      \node[rotate=90] at (-5.5,0) {VMSB,\enspace$\varepsilon = 1.0$};
      \node at (0,0) {\includegraphics[width=300pt]{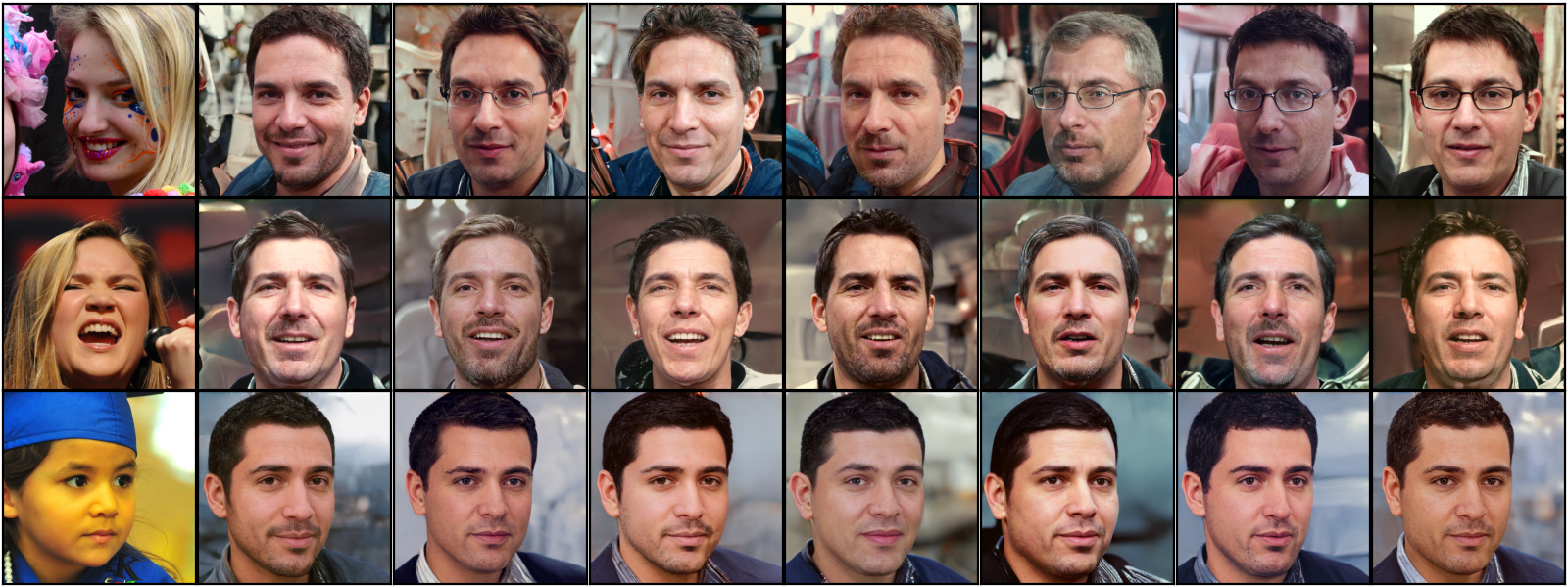}};
    \end{scope}
  \end{tikzpicture}
  \caption{Qualitative comparison between LightSB and VMSB for relatively high volatility, $\varepsilon=1.0$. Top (\textit{Male}\,$\bm{\to}$\,\textit{Female}): We find that VSBM has preserved more facial details, such as wearing glasses, than LightSB. Bottom (\textit{Adult}\,$\bm{\to}$\,\textit{Child}): VSBM was stable at retaining facial position even with high $\varepsilon$.} \label{fig:app_alae1}
  \vspace*{\fill}
\end{figure}

\end{document}